%% file: main.tex
\documentclass[10pt,journal,compsoc]{IEEEtran}
\ifCLASSOPTIONcompsoc
  \usepackage[nocompress]{cite}
\else
  \usepackage{cite}
\fi
\ifCLASSINFOpdf
\else
\fi
\usepackage{amsmath}
\usepackage{algorithmic}

\usepackage{array}

\usepackage{hyperref}       
\usepackage{url}            
\usepackage{booktabs}       
\usepackage{amsfonts}       
\usepackage{nicefrac}       
\usepackage{microtype}      

\usepackage{framed}
\usepackage{wrapfig}
\usepackage{url}
\usepackage{graphicx}
\usepackage{subfigure}
\usepackage{booktabs} 
\usepackage{amsmath}
\usepackage[shortlabels]{enumitem}
\usepackage{ntheorem}
\setlist[enumerate]{nosep}
\newtheorem{theorem}{Theorem}

\newtheorem{assumption}{Assumption}
\newtheorem{definition}{Definition} 

\newtheorem*{proof}{Proof}
\usepackage{amsfonts,amssymb}
\usepackage{multirow}
\usepackage{bm}
\usepackage{pifont}
\usepackage{algorithm}
\newcommand{\tabincell}[2]{\begin{tabular}{@{}#1@{}}#2\end{tabular}}
\newcommand{\x}{\bm{x}}
\newcommand{\z}{\bm{z}}
\newcommand{\GADmethod}{KLOD}
\newcommand{\PADmethod}{KLODS}

\usepackage{arydshln}
\usepackage{makecell}
\usepackage{tikz}
\usetikzlibrary{shapes,arrows}
\usepackage{caption}

\newcommand{\m}{\bm{\mu}}
\newcommand{\s}{\bm{\Sigma}}
\newcommand{\n}{\mathcal{N}}
\newcommand{\e}{\varepsilon}

\newcommand{\fKLnn}{KL(\mathcal{N}(\m_1,\s_1)||\mathcal{N}(\m_2,\s_2))}
\newcommand{\bKLnn}{KL(\mathcal{N}(\m_2,\s_2)||\mathcal{N}(\m_1,\s_1))}

\usepackage[shortlabels]{enumitem}
\setlist[enumerate]{nosep}
\usepackage{xr}
\usetikzlibrary{matrix}
\allowdisplaybreaks[4]

\usepackage{chngcntr}
\begin{document}
%
\title{Kullback-Leibler Divergence-Based Out-of-Distribution Detection  with Flow-Based Generative Models}
%
%
%
%
\author{Yufeng~Zhang, Jialu Pan, Wanwei~Liu, Zhenbang~Chen,  Kenli~Li, Ji~Wang, Zhiming~Liu, Hongmei~Wei
	\IEEEcompsocitemizethanks{
		\IEEEcompsocthanksitem Yufeng Zhang, Jialu Pan, and Kenli Li are with the College of Computer Science and Electronic Engineering, Hunan University, Changsha, China.\protect\\
		E-mail: yufengzhang@hnu.edu.cn, jialupan@hnu.edu.cn, lkl@hnu.edu.cn
		\IEEEcompsocthanksitem Wanwei Liu, Zhenbang Chen, and Ji Wang are with the College of Computer, National University of Defense Technology, Changsha, China.\protect\\
		E-mail: wwliu@nudt.edu.cn, zbchen@nudt.edu.cn, wj@nudt.edu.cn
		\IEEEcompsocthanksitem Zhiming Liu is with the Centre for Research and Innovation in Software Engineering, Southwest University, Chongqing, China.\protect\\
		E-mail: zhimingliu88@swu.edu.cn
		\IEEEcompsocthanksitem Hongmei Wei is with the National Research Center of Parallel Computer Engineering and Technology, China.\protect\\
		E-mail: wei\_hongm@163.com%
		\IEEEcompsocthanksitem Kenli Li and Ji Wang are the corresponding authors.
	}
	\thanks{Manuscript received xx xx, 2022; revised xx xx, 2022.}
}
\markboth{Journal of \LaTeX\ Class Files,~Vol.~14, No.~8, August~2015}%
{Shell \MakeLowercase{\textit{et al.}}: Bare Demo of IEEEtran.cls for Computer Society Journals}
%



\IEEEtitleabstractindextext{%
\begin{abstract}
Recent research has revealed that deep generative models including flow-based models and Variational Autoencoders may assign higher likelihoods to out-of-distribution (OOD) data than in-distribution (ID) data. However, we cannot sample  OOD data from the model. This counterintuitive phenomenon has not been satisfactorily explained and brings obstacles to OOD detection with flow-based models. 
In this paper, we prove theorems to investigate the Kullback-Leibler divergence in flow-based model and  give two explanations for the above phenomenon. 
Based on our theoretical analysis,  
we propose a new method \PADmethod\ to leverage KL divergence and local pixel dependence of representations to perform anomaly detection. 
Experimental results on prevalent benchmarks demonstrate the effectiveness and robustness of our method.
For group anomaly detection, our method achieves 98.1\% AUROC on average with a small batch size of 5. On the contrary, the baseline typicality test-based method only achieves 64.6\% AUROC on average due to its failure on challenging problems. Our method also outperforms the state-of-the-art method by 9.1\% AUROC. For point-wise anomaly detection, our method achieves 90.7\% AUROC on average and outperforms the baseline by 5.2\% AUROC. 
Besides, our method has the least notable failures and is the most robust one.
\end{abstract}

}

\maketitle

\IEEEdisplaynontitleabstractindextext

%
\IEEEpeerreviewmaketitle

\IEEEraisesectionheading{\section{Introduction}\label{sec:intro}}

%
%
%
%

\IEEEPARstart{A}{nomaly} detection is the process of ``finding patterns in data that do not conform to expected behavior'' \cite{anomaly_detection_survey_2009,deep_anomaly_detection_survey_2021}.  
Under an unsupervised learning setting, the model is trained on a set of unlabeled data $\{\bm{x}_1, \cdots,\bm{x}_n\}$ which are drawn independently from an unknown distribution $p^\ast$. Group anomaly detection (GAD) \cite{GADSurvey2018} aims to determine whether a given group of test inputs $\{\widetilde{\bm{x}}_1,\cdots,\widetilde{\bm{x}}_m\}(m > 1)$ is sampled from $p^\ast$.  
Typical applications of GAD include discovering high-energy particle physics, 
\cite{oneclassSMM2013},
anomalous galaxy clusters in astronomy \cite{Supportmeasure2015,HierarchicalGAD2011}
, unusual vorticity in fluid dynamics  \cite{GAD-using-Genre-Model2011},
and stealthy attacks \cite{GADSurvey2018,finding-rats-in-cats-2019}. 
Point-wise anomaly detection (PAD) \cite{anomaly_detection_survey_2009, chalapathy2019deep} aims to determine whether an individual input is sampled from $p^\ast$.
PAD is applied in many areas including detecting intrusion \cite{anomaly_detection_survey_2009}, 
fraud \cite{credit_card_fraud_2019},
malware \cite{malware_detection_survey_2017}, and
medical anomalies \cite{anomaly_detection_survey_2009}.
It is worth noting that GAD cannot be implemented by PAD because the individual members of the input group may not be anomalies \cite{deep_anomaly_detection_survey_2021, GADSurvey2018, GADGenre_nips2011}.
In literature, the term \textit{anomaly} is also referred to as outlier, peculiarity, out-of-distribution (OOD) data, \textit{etc}. In the following, we mainly use terms \textit{OOD data} and \textit{anomaly} as in most related works.

\textbf{Counterintuitive Phenomenon}.
This paper focuses on unsupervised OOD detection using explicit deep generative models (DGM) including flow-based models and Variational Autoencoders (VAE).
Recent research shows that explicit deep generative models including flow-based models \cite{kingma2018glow,dinh2016realnvp}, VAE \cite{kingma2013auto}, and auto-regressive models~\cite{van2016conditional,salimans2017pixelcnn++} are not capable of distinguishing OOD data from in-distribution (ID) data (training data) according to the model likelihood (\textit{i.e.}, Type II errors) \cite{nalisnick2018deep,shafaei2018digitnotcat, choi2018generative, vskvara2018generative,nalisnick2019detecting,whyflowfailood}. 
For example, as shown in Figures  \ref{fig:logp_compare_problem}\subref{fig:logpx_fashionmnist_notmnist_mnist} and \ref{fig:logp_compare_problem}\subref{fig:logpx_cifar10_cifar100_svhn} in the supplementary material, Glow \cite{kingma2018glow} assigns  higher likelihoods for SVHN (MNIST) when trained on CIFAR-10 (FashionMNIST). 
Figure \ref{fig:logp_cifar10_vs_svhn_residual_flow} in the supplementary material shows similar results in recent proposed residual flows \cite{chen2019residualflows}.  
However, as pointed out by Nalisnick \textit{et al.} \cite{nalisnick2019detecting} \textit{we cannot sample OOD data from the model}.
We can also observe a similar phenomenon in class conditional Glow (GlowGMM), which contains a Gaussian mixture model  on the top layer with one Gaussian distribution for each class \cite{kingma2018glow, fetaya2019conditional,izmailov2019semisupervised}.  
For example, GlowGMM does not achieve the same performance as prevalent discriminative models such as ResNet \cite{resnet} on FashionMNIST. 
We observe that the centroids of different components are close to each other (see Figure \ref{fig:cond_glow_logp_of_centroids_under_all_other_components} in the supplementary material).
One component may assign higher likelihoods for other classes (see Table \ref{tbl:cond_glow_fashionmnist_logpz_classification} in the supplementary material). 
However, \textit{we always sample images of the correct class from the corresponding component}.  


Nalisnick \textit{et al.} explain the above phenomenon by the discrepancy of the typical set and high probability density regions of the model distribution \cite{nalisnick2019detecting}. They propose using typicality test to detect OOD data.  However, their explanation and method fail on problems where the likelihoods of ID and OOD data coincide (\textit{e.g.}, CIFAR-10 vs CIFAR-100, CelebA vs CIFARs).  
In this paper,  
we manipulate the model likelihoods such that ID and OOD data have coinciding likelihoods (see Subsection \ref{sec:attack_likelihood}). Such manipulation could make all existing likelihood-based OOD detection methods \cite{nalisnick2019detecting,2020InputComplexity,understanding_anomaly_2020_nips,morningstar21KDE} fail.
Some researchers investigate the behaviors of flow-based models in OOD detection. 
Kirichenko \textit{et al.} reveal that flow-based model learns local pixel correlations and generic image-to-latent-space transformations  \cite{whyflowfailood}. Such learned knowledge may also exist in OOD dataset. 
Zhang \textit{et al.} state that the estimation error of the flow-based model is the reason for the failure of anomaly detection \cite{understandingfailures}.  


\textbf{Research Questions}.
Currently, the above counterintuitive phenomenon has not been explained  satisfactorily.
In this paper, we rethink the existing  conclusions relating to OOD detection using flow-based model. 
We focus on the following two research questions:
\begin{itemize}
	\setlength{\itemsep}{0pt}
	\setlength{\parskip}{0pt}
	\item \textbf{Q1: Explanation\footnote{We focus on the reason behind Q1 rather than aiming to sample OOD data in this paper.}}. Why can we not sample OOD data from flow-based model? We need a unified answer to this question whenever OOD data have lower, higher, or coinciding likelihoods. 
	
	\item \textbf{Q2: OOD detection}. How to detect OOD data using flow-based model and VAE without supervision? 
\end{itemize}

We start our research from the sampling process. 
Flow-based model constructs diffeomorphism $\z=f(\x)$ from visible (data) space to latent space. The model maps each input data point $\x$ to a unique representation $\z$ in latent space. We can sample noise $\varepsilon$ from prior (usually standard Gaussian distribution) and generate new data $f^{-1}(\varepsilon)$.
So we should ask \textit{why we cannot sample the representations of OOD data from prior}. 
In this paper, we explain why we cannot sample OOD data. We abandon the model likelihood and leverage Kullback-Leibler (KL) divergence and local pixel dependence of representations for OOD detection.

\textbf{Contributions}. The contributions of this paper are:
\begin{enumerate}
	\setlength{\itemsep}{0pt}
	\setlength{\parskip}{0pt}
	\item We prove several theorems to investigate the KL divergence in flow-based model. 
	We answer why we cannot sample OOD data from two perspectives. The first answer reveals the large KL divergence between the distribution of representations of OOD data and the prior. The second answer states that the representations of OOD data locate in specific directions. 
	
	\item 
	We propose a unified OOD detection method in three steps based on our analysis. 
	Firstly, we propose leveraging the KL divergence between the distribution of representations and prior for GAD. We also propose using fitted Gaussian to estimate the (lower bound of) KL divergence.  
	Secondly, we decompose the KL divergence and leverage the last-scale KL divergence for OOD detection. 
	Finally, we leverage the local pixel dependence of representations to improve our method further and support PAD.
	\item We conduct experiments to demonstrate the effectiveness and robustness of our method.
\end{enumerate}

The remaining part of this paper is organized as follows.
Section \ref{sec:related_work} discusses the related work. 
Section \ref{sec:problem} discusses problem settings.
Section \ref{sec:whynotexplain} presents our theoretical analysis to answer Q1. 
Section \ref{sec:GADmethod} elaborates on the details of our OOD detection method. 
Section \ref{sec:experiment} presents experimental results. 
Finally, Section \ref{sec:conclusion} concludes. 
More details of the methods, experimental results, discussion, and related work are presented in the supplementary material.

\vspace{-10pt}
\section{Related Work}\label{sec:related_work}
We discuss the most related work here. More discussion is presented in Section \ref{sec:more_relatedwork} in the supplementary material.

\textbf{GAD and PAD}. In \cite{GADSurvey2018}, Toth \textit{et al.} give a survey on GAD methods and plenty of real-world GAD applications. In \cite{chalapathy2019deep}, Chalapathy \textit{et al.} survey  a wide range of deep learning-based GAD and PAD methods. 
In \cite{deep_anomaly_detection_survey_2021}, Pang \textit{et al.} review the deep learning-based anomaly detection methods.
It is worth noting that in GAD an individual data point in the input group can be normal \cite{deep_anomaly_detection_survey_2021, GADSurvey2018, GADGenre_nips2011}. So GAD and PAD have different contexts. 
According to the availability of supervision information, OOD detection can be classified into supervised, semi-supervised, and unsupervised settings. 
In this paper, we focus on unsupervised OOD detection using flow-based model, so we mainly compare with methods in the same category.

\textbf{OOD Detection Using Flow-Based Model}.
Generally, it seems straightforward to use model likelihood $p(\bm{x})$ (if any) of a generative model to detect OOD data~\cite{Pimentel2014A, GADSurvey2018}. 
However, these methods fail when OOD data have higher or similar likelihoods.
Choi \textit{et al.} propose using the Watanabe-Akaike Information Criterion (WAIC) to detect OOD data~\cite{choi2018generative}.  WAIC penalizes points that are sensitive to the particular choice of posterior model parameters. However, Nalisnick \textit{et al.} ~\cite{nalisnick2019detecting}  could not reproduce the results of WAIC.
Choi \textit{et al.}  also propose using typicality test in the latent space to detect OOD data. Our results reported in Subsection \ref{sec:attack_likelihood} demonstrate that typicality test in the latent space can be attacked.  
Sabeti \textit{et al.} propose detecting anomalies based on typicality~\cite{sabeti2019data}, but their method is not suitable for deep generative models. 
Nalisnick \textit{et al.} propose using typicality test on model distribution (Ty-test) for GAD \cite{nalisnick2019detecting}.
Jiang \textit{et al.} propose GOD2KS which combines random projection and two-sample KS test to perform GAD based on flow-based model \cite{jiang2022revisiting}.
Ren \textit{et al.}  propose to use likelihood ratios for OOD detection\cite{ren2019likelihood-ratio}.  Serr\`{a} \textit{et al.} propose using likelihood compensated by input complexity for OOD detection \cite{2020InputComplexity}.
In \cite{understanding_anomaly_2020_nips}, Schirrmeister \textit{et al.} find the likelihood contributed by the last scale of Glow ($L_{last}$) is a better criterion than $\log p(\bm{x})$ for PAD. We find $L_{last}$ should not be explained as likelihood consistently for OOD data. See Section \ref{sec:more_relatedwork} in the supplementary material for more discussion.
%
In \cite{morningstar21KDE}, Morningstar \textit{et al.} train density estimator (DoSE) and one-class SVM on the statistics of deep generative models to detect OOD data. 
Before this writing, GOD2KS \cite{jiang2022revisiting} and DoSE \cite{morningstar21KDE} are the SOTA GAD and PAD methods applicable to flow-based models under unsupervised setting, respectively. 
We will show that many baseline methods could degenerate into being not better than random guessing under data manipulation.
These results demonstrate the difficulty of OOD detection using flow-based model.




\vspace{-8pt}

\section{Problem Settings}
\label{sec:problem}
This paper mainly focuses on flow-based generative model, which constructs diffeomorphism $\z=f(\x)$ from visible space $\mathcal{X}$ to latent space $\mathcal{Z}$ \cite{kingma2018glow,dinh2016realnvp,dinh2014nice,papamakarios2019flow_model_survey}. Our work also involves Variational Autoencoder (VAE) \cite{kingma2013auto}.  
Please refer to Section \ref{sec:background} in the supplementary material for background.
In this section, we first discuss how to manipulate the model likelihoods. Then we note the target problems of this work.

 \vspace{-8pt}
\subsection{Manipulating Likelihoods}  \label{sec:attack_likelihood}
In \cite{nalisnick2019detecting}, Nalisnick \textit{et al.} conjecture that the counterintuitive phenomena in Q1 stem from the distinction of high probability density regions and the typical set of the model distribution~\cite{nalisnick2019detecting, choi2018generative}. For example, Figure \ref{fig:typical_set} in the supplementary material shows the typical set of $d$-dimensional standard Gaussian distribution, which is an annulus with a radius of $\sqrt{d}$ ~\cite{vershynin2018high}. When sampling from the Gaussian distribution, it is highly likely to get points in the typical set rather than the highest density region ($i.e.$ the center) or the lowest density region far from the mean. Based on this explanation, Nalisnick \textit{et al.} propose using typicality test (Ty-test in short) to detect OOD data \cite{nalisnick2019detecting}. 
However, their explanation and method do not apply to problems where  OOD data reside in the typical set of model distribution (\textit{i.e.}, OOD data has coinciding likelihoods with ID data).
Researchers have also proposed other likelihood-related OOD detection methods, including input complexity compensated likelihood \cite{2020InputComplexity}, likelihood contributed by the last scale \cite{understanding_anomaly_2020_nips}, and DoSE \cite{morningstar21KDE}. 
In the following, we show how to manipulate OOD data to make the likelihood  of ID and OOD dataset coincide. Such manipulation could make all existing likelihood-based methods fail.

\textbf{M1: Manipulating $p(\z)$ by Rescaling $\z$ to Typical Set of Prior}. We train Glow with 768-dimensional standard Gaussian prior on FashionMNIST. Figure \ref{fig:logp_compare_problem}\subref{fig:logpz_fashionmnist_notmnist_mnist} in the supplementary material shows the histogram of log-likelihood of representations under prior (\textit{i.e.}, $\log p(\bm{z})$)\footnote{In official Glow model, $\log p(\z)$ is implemented as the log-likelihood of the representation of the last scale of Glow under prior.}. 
Note that $\log p(\bm{z})$ of FashionMNIST is around $-768\times (0.5\times \text{ln} 2\pi e)\approx -1089.74$,
which is the log-probability of typical set of the prior~\cite{cover2012elements}. 
Here it seems that we can detect OOD data by $p(\bm{z})$ or typicality test in the latent space ~\cite{choi2018generative}.
However, as shown in Figure \ref{fig:typical_set} in the supplementary material, we can decode each OOD data point $\x$ as $\bm{z}=f(\x)$ and rescale $\z$ to the typical set by setting $\bm{z}'=\sqrt{d}\times \bm{z}/|\bm{z}|$ ($d=768$).
Then we decode $\z'$ to generate image $\x'=f^{-1}(\z')$. We find that $\bm{x}'$ corresponds to the similar image with $\bm{x}$. Figure \ref{fig:sample_images_rescale_to_typical_set_trained_glow_on_fashionmnist} in the supplementary material shows some examples of $\x'$. 
These results demonstrate that flow-based model cannot expel representations of OOD data from the typical set of the prior. 
Note that, Glow model uses multi-scale architecture and has three stages of representations with different scales. In our experiments, rescaling the last scale yields similar results as rescaling all scales simultaneously (see Figure \ref{fig:sample_images_rescale_to_typical_set_trained_glow_on_fashionmnist} in the supplementary material).
To the best of our knowledge, we are the first to discover that the latents rescaled to the typical set of prior still can be mapped back to legal images. 
In this paper, we will see that, such manipulation can make multiple exsiting OOD detection methods fail.

\textbf{M2: Manipulatinging $p(\x)$ by Adjusting Contrast}. 
Nalisnick \textit{et al.} find that the likelihoods can be manipulated by adjusting the variance of inputs
~\cite{nalisnick2018deep}.  
As shown in Figure \ref{fig:logp_compare_problem}\subref{fig:logpx_cifar10_svhn_contrast_glow} in the supplementary material, SVHN with increased contrast by a factor of 2.0 has coinciding likelihood distribution with CIFAR-10 on Glow trained on CIFAR-10. So it is impossible to detect OOD data by $p(\bm{x})$ or typicality test on the model distribution (see Figure \ref{fig:logp_compare_problem}\subref{fig:logpx_cifar10_cifar100_svhn} in the supplementary material too). In our experiments, we can manipulate the likelihoods of OOD dataset in this way for almost all problems (see Figure \ref{fig:logpx_glow_fashionmnist_vs_others}$\sim$\ref{fig:logpx_glow_celeba_vs_others} in the supplementary material). 
We will see that (in Section \ref{sec:experiment}) multiple existing OOD detection methods could degenerate into being not better than random guessing under data manipulation. 
Similarly, in VAE, we can also manipulate the likelihoods by adjusting the contrast of input images. 

\textbf{Summary}. We can manipulate both $p(\x)$ and $p(\z)$ of OOD data without knowing the model parameters. 
In this paper, we abandon the model likelihood and propose an OOD detection method that is robust to data manipulations.

\vspace{-10pt}
\subsection{Problems}

We use ID vs OOD to represent an OOD detection problem and use ``ID (OOD) representations'' to denote the representations of ID (OOD) data.
According to the statistics of OOD dataset, we group OOD detection problems into two categories:
\begin{itemize}
	\setlength{\itemsep}{0pt}
	\setlength{\parskip}{0pt}
	\item \textbf{Category I: smaller/similar variance, higher/similar likelihoods}. OOD dataset has smaller or similar variance with ID dataset and tends to have higher or similar likelihoods;
	\item \textbf{Category II: larger variance, lower likelihoods}. OOD dataset has larger variance than ID data and tends to have lower likelihoods.
\end{itemize}

As shown in Subsection \ref{sec:attack_likelihood}, we can use data manipulation \textbf{M2} (adjusting contrast) to convert one problem from one category to another.


\input{sec3.tex}
\vspace{-10pt}
\section{Anomaly Detection Method}\label{sec:GADmethod}
In this section, we elaborate on our OOD detection method in three steps in three subsections, respectively. In Subsection \ref{sec:KL_GAD}, we propose leveraging KL divergence for OOD detection. In Subsection \ref{sec:last_scale}, we reduce the computation cost. Finally, in Subsection \ref{sec:PADmethod}, we present a unified OOD method supporting PAD and GAD with small batch sizes. Please refer to Figure \ref{fig:bigpicture2} and Figure \ref{fig:flowchart} in the supplementary material for an overview when reading this section.

\vspace{-10pt}
\subsection{Step 1: Leveraging KL divergence}\label{sec:KL_GAD}
Answer 1 in Subsection \ref{sec:answer1} reminds us to detect OOD data by estimating $KL(p||p_Z^r)$, where $p$ is the distribution of representations of inputs. %
However, when only samples are available, divergence estimation is provable hard, and the estimation error decays slowly in high dimension space \cite{Hoijtink_introductionto,Nguyen07estimatingdivergence, Rubenstein2019estimation}. This brings difficulty in applying existing divergence estimation \cite{estimation2005QingWang,estimation2009QingWang,estimate2010Nguyen,ensembleEstimation2014, Rubenstein2019estimation} to high dimensional problems with small sample size. 
Luckily, as shown in Table \ref{tbl:SW_test} in the supplementary material, we observe that both ID data and OOD data of \textit{Category I} problems (smaller/similar variance, higher/similar likelihood) follow a Gaussian-like distribution. This provides us with a facility to estimate the KL divergence for GAD. 

\vspace{-10pt}
\subsubsection{Flow-based Model}
\label{sec:investigate_flow_latents}
\textbf{ID Data}. 
As discussed in Section \ref{sec:whynotexplain}, we can use a Gaussian distribution $\mathcal{N}_p$ to approximate $p_Z$. Here we use sample expectation $\bm{\widetilde{\mu}}$ and covariance $\bm{\widetilde{\Sigma}}$ of representations to estimate the parameters of $\mathcal{N}_p$ \footnote{This is equal to using maximum likelihood estimation \cite{PRML}.}.
Experiments also show that we can generate high-quality images by sampling from  $\mathcal{N}_p$ rather than the prior (standard Gaussian distribution). 
Now we can calculate the KL divergence between two Gaussian distributions $\mathcal{N}(\bm{\widetilde{\mu}}, \bm{\widetilde{\Sigma}})$ and $\mathcal{N}(\bm{\mu},\bm{\Sigma})$ analytically by 
\begin{align}
	& KL(\mathcal{N}(\bm{\widetilde{\mu}}, \bm{\widetilde{\Sigma}})||\mathcal{N}(\bm{\mu},\bm{\Sigma})) \label{equ:KL_compuation}\\
	= 	& \dfrac{1}{2}\Big\{\log \dfrac{|\bm{\Sigma}|}{|\widetilde{\bm{\Sigma}}|}+\text{Tr}(\bm{\Sigma}^{-1}\widetilde{\bm{\Sigma}})+(\bm{\mu}-\widetilde{\bm{\mu}})^T\bm{\Sigma}^{-1}(\bm{\mu}-\widetilde{\bm{\mu}})-n\Big\} \nonumber
\end{align}
When the prior ($\mathcal{N}(\bm{\mu},\bm{\Sigma})$) is standard Gaussian distribution $\mathcal{N}(0,I)$, Equation \eqref{equ:KL_compuation} equals to  
\begin{equation}\label{equ:KL_computation_as_criterion}
\dfrac{1}{2} \big\{-\log|\widetilde{\bm{\Sigma}}|+\text{Tr}(\widetilde{\bm{\Sigma}})+\widetilde{\bm{\mu}}^{\top}\widetilde{\bm{\mu}}-n\big\}
\end{equation}
where generalized variance $|\widetilde{\bm{\Sigma}}|$  and total variation $\text{Tr}(\widetilde{\bm{\Sigma}})$ both measure the dispersion of representations.  $KL(\n_p||p^r_Z)$ can be calculated in $O(n^3)$ where $n$ is the dimension.
%

\textbf{OOD Data in \textit{Category I} Problems}. 
As discussed in Subsection \ref{sec:whynotexplain}, OOD representations of \textit{Category I} problems (smaller/similar variance, higher/similar likelihood) tend to follow a Gaussian-like distribution. 
Similar to ID data, we can use fitted Gaussian distribution $\mathcal{N}_q$ to approximate $q_Z$ and estimate $KL(q_Z||p_Z^r)$.



\textbf{OOD Data in \textit{Category II} Problems}. 
Our normality test results (see Table \ref{tbl:SW_test} in the supplementary material) show that OOD representations in \textit{Category II} problems (larger variance, lower likelihood) do not follow a Gaussian-like distribution. However, we find that Equation \eqref{equ:KL_computation_as_criterion}  performs  even better on \textit{Category II} problems. The rationality of using Equation \eqref{equ:KL_computation_as_criterion} for \textit{Category II} problems can be explained both intuitively and theoretically. 

Intuitively, the first two items of Equation \eqref{equ:KL_computation_as_criterion} compensate each other. For \textit{Category I} problems (smaller/similar variance, higher/similar likelihood), OOD representations are less dispersed than ID representations and have a larger $-\log|\widetilde{\bm{\Sigma}}|$.  For \textit{Category II} problems, OOD representations tend to be more dispersed and have a larger $\text{Tr}(\widetilde{\bm{\Sigma}})$. Besides, we  find  OOD representations tend to have a larger $\widetilde{\bm{\mu}}^{\top}\widetilde{\bm{\mu}}$ than ID representations. Thus, Equation \eqref{equ:KL_computation_as_criterion} always produces a larger result for OOD than ID data. Note that the term $\widetilde{\bm{\mu}}^{\top}\widetilde{\bm{\mu}}$ alone cannot achieve high performance in GAD. It can also be manipulated  by moving the center of dataset (i.e., adding a vector to the input dataset). We can treat Equation \eqref{equ:KL_computation_as_criterion} as a more comprehensive statistic than that used in t-test, Maximum Mean Discrepancy, \textit{etc}. 

Theoretically, the following Theorem \ref{thm:lower_bound_kl_ood_pr} can explain the rationality of using Equation \ref{equ:KL_computation_as_criterion} in \textit{Category II} problems.

\begin{theorem}\label{thm:lower_bound_kl_ood_pr}
(see \cite{pardo2018statistical}) Let $\mathcal{N}_1(\bm{\mu}_1,\bm{\Sigma}_1)$ and $\mathcal{N}_1(\bm{\mu}_2,\bm{\Sigma}_2)$ be two $n$-dimensional Gaussian distributions. Assume that $\bm{Z}\sim P_Z(\z) $ is an arbitrary $n$-dimensional continuous random variable with mean vector $\bm{\mu}_1$ and covariance matrix $\bm{\Sigma}_1$, then
\begin{align}\nonumber
	KL(\mathcal{N}_1(\bm{\mu}_1,\bm{\Sigma}_1)||\mathcal{N}_1(\bm{\mu}_2,\bm{\Sigma}_2))
	\leq KL(P_Z(\z)||\mathcal{N}_1(\bm{\mu}_2,\bm{\Sigma}_2))
\end{align}
\end{theorem}

According to Theorem \ref{thm:lower_bound_kl_ood_pr}, when we use fitted Gaussian $\n_q$ from OOD representations, $KL(\n_q||p_Z^r)$ is a lower bound of $KL(q_Z||p_Z^r)$.  If the lower bound is large, $KL(q_Z||p_Z^r)$ must be large.


\textbf{Summary}. 
Equation \eqref{equ:KL_computation_as_criterion} is a unified conservative criterion for GAD due to the following reasons. 
\begin{enumerate}
\item For ID data, Equation \eqref{equ:KL_computation_as_criterion} approximates $KL(p_Z||p_Z^r)$ and should be small;
\item For OOD data whose representations follow a Gaussian-like distribution, Equation \eqref{equ:KL_computation_as_criterion} approximates $KL(q_Z||p_Z^r)$ and should be large;
\item For OOD data whose representations do not follow a Gaussian-like distribution, Equation \eqref{equ:KL_computation_as_criterion} computes the lower bound of $KL(q_Z||p_Z^r)$. If the lower bound is large, then $KL(q_Z||p_Z^r)$ must be large.
\end{enumerate}

Note that Equation \eqref{equ:KL_computation_as_criterion} also applies to Gaussian prior with diagonal covariance $diag(\bm{\sigma}^2)$ and mean $\bm{\mu}$. In such a case, we only need to normalize the data by a linear operation $Z'=(Z-\bm{\mu})/\bm{\sigma}$ while keeping $KL(p_{Z}||\mathcal{N}(\bm{\mu}, diag(\bm{\sigma}^2)))=KL(p_{Z'}||\mathcal{N}(0,I))$ (by Theorem \ref{thm:all_enough_distance_guarantee}). This equals to using Equation \eqref{equ:KL_compuation} directly. We also note that we are not pursuing precise divergence estimation or parameter estimation that are proven to be hard with very small batch sizes in high-dimensional problems.
\vspace{-8pt}
\subsubsection{VAE}\label{sec:method_vae}

It is well-known that VAE and its variations learn independent representations~\cite{burgess2018understanding, factorVAE,isolating2018Chen,kumar2017variational,locatello2019challenging}. 
In VAE, the probabilistic encoder $q_{\phi}(\bm{z}|\bm{x})$ is often chosen as Gaussian form $\mathcal{N}(\bm{\mu}(\bm{x}),diag(\bm{\sigma}(\bm{x})^2))$, where $\bm{z}\sim q_{\phi}(\bm{z}|\bm{x})$ is used as sampled representation, $\bm{\mu}(\bm{x})$ is used as mean representation.
The KL term in variational evidence lower bound objective (ELBO, see Equation \eqref{equ:VAE_ELBO} in the supplementary material) can be rewritten as 
$E_{p(x)}[KL(q_{\phi}(\bm{z}|\bm{x})||p(\bm{z}))]=I(\bm{x};\bm{z})+KL(q(\bm{z})||p(\bm{z}))$,
where $p(\bm{z})$ is the prior, $q(\bm{z})$ the aggregated posterior, and $I(\bm{x};\bm{z})$ the mutual information between $\bm{x}$ and $\bm{z}$ \cite{hoffman2016elbo}. Here the term $KL(q(\bm{z})||p(\bm{z}))$ pulls $p_Z$ to the Gaussian prior and encourages independent sampled representations. We also investigate the representations in VAE. The results show that: 
\begin{enumerate}
\item ID representations in VAE do not always have $p$-value greater than 0.05 in Shapiro-Wilk (normality) test;
\item the representations of all OOD datasets do not have $p$-value greater than 0.05 in normality test;
\item the representations of OOD datasets are more correlated (see Figure \ref{fig:heatmap_correlation_fashionmnist_mnist_notmnist_VAE}$\sim$\ref{fig:histo_correlation_VAE_train_cifar10} in  the supplementary material).
\end{enumerate}
Furthermore, there is no theoretical guarantee that $KL(q_Z||p_Z^r)$ is large enough because Theorem \ref{thm:all_enough_distance_guarantee} does not apply to non-diffeomorphisms. Nevertheless, we find that Equation \eqref{equ:KL_computation_as_criterion} also works for GAD with VAE.

\vspace{-5pt}
\subsection{Step 2: Leveraging Last-Scale KL Divergence}\label{sec:last_scale}
Although we can use Equation \eqref{equ:KL_computation_as_criterion} as a preliminary criterion for GAD, it is  expensive to compute the sample covariance of representations in $O(n^3)$ when the dimension $n$ reaches several thousand in flow-based model. We propose to use the last scale of representations instead.

Glow model uses multi-scale architecture and has three stages of representations \cite{glowopenai}. At the end of the first two stages, outputs are split into two parts $\bm{h}_i$ and $\bm{z}_i$ ($i=1,2$), where $\bm{h}_i$ is processed by the next stage. The output of the final stage (\textit{i.e.}, $\bm{z}_3$) contains a quarter of the whole dimensions. 
Among the three scales, the last scale is the most special one. Interpolating between two representations of the last scale can generate gradually varying images between two real-world images. 
Schirrmeister \textit{et al.} have shown that Glow network scales manifest a hierarchy of features \cite{understanding_anomaly_2020_nips}. Earlier scales learn low-level features that may be generic in different datasets. The last scale learns high-level features that are more specific to the training dataset. The results in \cite{understanding_anomaly_2020_nips} also demonstrate that the likelihood contributed by the last scale is a better metric than the whole likelihood for OOD detection. Other work such as \cite{HierarchicalVAEsKnowDontKnow} also demonstrates the effectiveness of the higher scale. 
Therefore, the last scale of OOD representations should differ more from ID representations than earlier stages.
More precisely, let $q_{Z_1}\sim q_{Z_3}$ be the marginal distribution of the three scales of OOD representations, respectively. We should observe $KL(q_{z_3}||\n(0,I))>KL(q_{Z_i}||\n(0,I))\ (i\in\{1,2\})$. 

Theoretically, 
similar to Theorem  \ref{thm:decompose_KL_ID}, we can decompose the whole KL divergence into local divergence inside each scale and total correlation between different scales as follows. 
\begin{align}\label{equ:decompose_KL_to_scales}
	&KL(p_Z(\z)||\mathcal{N})=\underbrace{KL(p_Z(\z)||p_{Z_{1,2}}(\z_1\z_2)p_{Z_3}(\z)}_{\text{total\ correlation between scales}  } \nonumber \\
	&+ \underbrace{ KL(p_{Z_{1,2}}(\z_1\z_2)||\mathcal{N})}_{\text{KL divergence from first two scales}}
	 + \underbrace{KL(p_{Z_{3}}(\z)||\mathcal{N})}_{\text{last-scale KL divergence}}
\end{align} 
where $\z=\z_1\z_2\z_3$, $\z_1\z_2\sim p_{Z_{1,2}}$, $\z_3\sim p_{Z_3}$ and $\mathcal{N}$ is standard Gaussian distribution. Figure \ref{fig:bigpicture2} shows the decomposition.
We call the last item of Equation \eqref{equ:decompose_KL_to_scales} as \textit{last-scale KL divergence}.
The rationality of using last-scale KL divergence as the criterion for OOD detection is based on the following inequality.
\vspace{-3pt}
\begin{align}\label{equ:KL_last_scale_metric}
	KL(q_{Z_3}||\n)>KL(p_{Z_3}||\n)
\end{align}

where $q_{Z_3}$ and $p_{Z_3}$ are the marginal distributions of the last scale of OOD and ID representations, respectively. Since the last scale contains fewer dimensions, we can efficiently calculate the last-scale KL divergence. For the non-Gaussian case, we can still rely on Theorem \ref{thm:lower_bound_kl_ood_pr} to compute the lower bound.
\vspace{-10pt}
\subsection{Step 3: Leveraging Group-Wise KL divergence in the Last Scale}\label{sec:PADmethod}
Up to now, we are still facing two issues. First, when batch size is small (e.g., $<$5), the performance of last-scale KL divergence is unsatisfactory. Second, the last-scale KL divergence does not support PAD. In this subsection, we address these two issues. The key idea is splitting representation into groups.

The factorizability of standard Gaussian distribution allows us to investigate representations in groups. 
Intuitively, if $\z\sim \mathcal{N}(0,I)$, then each dimension group of $\z$ follows $\mathcal{N}(0,I)$; Otherwise, it is unlikely that each part of $\z$ follows $\mathcal{N}(0,I)$.  Thus, we can split one single $\z$ into multiple subvectors and investigate these subvectors separately. This also generates multiple samples from one data point artificially. 
Formally, we split random vector $Z$ into $k$ $l$-dimensional ($k=n/l$) subvectors $\bar{Z}_1,\dots,\bar{Z}_k$. We note the marginal distribution of $\bar{Z}_i$ as $p_{\bar{Z}_i}$ ($1\leq i \leq k$). Then we can use the following Theorem \ref{thm:decompose_KL_PAD} to further decompose  the last-scale KL divergence.
\begin{theorem} \label{thm:decompose_KL_PAD}
Let  $X\sim p^*_{X}$ be an $n$-dimensional random vector. Note $X=\bar{X}_1\dots \bar{X}_k$ where  $\bar{X}_i\sim p^*_{\bar{X}_i}$ be the $i$-th $l$-dimensional ($k=n/l$) subvector of $X$, $\bar{X}_{ij}\sim p^*_{\bar{X}_{ij}}$ is the $j$-th element of $\bar{X}_{i}$. Then, 
\begin{scriptsize}
\vspace{-5pt}
\begin{align}
	\nonumber &KL(p^*_X(\x)||\mathcal{N}(0,I_n))\\
	=&\underbrace{KL(p^*_X(\x)||\prod_{i=1}^k p^*_{\bar{X}_i}(\x))}_{\text{$I_g[p^*_X]$}} +  \underbrace{\sum_{i=1}^k KL(p^*_{\bar{X}_i}(\x)||\mathcal{N}(0,I_l))}_{\text{$D_g[p^*_X]
		=\sum_{i=1}^k D^i_g[p^*_{\bar{X}_i}]$}\atop{\text{Group-wise KL divergence}} }  \label{equ:decompose_KL_group} \\
	\nonumber = & \underbrace{KL(p^*_X(\x)||\prod_{i=1}^k p^*_{\bar{X}_i}(\x))}_{\text{$I_g[p^*_X]$}} + 
	\underbrace{\sum_{i=1}^{k}KL(p^*_{\bar{X}_i}(\x)||\prod_{j=1}^l p^*_{\bar{X}_{ij}}(x))  }_{   \text{$I_{l}[p^*_{X}]=\sum_{i=1}^{k}I^i_{d}[p^*_{\bar{X}_i}] $}}\\
	&+ \underbrace{\sum_{i=1}^n KL(p^*_{X_i}(\z)||\mathcal{N}(0,1))}_{\text{$D_d[p^*_X]$}}\label{equ:decompose_KL_3parts}
\end{align}
\end{scriptsize}
\end{theorem}
\vspace{-5pt}
\begin{proof}
The proof of Theorem \ref{thm:decompose_KL_PAD} is similar to Theorem \ref{thm:decompose_KL_ID}. See Subsection \ref{sec:proof_theorem_decompose_KL_PAD} in the supplementary material for details.
$\hfill\square$ 
\end{proof}


In Equation \eqref{equ:decompose_KL_group},  $I_g$ is the generalized mutual information between dimension groups \cite{mutual_info_esti_2013}. $D_g$ is \textit{group-wise KL divergence}. Furthermore, in Equation \eqref{equ:decompose_KL_3parts} $D_g$ is decomposed as $I_l+D_d$, where $I_l$ is the generalized mutual information inside each group, $D_d$ is dimension-wise KL divergence that also occurs in Equation \eqref{equ:decompose_KL_basic}. Combining Equation \eqref{equ:decompose_KL_basic} and \ref{equ:decompose_KL_3parts}, we have $I_d=I_g+I_l$ and $D_g=I_l+D_d$.
Equation 
\eqref{equ:decompose_KL_group} distributes more divergence into the second term than Equation \eqref{equ:decompose_KL_basic}. 
In principle, there are multiple strategies to split $Z$ into $k$ subvectors $\bar{Z}_1,\dots,\bar{Z}_k$. 
The splitting strategy affects how the whole KL divergence is distributed into $I_g$ and $D_g$ in Equation \eqref{equ:decompose_KL_group}.
When $k=n$, Equation \eqref{equ:decompose_KL_group} is equal to Equation \eqref{equ:decompose_KL_basic}.

As shown in Figure \ref{fig:bigpicture2}, we can apply Theorem \ref{thm:decompose_KL_PAD} on $p_{Z_3}$ and $q_{Z_3}$ and get
\begin{align}
KL(p_{Z_3}||p^r_Z) =I_g[p_{Z_3}]+D_g[p_{Z_3}]=I_g[p_{Z_3}]+\sum\nolimits_{i=1}^{k}D^i_g[p_{\bar{Z}_i}] \nonumber\\
KL(q_{Z_3}||p^r_Z) =I_g[q_{Z_3}]+D_g[q_{Z_3}]=I_g[q_{Z_3}]+\sum\nolimits_{i=1}^{k}D^i_g[q_{\bar{Z}_i}] \nonumber
\end{align}
where $p_{\bar{Z}_i}$, $q_{\bar{Z}_i}$ are the marginal distributions of subvectors of the last scale of ID and OOD  representations, respectively. 
Combining Equation \eqref{equ:KL_last_scale_metric}, we can know
\begin{equation}\label{equ:decomposed_criterion}
\begin{aligned}
I_g[q_{Z_3}]+D_g[q_{Z_3}]
>
I_g[p_{Z_3}]+D_g[p_{Z_3}]
\end{aligned}
\end{equation}
\textbf{Final Criterion}.
Based on  the analysis up to now, we can obtain a final criterion for both GAD and PAD. Figure \ref{fig:bigpicture2} shows our analysis in this Section. For ID data, $KL(p_Z||\n)$ is trained to be small (see Subsection \ref{sec:general_case}).
According to Equation \eqref{equ:decompose_KL_to_scales}, the last-scale KL divergence $KL(p_{Z_3}||\n)=I_g[p_{Z_3}]+D_g[p_{Z_3}]$ must be smaller. 
We can assume the mutual information between  groups $I_g[p_{Z_3}]$ is sufficiently small, \textit{i.e.},  $I_g[p_{Z_3}]< \varepsilon$. To make Equation \eqref{equ:decomposed_criterion} hold, it suffices that the group-wise KL divergence part satisfies $D_g[q_{Z_3}]> D_g[p_{Z_3}]+\varepsilon$. 
If we choose an appropriate splitting strategy and distribute more divergence to group-wise KL divergence part ($D_g[q_{Z_3}]$) in Equation \eqref{equ:decompose_KL_group}, it is highly likely that we can make 
\vspace{-3pt}
\begin{align}\label{equ:groupwise_kl_lastscale_criterion}
	D_g[q_{Z_3}]>D_g[p_{Z_3}]
\end{align}
Then we can use group-wise KL divergence of the last scale $D_g$ as the criterion to detect OOD data. 

The remaining problems are: (1) how to choose a strategy to split $Z$ into $k$ subvectors so that more divergence is distributed into $D_g$ and (2) how to leverage group-wise KL divergence for OOD detection.
\vspace{-10pt}
\subsubsection{Splitting Strategy: Leveraging Local Pixel Dependence}\label{sec:splitting_strategy}
From Equation \eqref{equ:decompose_KL_group} and \eqref{equ:decompose_KL_3parts}, we can know a good splitting strategy should retain enough intragroup dependence in $I_l[q_{Z_3}]$ to make group-wise KL divergence part satisfy $D_g[q_{Z_3}]>D_g[p_{Z_3}]+\varepsilon$. 

Take the Glow model for example, the last scale has a shape of $(H\times W\times C)$ \footnote{The shape of the last scale of the representation in Glow is $4\times 4\times 48$.} where $H,W,C$ are the height, width, and the channels, respectively. We can split the last scale into multiple groups.
The most natural choices are as follows. 
\begin{enumerate}
\item \textit{\textbf{horizontal}}: treat dimensions in the same pixel position in different channels as one group and split $\z$ as $H\times W$ $C$-dimensional vectors; 
\item \textit{\textbf{vertical}}: treat dimensions in one channel as one group and split $\z$ as $C$ ($H\times W$)-dimensional vectors.
\end{enumerate}
Here horizontal strategy retains inter-channel dependence into group-wise KL divergence part (\textit{i.e.}, $D_g$). Vertical strategy retains pixel dependence into $D_g$.


Figure \ref{fig:splitting_strategy} in the supplementary material shows the idea behind this subsection.
Precisely, we split a single representation $\z$ into $k$ subvectors $\z_1,\dots, \z_k$ and treat $\z_i$ as a sample of random vector $\bar{Z}_i\sim p_{\bar{Z}_i}$. Then we can treat $\z_1,\dots, \z_k$ as $k$ samples of one random vector $\bar{Z}_m$ which follows a mixture of distributions $p_{\bar{Z}_m}=(1/k)\Sigma_{i=1}^k p_{\bar{Z}_i}$. If the $r$-th element $\bar{Z}_{i,r}$ and $s$-th element $\bar{Z}_{i,s}$ are strongly correlated for all $1\leq i\leq k$, we can say that $\bar{Z}_{m,r}$ and $\bar{Z}_{m,s}$ are also strongly correlated. More generally, if $\bar{Z}_1,\dots,\bar{Z}_k$ have a similar dependence structure, $\bar{Z}_m$ would also have a similar dependence structure.
Based on this intuition, we conduct experiments and find that OOD representations  manifest local pixel dependence. 
For example, we test ImageNet32 on Glow trained on SVHN. For each OOD dataset, we visualize the correlation between pixels. We find that in almost all channels each pixel always has stronger correlation with its neighbors. For example, Figure \ref{fig:heatmap_correlation_between_pixels_cifar10_imagenet_under_svhn_glow} in the supplementary material shows the correlation between each pixel with its neighbors in a randomly selected channel.
Therefore, we can say that $\bar{Z}_1,\dots,\bar{Z}_k$ tend to have a similar dependence structure. This means that the vertical strategy tends to retain more divergence to $D_g$.  
On the contrary, we cannot observe a similar dependence structure between channels when using the horizontal strategy. 
Thus, the vertical strategy can leverage pixel dependence of representations and is more suitable for OOD detection. 
Besides, we have also tried other splitting strategies. Evaluation results show that the vertical strategy is the best one. 

\vspace{-10pt}
\subsubsection{How to leverage Group-wise KL Divergence in the Last Scale}
We want to leverage group-wise KL divergence $D_g$ for OOD detection.
For ID data, we treat each representation as $k$ data points sampled from a mixture of distributions $p_{\bar{Z}_m}=(1/k) \sum_{i=1}^k p_{\bar{Z}_i}$ where $p_{\bar{Z}_i}(1\leq i\leq k)$ is very close to $\mathcal{N}(0,I)$. Thus, we can use a single Gaussian distribution 
$\n_{\bar{Z}_s}$ to approximate each $p_{\bar{Z}_i}$.    
Therefore, $D_g[p_{Z_3}]$ can be approximated as
\vspace{-5pt}
\begin{equation}\label{equ:approximation_split_KL}
\begin{aligned}
D_g[p_{Z_3}]  = \sum_{i=0}^k KL(p_{\bar{Z}_i}(\z)||\mathcal{N}(0,I))
 \approx k\times KL(\n_{\bar{Z}_s}||\mathcal{N}(0,I))
\end{aligned}
\end{equation}
Now we can plug Equation \eqref{equ:KL_computation_as_criterion} in Equation 
\eqref{equ:approximation_split_KL}, except that each representation $\z$ is treated as $k$ samples of $p_{\bar{Z}_m}$.

For OOD data, we cannot use a single Gaussian distribution to approximate $q_{\bar{Z}_m}=(1/k) \sum_{i=1}^k q_{\bar{Z}_i}$ when $q_{\bar{Z}_i}$ are far from each other or $q_Z$ is not Gaussian-like. Nevertheless, we  can still use fitted Gaussian and Equation \eqref{equ:KL_computation_as_criterion} to compute the lower bound according to Theorem \ref{thm:lower_bound_kl_ood_pr}. 



\textbf{Summary}. 
Overall, we get the following answer to Q2. 
\vspace{-5pt}
\begin{framed}
\textbf{\textit{Answer to Q2}}: We use group-wise KL divergence in last scale (\textit{i.e.}, $D_g$ in Equation \ref{equ:decompose_KL_group}) as a unified criterion for both GAD and PAD with flow-based generative models. 
\end{framed}
\begin{figure}[t]
\centering
\includegraphics[width=7cm]{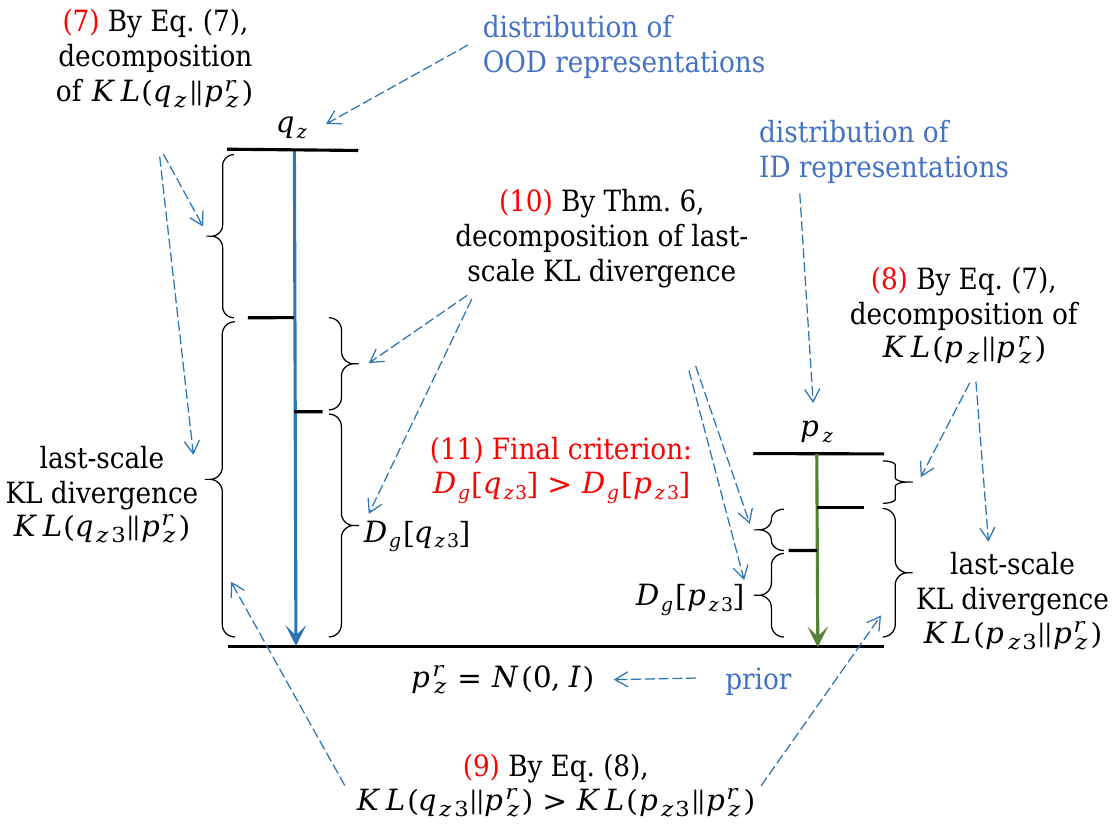}
\vspace{-5pt}
\caption{Decomposition of KL divergence for OOD detection. Steps (1) $\sim$ (6) are shown in Figure \ref{fig:theory_and_algorithm}. 
}  
\label{fig:bigpicture2}
\end{figure}

\vspace{-10pt}

\subsection{Algorithm}\label{sec:ood_detection}

Algorithm \ref{alg:GAD} shows the details of our OOD detection method. 
The inputs are a set of data points $\bm{X}=\{\bm{x}_1,\cdots,\bm{x}_m\}(m\geq 1)$, where each $\x_i$ is an individual input.
Here we support both GAD and PAD (in the case of $m=1$). 
For each input $\x_i$, we first compute the representation  $\z_i=f(\x_i)$  (line \ref{alg:getz}).
Here $f$ represents flow-based model or the encoder part of VAE.
Then we normalize representations $\{\bm{z}_1,\cdots,\bm{z}_m\}$ as  $\bar{\z}_i=(\bm{z}_i-\bm{\mu})/\bm{\sigma}$, where $\bm{\mu}$ and $diag(\bm{\sigma}^2))$ are the mean and covariance matrix of Gaussian prior, respectively (line \ref{alg:normalize}). If  $\z_i$ has multiple stages, we choose only the last-scale representation to leverage last-scale KL divergence (line \ref{alg:last_scale}, Section \ref{sec:last_scale}). Then we split $\bar{\z}_i$ as $C$ $(H\times W)$-dimensional subvectors to leverage local pixel dependence as discussed in Subsection \ref{sec:splitting_strategy}. We collect the subvectors from each $\x_i$ and treat $\bar{\bm{Z}}$ as $m\times C$ data points sampled from a mixture of distributions (line \ref{alg:collectz}).
Then we calculate the sample covariance $\widetilde{\bm{\Sigma}}$ and sample mean $\widetilde{\bm{\mu}}$ of $\bar{\bm{Z}}$ (line \ref{alg:calculate}).
Finally, we use the following \textit{anomaly score} 
\vspace{-3pt}
\begin{align}\label{equ:score}
score=(1/2) \big\{-\log|\widetilde{\bm{\Sigma}}|+\text{Tr}(\widetilde{\bm{\Sigma}})+\widetilde{\bm{\mu}}^{\top}\widetilde{\bm{\mu}}-n\big\}	
\end{align}
as the criterion (line \ref{alg:criterion}). For ID data, $score$ is the estimated group-wise KL divergence in last scale (\textit{i.e.}, $D_g[p_{Z_3}]$) except for neglecting the constant $k$ in Equation \eqref{equ:approximation_split_KL}. For OOD data, $score$ is the lower bound of $D_g[q_{Z_3}]$.
The larger $score$ is, the more like OOD the input.
If $score$ is greater than a threshold $t$,  the input  is determined as OOD data (line \ref{alg:ood}). Otherwise, the input  is determined as ID data (line \ref{alg:id}). 

We name our method as \textbf{\PADmethod}\ for \textit{KL divergence-based Out-of-Distribution Detection with Split representations}. We also call our method without split representations as \textbf{\GADmethod}.

\begin{algorithm}[t]
\caption{KL divergence-based Out-of-Distribution Detection with Split representations (\PADmethod\ )} 
\label{alg:GAD}  
\begin{algorithmic}[1] 
\STATE {\bfseries Input:}  
$f(\bm{x})$: a well-trained flow-based model or the encoder of VAE using Gaussian prior $\mathcal{N}(\bm{\mu},diag(\bm{\sigma}))$;
$\bm{X}=\{\bm{x}_1,\cdots,\bm{x}_m\}(m\geq 1)$: a batch of inputs; $(H,W,C)$: the height, width, and channels of last scale of representations.
$t$: threshold
\STATE $\bar{\bm{Z}}=\emptyset$
\FOR {$i=1$ to $m$}
\STATE $\bm{z}_i=f(\bm{x}_i)$\label{alg:getz}
\STATE $\bar{\bm{z}}_i=(\bm{z}_i-\bm{\mu})/\bm{\sigma}$\label{alg:normalize}
	\IF {$\bm{z}_i$ consists of multiple stages} \label{alg:last_scale}
	\STATE $\bar{\bm{z}}_i=$ last scale of $\bar{\bm{z}}_i$
	\ENDIF
\STATE split  $\bar{\z}_i$ as $C$ $(H\times W)$-dimensional subvectors $\bar{\z}_{i,1}\dots \bar{\z}_{i,C}$ \label{alg:split}
\STATE $\bar{\bm{Z}}=\bar{\bm{Z}}\cup \{\bar{\z}_{i,1},\dots,\bar{\z}_{i,C}\}$ \label{alg:collectz}
\ENDFOR
\STATE calculate sample covariance $\widetilde{\bm{\Sigma}}$ and sample mean $\widetilde{\bm{\mu}}$ of $\bar{\bm{Z}}$\label{alg:calculate}
\STATE $score=(1/2) \big\{-\log|\widetilde{\bm{\Sigma}}|+\text{Tr}(\widetilde{\bm{\Sigma}})+\widetilde{\bm{\mu}}^{\top}\widetilde{\bm{\mu}}-n\big\}$ \label{alg:criterion}
\IF {$score > t$} \label{alg:line_threshold}
\STATE return ``$\bm{X}$ is OOD data'' \label{alg:ood}
\ELSE
\STATE return ``$\bm{X}$ is ID data''\label{alg:id}
\ENDIF		
\end{algorithmic} 
\end{algorithm}

\vspace{-10pt}
\subsection{Summary}
In Figure 1, we have illustrated our analysis steps in explaining why we cannot sample OOD data. In Figure \ref{fig:bigpicture2}, we summarize how to leverage KL divergence for OOD detection in Section \ref{sec:GADmethod}.
To help readers have a bird's eye view of our whole work, \textbf{we summarize all critical steps in a big flowchart in Figure \ref{fig:flowchart} in the supplementary material}.

\vspace{-10pt}
\section{Experiments}\label{sec:experiment}

We conduct experiments to evaluate the effectiveness and robustness of our OOD detection method.

%
\vspace{-10pt}
\subsection{Experimental Setting}\label{sec:exp_setting}
\textbf{Benchmarks}.
We evaluate our method with prevalent benchmarks in deep anomaly detection research \cite{nalisnick2018deep, nalisnick2019detecting, lee2018simple,shafaei2018digitnotcat,hendrycks2016baseline,hendrycks2018deep}, including 
Constant,
Uniform,
MNIST \cite{mnist}, 
FashionMNIST \cite{fashionmnist},
notMNIST \cite{notmnist}, 
KMNIST \cite{KMNIST},
Omniglot \cite{Omniglot},
CIFAR-10/100 \cite{cifar}, 
SVHN \cite{svhn}, 
CelebA \cite{celeba},
TinyImageNet \cite{tinyimagenet},
ImageNet32 \cite{imagenet},
and LSUN \cite{LSUN}.
We use different dataset compositions falling into \textit{Category I} (smaller/similar variance, higher/similar likelihoods, \textit{e.g.}, CIFAR-10 vs SVHN) and \textit{Category II} (larger variance, lower likelihoods, \textit{e.g.}, SVHN vs CIFAR-10) problems. 
All datasets are resized to $32\times 32 \times 3$ for consistency.
We use $S$-C($k$) ($k\ge 0$) to denote dataset $S$ with adjusted contrast by a factor $k$. 
More details of the benchmarks are described in Section \ref{sec:benchmarks} in the supplementary material.

\textbf{Baselines}. 
We choose the following recently published OOD detection methods as baselines.

\textbf{GAD}:
\begin{enumerate}
\setlength{\itemsep}{0pt}
\setlength{\parskip}{0pt}
\item \textbf{$t$-test}: two-sample students' $t$-test for a difference in means in the empirical likelihoods. 

\item \textbf{Kolmogorov-Smirnov test (KS-test)}: two-sample KS-test to the likelihood empirical distribution functions.

\item \textbf{Maximum Mean Discrepancy (MMD)} \cite{MMD2012}: two-sample MMD test.

\item \textbf{Kernelized Stein Discrepancy (KSD)} \cite{KSD2016}: test for Goodness of Fit to the generative model.

\item \textbf{Annulus Method} \cite{choi2018generative}: Typicality test in latent space. Inputs whose latents are far from the annulus with radius $\sqrt{n}$ are classified as OOD data. 

\item \textbf{Ty-test} \cite{nalisnick2019detecting}: typicality test in model distribution. 

\item \textbf{GOD2KS} \cite{jiang2022revisiting}: combining  random projection and two-sample KS test.

\end{enumerate}
Among the above GAD methods, Annulus Method, Ty-test, and GOD2KS are the best ones. 
We reimplement Annulus Method and Ty-test to produce more results. 

\textbf{PAD}:
\begin{enumerate}
	\setlength{\itemsep}{0pt}
	\setlength{\parskip}{0pt}
\item \label{input_complexity} $\bm{\mathcal{S}}$ \cite{2020InputComplexity}: input complexity compensated likelihood.

\item $\bm{L}_{last}$ \cite{understanding_anomaly_2020_nips}: likelihood contributed by the last-scale representation of Glow. 

\item \textbf{DoSE}  \cite{morningstar21KDE}: density estimators on the statistics of models to detect OOD data. 
\item \textbf{ODIN}  \cite{ODIN2017}: Liang \textit{et al.} introduce ODIN method for OOD detection.
\item \textbf{Joint confidence loss} \cite{confidenceloss2018ICLR}: Lee \textit{et al.} introduce joint confidence loss for OOD detection. 
\item \textbf{Joint confidence loss+ODIN} \cite{confidenceloss2018ICLR}: combination of Joint confidence loss and ODIN (better than each method alone). 
\end{enumerate}
For a more comprehensive evaluation, we reimplement the first three PAD baselines applicable to flow-based model. 
DoSE is the SOTA PAD method applicable to flow-based model.
The rest baselines apply to classification networks rather than flow-based models.
See Section \ref{sec:more_discussion} in the supplementary material for more discussion about baselines.

\textbf{Models}.
We use the official Glow model \cite{glowopenai} and the model released by the authors of Ty-test (DeepMind \cite{glowdeepmind}). 
See Section \ref{sec:appendix_model_details} in the supplementary material for details.

\textbf{Metrics}. We use the same metrics as baseline methods in their original publications. These metrics include  false positive rate (\textbf{FPR}), true positive rate (\textbf{TPR}), threshold-independent metrics area under the receiver operating characteristic curve (\textbf{AUROC}) and area under the precision-recall curve (\textbf{AUPR}) \cite{buckland1994roc}, and threshold-dependent Equal Error Rate (\textbf{EER}). 
We treat OOD data as positive data.   
For GAD, each dataset is shuffled and then divided into groups of size $m$. 
We run each method for 5 times and show ``mean$\pm$ standard deviation'' for each GAD problem.

\vspace{-10pt}
\subsection{Experimental Results}\label{sec:experimental_results}

\begin{figure*}[t]
\centering
\vspace{-0pt}
\subfigure{
\begin{minipage}[t]{4.1cm}
	\includegraphics[width=4.1cm]{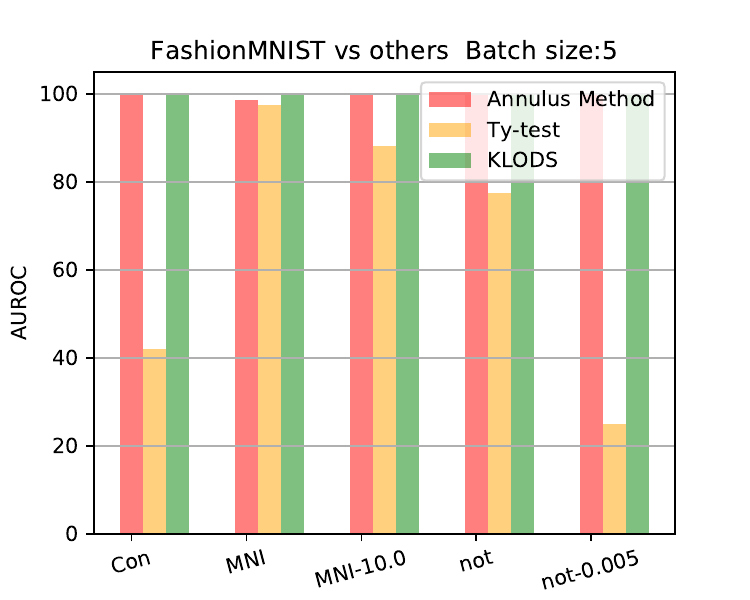}
	\label{fig:GAD_bar_fashion_vs_others_bs5}
\end{minipage}
}
\subfigure{
\begin{minipage}[t]{4.1cm}
	\centering
	\includegraphics[width=4.1cm]{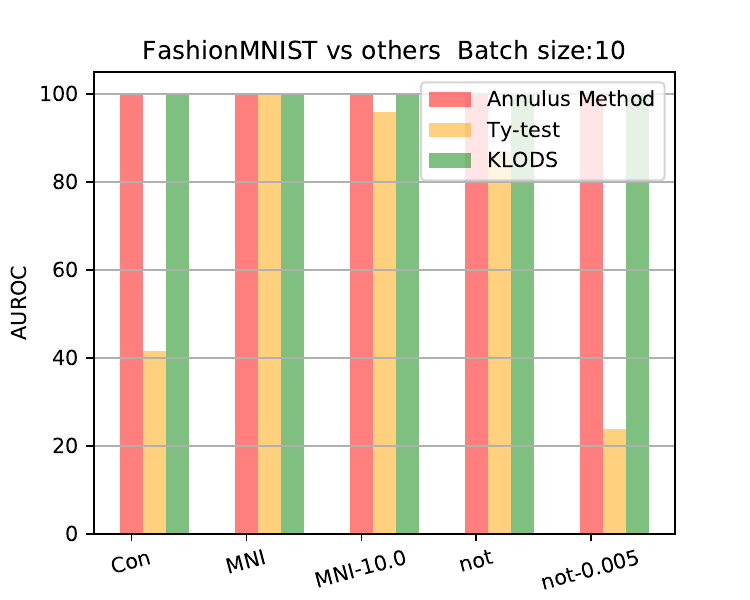}
	\label{fig:GAD_bar_fashion_vs_others_bs10}
\end{minipage}%
}%
\subfigure{
\begin{minipage}[t]{4.2cm}
	\includegraphics[width=4.2cm]{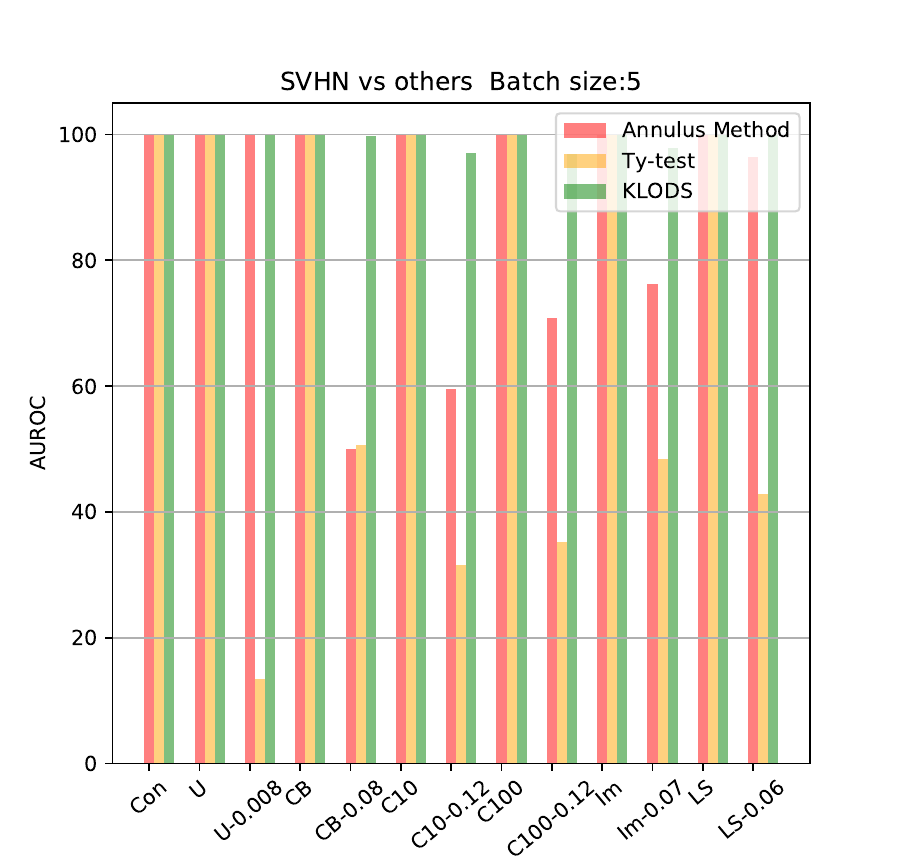}
	\label{fig:GAD_bar_svhn_vs_others_bs5}
\end{minipage}
}
\subfigure{
\begin{minipage}[t]{4.2cm}
	\includegraphics[width=4.2cm]{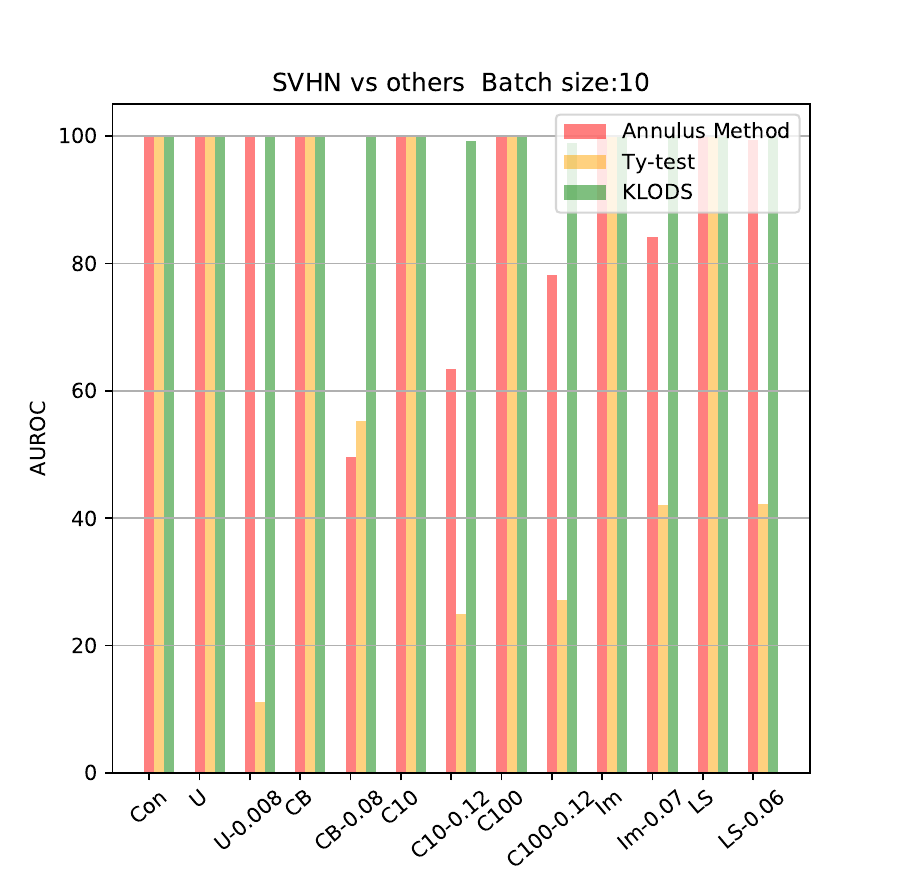}
	\label{fig:GAD_bar_svhn_vs_others_bs10}
\end{minipage}
}
\vspace{-20pt}

\subfigure{
\begin{minipage}[t]{4.3cm}
	\includegraphics[width=4.3cm]{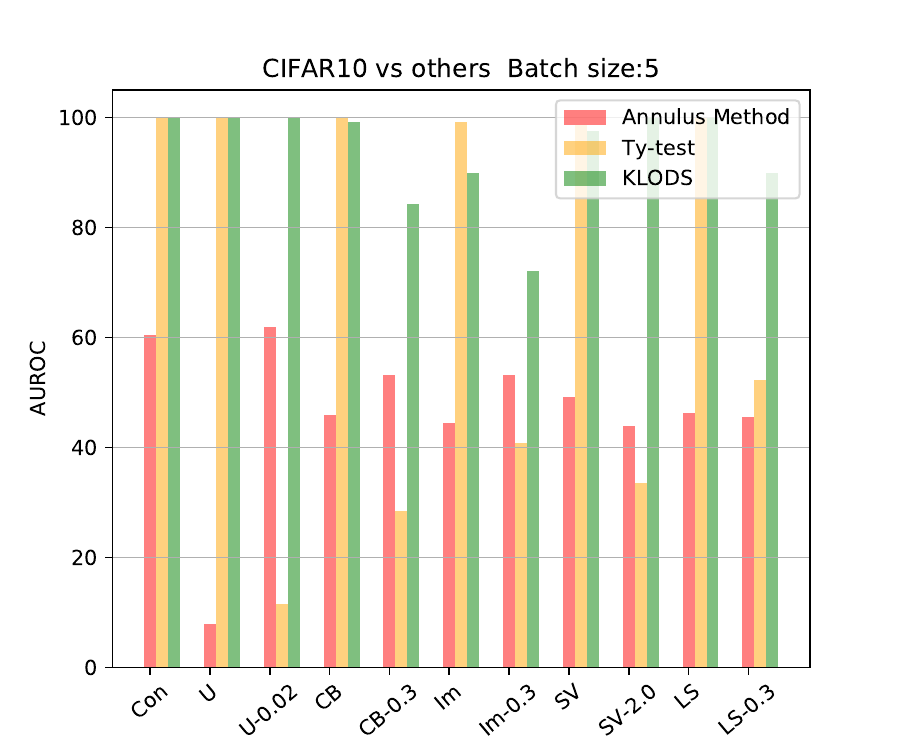}
	\label{fig:GAD_bar_cifar10_vs_others_bs5}
\end{minipage}
}
\subfigure{
\begin{minipage}[t]{4.3cm}
	\centering
	\includegraphics[width=4.3cm]{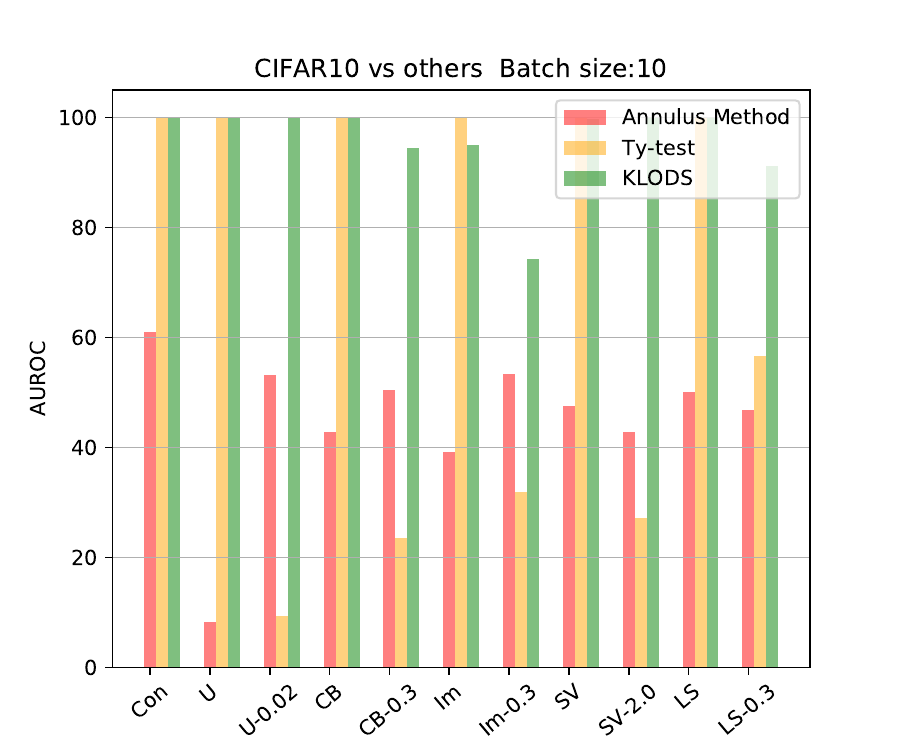}
	\label{fig:GAD_bar_cifar10_vs_others_bs10}
\end{minipage}%
}%
\subfigure{
\begin{minipage}[t]{4.25cm}
	\includegraphics[width=4.15cm]{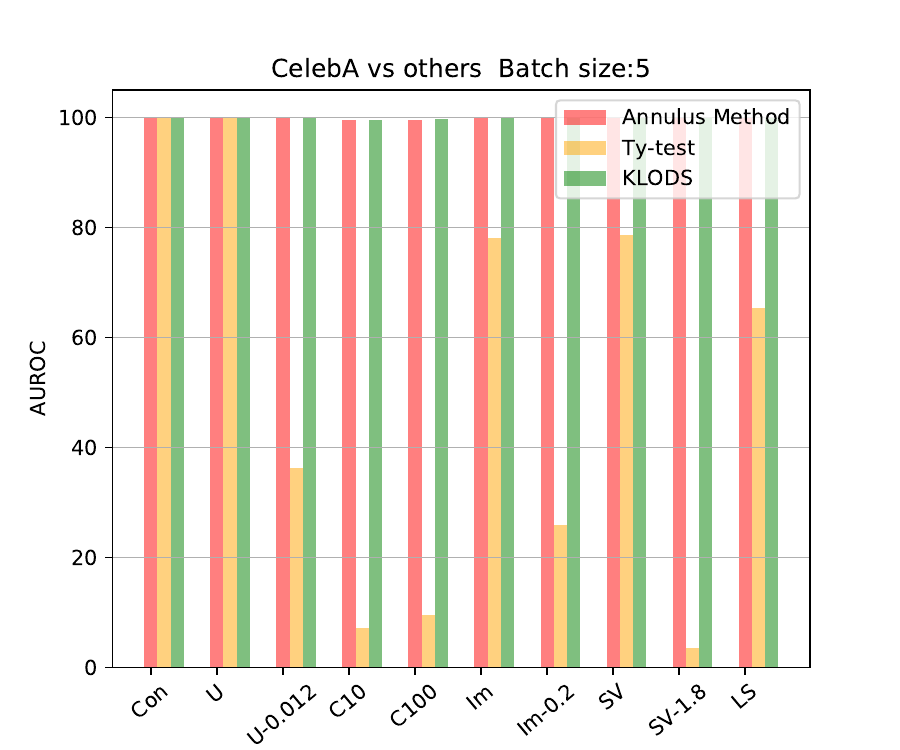}
	\label{fig:GAD_bar_celeba_vs_others_bs5}
\end{minipage}
}
\subfigure{
\begin{minipage}[t]{4.25cm}
	\includegraphics[width=4.15cm]{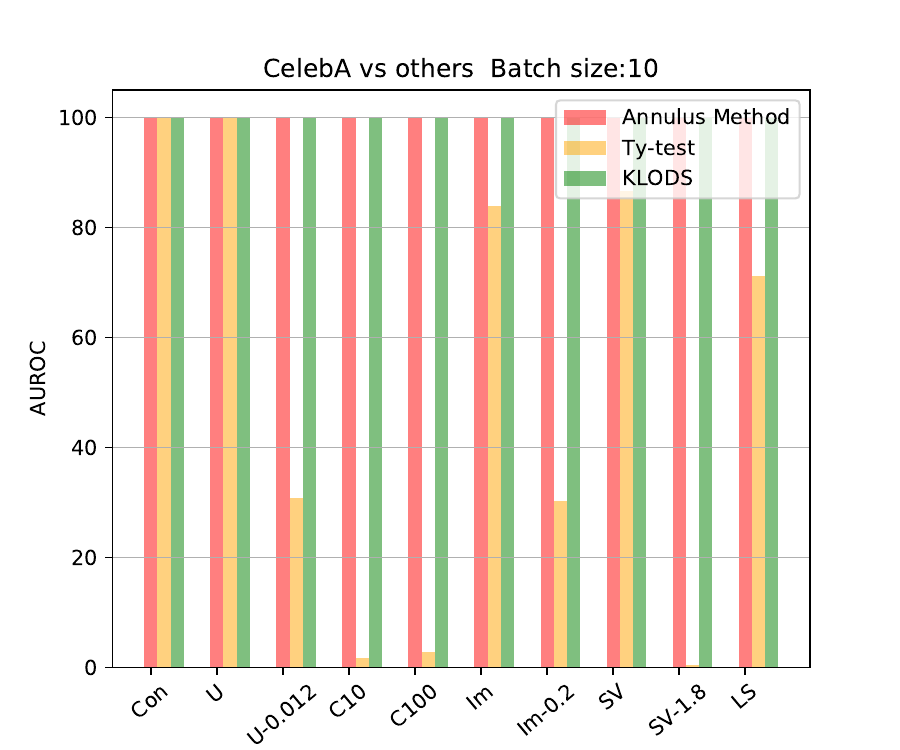}
	\label{fig:GAD_bar_celeba_vs_others_bs10}
\end{minipage}
}
\vspace{-15pt}
\caption{GAD results (AUROC) on Glow with batch sizes 5 and 10. The X-axis are labeled with OOD datasets including Con: Constant, MNI: MNIST, not: notMNIST, U: Uniform, CB: CelebA, C10/100: CIFAR-10/100, Im: ImageNet, SV: SVHN, LS: LSUN. 
The corresponding numerical results are shown in Table \ref{tbl:glow_fashionmnist_svhn_cifar10_celeba_DOCR_TC_M_bs_5_10_S2}, \ref{tbl:glow_fashionmnist_svhn_cifar10_celeba_DOCR_TC_M_bs_5_10_S2_vs_annulus_method}, and \ref{tbl:glow_fashionmnist_svhn_cifar10_celeba_DOCR_TC_M_EER_bs_5_10_S2} in the supplementary material.}  
\label{fig:GAD_results_bar}
\end{figure*}

\subsubsection{Group Anomaly Detection}\label{sec:GAD_results}
\textbf{Main Results on Unconditional Glow}.

\textbf{FashionMNIST vs Others}.
Table \ref{tbl:glow_GAD_many_baselines} in the supplementary material shows the GAD results of Glow trained on FashionMNIST.
The ID column reflects false positive rate (ideally should be 0). The MNIST and notMNIST columns reflect true positive rate (ideally should be 1).  The authors of baselines apply bootstrap procedure on validation data to establish thresholds. 
See Section \ref{appendix:sec:GADresults} in the supplementary materials for the details on how we establish thresholds.
We can see that all methods cannot achieve satisfactory results with small batch size $m=2$. Our method achieves the highest true positive rate with the lowest false positive rate for larger batch sizes (\textit{i.e.}, 10, 25).

The first two subfigures of Figure \ref{fig:GAD_results_bar} show the comparison of \PADmethod, our reimplementation of Ty-test, and Annulus Method on FashionMNIST vs Others. The corresponding numerical results of Figure \ref{fig:GAD_results_bar} are shown in Table \ref{tbl:glow_fashionmnist_svhn_cifar10_celeba_DOCR_TC_M_bs_5_10_S2}, \ref{tbl:glow_fashionmnist_svhn_cifar10_celeba_DOCR_TC_M_bs_5_10_S2_vs_annulus_method}, and \ref{tbl:glow_fashionmnist_svhn_cifar10_celeba_DOCR_TC_M_EER_bs_5_10_S2} in the supplementary material. In our reimplementation, Annulus Method achieves much better results than that reported in \cite{nalisnick2019detecting} (and Table \ref{tbl:glow_GAD_many_baselines} in the supplementary material).
Nevertheless, our method outperforms all baselines significantly.

\textbf{SVHN/CIFAR-10/CelebA vs Others}. Figure \ref{fig:GAD_results_bar} also shows the GAD results on Glow trained on SVHN, CIFAR-10, and CelebA. Our method is the best one. 
We adjust the contrast of OOD dataset to make the likelihood distributions of ID and OOD data coincide. For these kinds of problems, the performance of Annulus Method and Ty-test degenerate severely.
Our method is more robust against data manipulation. 

CelebA vs CIFAR-10/100 are challenging for Ty-test as reported by \cite{nalisnick2019detecting}.  Our method can achieve 100\% AUROC with batch size 10.  In our experiments, it is hard to make the likelihood distributions of CelebA train and test split fit very well on the official Glow model \footnote{We stop training after 2,000 epochs.}. This affects the performance of Ty-test. Please see Section \ref{appendix:sec:GADresults} in the supplementary material for more discussion.


CIFAR-10 vs CIFAR-100 is one of the most challenging problems. Annulus Method and Ty-test achieve 47.2\% and 72.4\% AUROCs with batch size $m=200$, respectively.
\GADmethod\ and \PADmethod\ achieve around 70\% AUROC when the batch size $m=200$.
We think this is due to unsuccessful model and the similarity between ID and OOD datasets.
Please see Section \ref{sec:more_discussion} in the supplementary material for more discussion on CIFAR-10 vs CIFAR-100.

\textbf{Smaller Batch Sizes}. \PADmethod\ outperforms Ty-test when batch size is smaller (\textit{i.e.}, $2\sim 4$). 
See Table \ref{tbl:glow_fashionmnist_svhn_cifar10_celeba_DOCR_TC_M_bs_2_4_S2} in the supplementary material for details.

\textbf{Comparison with GOD2KS}. Table \ref{tbl:glow_fashionmnist_svhn_cifar10_celeba_DOCR_TC_M_bs_5_10_S2_vs_GOD2KS} in the supplementary material compares our method and GOD2KS \cite{jiang2022revisiting} on Glow. We use the same problems reported in \cite{jiang2022revisiting}. Our method outperforms GOD2KS. 




\textbf{Robustness}. 
The results presented above have demonstrated the robustness of our method against data manipulation method \textbf{M2} (adjusting contrast). Experimental results show that \PADmethod\ achieves the same performance under  \textbf{M1} (rescaling representations), except that a slightly larger batch size (+5) is needed for CIFAR-10-related problems. As shown in Figure \ref{fig:GAD_results_bar}, Table \ref{tbl:glow_fashionmnist_svhn_cifar10_celeba_DOCR_TC_M_bs_5_10_S2}, \ref{tbl:glow_fashionmnist_svhn_cifar10_celeba_DOCR_TC_M_bs_5_10_S2_vs_annulus_method}, and \ref{tbl:glow_fashionmnist_svhn_cifar10_celeba_DOCR_TC_M_EER_bs_5_10_S2}  in the supplementary material, Annulus Method and Ty-test is affected by data manipulation M2 (adjusting contrast). Besides, \textit{Annulus Method achieves exactly \textbf{0 AUROCs} for all problems under data manipulation \textbf{M1} (rescaling representations)}. This is because all OOD representations are rescaled to the annulus of typical set of prior and hence definitely closer to the typical set annulus than ID representations (see Section \ref{sec:attack_likelihood}). 
The results are omitted for brevity. Currently, the performance of GOD2KS under data manipulations is still not clear.

\vspace{-5pt}
\begin{framed}
	\textbf{Summary}. For GAD, our method achieves 98.1\% AUROC, 98.2\% AUPR, and 4.6\% EER on average with batch size 5 and outperforms Ty-test by 33.5\%, 29.2\%, 29.3\% on average in AUROC, AUPR, and EER, respectively.
	Our method also outperforms GOD2KS by 9.1\%, 12.1\% on average in AUROC and AUPR with batch size 5, respectively. 
	Our method is robust against data manipulations, while the baseline methods Ty-test and Annulus Method can be attacked in almost all cases.
\end{framed}

\vspace{-5pt}
\noindent\textbf{More Results}.

\textbf{Mixture of OOD Datasets}.
We also use the mixture of two datasets among SVHN, CelebA, and CIFAR-10 as one OOD dataset. We can treat samples from multiple distributions as from a mixture of distributions.  We randomly choose 5,000 samples from each dataset and get 10,000 samples as one OOD dataset. Table \ref{tbl:glow_mixture_ood_datasets_DOCR_TC_M_bs_5_10_S2} in the supplementary material shows the results of \PADmethod.
Our method outperforms Ty-test by 38.9\% AUROC on average with batch size 5.

\textbf{Ablation Study}. We compare the following four methods to evaluate how the techniques proposed in Section \ref{sec:GADmethod} affect the performance.
\begin{enumerate}
\item \textbf{Ty-test}: the baseline.
\item \textbf{\GADmethod-all}: GAD with all scales of representation, without splitting dimensions.
\item \textbf{\GADmethod}: GAD  using the last-scale representation, without splitting dimensions.
\item \textbf{\PADmethod}: GAD using the last-scale representation with splitting dimensions.

\end{enumerate}

Table \ref{tbl:glow_fashionmnist_svhn_cifar10_celeba_ablation_study} in the supplementary material shows the results. Neglecting CIFAR10 vs ImageNet32-C(0.3), the order of the methods by performance is \PADmethod\ $>$ \GADmethod\ $>$ \GADmethod-all $>$ Ty-test. The only exception is CIFAR10 vs ImageNet32-C(0.3), where KLOD outperforms KLODS. The low contrast leads to weak local pixel dependence and affects our splitting strategy.
Overall, we can see that both using the last scale and splitting dimensions into groups can improve the performance of GAD.  Note that splitting dimensions also makes PAD feasible. Besides, when the batch size is smaller (\textit{e.g.}, 5), \PADmethod\ outperforms \GADmethod\ more obviously. More results are not shown for brevity.



\textbf{One-vs-Rest}.
We conduct one-vs-rest evaluation on MNIST. For each class from 0 to 9, we use images in that class as ID data and the rest classes as OOD data. We train one Glow model under each setting for 120 epochs. As shown in Table \ref{tbl:glow_GAD_mnist_one_vs_rest} in the supplementary material, our method achieves 85.4\% AUROC and 85.8\% AUPR on average with batch size 5, outperforming the baseline by 8\% and 5\%, respectively.



\textbf{GAD on GlowGMM}.
We train  GlowGMM on FashionMNIST. 
We treat each class as ID data and the rest as OOD data. \PADmethod\ can achieve 100\% AUROC on average when batch size is 25. On the contrary, Ty-test is worse than random guessing in most cases. See Figure \ref{fig:glowgmm_fashionmnist_bar} and Table \ref{tbl:cond_glow_fashionmnist_CTR_TC} in the supplementary material for results.
Experimental results also demonstrate that each component may assign higher likelihoods to other classes (See Table \ref{tbl:cond_glow_fashionmnist_logpz_classification} in the supplementary material). 

\textbf{Generating OOD Images Using GlowGMM}.  
In GlowGMM, we can generate \textit{high-quality} OOD images.
See Section \ref{sec:more_results} in the supplementary material for more discussion. 

\textbf{GAD on VAE}.
We train convolutional VAE with 8-/16-/32-dimensional latent space on FashionMNIST, SVHN, and CIFAR-10, respectively.   
The latent space is not large enough, so we did not split representations and only used \GADmethod\ in experiments. 
The results are shown in Figure \ref{fig:GAD_VAE_results_bar} and Table \ref{tbl:vae_fashionmnist_svhn_cifar10_CTR_TC}  in the supplementary material.
\GADmethod\ achieves 99.9\% AUROC on average when $m=25$ for most problems. CIFAR-10 vs CIFAR-100 is also the most challenging problem on VAE.
\GADmethod\ needs a batch size of 150 to achieve 98\%+ AUROC (See Table \ref{tbl:vae_cifar10_cifar100_imagenet32_ctr_tc} in the supplementary material). Nevertheless, \GADmethod\ still outperforms Ty-test. Again, Ty-test can be attacked by data manipulations \textbf{M2} (adjusting contrast). 
Finally, as pointed out by \cite{An2015VariationalAB}, for vanilla VAE the reconstruction probability is not a reliable criterion for OOD detection (See Table \ref{tbl:vae_cifar10_reconstruction_probability} in the supplementary material). 
\vspace{-5pt}
\subsubsection{Point-wise Anomaly Detection}

The PAD results of \PADmethod, $\mathcal{S}$, $L_{last}$, and DoSE are shown in Table \ref{tbl:PAD_all_results}.

\textbf{SVHN vs Others}.
The problems above the dash line in Table \ref{tbl:PAD_all_results} fall in \textit{Category II} (larger variance, lower likelihoods). \PADmethod\ can achieve 98.8\%$+$ AUROC and outperforms the baselines. 
In \cite{2020InputComplexity}, although the authors state that their method $\mathcal{S}$ can detect OOD data with more complexity than ID data (roughly \textit{Category II}), they did not evaluate their method thoroughly on \textit{Category II} problems. We find $\mathcal{S}$ does not perform well on these problems.

The problems for SVHN vs others below the dash line in  Table \ref{tbl:PAD_all_results} fall in \textit{Category I} (smaller/similar variance, higher likelihoods). For these problems, 
$L_{last}$ and DoSE degenerate into being not better than random guessing.
\PADmethod\ is comparable with  $\mathcal{S}$ and outperforms $L_{last}$ and DoSE significantly. The reason is all the distributions of $\log p(\bm{x})$, $\log p(\bm{z})$, and $\log p(\bm{x})$ contributed by the last scale overlap with those of ID data. Figure \ref{fig:logpx_glow_SVHN_vs_others} and \ref{fig:logpx_last_scale_glow_SVHN_vs_others_hierarchical_model} in the supplementary material shows the histograms of these three statistics. 
These issues make DoSE fail because DoSE relies on the effectiveness of its based statistics.

\begin{table}[t]
	\centering
	\scriptsize
	\caption{PAD results (AUROC in percentage) on Glow. 
		Notable failures (below 60\%) are underlined.
	}
	\vspace{-5pt}
	\label{tbl:PAD_all_results}
	\begin{tabular}{ c l c c c c }
		\toprule[1pt]
		ID&OOD& $\mathcal{S}$ & $L_{last}$ & DoSE& \PADmethod\\ 
		\cline{1-6}
		\multirow{14}*{\rotatebox{90}{\tabincell{c}{ SVHN}}}
		&Uniform & \textbf{100.0} & \textbf{100.0} & \textbf{100.0} & \textbf{100.0}  \\ 
		&ImageNet32 & 78.7 & 99.8 & \textbf{99.9} & \textbf{99.9}  \\ 
		&CelebA & 83.1 & 100.0 & \textbf{100.0} & \textbf{100.0}  \\ 
		&CIFAR-10 & \underline{43.8} & 97.7 & 96.2 & \textbf{98.9}  \\ 
		&CIFAR-100 & \underline{44.9} & 97.3 & 96.5 & \textbf{98.8}   \\ 
		&LSUN & 91.8 & \textbf{100.0} & 91.6 & \textbf{100.0}\\
		\cdashline{2-6}[1pt/1pt]
		&Uniform-C(0.008) & 97.9 & 0.0 & 96.8 &\textbf{98.6} \\
		&CelebA-C(0.08)  & 81.4 & \underline{41.6} & \underline{48.0} & \textbf{82.2}   \\ 
		&CIFAR-10-C(0.12) & \textbf{75.3} & 47.7 & \underline{50.5} & 72.5  \\ 
		&CIFAR-100-C(0.12) & 75.2 & \underline{48.6} & \underline{54.5} & \textbf{75.3}   \\ 
		&ImageNet32-C(0.07)  & \textbf{99.6} & \underline{42.2} & \underline{55.7} & 84.0   \\ 
		&LSUN-C(0.06) & 81.3 & 3.0 & 69.5 & \textbf{91.6}\\
		&notMNIST & \textbf{100.0} & 98.7 &99.9&99.6  \\
		&Constant  & \textbf{100.0} & \underline{0.4} & 99.9 & 99.8  \\
		\hline
		\multirow{10}*{\rotatebox{90}{\tabincell{c}{CelebA}}}
		&Constant& 98.0  & 99.8 & 99.9 & \textbf{100.0}  \\
		&Uniform & 91.0   & 100 &\textbf{100.0} &  \textbf{100.0}  \\
		&Uniform-C(0.012) & 97.2 & 98.1 & 90.9 & \textbf{99.5} \\
		&ImageNet & \underline{16.5}  & 99.7 & 99.8   & \textbf{100.0}  \\
		&ImageNet-C(0.2) & 88.5 & \textbf{97.9} & 91 & 93.3\\
		&CIFAR-10 & \underline{55.0}  & 90.4 &\textbf{94.9} & 69.0\\
		&CIFAR-100 & \underline{53.2}  & 90.6 &\textbf{95.6} & 72.3\\
		&SVHN & 83.9  & 99.3  & \textbf{99.7}&  94.7  \\
		&SVHN-C(1.8) & 90.5 & \textbf{99.9} & 85.2 & 98.9 \\
		&LSUN & 65.4 & \textbf{99.6} & 84.9 & 99.2 \\
		\hline
		 
		 \multirow{11}*{\rotatebox{90}{\tabincell{c}{CIFAR-10}}}
		&Constant& \textbf{100} & \underline{1.4} & 99.8 & 98.9 \\ 
		&Uniform&\textbf{100} & \textbf{100} & \textbf{100} & \textbf{100} \\ 
		&Uniform-C(0.2) & 98.8 & 1.9 & 64.7 & \textbf{99.7}\\
		&CelebA& 86.3 & 96.6 & \textbf{99.5} & 85.2 \\ 
		&CelebA-C(0.3) & \textbf{95.0} & 7.8 & \underline{46.5} & 64.9 \\ 
		&SVHN & 95.0 & 92.9 & \textbf{95.5} & 82.6 \\ 
		&SVHN-C(2.0) & 94.0 & \textbf{98.9} & 93.7 & 95 \\
		
		&TinyImageNet&71.6 & \textbf{90.7} & 76.7 & 83.9 \\ 
		&CIFAR-100&\textbf{73.6} & 60.0  & \underline{57.1} & \underline{54.1} \\ 
		&LSUN & 91.1 & 82.8 & 98.0 & \textbf{98.9}\\
		&LSUN-C(0.3) & \textbf{96.4} & 94.8 & 61.2 & 83.3\\
		\hline
		& \textbf{average} & 82.7 & 73.7 & 85.5 & \textbf{90.7}\\
		& \textbf{\#notable failures} & 5 & 5 & 6 & \textbf{1}\\
		\bottomrule[1pt]
	\end{tabular}
\end{table}

\textbf{CelebA vs Others}.
The performance of $\mathcal{S}$ degenerates severely on these problems. Our method is slightly affected because the likelihoods of the train and test split of CelebA do not fit very well (see Figure \ref{fig:logpx_glow_celeba_vs_others} in the supplementary material). 



\textbf{CIFAR-10 vs Others}.
As discussed in the last subsection, Glow model fails to generate high-quality CIFAR-10-like images. Our method is affected on CIFAR-10 vs others. As discussed before, we argue that it is hard to require an ``unsuccessful'' model can detect OOD data. 

Finally, our method has only one notable failure (\textit{i.e.}, below 60\% AUROC). $\mathcal{S}$, $L_{last}$, and DoSE have 5, 5, and 6 notable failures in total, respectively.
\noindent\textbf{Other comparisons}.

We compare \PADmethod\ with Joint confidence loss, ODIN, and Joint confidence loss+ODIN.
These three baseline methods do not apply to flow-based model. The results are shown in Table \ref{tbl:PAD_vs_confidence_loss_ODIN} in the supplementary material, where we use the same datasets reported in \cite{confidenceloss2018ICLR}. Our method is the best one. 
See Section \ref{sec:more_discussion} in the supplementary material for more discussion on our method and results.
\vspace{-5pt}
\begin{framed}
	\textbf{Summary}. For PAD, our method achieves 90.7\% AUROC on average and outperforms the SOTA baseline DoSE by 5.2\% in AUROC. 
	Our method also has the least notable failures. 
\end{framed}

\vspace{-7pt}

\vspace{-3pt}
\section{Conclusion} \label{sec:conclusion}
In this paper, we prove theorems to investigate KL divergences in flow-based models. We observe the normality of ID and OOD representations in flow-based model for a wide range of problems. Based on our analysis, we explain why we cannot sample OOD data from flow-based model from two perspectives. 
We propose leveraging KL divergence for OOD detection. 
We further decompose the KL divergence to leverage the last-scale KL divergence of Glow model. Furthermore, we split representations into groups to leverage group-wise KL divergence as the final OOD detection criterion.
Experimental results have demonstrated the effectiveness and robustness of our method.

\ifCLASSOPTIONcaptionsoff
  \newpage
\fi

\vspace{-10pt}
\bibliographystyle{IEEEtran}
\bibliography{main}

\appendices
\renewcommand\appendix{\par
	\setcounter{section}{0}
	\setcounter{subsection}{0}
	\gdef\thesection{附录 \Alph{section}}}

\counterwithin{figure}{section}
\counterwithin{table}{section}

\section{Background}\label{sec:background}
\textbf{Flow-based generative model}  constructs diffeomorphism $f$ from visible space $\mathcal{X}$ to latent space $\mathcal{Z}$ \cite{kingma2018glow,dinh2016realnvp,dinh2014nice,papamakarios2019flow_model_survey}. The model uses a series of  diffeomorphisms implemented by multilayered neural networks
\begin{equation}\nonumber
	\setlength{\abovedisplayskip}{3pt}
	\setlength{\belowdisplayskip}{3pt}
	\bm{x} \stackrel{f_1}{\longleftrightarrow} \bm{h_1} \stackrel{f_2}{\longleftrightarrow} \bm{h_2} \dots \stackrel{f_n}{\longleftrightarrow} \bm{z} 
\end{equation}
like flow. The whole bijective transformation 
$f(\bm{x})=f_n\circ f_{n-1} \cdots f_1(\bm{x})$ can be seen as encoder, and the inverse function $f^{-1}(\bm{z})$ is used as decoder.
According to the change of variable rule, the probability density function of the model can be formulated as 
\begin{equation}
	\setlength{\abovedisplayskip}{3pt}
	\setlength{\belowdisplayskip}{0pt}
	\begin{aligned}
		\log p_{X}(\bm{x}) & =\log p_{Z}(f(\bm{x}))+\log \left|{\rm det}\dfrac{\partial \bm{z}}{\partial \bm{x}^T}\right| \\
		& = \log p_{Z}(f(\bm{x}))+\sum\nolimits_{i=1}^n\log \left|{\rm det}\dfrac{\partial \bm{h}_i}{\partial \bm{h}_{i-1}^T}\right| \label{equ:change_of_var} 
	\end{aligned}
\end{equation}
where $\bm{x}=\bm{h}_0, \bm{z}=\bm{h}_n, \frac{\partial \bm{h}_i}{d\bm{h}_{i-1}^T}$ is the Jacobian of $f_i$, $\det$ is the determinant.

Here prior $p_{Z}(\bm{z})$ is chosen as tractable density function. For example, the most popular prior is standard  Gaussian distribution $\mathcal{N}(0, I)$, which makes $\log p_{Z}(\bm{z})=-(1/2)\times\sum_i \bm{z}_i^2+C$ ($C$ is a constant).
After training, one can sample noise $\varepsilon$ from prior and generate new samples $f^{-1}(\varepsilon)$.

In this paper, we replace prior $\mathcal{N}(0, I)$ with fitted Gaussian distributions $\n(\widetilde{\bm{\mu}},\widetilde{\bm{\Sigma}})$ from OOD representations, where $\widetilde{\bm{\mu}}$ and $\widetilde{\bm{\Sigma}}$ are the sample mean and covariance, respectively. Then we sample noise $\varepsilon' \sim \n(\widetilde{\bm{\mu}},\widetilde{\bm{\Sigma}})$ and generate OOD samples $f^{-1}(\varepsilon')$.


\textbf{Variational Autoencoder (VAE)} is directed graphical model approximating the data distribution $p(\bm{x})$ with encoder-decoder architecture \cite{kingma2013auto}. The probabilistic encoder $q_{\phi}(\bm{z}|\bm{x})$ approximates the unknown intractable posterior $p(\bm{z}|\bm{x})$. The probabilistic decoder 
$p_{\theta}(\bm{x}|\bm{z})$ approximates $p(\bm{x}|\bm{z})$. 
In VAE, the variational lower bound of the marginal likelihood of data points (ELBO) 
\begin{equation}
	\begin{aligned}\label{equ:VAE_ELBO}
		\mathcal{L}(\theta,\phi)=\dfrac{1}{N}\sum _{i=1}^N E_{z\sim q_{\phi}}[\log p_{\theta}(\bm{x}^i|\bm{z})]-KL(q_{\phi}(\bm{z}|\bm{x}^i)||p(\bm{z}))
	\end{aligned}
\end{equation}
can be optimized using stochastic gradient descent.
After training, one can sample $\bm{z}$ from prior $p(\bm{z})$ and use the decoder $p_{\theta}(\bm{x}|\bm{z})$ to generate new samples.

\section{Definitions}\label{sec:appendix_divergence}
\vspace{-0pt}
\begin{definition}[$\phi$-divergence]
	\label{def:phi_divergence}
	The $\phi$-divergence between two densities $p(\bm{x})$ and $q(\bm{x})$ is defined by 
	\begin{equation}\label{def:phi_distance}
		D_{\phi}(p,q)=\int \phi (p(\bm{x})/q(\bm{x}))q(\bm{x})\mathrm{d}\bm{x},
	\end{equation}
	where $\phi$ is a convex function on $[0,\infty)$ such that $\phi(1)=0$. When $q(\bm{x})=0$, $0\phi(0/0)=0$ and $0\phi(p/0)=\lim_{t\rightarrow\infty}\phi(t)/t$\cite{ali1966general}.
\end{definition}
\vspace{-0pt}
$\phi$-divergence family is used widely in machine learning fields. As shown in Table \ref{tbl:phi_divergence}, many commonly used measures, including the KL divergence, Jensen-Shannon divergence, and squared Hellinger distance, belong to the $\phi$-divergence family. Many $\phi$-divergences are not proper distance metrics and do not satisfy the triangle inequality.

\begin{table} [!h]
	
	\vspace{-0pt}
	\caption{Examples of $\phi$-divergence family} 
	\label{tbl:phi_divergence}  
	\begin{center}  
		\begin{tabular}{c c}  
			\toprule[1pt]
			$\phi(x)$ & Divergence  \\
			\hline 
			$x\log x-x+1$ & Kullback-Leibler \\
			$-\log x + x -1$ & Minimum Discrimination Information \\
			$(x-1)\log x $ & $J$-Divergence\\
			$\frac{1}{2}\vert 1-x\vert$ & Total Variation Distance\\
			$(1-\sqrt{x})^2$ & Squared Hellinger distance\\	
			$x\log \frac{2x}{x+1}+\log \frac{2}{x+1}$ & Jensen-Shannon divergence \\			
			\bottomrule[1pt] 
		\end{tabular}  
	\end{center}  
\end{table}

\section{Tables}
Table \ref{tbl:notations} summarizes the notations of distributions and KL divergences involved in our analysis.
\begin{table} [ht]
	\small
	\vspace{-0pt}
	\caption{Distributions and KL divergences in our analysis.} 
	\label{tbl:notations}  
	\begin{center}  
		\begin{tabular}{|c|l|}  
			\hline
			\textbf{Notations} & \textbf{Explanations}\\  
			\hline
			$\z=f(\x)$ & the flow-based model function\\\hline
			$p_X$  & the distribution of ID data \\\hline
			$q_X$ & the distribution of OOD data\\\hline
			$p_Z$ & the distribution of ID representations\\\hline  
			$q_Z$ & the distribution of OOD representations\\\hline
			$p^r_Z$ & the prior of the flow-based model\\\hline
			\multirow{2}{*}{$p^r_X$} &\multirow{2}{*}{\makecell[l]{the model induced distribution such \\that $Z_r\sim p_Z^r$ and $X_r=f^{-1}(Z_r)\sim p^r_X$}}\\
			& \\
			\hline
			\multirow{2}{*}{$KL(p_X||q_X)$} & \multirow{2}{*}{\makecell[l]{the KL divergence between $p_X$ and $q_X$,\\ \textit{assumed to be any large (by Assumption 1)}.}} \\
			& \\\hline
			\multirow{2}{*}{$KL(p_Z||q_Z)$} & \multirow{2}{*}{\makecell[l]{the KL divergence between $p_Z$ and $q_Z$,\\ \textit{equals to $KL(p_X||q_X)$ (by Theorem \ref{thm:all_enough_distance_guarantee})}.}}\\
			& \\\hline
			\multirow{2}{*}{$KL(p_Z||p^r_Z)$} & \multirow{2}{*}{\makecell[l]{the KL divergence between $p_Z$ and prior,\\ \textit{trained to be small (by Assumption 2)}.}}\\
			& \\
			\hline
			\multirow{5}{*}{$KL(p^r_Z||q_Z)$} & \multirow{5}{*}{\makecell[l]{the KL divergence between prior and $q_Z$,\\ influenced by $KL(p_Z||q_Z)$ and $KL(p_Z||p^r_Z)$. \\\textit{By relaxed triangle inequality (Theorem \ref{thm:triangle_n1_n2_n3}),} \\\textit{a small $KL(p_Z||q_Z)$ and a large $KL(p_Z||p^r_Z)$ } \\\textit{imply  $KL(p^r_Z||q_Z)$ is large}.}}\\
			& \\
			& \\
			& \\
			& \\
			\hline
			\multirow{4}{*}{$KL(q_Z||p^r_Z)$} & \multirow{4}{*}{\makecell[l]{the KL divergence between prior and $q_Z$, \\ influenced by $KL(p^r_Z||q_Z)$. \textit{By the approximate}\\ \textit{ symmetry property (Theorem \ref{thm:duality_small_KL_general}), a large} \\ \textit{$KL(p^r_Z||q_Z)$ must lead to a large $KL(q_Z||p^r_Z)$}.}}\\
			& \\
			& \\
			& \\
			\hline
		\end{tabular}  
	\end{center}  
\end{table}

Table \ref{tbl:values_e_bound_f} shows some approximate values of the
supremum of KL divergence in Theorem \ref{thm:duality_small_KL_general}.
\begin{table} [h]
	\small
	\vspace{-0pt}
	\caption{Some approximate values of the
		supremum of KL divergence} 
	\label{tbl:values_e_bound_f}  
	\begin{center}  
		\begin{tabular}{ccccccc}  
			\toprule[1pt]
			$\varepsilon$  & 0.001 & 0.005 & 0.01 & 0.05 & 0.1 & 0.5 \\
			$\sup $ & 0.001 & 0.006 & 0.011 & 0.069 & 0.016 & 1.732\\	
			\bottomrule[1pt] 
		\end{tabular}  
	\end{center}  
\end{table}

\begin{table}[t]
	\centering
	\scriptsize
	\caption{Results of Generalized Shapiro-Wilk test for multivariate normality on the representations of datasets under Glow. See Section \ref{sec:exp_setting} for the explanation of dataset names. For each dataset, we randomly select 2000 inputs for normality test.  The larger $W$ and $p$ are, the more Gaussian-like the distribution is. When $p\geq 0.05$, there is no evidence to reject the normality hypothesis. 
		In our experiments, ID representations under all models manifest strong normality. For \textit{Category I} problems, all OOD representations except for SVHN vs Constant manifest normality. Some OOD representation (e.g., Unform, ImageNet32) even has a higher $p$-value than ID data (CelebA).}
	\label{tbl:SW_test}
	\begin{tabular}{llcll}
		\toprule[1pt]
		ID & Input(\textbf{ID}/OOD) & Category & $W$ & $p$-value \\ \hline
		\multirow{4}*{ \rotatebox{90}{\textbf{Fashion.}} } & \textbf{Fashion.} & - & 0.9996 & \textbf{0.9479} \\ \cline{2-5}
		& Constant & I & 0.9992 & \textbf{0.5872}\\ \cline{2-5}
		& Constant-C(0.1) & I & 0.9995 & \textbf{0.9212 }\\ \cline{2-5}
		& MNIST & I & 0.9985 &  \textbf{0.0733 } \\ \cline{2-5}
		& MNIST-C(10.0) & I &  0.9991 &  \textbf{0.4114}  \\ \cline{2-5}
		& notMNIST & I &  0.9989 & \textbf{0.2337}   \\ \cline{2-5}
		& notMNIST-C(0.005) & I & 0.9993 & \textbf{0.6411}   \\ \hline
		
		\multirow{9}*{\rotatebox{90}{\textbf{SVHN}}} 
		& \textbf{SVHN} & - & 0.9993 & \textbf{0.6227} \\ \cline{2-5}
		& Constant & I & 0.9911 & 9.6e-10\\ \cline{2-5}
		& Constant-C(0.1) & I & 0.9992 & \textbf{0.5442 }\\ \cline{2-5}
		& Uniform & II & 0.9992 & \textbf{0.5273}\\ \cline{2-5}
		& Uniform-C(0.008) & II & 0.9993 & \textbf{0.6203}\\ \cline{2-5}
		& CelebA & II & 0.9336 & < 2.2$e$-16   \\ \cline{2-5}
		& CelebA-C(0.08) & I &  0.9993 & \textbf{0.6503}  \\ \cline{2-5}
		& CIFAR-10 & II & 0.99429 & 5.7$e$-07  \\ \cline{2-5}
		& CIFAR-10-C(0.12) & I & 0.9995 & \textbf{0.8838}  \\ \cline{2-5}
		& CIFAR-100 & II & 0.9528 & < 2.2e-16 \\ \cline{2-5}
		& CIFAR-100-C(0.12) & I & 0.9985 & \textbf{0.0760} \\ \cline{2-5}
		& ImageNet32 & II & 0.8618 & < 2.2$e$-16 \\ \cline{2-5}
		& ImageNet32-C(0.07) & I &  0.9670 & < 2.2$e$-16  \\ \hline
		
		\multirow{8}*{\rotatebox{90}{\textbf{CIFAR-10}}} 
		& \textbf{CIFAR-10} & - & 0.9995 & \textbf{0.9064} \\ \cline{2-5}
		
		& Constant & I & 0.9992 & \textbf{0.5512 }\\ \cline{2-5}
		& Constant-C(0.1) & I & 0.9991 & \textbf{0.4725 }\\ \cline{2-5}
		& Uniform & I & 0.70958 & <2.2e-16\\ \cline{2-5}
		& Uniform-C(0.02) & II & 0.99931 & \textbf{0.6964}\\ \cline{2-5}
		& CIFAR-100 & I & 0.9994 & \textbf{0.8426} \\ \cline{2-5}
		& CelebA & I & 0.9987 & \textbf{0.1390} \\ \cline{2-5}
		& CelebA-C(0.3) & I & 0.9994 &  \textbf{0.7960}  \\ \cline{2-5}
		& ImageNet32 & I & 0.9977 & 0.0048  \\ \cline{2-5}
		& TinyImageNet & I &  0.9995 & \textbf{0.3092}   \\ \cline{2-5}
		& SVHN & I & 0.9989 & \textbf{0.2532}  \\ \cline{2-5}
		& SVHN-C(2.0) & I & 0.9989 & \textbf{0.2547}  \\ \hline
		
		\multirow{7}*{\rotatebox{90}{\textbf{CelebA}}} 
		& \textbf{CelebA} & - & 0.9992 & \textbf{0.6064}\\ \cline{2-5}
		& Constant & I & 0.9989 & \textbf{0.2605 }\\ \cline{2-5}
		& Constant-C(0.1) & I & 0.9984 & \textbf{0.7184 }\\ \cline{2-5}
		& Uniform & II & 0.9993 & \textbf{0.6922}\\ \cline{2-5}
		& Uniform-C(0.012) & II & 0.9992 & \textbf{0.5815}\\ \cline{2-5}
		& CIFAR-10 & I & 0.9992 & \textbf{0.5953} \\ \cline{2-5}
		& CIFAR-100 & I & 0.9990 & \textbf{0.3313 }\\ \cline{2-5}
		& ImageNet32 & I & 0.9993 &  \textbf{0.6410}  \\ \cline{2-5}
		& ImageNet32-C(0.2) & I & 0.9990 & \textbf{0.3676}  \\ \cline{2-5}
		& SVHN & I & 0.9991 & \textbf{0.4351} \\ \cline{2-5}
		& SVHN-C(1.8) & I & 0.9990 &  \textbf{0.3600} \\ \bottomrule[1pt]
	\end{tabular}
\end{table}	

\section{Proofs}

\subsection{Proof of Theorem \ref{thm:decompose_KL_ID}}\label{sec:proof_decompose_kl_ID}
\begin{proof}
	{\small
		\begin{align}
			&KL(p^*_{X}(\x)||\mathcal{N}(0,I_n))\nonumber\\
			=&\mathbb{E}_{p^*_{X}(\x)}\Big[ \log\Big(\dfrac{p^*_{X}(\x)}{\mathcal{N}(0,I_n)}\Big)\Big]\nonumber\\
			=&\mathbb{E}_{p^*_{X}(\x)}\Big[ \log\Big(\dfrac{p^*_{X}(\x)}{\prod_{i=1}^n p^*_{X_i}(x)} \dfrac{\prod_{i=1}^n p^*_{X_i}(x)}{\mathcal{N}(0,I_n)}\Big)\Big]\nonumber\\
			=&\mathbb{E}_{p^*_{X}(\x)}\Big[ \log\Big(\dfrac{p^*_{X}(\x)}{\prod_{i=1}^n p^*_{X_i}(x)}\Big)\Big] +  \mathbb{E}_{p^*_{X}(\x)}\Big[\log\Big( \dfrac{\prod_{i=1}^n p^*_{X_i}(x)}{\prod_{i=1}^n \mathcal{N}(0,1)}\Big)\Big]\nonumber\\
			=&\mathbb{E}_{p^*_{X}(\x)}\Big[ \log\Big(\dfrac{p^*_{X}(\x)}{\prod_{i=1}^n p^*_{X_i}(x)}\Big)\Big] +  \mathbb{E}_{p^*_{X}(\x)}\Big[\sum_{i=1}^n \log\Big( \dfrac{p^*_{X_i}(x)}{\mathcal{N}(0,1)}\Big)\Big]\nonumber\\
			=&\mathbb{E}_{p^*_{X}(\x)}\Big[ \log\Big(\dfrac{p^*_{X}(\x)}{\prod_{i=1}^n p^*_{X_i}(x)}\Big)\Big] +  \sum_{i=1}^n \mathbb{E}_{p^*_{X_i}(x)}\Big[ \log\Big( \dfrac{p^*_{X_i}(x)}{\mathcal{N}(0,1)}\Big)\Big]\nonumber\\
			=&KL(p^*_{X}(\x)||\prod_{i=1}^n p^*_{X_i}(x))+ \sum_{i=1}^n KL(p^*_{X_i}(x)||\mathcal{N}(0,1))\nonumber
		\end{align}
	}
	$\hfill\square$ 
\end{proof}

\subsection{Proof of Theorem \ref{thm:decompose_KL_PAD}}\label{sec:proof_theorem_decompose_KL_PAD}
\begin{proof}
	We can use the similar deduction in Theorem \ref{thm:decompose_KL_ID} and get Equation \eqref{equ:decompose_KL_group}.
	\begin{equation}
		\begin{aligned}
			\nonumber &KL(p^*_X(\x)||\mathcal{N}(0,I_n))\\
			\nonumber =&\mathbb{E}_{p^*_X(\x)}\Big[ \log\Big(\dfrac{p^*_X(\x)}{\prod_{i=1}^k p^*_{\bar{X}_i}(\x)} \dfrac{\prod_{i=1}^k p^*_{\bar{X}_i}(\x)}{\mathcal{N}(0,I_n)}\Big)\Big]\\
			= & I_g[p^*_X] + D_g[p^*_X]
		\end{aligned}
	\end{equation}
	Then we apply Theorem \ref{thm:decompose_KL_ID} on each $D_g^i[p^*_{\bar{X}_i}]$ and have 
	\begin{equation} \label{equ:decompose_D_g_i}
		\begin{aligned}
			& KL(p^*_{\bar{X}_i}(\x)||\mathcal{N}(0,I_l))\\
			= &  KL(p^*_{\bar{X}_i}(\x)||\prod_{j=1}^l p^*_{\bar{X}_{ij}}(x)) + \sum_{j=1}^l KL(p^*_{\bar{X}_{ij}}(x)||\mathcal{N}(0,1))
		\end{aligned}
	\end{equation}
	Finally, combining Equation \eqref{equ:decompose_KL_group} and \ref{equ:decompose_D_g_i} we can obtain Equation \eqref{equ:decompose_KL_3parts}.
	$\hfill\square$ 
\end{proof}

\section{Summary of Our Work}
The flowchart in Figure \ref{fig:flowchart} can help readers to have an overview of our work.
\begin{figure*}[h!]
	\centering
	\includegraphics[width=16.5cm]{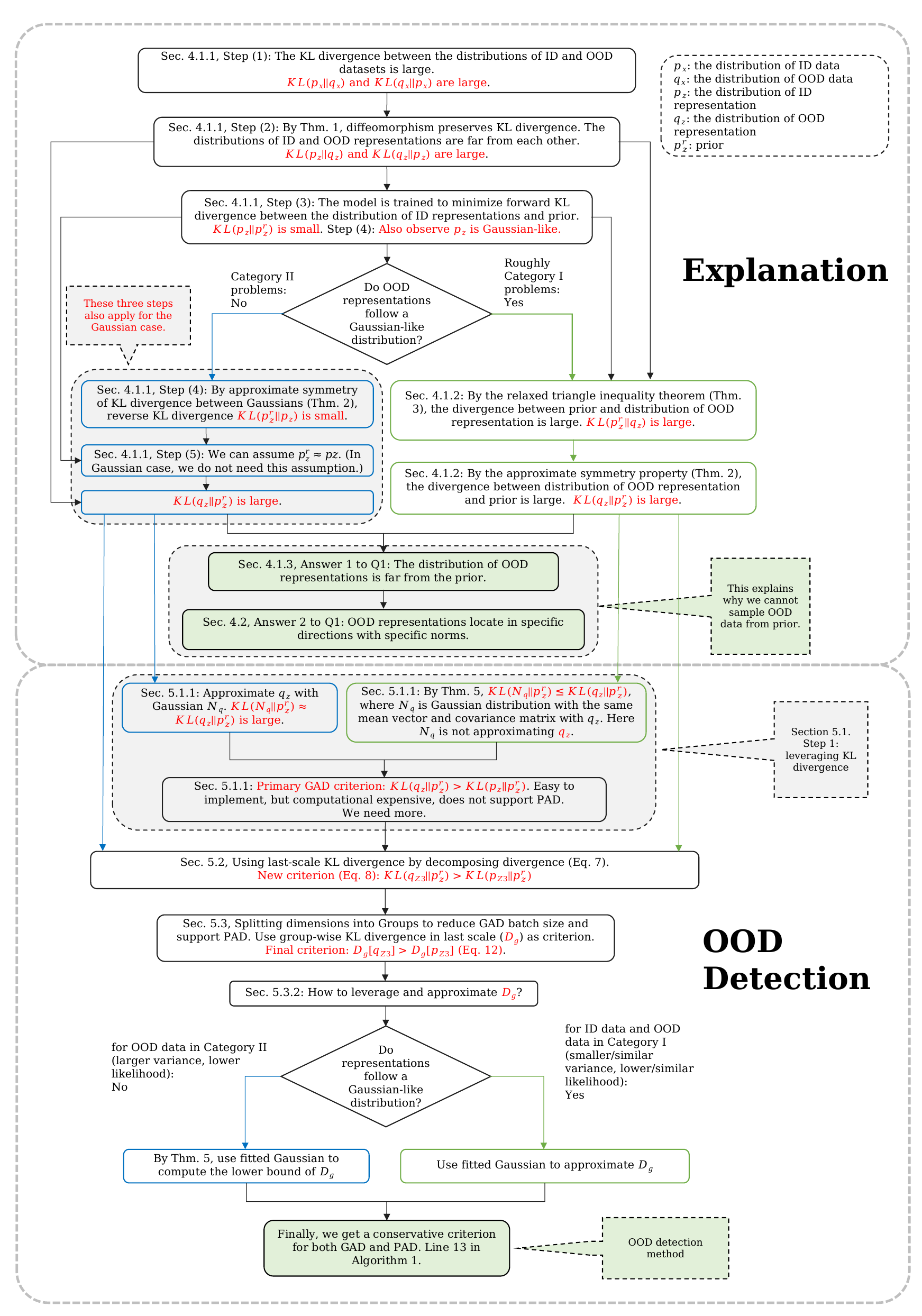}
	\vspace{-5pt}
	\caption{Overview of our method. The top half explains why we cannot sample OOD data from the flow-based model. The bottom half is for OOD detection method. Green lines are for the Gaussian case. Blue lines are for the non-Gaussian case. Please also refer to Figure \ref{fig:theory_and_algorithm} and \ref{fig:bigpicture2} in the main text. 
	}  
	\label{fig:flowchart}
\end{figure*}

\clearpage

\section{Benchmarks}\label{sec:benchmarks}
Here we briefly introduce the benchmarks used in our experiments:
\begin{enumerate}
	\item Constant: The Constant dataset consists of images with all pixels equal to the same constant $C\sim U\{0,255\}$.
	\item Uniform: The Uniform dataset consists of images with each pixel sampled independently from $U\{0,255\}$.
	\item MNIST \cite{mnist}: MNIST is a dataset of handwritten digits.
	\item FashionMNIST \cite{fashionmnist}: FashionMNIST is a dataset of images of clothes and shoes.
	\item notMNIST \cite{notmnist}: notMNIST is a dataset of fonts and extracting glyphs similar to MNIST. 
	\item KMNIST \cite{KMNIST}: KMNIST is a dataset of Japanese characters.
	\item Omniglot \cite{Omniglot}: Omniglot is a dataset of handwritten digits of a set of alphabets.
	\item CIFAR-10/100 \cite{cifar}: CIFAR-10/100 are datasets of natural images including animals and vehicles.
	\item SVHN \cite{svhn}: SVHN is a dataset of street view housing numbers.
	\item CelebA \cite{celeba}: CelebA is a dataset of face images of celebrities.
	\item TinyImageNet \cite{tinyimagenet}: TinyImageNet is a subset of ImageNet. 
	\item ImageNet32 \cite{imagenet}: Imagenet32 is a dataset of small images called the down-sampled version of Imagenet.
	\item LSUN \cite{LSUN}: LSUN is a dataset of scene categories including bedrooms, classroom, \textit{etc}.
\end{enumerate} 

All datasets are resized to $32\times 32 \times 3$ for consistency. The size of each test dataset is fixed to 10,000 for comparison. For grayscale datasets of size $28\times 28\times 1$, we replicate channels and pad zeros around images. We use the same method to process LSUN as the baseline method GOD2KS \cite{jiang2022revisiting}. 
See  Figure \ref{fig:dataset_examples} in the supplementary material for example images of different datasets.

\section{Model Details}\label{sec:appendix_model_details}

We use the released model (checkpoints) by the author of the baseline to conduct experiments if possible.
Otherwise, we train the model by ourselves.  

The authors of GAD baseline method (\textit{i.e.}, Ty-test) reimplement Glow with PyTorch and release only one model checkpoint trained on CIFAR-10 \cite{nalisnick2019detecting}. We use their model for CIFAR-10 vs others. For other problems, we train the official Glow model \cite{glowopenai} by ourselves.
For VAE, we train convolutional VAE and use sampled representation for all problems. 
Among baseline methods, only the authors of $L_{last}$ released their model checkpoints. We use their checkpoints to produce results on problems not evaluated in the original paper.


The Glow model consists of three stages, each containing 32 coupling layers with width 512. After each stage, the latent variables are split into two parts. One half is treated as the final representations and another half is processed by the next stage. 
We use additive coupling layers for grayscale datasets and CelebA and use affine coupling layers for SVHN and CIFAR-10. 
We find no difference between these two coupling layers for OOD detection.
All priors are standard Gaussian distribution except for CIFAR-10, which has learned mean and diagonal covariance. All models are trained using Adamax optimization method with a batch size of 64. The learning rate is increased from 0 up to 0.001 in the first 10 epochs and keeps invariable in the remaining epochs. Flow-based models are  resource consuming. We train Glow on FashionMNIST/SVHN/CelebA32 for 130/390/2000 epochs, respectively. The training curve of Glow on CelebA32 is shown in Figure \ref{fig:traing_curve}.
We have also conducted experiments using the checkpoint released by OpenAI \cite{glowopenai} for CIFAR-10 vs others. The results are similar.

\begin{figure}[t]
	\centering
	\centering
	\includegraphics[width=6cm]{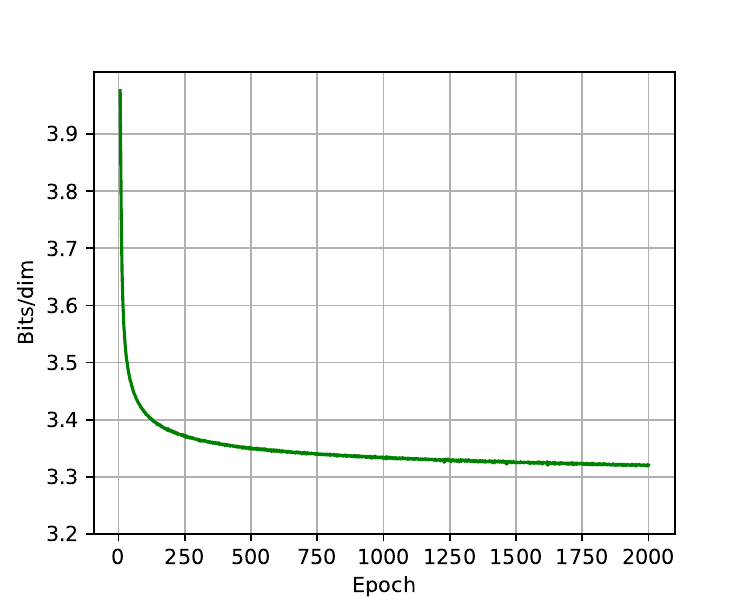}
	\label{fig:training_curve_celeba}
	\caption{The training curve of Glow on CelebA32.}
	\label{fig:traing_curve}
\end{figure}



For VAE, we use convolutional architecture in the encoder and decoder. The encoder consists of three $4\times 4 \times 64$ convolution layers. On top of convolutional layers, two dense layer heads output the mean $\bm{\mu}(\bm{x})$ and the standard variance $\bm{\sigma}(\bm{x})$ respectively. The decoder has the mirrored architecture as the encoder. All activations are LeakyReLU with $\alpha=0.3$. For FashionMNIST, SVHN, and CIFAR-10, we use 8-, 16- and 32-dimensional latent space, respectively. Models are trained using Adam without dropout. The learning rate is $5\times 1^{-4}$ with no decay. 

\textbf{Details of Baseline Methods}. 
\begin{enumerate}
	\item $\mathcal{S}$. 
	The authors of $\mathcal{S}$ modified the official Glow model by using zero padding and removing ActNorm layer \cite{2020InputComplexity}.
	They did not explain the reason for such modification. In principle, such modification to models should not affect the baseline method. Since the authors did not release their source code and model checkpoint, we reimplement $\mathcal{S}$ method using the official Glow model \cite{glowdeepmind} for those problems not evaluated in \cite{2020InputComplexity}. We also use FLIF \cite{flif} as the compressor, which is the best compressor in \cite{2020InputComplexity}. We find the performance of $\mathcal{S}$ degenerates on the official Glow model. 
	
	\item $L_{last}$. We use the model checkpoints released by the authors of $L_{last}$ for results not reported in the original publication
	\cite{understanding_anomaly_2020_nips}. 
	
	\item DoSE. The authors of DoSE did not release their source code and model. We reimplement DoSE$_{\text{SVM}}$ on the official Glow model for problems not evaluated in \cite{morningstar21KDE}. We use the same parameters as the original paper. The performance of DoSE on the official Glow model trained on SVHN is slightly better than the results in \cite{morningstar21KDE}. We also find the performance of DoSE degenerates severely on CelebA32 vs CIFAR-10/100 compared with the results reported in their original publication (below 75\% AUROC). We did not present these results in the main text. 
\end{enumerate}

\section{Discussion}\label{sec:more_discussion}



\textbf{Normality of Representations}.
The normality of ID and OOD representation facilitates our theoretical analysis and OOD detection algorithm on flow-based model. 
In our experiments, we find that the normality of OOD representation is a widely existing phenomenon under flow-based models. 
We are investigating the underlying reason.
Most importantly, our method performs even better on \textit{Category II} problems, although the criterion computes the lower bound of the KL divergence.

Both Flow-based model and VAE are trained to minimize KL divergence between $p_Z$ and prior. However, it seems that the normality of ID/OOD representations is a characteristic of flow-based model.

In principle, we can construct latents following any distribution and decode these latents back to data space to construct an OOD dataset. 
Note that such manipulation does not necessarily make our OOD detection method fail, although it can violate the
normality hypothesis. Besides, such manipulation is much more difficult to conduct than the data manipulations presented in this paper because we need the model parameters. 

\textbf{PAD Results}.
From the results of three Glow models trained on SVHN/CelebA/CIFAR-10 (Table \ref{tbl:PAD_all_results} in the main text), 
we can see that the more high-quality generated images are, the better performance our method can achieve.
This is consistent with our theoretical analysis. 
Our method has the a solid theoretical foundation.
We believe that our method can achieve better performance with the increasing success of flow-based model in the future.

\textbf{Limitation}.
Our method requires the model to capture the distribution of training data. 
For example, Glow trained on CIFAR-10 does not generate meaningful images. Thus, the KL divergence between the distributions of ID representation and prior is not small enough. So our theoretical analysis does not apply to this problem well. This is why our method does not achieve high AUROC on CIFAR-10 vs CIFAR-100.
We think it is hard to achieve high AUROC with a model which does not succeed in generating meaningful images (see Figure \ref{fig:generated_images_glow} in the supplementary material). In such a case, the model generates ``OOD data'' that differ from the training set. Besides, the similarity between CIFAR-10 and CIFAR-100 also brings obstacles to OOD methods.

There are two possible solutions for the most challenging problem CIFAR-10 vs CIFAR-100. The first one is to improve the model. Modeling data is a long-standing goal of unsupervised learning \cite{PRML}. Up to now, it is still hard to generate high-quality CIFAR-10-like images using unconditional flow-based models. The second possible solution is to use a more sensitive criterion to estimate KL divergence or dependence.
We leave this direction as future work. 

Ty-test applies to flow-based model, VAE, and auto-regressive model. Our method applies to models which learn independent or disentangled representations \cite{disentanglement_challenge_2019, higgins2018definition, eastwood2018a,higgins2017beta, factorVAE, isolating2018Chen, kumar2017variational}, not including auto-regressive model.

\textbf{Other Baselines}.

In this paper, we mainly choose recently proposed methods applicable to flow-based model as baselines.
We did not choose WAIC \cite{choi2018generative} whose results could not be reproduced by  Nalisnick \textit{et al} \cite{nalisnick2019detecting}. 
We did not choose the likelihood ratios method \cite{ren2019likelihood-ratio} as the baseline either for several reasons. First, in \cite{2020InputComplexity}, Serr\`a \textit{et al.}
interpret their method $\bm{\mathcal{S}}$ as a likelihood-ratio test statistic and achieve better performance than likelihood ratios. 
Second, the authors of the likelihood ratios method \cite{ren2019likelihood-ratio} did not report results on flow-based models. So we choose method $\bm{\mathcal{S}}$ rather than likelihood ratios as baseline.
Finally, we choose SOTA method DoSE as baseline, which is better than $\log p(\x)$, $\log p(\z)$, WAIC, and likelihood ratio as reported in \cite{morningstar21KDE}.


In \cite{jiang2022revisiting}, the authors propose GOD2KS mainly for GAD. In the appendix of \cite{jiang2022revisiting}, the authors also use data augmentation to support PAD using GOD2KS. However, they only report a few PAD results based on RealNVP. The PAD results of GOD2KS on CIFAR-10 vs SVHN/CelebA/LSUN with RealNVP are 85\%/57\%/46\% AUROCs, respectively. Our method achieves 82.6\%/85.2\%/99.2\% AUROCs on the same three problems with Glow. Due to this situation, we did not use GOD2KS as PAD baseline.

Just after we receive the first round of review comments, Osada \textit{et al.}  propose PRE method \cite{2023WACV-recons-error-OOD}, which uses reconstruction error and typicality-based penalty to perform point-wise anomaly detection with flow-based model. PRE uses the original Equation (4) in \cite{2023WACV-recons-error-OOD} as anomaly score. The larger the score, the more likely the input is OOD. We did not choose PRE as baseline for two reasons. First, the authors of PRE do not use as many dataset compositions in evaluation as ours. They also use ID datasets of different sizes. It is hard to compare two methods in this situation. Second, similar to Annulus Method, PRE can also be attacked by data manipulation \textbf{M1} (rescaling representations). For each OOD dataset $S=\{\x\}$, we can use data manipulation \textbf{M1} (see Subsection \ref{sec:attack_likelihood}) to construct an OOD dataset $S'=\{\x'=f^{-1}(\sqrt{d}\frac{f(\x)}{|f(\x)|})|\x\in S\}$ (see Figure \ref{fig:typical_set} and Figure \ref{fig:sample_images_rescale_to_typical_set_trained_glow_on_fashionmnist}). 
For input $\x'\in S'$, the representation $\z'=f(\x')$ locates in the typical set annulus of prior precisely and Equation (4) in \cite{2023WACV-recons-error-OOD} equals 0. This can make PRE method achieve near 0 AUROC.

\textbf{Other Comparisons}.
Explicit generative models, including autoregressive models, flow-based models, and VAEs, can provide users with likelihoods (or lower bound). An ideal explicit generative model should: 1) generate new data from the training data distribution and 2) provide likelihood indicating the confidence of whether the data belongs to the training data distribution.  
Implicit generative models (\textit{i.e.}, GAN) do not produce likelihood, so these two kinds of models are under different settings.
Commonly, explicit generative models are compared together in anomaly detection publications. All the baseline methods applying to flow-based model are compared with explicit generative models in evaluation. (\textit{e.g.}, \cite{choi2018generative, whyflowfailood, 2020InputComplexity, ren2019likelihood-ratio, understanding_anomaly_2020_nips, morningstar21KDE}). 

We notice that the existing hybrid model \cite{CVAE-GAN-OOD-2019} achieves better performance on leave-one-out setting on MNIST \footnote{The author of \cite{CVAE-GAN-OOD-2019} did not report experimental results on cross-dataset problems.}. It is unfair to compare a flow-based method with a hybrid model combining explicit and implicit generative models. 
For example, existing work has shown that the combination of GAN and flow-based model can improve the quality of the generated images of flow-based model.   For example, in Flow-GAN \cite{flow-GAN}, flow-based model is used to avoid mode collapse. Adversarial training is used to improve the image quality of flow-based model. The method is to sample noise $\z$ from the typical set of prior and use a discriminator to distinguish $f^{-1}(\z)$ and training data. In our data manipulation \textbf{M1} (rescaling representations to the typical set), we find that Glow trained by maximum likelihood estimation cannot expel OOD representation from the typical set of prior (see Subsection \ref{sec:attack_likelihood}). In Flow-GAN, adversarial training tends to compel Glow ($f^{-1}$) to map latents in the typical set of prior to in-distribution images. Importantly, our theoretical analysis also applies to flow-GAN whose loss function includes the basic loss function of Glow.
We did not conduct experiments on Flow-GAN \cite{flow-GAN} because Flow-GAN uses old flow-based models NICE \cite{dinh2014nice} and RealNVP \cite{dinh2016realnvp}. Besides, adversarial training would affect the divergences in flow-based model.
We will explore the properties of hybrid models which combine the SOTA flow-based model and GAN in the future. 

We conduct comprehensive experiments to evaluate our method on different dataset compositions falling into \textit{Category I} and \textit{II}. Flow-based models have different behaviors for these two categories of problems.
Commonly, one OOD detection method's performance may vary on different problems.
Several existing OOD detection methods have been evaluated with very few datasets (\textit{e.g.}, only CIFAR-10) as the training dataset. We did not compare our method with such methods for two reasons. First, there is no result reported on more problems. Second, our method requires the model to succeed in modeling the training dataset.  Unluckily, flow-based model is not as successful as other datasets on CIFAR-10. This affects our method. We think it is not comprehensive if using only CIFAR-10 as training dataset in evaluation. 

Researchers also propose ensembling algorithms for anomaly detection \cite{ensemble-OOD-zhao-2015}, which is orthogonal to our method. We plan to explore this direction in the future. 

\textbf{Models}.
We did not conduct more experiments on flow-based models with various architectures. In principle, a more expressive model can make the forward KL divergence smaller, and our method can benefit more.

Our theoretical analysis relies on the assumption that the KL divergence between the distributions of ID and OOD representations is large (see Section \ref{sec:general_case} in the main text). So our analysis does not apply to VAE and autoregressive models directly. According to the Brouwer Invariance of Domain Theorem \cite{Brouwer1911invarianceofdomain}, $R^n$ cannot be homeomorphic to $R^m$ if $n \neq m$ .
The Brouwer Invariance of Domain Theorem also implies that there is no dead neuron in flow-based model. Otherwise, we can construct diffeomorphism from high to low dimensional space. 
For VAE, a high-dimensional latent space may contain nearly dead neurons. This may reduce the performance of our method.
We did not conduct experiments on other VAE variations, \textit{e.g.}, $\beta$-VAE \cite{higgins2017beta}, FactorVAE \cite{factorVAE}, $\beta$-TCVAE \cite{isolating2018Chen},  and DIP-VAE \cite{kumar2017variational}. These variations add more regularization strength on disentanglement and have more independent representations than vanilla VAE \cite{locatello2019challenging}. 
We will conduct experiments on larger VAE models and variations in the future.

Finally, we use the model checkpoints released by baselines as long as possible. These released models should be tuned elaborately for their methods, so our method benefits less from fine-tuning.

\section{More Related Work}\label{sec:more_relatedwork}
\textbf{OOD Detection}.
In \cite{understanding_anomaly_2020_nips}, Schirrmeister \textit{et al.} find the likelihood contributed by the last scale of Glow ($L_{last}$) is  better than $\log p(\bm{x})$ for PAD. 
Their method relies on the decomposition of the likelihood (original Equation (3) in \cite{understanding_anomaly_2020_nips}). Such decomposition requires that the split two parts at each stage of Glow are independent. This may not hold for OOD data due to covariate shift. 
Experimental results show that the last-scale log-likelihood of OOD data may be larger than 0 (See Figure \ref{fig:logpx_last_scale_glow_SVHN_vs_others_hierarchical_model} in the supplementary material). So the criterion used by $L_{last}$ should not be explained as likelihood for OOD data. 
Finally, as shown in Table \ref{tbl:PAD_all_results}, $L_{last}$ is also affected by data manipulation.

\textbf{Theoretical Analysis}. Previous works \cite{papamakarios2019flow_model_survey, masked_autoregressive_flow_2017_nips} analyze the training objective of flow-based model in KL divergence form. We apply the property of diffeomorphism to investigate the divergences between distributions in flow-based models in the setting of OOD detection. We also propose new theorems on the properties of KL divergence between Gaussian distributions for further analysis.
Currently, there is no similar work on the properties of KL divergence between Gaussian distributions. Theorems \ref{thm:duality_small_KL_general}, and \ref{thm:triangle_n1_n2_n3} can be used as basic theorems in machine/deep learning and information theory. For example, after we post the manuscript containing the proofs of Theorems \ref{thm:duality_small_KL_general} and \ref{thm:triangle_n1_n2_n3}  on Arxiv \cite{zhang2021properties}, the relaxed triangle inequality (Theorem \ref{thm:triangle_n1_n2_n3}) has been used in constrained variational policy optimization for safe reinforcement learning \cite{liu2022constrainedSafeRL}.

\textbf{Other Approximator and Divergence Estimation}.
In principle, GMM can approximate a target density better than a single Gaussian distribution \cite{2001Advanced}. We tried to use GMM to approximate the distribution of representations and use Monte Carlo sampling to estimate the KL divergence between GMMs. The results show GMM is worse than using Gaussian distribution for OOD detection. The reasons are twofold: a) it is inappropriate to use GMM for modeling ID representations that follow a Gaussian-like distribution. b) the batch size is too small (usually 5 $\sim 10$) to estimate the parameters of GMM. 

We also tried the SOTA $\phi$-divergence estimation method applicable for VAE, \textit{i.e.} RAM-MC \cite{Rubenstein2019estimation}. Results show that RAM-MC can also be affected by data manipulation \textbf{M2} (adjusting contrast, see Section \ref{sec:attack_likelihood}) \footnote{This does not prove that RAM-MC is not applicable to general-purpose divergence estimation.}. 

\textbf{OOD Sampling}. In this paper, we sample OOD data using the fitted Gaussian from OOD representations to verify that OOD representations reside in specific directions. Such directions can be partially captured by the sample mean and sample covariance. 
In \cite{gambardella2019transflow}, the authors sample noise $\varepsilon\sim \n(\bm{\widetilde{\mu}}, I)$ and generate OOD data $f^{-1}(\varepsilon)$ using flow-based model, where $\widetilde{\bm{\mu}}$ is the sample mean of OOD representations. They did not use the sample covariance of OOD representations. Their manuscript is released contemporaneously with the first edition of this paper \cite{OOD_yufeng_2020_v1}.
Some researchers have explored other OOD sampling methods. Sinha \textit{et al.} use various operations including Jigsaw, Stitching on normal images to generate negative data \cite{negative-data-sampling2021}. These negative data can be used to help GAN to improve generation quality and OOD detection ability. However, their performance of OOD detection on CIFAR-10 vs others is worse than our method. Dionelis \textit{et al}. propose to generate samples \cite{Dionelis_2020} on the boundary of the support of data distributions which is learned by flow-based model. Their method does not modify the parameters of flow-based model and has the same OOD detection ability as the original flow-based model. In this paper, we sample OOD data in order to verify that OOD representations reside in specific directions that can be partially captured by the sample mean and covariance of representations.
We will explore more work on OOD sampling in the future.

\textbf{Local Pixel Dependence}. In \cite{whyflowfailood}, Kirichenko \textit{et al.} reshape the representations of flow-based models to the original input shape ($32\times 32\times 3$) and analyze the induction biases of flow-based model. Their work reveals the reshaped representation manifests local pixel dependence. 
Our work show that the representations with a raw shape ($4\times 4\times 48$) also manifest local pixel dependence.

\textbf{OOD Detection With Auxiliary Data}.
OOD detection can be improved with the help of an auxiliary outlier dataset.
In \cite{understanding_anomaly_2020_nips}, Schirrmeister \textit{et al.} improve likelihood-ratio-based method by the help of a huge outlier dataset (80 Million Tiny ImageNet). 
This is not unsupervised learning due to the exposure to outliers in training as like \cite{hendrycks2018deep}. Besides, the huge outlier dataset includes almost all the image classes in the testing phase. We did not compare with such methods  due to different problem settings.

\textbf{Classification of Problems}. We classify OOD problems into \textit{Category I} and \textit{II}  according to the variance and likelihoods of datasets. This criterion is roughly similar to the complexity used in \cite{2020InputComplexity}. See Figure \ref{fig:complexity_length} in the supplementary material for details.


\section{More Experimental Results} \label{sec:more_results}

\subsection{GAD Results on Glow}\label{appendix:sec:GADresults}

\textbf{FashionMNIST vs Others}. 
Table \ref{tbl:glow_GAD_many_baselines} shows the GAD results of Glow trained on FashionMNIST.
The results of baselines are referenced from \cite{nalisnick2019detecting}, in which the authors use bootstrap method to establish thresholds.
We use a low false positive rate to establish a threshold $t$ (see Algorithm \ref{alg:GAD}) and then compute the corresponding true positive rate with $t$. 
Take FashionMNIST vs MNIST with batch size $m=2$ as example. We find the threshold $t$ corresponding to false positive rate 0.01 which is the second lowest one among all the baselines (column 2 in Table \ref{tbl:glow_GAD_many_baselines}). Then we use $t$ to compute the corresponding true positive rate $0.43\pm 0.02$. For larger batch sizes ($m=10, 25$), we set the thresholds $t$, achieving a  most rigorous 0 false positive rate.

\textbf{SVHN/CIFAR-10/CelebA vs Others}.
Table \ref{tbl:glow_fashionmnist_svhn_cifar10_celeba_DOCR_TC_M_bs_5_10_S2}, \ref{tbl:glow_fashionmnist_svhn_cifar10_celeba_DOCR_TC_M_bs_5_10_S2_vs_annulus_method} and \ref{tbl:glow_fashionmnist_svhn_cifar10_celeba_DOCR_TC_M_EER_bs_5_10_S2} shows numerical GAD results corresponding to Figure \ref{fig:GAD_results_bar} in the main text. Table \ref{tbl:glow_fashionmnist_svhn_cifar10_celeba_DOCR_TC_M_bs_2_4_S2} shows GAD results with smaller batch sizes 2 and 4. 

\textbf{CelebA vs CIFAR-10/100} are  challenging for Ty-test. In principle, if the train and test split of ID dataset have coinciding likelihoods, the worst AUROC of Ty-test should be around 50\%.
But Ty-test only achieves 1.7\% and 2.9\% AUROCs on these two problems.
The reasons are two-fold.
First, CIFAR-10/100 have  coinciding likelihoods with CelebA. Please see Figure \ref{fig:logpx_glow_celeba_vs_others} for details. Second, we find it is hard to make the likelihood distributions of CelebA train and test split fit very well on the official Glow model even within 2,000 epochs (see Figure \ref{fig:traing_curve} for training curve). 
The likelihoods of CIFAR-10/100 are closer to CelebA train set than the CelebA test set. This misleads Ty-test to make wrong decisions (below 10\% AUROC).
This also makes Ty-test perform worse when batch size is larger because a larger batch size eliminates randomness (see Table \ref{tbl:glow_fashionmnist_svhn_cifar10_celeba_DOCR_TC_M_bs_5_10_S2}).
On the contrary, our method is not affected by such possible underfitting or overfitting.

\textbf{Comparison with GOD2KS}. Table \ref{tbl:glow_fashionmnist_svhn_cifar10_celeba_DOCR_TC_M_bs_5_10_S2_vs_GOD2KS} shows the comparison of our method with GOD2KS.  Our method is better.

\begin{table*} [h]
	\vspace{-0pt}
	\scriptsize
	\caption{GAD Results on Glow trained on FashionMNIST. The ID column reflects FPR (ideally should be 0) and the MNIST and notMNIST columns are TPR (ideally should be 1).  The results of baselines are referenced from \cite{nalisnick2019detecting}. Notable failures (under 0.5 TPR) are underlined.} 
	\label{tbl:glow_GAD_many_baselines}  
	\begin{center}  
		\begin{tabular}{lrrr| rrr|rrr}  
			\toprule[1pt]
			& \multicolumn{3}{c}{$m$=2}
			&\multicolumn{3}{c}{$m$=10} 
			&\multicolumn{3}{c}{$m$=25}  
			\\
			\hline   
			Method & \multicolumn{1}{c}{ID} & \multicolumn{1}{c}{MNIST} & \multicolumn{1}{c}{notMNIST} & \multicolumn{1}{c}{ID} &\multicolumn{1}{c}{MNIST}  & \multicolumn{1}{c}{notMNIST} & \multicolumn{1}{c}{ID} & \multicolumn{1}{c}{MNIST} & \multicolumn{1}{c}{notMNIST}\\\hline

			Ty-test 
			& 0.02$\pm$0.01 & \underline{0.14$\pm$0.10} & \underline{0.08$\pm$0.04} 
			&0.02$\pm$0.02 &\textbf{1.00$\pm$0.00} & 0.69$\pm$0.11
			&0.01$\pm$0.00 & \textbf{1.00$\pm$0.00} & \textbf{1.00$\pm$0.00} \\
			$t$-test 
			&0.01$\pm$0.00 &\underline{0.08$\pm$0.00} &0.06$\pm$0.00
			&0.01$\pm$0.00 & \textbf{1.00$\pm$0.00} & 0.67$\pm$0.01
			&0.01$\pm$0.00 & \textbf{1.00$\pm$0.00} &0.99$\pm$0.00
			\\
			KS-Test 
			&\textbf{0.00$\pm$0.00} & \underline{0.00$\pm$0.00} & \underline{0.00$\pm$0.00}
			&0.01$\pm$0.00 & \textbf{1.00$\pm$0.00} & 0.61$\pm$0.01
			&\textbf{0.00$\pm$0.00} & \textbf{1.00$\pm$0.00} & 0.98$\pm$0.01
			\\
			
			Max Mean Dis. 
			& 0.05$\pm$0.02 & \underline{0.17$\pm$0.06} & \underline{0.04$\pm$0.03} 
			&0.02$\pm$0.02 & 0.63$\pm$0.12 & \underline{0.37$\pm$0.24}
			&0.04$\pm$0.04 & \textbf{1.00$\pm$0.00} &\textbf{1.00$\pm$0.00}
			
			\\
			Kern. Stein Dis.
			&0.05$\pm$0.05 & \underline{0.16$\pm$0.14} &\underline{0.01$\pm$0.01}
			&0.01$\pm$0.01 & \underline{0.21$\pm$0.11} & \underline{0.01$\pm$0.00} 
			&0.02$\pm$0.03&0.76$\pm$0.21 & \underline{0.00$\pm$0.00}
			
			\\
			Annulus Method
			&0.01$\pm$0.01 & \underline{0.00$\pm$0.00} & \textbf{0.96$\pm$0.03}
			&0.02$\pm$0.00 & \underline{0.00$\pm$0.00} & \textbf{1.00$\pm$0.00}
			&0.03$\pm$0.03 &\underline{0.00$\pm$0.00} & \textbf{1.00$\pm$0.00}
			\\
			
			\PADmethod
			&0.01$\pm$0.00 &\underline{\textbf{0.43$\pm$0.02}} & 0.95$\pm$0.00
			&\textbf{0.00$\pm$0.00} &\textbf{1.00$\pm$0.00} & \textbf{1.00$\pm$0.00}
			&\textbf{0.00$\pm$0.00} & \textbf{1.00$\pm$0.00} & \textbf{1.00$\pm$0.00}\\
			\bottomrule[1pt]   			
		\end{tabular}  
	\end{center}  
\end{table*}

\begin{table*} [ht]
	\scriptsize
	\caption{GAD Results (AUROC and AUPR in percentage) of \PADmethod\ and Ty-test on Glow with batch sizes 5 and 10. The higher  the better.
		The performance of Ty-test on CelebA vs CIFAR-10/100 decreases when the batch size is larger. See our explanation in Section \ref{sec:GAD_results} in the main text.
	} 
	\label{tbl:glow_fashionmnist_svhn_cifar10_celeba_DOCR_TC_M_bs_5_10_S2}  
	\begin{center}  
		\begin{tabular}{cllr rr r|r rr r}  
			\toprule[1pt]
			\multirow{3}*{ID$\downarrow$}&\multirow{3}*{OOD$\downarrow$}&Batch size$\rightarrow$&
			\multicolumn{4}{c}{$m$=5}  &\multicolumn{4}{c}{$m$=10}
			\\
			\cline{3-11}   
			& & Method$\rightarrow$& \multicolumn{2}{c}{\PADmethod} & \multicolumn{2}{c}{Ty-test} 
			& \multicolumn{2}{c}{\PADmethod} & \multicolumn{2}{c}{Ty-test} \\
			\cline{3-11} 
			&&Metric$\rightarrow$
			&  \multicolumn{1}{c}{AUROC} 
			& \multicolumn{1}{c}{AUPR}
			&  \multicolumn{1}{c}{AUROC} 
			& \multicolumn{1}{c}{AUPR}
			&  \multicolumn{1}{c}{AUROC} 
			& \multicolumn{1}{c}{AUPR}
			&  \multicolumn{1}{c}{AUROC} 
			& \multicolumn{1}{c}{AUPR}\\
			\hline
			
			\multirow{5}*{\rotatebox{90}{\tabincell{c}{Fash.}}}

			&\multicolumn{2}{l}{Constant}
			\input{results/S2/glow_fashionmnist_vs_constant_bs_5_10_S2.tex}


			&\multicolumn{2}{l}{MNIST} 
			\input{results/S2/glow_fashionmnist_vs_mnist_bs_5_10_S2.tex}
			
			&\multicolumn{2}{l}{MNIST-C(10.0)}
			\input{results/S2/glow_fashionmnist_vs_mnist_divergence_bs_5_10_S2.tex}
			& \multicolumn{2}{l}{notMNIST} 
			\input{results/S2/glow_fashionmnist_vs_notmnist_bs_5_10_S2.tex}
			& \multicolumn{2}{l}{notMNIST-C(0.005)}
			\input{results/S2/glow_fashionmnist_vs_notmnist_gray_bs_5_10_S2.tex}\\
			\hline
			\multirow{13}*{\rotatebox{90}{SVHN}} 
			
			& \multicolumn{2}{l}{Constant} 
			\input{results/S2/glow_svhn_vs_constant_bs_5_10_S2.tex}
			
			& \multicolumn{2}{l}{Uniform} 
			\input{results/S2/glow_svhn_vs_uniform_noise_bs_5_10_S2.tex}

			& \multicolumn{2}{l}{Uniform-C(0.008)} 
			\input{results/S2/glow_svhn_vs_uniform_noise_gray_bs_5_10_S2.tex}

			& \multicolumn{2}{l}{CelebA} 
			\input{results/S2/glow_svhn_vs_celeba32_bs_5_10_S2.tex}
			& \multicolumn{2}{l}{CelebA-C(0.08)}
			\input{results/S2/glow_svhn_vs_celeba32_gray_bs_5_10_S2.tex}
			& \multicolumn{2}{l}{CIFAR-10} 
			\input{results/S2/glow_svhn_vs_cifar10_bs_5_10_S2.tex}
			& \multicolumn{2}{l}{CIFAR-10-C(0.12)}
			\input{results/S2/glow_svhn_vs_cifar10_gray_bs_5_10_S2.tex}
			& \multicolumn{2}{l}{CIFAR-100} 
			\input{results/S2/glow_svhn_vs_cifar100_bs_5_10_S2.tex}
			& \multicolumn{2}{l}{CIFAR-100-C(0.12)}
			\input{results/S2/glow_svhn_vs_cifar100_gray_bs_5_10_S2.tex}
			& \multicolumn{2}{l}{ImageNet32} 
			\input{results/S2/glow_svhn_vs_imagenet32_bs_5_10_S2.tex}
			& \multicolumn{2}{l}{ImageNet32-C(0.07)}
			\input{results/S2/glow_svhn_vs_imagenet32_gray_bs_5_10_S2.tex}\\
			& \multicolumn{2}{l}{LSUN}
			&\textbf{100.0$\pm$0.0} & \textbf{100.0$\pm$0.0} & \textbf{100.0$\pm$0.0} & \textbf{100.0$\pm$0.0} &\textbf{100.0$\pm$0.0} & \textbf{100.0$\pm$0.0} & \textbf{100.0$\pm$0.0} & \textbf{100.0$\pm$0.0}\\ 
			& \multicolumn{2}{l}{LSUN-C(0.06)}
			&\textbf{99.9$\pm$0.0} & \textbf{99.0$\pm$0.0} & 42.8$\pm$0.6 & 42.5$\pm$0.3 
			&\textbf{100.0$\pm$0.0} &\textbf{100.0$\pm$0.0} & 42.3$\pm$0.5 & 42.2$\pm$0.1\\
			\hline
			\multirow{11}*{\rotatebox{90}{CIFAR-10}}
			
			& \multicolumn{2}{l}{Constant}
			\input{results/S2/glow_cifar10_vs_Constant_bs_5_10_S2.tex}
			
			& \multicolumn{2}{l}{Uniform}
			& \textbf{100.0$\pm$0.0} & \textbf{100.0$\pm$0.0} & \textbf{100.0$\pm$0.0} & \textbf{100.0$\pm$0.0} & \textbf{100.0$\pm$0.0} & \textbf{100.0$\pm$0.0} & \textbf{100.0$\pm$0.0} & \textbf{100.0$\pm$0.0} \\
			& \multicolumn{2}{l}{Uniform-C(0.02)}
			\input{results/S2/glow_cifar10_vs_uniform_gray_bs_5_10_S2.tex}
			
			& \multicolumn{2}{l}{CelebA}
			\input{results/S2/glow_cifar10_vs_CelebA32_bs_5_10_S2.tex}
			& \multicolumn{2}{l}{CelebA-C(0.3)}
			\input{results/S2/glow_cifar10_vs_CelebA32-C0.3_bs_5_10_S2.tex}
			& \multicolumn{2}{l}{ImageNet32} 
			\input{results/S2/glow_cifar10_vs_Imagenet32_bs_5_10_S2.tex}
			& \multicolumn{2}{l}{ImageNet32-C(0.3)}
			\input{results/S2/glow_cifar10_vs_Imagenet32-C0.3_bs_5_10_S2.tex}
			& \multicolumn{2}{l}{SVHN} 
			\input{results/S2/glow_cifar10_vs_SVHN_bs_5_10_S2.tex}
			& \multicolumn{2}{l}{SVHN-C(2.0)}
			\input{results/S2/glow_cifar10_vs_SVHN-C2.0_bs_5_10_S2.tex}\\
			& \multicolumn{2}{l}{LSUN}
			&\textbf{100.0$\pm$0.0}&\textbf{100.0$\pm$0.0} & 99.9$\pm$0.0 & 99.9$\pm$0.0
			&\textbf{100.0$\pm$0.0}&\textbf{100.0$\pm$0.0}&\textbf{100.0$\pm$0.0}&\textbf{100.0$\pm$0.0}\\
			& \multicolumn{2}{l}{LSUN-C(0.3)}
			& \textbf{90.0$\pm$0.2} & \textbf{90.8$\pm$0.2} & 52.2$\pm$0.8 & 48.8$\pm$0.4
			& \textbf{91.2$\pm$0.2} & \textbf{92.0$\pm$0.2} & 56.6$\pm$0.4 & 51.3$\pm$0.3\\
			\hline
			
			\multirow{10}*{\rotatebox{90}{CelebA}} 
			& \multicolumn{2}{l}{Constant}
			\input{results/S2/glow_celeba32_vs_constant_bs_5_10_S2.tex}
			& \multicolumn{2}{l}{Uniform}
			\input{results/S2/glow_celeba32_vs_uniform_noise_bs_5_10_S2.tex}
			& \multicolumn{2}{l}{Uniform-C(0.012)}
			\input{results/S2/glow_celeba32_vs_uniform_noise_gray_bs_5_10_S2.tex}
			& \multicolumn{2}{l}{CIFAR-10}
			\input{results/S2/glow_celeba32_vs_cifar10_bs_5_10_S2.tex}
			& \multicolumn{2}{l}{CIFAR-100} 
			\input{results/S2/glow_celeba32_vs_cifar100_bs_5_10_S2.tex}
			& \multicolumn{2}{l}{ImageNet32}
			\input{results/S2/glow_celeba32_vs_imagenet32_bs_5_10_S2.tex}
			& \multicolumn{2}{l}{ImageNet32-C(0.2)}
			& \textbf{100.0$\pm$0.0} & \textbf{100.0$\pm$0.0} & 26.0$\pm$0.3 & 36.4$\pm$0.2  	
			& \textbf{100.0$\pm$0.0} & \textbf{100.0$\pm$0.0} & 18.2$\pm$0.3 & 33.8$\pm$0.0\\
			& \multicolumn{2}{l}{SVHN} 
			\input{results/S2/glow_celeba32_vs_svhn_bs_5_10_S2.tex}
			& \multicolumn{2}{l}{SVHN-C(1.8)} 
			\input{results/S2/glow_celeba32_vs_svhn_divergence_bs_5_10_S2.tex}\\
			& \multicolumn{2}{l}{LSUN}
			& \textbf{100.0$\pm$0.0} & \textbf{100.0$\pm$0.0} & 65.4$\pm$0.3 & 64.3$\pm$0.2  	
			& \textbf{100.0$\pm$0.0} & \textbf{100.0$\pm$0.0} & 71.2$\pm$0.3 & 70.0$\pm$0.5\\
			
			\hline
			&\textbf{average}& & \textbf{98.1} & \textbf{98.2} & 64.6 & 69.0 & \textbf{98.8} & \textbf{98.9} & 64.0 & 68.4\\
			\bottomrule[1pt]   			
		\end{tabular}  
	\end{center}  
\end{table*}

\begin{table*} [ht]
	\scriptsize
	\caption{GAD Results (AUROC and AUPR in percentage) of \PADmethod\ and Annulus Method (R) (reimplementation) on Glow with batch sizes 5 and 10. The higher  the better.  \textit{Our reimplementation of Annulus Method achieves much better results than that reported in \cite{nalisnick2019detecting} (referenced by Table \ref{tbl:glow_GAD_many_baselines})}.
	} 
	\label{tbl:glow_fashionmnist_svhn_cifar10_celeba_DOCR_TC_M_bs_5_10_S2_vs_annulus_method}  
	\begin{center}  
		\begin{tabular}{cllr rr r|r rr r}  
			\toprule[1pt]
			\multirow{3}*{ID$\downarrow$}&\multirow{3}*{OOD$\downarrow$}&Batch size$\rightarrow$&
			\multicolumn{4}{c}{$m$=5}  &\multicolumn{4}{c}{$m$=10}
			\\
			\cline{3-11}   
			& & Method$\rightarrow$& \multicolumn{2}{c}{\PADmethod} & \multicolumn{2}{c}{Annulus Method (R)} 
			& \multicolumn{2}{c}{\PADmethod} & \multicolumn{2}{c}{Annulus Method (R)} \\
			\cline{3-11} 
			&&Metric$\rightarrow$
			&  \multicolumn{1}{c}{AUROC} 
			& \multicolumn{1}{c}{AUPR}
			&  \multicolumn{1}{c}{AUROC} 
			& \multicolumn{1}{c}{AUPR}
			&  \multicolumn{1}{c}{AUROC} 
			& \multicolumn{1}{c}{AUPR}
			&  \multicolumn{1}{c}{AUROC} 
			& \multicolumn{1}{c}{AUPR}\\
			\hline
			
			\multirow{5}*{\rotatebox{90}{\tabincell{c}{Fash.}}}

			&\multicolumn{2}{l}{Constant}
			& \textbf{100.0$\pm$0.0} & \textbf{100.0$\pm$0.0} & \textbf{100.0$\pm$0.0} & \textbf{100.0$\pm$0.0} & \textbf{100.0$\pm$0.0} & \textbf{100.0$\pm$0.0} & \textbf{100.0$\pm$0.0} & \textbf{100.0$\pm$0.0} \\


			&\multicolumn{2}{l}{MNIST} 
			& \textbf{99.8$\pm$0.0} & \textbf{99.8$\pm$0.0} & 98.6$\pm$0.0 & 98.7$\pm$0.0 & \textbf{100.0$\pm$0.0} & \textbf{100.0$\pm$0.0} & 99.9$\pm$0.0 & 99.9$\pm$0.0 \\
			
			&\multicolumn{2}{l}{MNIST-C(10.0)}
			& \textbf{100.0$\pm$0.0} & \textbf{100.0$\pm$0.0} & \textbf{100.0$\pm$0.0} & \textbf{100.0$\pm$0.0} & \textbf{100.0$\pm$0.0} & \textbf{100.0$\pm$0.0} & \textbf{100.0$\pm$0.0} & \textbf{100.0$\pm$0.0} \\
			
			& \multicolumn{2}{l}{notMNIST} 
			& \textbf{100.0$\pm$0.0} & \textbf{100.0$\pm$0.0} & \textbf{100.0$\pm$0.0} & \textbf{100.0$\pm$0.0} & \textbf{100.0$\pm$0.0} & \textbf{100.0$\pm$0.0} & \textbf{100.0$\pm$0.0} & \textbf{100.0$\pm$0.0} \\
			& \multicolumn{2}{l}{notMNIST-C(0.005)}
			& \textbf{100.0$\pm$0.0} & \textbf{100.0$\pm$0.0} & \textbf{100.0$\pm$0.0}& \textbf{100.0$\pm$0.0} & \textbf{100.0$\pm$0.0} & \textbf{100.0$\pm$0.0} & \textbf{100.0$\pm$0.0} & \textbf{100.0$\pm$0.0} \\
			\hline
			\multirow{13}*{\rotatebox{90}{SVHN}} 
			
			& \multicolumn{2}{l}{Constant} 
			& \textbf{100.0$\pm$0.0} & \textbf{100.0$\pm$0.0} & \textbf{100.0$\pm$0.0} & \textbf{100.0$\pm$0.0} & \textbf{100.0$\pm$0.0} & \textbf{100.0$\pm$0.0} & \textbf{100.0$\pm$0.0} & \textbf{100.0$\pm$0.0} \\
			
			& \multicolumn{2}{l}{Uniform} 
			& \textbf{100.0$\pm$0.0} & \textbf{100.0$\pm$0.0} & \textbf{100.0$\pm$0.0} & \textbf{100.0$\pm$0.0} & \textbf{100.0$\pm$0.0} & \textbf{100.0$\pm$0.0} & \textbf{100.0$\pm$0.0} & \textbf{100.0$\pm$0.0} \\

			& \multicolumn{2}{l}{Uniform-C(0.008)} 
			& \textbf{100.0$\pm$0.0} & \textbf{100.0$\pm$0.0} & \textbf{100.0$\pm$0.0} & \textbf{100.0$\pm$0.0} & \textbf{100.0$\pm$0.0} & \textbf{100.0$\pm$0.0} & \textbf{100.0$\pm$0.0} & \textbf{100.0$\pm$0.0} \\

			& \multicolumn{2}{l}{CelebA} 
			& \textbf{100.0$\pm$0.0} & \textbf{100.0$\pm$0.0} & \textbf{100.0$\pm$0.0} & \textbf{100.0$\pm$0.0} & \textbf{100.0$\pm$0.0} & \textbf{100.0$\pm$0.0} & \textbf{100.0$\pm$0.0} & \textbf{100.0$\pm$0.0} \\
			& \multicolumn{2}{l}{CelebA-C(0.08)}
			& \textbf{99.7$\pm$0.0} & \textbf{99.7$\pm$0.0} & 50.0$\pm$0.2 & 49.3$\pm$0.1 & \textbf{100.0$\pm$0.0} & \textbf{100.0$\pm$0.0} & 49.6$\pm$0.2 & 49.2$\pm$0.1 \\
			
			& \multicolumn{2}{l}{CIFAR-10} 
			& \textbf{100.0$\pm$0.0} & \textbf{100.0$\pm$0.0} & \textbf{100.0$\pm$0.0} & \textbf{100.0$\pm$0.0} & \textbf{100.0$\pm$0.0} & \textbf{100.0$\pm$0.0} & \textbf{100.0$\pm$0.0} & \textbf{100.0$\pm$0.0} \\
			
			& \multicolumn{2}{l}{CIFAR-10-C(0.12)}
			& \textbf{97.0$\pm$0.2} & \textbf{97.4$\pm$0.2} & 59.6$\pm$0.0 & 58.2$\pm$0.0 & \textbf{99.3$\pm$0.1} & \textbf{99.4$\pm$0.1} & 63.5$\pm$0.1 & 62.0$\pm$0.5 \\
			
			& \multicolumn{2}{l}{CIFAR-100} 
			& \textbf{100.0$\pm$0.0} & \textbf{100.0$\pm$0.0} & \textbf{100.0$\pm$0.0} & \textbf{100.0$\pm$0.0} & \textbf{100.0$\pm$0.0} & \textbf{100.0$\pm$0.0} & \textbf{100.0$\pm$0.0} & \textbf{100.0$\pm$0.0} \\
			
			& \multicolumn{2}{l}{CIFAR-100-C(0.12)}
			& \textbf{96.9$\pm$0.1} & \textbf{97.3$\pm$0.1} & 70.9$\pm$0.2 & 70.4$\pm$0.3 & \textbf{98.9$\pm$0.3} & \textbf{99.0$\pm$0.3} & 78.2$\pm$0.3 & 77.9$\pm$0.4 \\
			
			& \multicolumn{2}{l}{ImageNet32} 
			& \textbf{100.0$\pm$0.0} & \textbf{100.0$\pm$0.0} & \textbf{100.0$\pm$0.0} & \textbf{100.0$\pm$0.0} & \textbf{100.0$\pm$0.0} & \textbf{100.0$\pm$0.0} & \textbf{100.0$\pm$0.0} & \textbf{100.0$\pm$0.0} \\
			
			& \multicolumn{2}{l}{ImageNet32-C(0.07)}
			& \textbf{99.8$\pm$0.0} & \textbf{99.8$\pm$0.0} & 76.3$\pm$0.2 & 75.7$\pm$0.3 & \textbf{100.0$\pm$0.0} & \textbf{100.0$\pm$0.0} & 84.2$\pm$0.5 & 83.8$\pm$0.5 \\
			& \multicolumn{2}{l}{LSUN}
			& \textbf{100.0$\pm$0.0} & \textbf{100.0$\pm$0.0} & \textbf{100.0$\pm$0.0}& \textbf{100.0$\pm$0.0}
			& \textbf{100.0$\pm$0.0}& \textbf{100.0$\pm$0.0}& \textbf{100.0$\pm$0.0}& \textbf{100.0$\pm$0.0}\\
			& \multicolumn{2}{l}{LSUN-C(0.06)}
			&\textbf{99.9$\pm$0.0} & \textbf{99.0$\pm$0.0} & 96.4$\pm$0.1 & 96.1$\pm$0.1
			&\textbf{100.0$\pm$0.0} &\textbf{100.0$\pm$0.0} & 99.4$\pm$0.1  & 99.4$\pm$0.1 \\
			\hline
			
			\multirow{11}*{\rotatebox{90}{CIFAR-10}}
			
			& \multicolumn{2}{l}{Constant}
			& \textbf{100.0$\pm$0.0} & \textbf{100.0$\pm$0.0} & 62.0$\pm$1.9 & 67.2$\pm$2.5 & \textbf{100.0$\pm$0.0} & \textbf{100.0$\pm$0.0} & 66.7$\pm$4.5 & 75.6$\pm$2.8 \\

			& \multicolumn{2}{l}{Uniform}
			& \textbf{100.0$\pm$0.0} & \textbf{100.0$\pm$0.0} & 6.3$\pm$1.0 & 31.8$\pm$0.4 & \textbf{100.0$\pm$0.0} & \textbf{100.0$\pm$0.0} & 7.3$\pm$1.4 & 34.1$\pm$1.6 \\
			& \multicolumn{2}{l}{Uniform-C(0.02)}
			& \textbf{100.0$\pm$0.0} & \textbf{100.0$\pm$0.0} & 54.4$\pm$2.1 & 63.3$\pm$2.6 & \textbf{100.0$\pm$0.0} & \textbf{100.0$\pm$0.0} & 61.3$\pm$2.9 & 71.8$\pm$2.7 \\

			& \multicolumn{2}{l}{CelebA}
			& \textbf{99.2$\pm$0.1} & \textbf{99.4$\pm$0.1} & 36.8$\pm$1.9 & 48.3$\pm$2.2 & \textbf{100.0$\pm$0.0} & \textbf{100.0$\pm$0.0} & 45.4$\pm$4.9 & 60.6$\pm$4.1 \\
			
			& \multicolumn{2}{l}{CelebA-C(0.3)}
			& \textbf{84.3$\pm$0.3} & \textbf{84.4$\pm$0.4} & 55.9$\pm$2.8 & 63.0$\pm$2.7 & \textbf{94.5$\pm$0.3} & \textbf{94.7$\pm$0.3} & 48.4$\pm$2.8 & 62.1$\pm$2.8 \\
			
			& \multicolumn{2}{l}{ImageNet32} 
			& \textbf{90.0$\pm$0.2} & \textbf{92.1$\pm$0.1} & 45.8$\pm$1.9& 55.2$\pm$1.4  & \textbf{95.0$\pm$0.4} & \textbf{96.2$\pm$0.2} &43.8$\pm$2.5 & 57.7$\pm$2.1 \\
			
			& \multicolumn{2}{l}{ImageNet32-C(0.3)}
			& \textbf{72.0$\pm$0.3} & \textbf{72.6$\pm$0.4} & 47.2$\pm$2.6 & 56.1$\pm$2.5 & \textbf{74.3$\pm$0.6} & \textbf{74.8$\pm$0.8} & 51.3$\pm$3.5 & 64.5$\pm$2.7 \\
			
			& \multicolumn{2}{l}{SVHN} 
			& \textbf{97.6$\pm$0.2} & \textbf{97.8$\pm$0.2} & 47.3$\pm$1.6 & 56.9$\pm$2.4 & \textbf{99.8$\pm$0.0} & \textbf{99.8$\pm$0.0} & 49.4$\pm$2.4 & 63.0$\pm$2.4 \\
			
			& \multicolumn{2}{l}{SVHN-C(2.0)}
			& \textbf{100.0$\pm$0.0} & \textbf{100.0$\pm$0.0} & 45.0$\pm$1.5 & 56.0$\pm$1.9 & \textbf{100.0$\pm$0.0} & \textbf{100.0$\pm$0.0} & 39.4$\pm$0.9 & 54.9$\pm$1.5 \\
			& \multicolumn{2}{l}{LSUN}
			&  \textbf{100.0$\pm$0.0} & \textbf{100.0$\pm$0.0} & 29.1$\pm$1.2 & 44.6$\pm$2.2
			&  \textbf{100.0$\pm$0.0} & \textbf{100.0$\pm$0.0} & 29.9$\pm$4.2 & 46.5$\pm$4.2\\
			& \multicolumn{2}{l}{LSUN-C(0.3)}
			& \textbf{90.0$\pm$0.2} & \textbf{90.8$\pm$0.2} & 49.4$\pm$2.1 & 58.1$\pm$2.3
			& \textbf{91.2$\pm$0.2} & \textbf{92.0$\pm$0.2} & 43.6$\pm$1.1 & 58.1$\pm$1.0\\
			\hline
			
			\multirow{10}*{\rotatebox{90}{CelebA}} 
			& \multicolumn{2}{l}{Constant}
			& \textbf{100.0$\pm$0.0} & \textbf{100.0$\pm$0.0} & \textbf{100.0$\pm$0.0} & \textbf{100.0$\pm$0.0} & \textbf{100.0$\pm$0.0} & \textbf{100.0$\pm$0.0} & \textbf{100.0$\pm$0.0} & \textbf{100.0$\pm$0.0} \\
			
			& \multicolumn{2}{l}{Uniform}
			& \textbf{100.0$\pm$0.0} & \textbf{100.0$\pm$0.0} & \textbf{100.0$\pm$0.0} & \textbf{100.0$\pm$0.0} & \textbf{100.0$\pm$0.0} & \textbf{100.0$\pm$0.0} & \textbf{100.0$\pm$0.0} & \textbf{100.0$\pm$0.0} \\
			
			& \multicolumn{2}{l}{Uniform-C(0.012)}
			& \textbf{100.0$\pm$0.0} & \textbf{100.0$\pm$0.0} & \textbf{100.0$\pm$0.0} & \textbf{100.0$\pm$0.0} & \textbf{100.0$\pm$0.0} & \textbf{100.0$\pm$0.0} & \textbf{100.0$\pm$0.0} & \textbf{100.0$\pm$0.0} \\
			
			& \multicolumn{2}{l}{CIFAR-10}
			& \textbf{99.6$\pm$0.0} & \textbf{99.6$\pm$0.0} & 99.5$\pm$0.0 & 99.6$\pm$0.0 & \textbf{100.0$\pm$0.0} & \textbf{100.0$\pm$0.0} & \textbf{100.0$\pm$0.0} & \textbf{100.0$\pm$0.0} \\
			
			& \multicolumn{2}{l}{CIFAR-100} 
			& \textbf{99.8$\pm$0.0} & \textbf{99.8$\pm$0.0} & 99.6$\pm$0.0 & 99.6$\pm$0.0 & \textbf{100.0$\pm$0.0} & \textbf{100.0$\pm$0.0} & \textbf{100.0$\pm$0.0} & \textbf{100.0$\pm$0.0} \\

			& \multicolumn{2}{l}{ImageNet32}
			& \textbf{100.0$\pm$0.0} & \textbf{100.0$\pm$0.0} & \textbf{100.0$\pm$0.0} & \textbf{100.0$\pm$0.0} & \textbf{100.0$\pm$0.0} & \textbf{100.0$\pm$0.0} & \textbf{100.0$\pm$0.0} & \textbf{100.0$\pm$0.0} \\
			
			& \multicolumn{2}{l}{ImageNet32-C(0.2)}
			& \textbf{100.0$\pm$0.0} & \textbf{100.0$\pm$0.0} & \textbf{100.0$\pm$0.0} & \textbf{100.0$\pm$0.0} & \textbf{100.0$\pm$0.0} & \textbf{100.0$\pm$0.0} & \textbf{100.0$\pm$0.0} & \textbf{100.0$\pm$0.0} \\
			
			& \multicolumn{2}{l}{SVHN} 
			& \textbf{100.0$\pm$0.0} & \textbf{100.0$\pm$0.0} & \textbf{100.0$\pm$0.0} & \textbf{100.0$\pm$0.0} & \textbf{100.0$\pm$0.0} & \textbf{100.0$\pm$0.0} & \textbf{100.0$\pm$0.0} & \textbf{100.0$\pm$0.0} \\
			
			& \multicolumn{2}{l}{SVHN-C(1.8)} 
			& \textbf{100.0$\pm$0.0} & \textbf{100.0$\pm$0.0} & \textbf{100.0$\pm$0.0} & \textbf{100.0$\pm$0.0} & \textbf{100.0$\pm$0.0} & \textbf{100.0$\pm$0.0} & \textbf{100.0$\pm$0.0} & \textbf{100.0$\pm$0.0} \\
			& \multicolumn{2}{l}{LSUN}
			& \textbf{100.0$\pm$0.0}& \textbf{100.0$\pm$0.0} & \textbf{100.0$\pm$0.0}& \textbf{100.0$\pm$0.0}
			& \textbf{100.0$\pm$0.0}& \textbf{100.0$\pm$0.0} & \textbf{100.0$\pm$0.0}&\textbf{100.0$\pm$0.0}\\
			\hline
			&\textbf{average}& & \textbf{98.1} & \textbf{98.2} & 80.3 & 83.3 & \textbf{98.8} & \textbf{98.9} & 81.1 & 85.2\\
			\bottomrule[1pt]   			
		\end{tabular}  
	\end{center}  
\end{table*}

\begin{table*} [ht]
	\scriptsize
	\caption{GAD Results (EER in percentage) of \PADmethod, Ty-test and Annulus Method (Annulus.) on Glow with batch sizes 5 and 10. The lower the better.
	} 
	\label{tbl:glow_fashionmnist_svhn_cifar10_celeba_DOCR_TC_M_EER_bs_5_10_S2}  
	\begin{center}  
		\begin{tabular}{cll r r r |r r r}  
			\toprule[1pt]
			\multirow{3}*{ID$\downarrow$}&\multirow{3}*{OOD$\downarrow$}&Batch size$\rightarrow$&
			\multicolumn{3}{c}{$m$=5}  &\multicolumn{3}{c}{$m$=10}
			\\
			\cline{3-9}   
			& & Method$\rightarrow$& \multicolumn{1}{c}{\PADmethod} & \multicolumn{1}{c}{Ty-test} & \multicolumn{1}{c}{Annulus.} 
			& \multicolumn{1}{c}{\PADmethod} & \multicolumn{1}{c}{Ty-test} & \multicolumn{1}{c}{Annulus.}\\
			
			\hline
			
			\multirow{5}*{\rotatebox{90}{\tabincell{c}{Fash.}}}

			&\multicolumn{2}{l}{Constant}
			& \textbf{0.0$\pm$0.0} & 54.2$\pm$0.4 & \textbf{0.0$\pm$0.0} & \textbf{0.0$\pm$0.0} & 53.9$\pm$0.8 & \textbf{0.0$\pm$0.0} \\


			&\multicolumn{2}{l}{MNIST} 
			& \textbf{0.9$\pm$0.2} & 5.6$\pm$0.3 & 6.1$\pm$0.4 & \textbf{0.0$\pm$0.0} & 1.2$\pm$0.2 & 1.6$\pm$0.3\\

			&\multicolumn{2}{l}{MNIST-C(10.0)}
			& \textbf{0.0$\pm$0.0} & 14.7$\pm$0.5 & 0.6$\pm$0.1 & \textbf{0.0$\pm$0.0} & 7.3$\pm$0.2 &\textbf{0.0$\pm$0.0} \\

			& \multicolumn{2}{l}{notMNIST} 
			& \textbf{0.0$\pm$0.0} & 29.5$\pm$0.3 & \textbf{0.0$\pm$0.0} & \textbf{0.0$\pm$0.0} & 20.8$\pm$0.3 & \textbf{0.0$\pm$0.0} \\

			& \multicolumn{2}{l}{notMNIST-C(0.005)}
			& \textbf{0.0$\pm$0.0} & 44.5$\pm$0.4 & \textbf{0.0$\pm$0.0} &  \textbf{0.0$\pm$0.0} & 39.4$\pm$0.5 & \textbf{0.0$\pm$0.0} \\
			\hline
			\multirow{13}*{\rotatebox{90}{SVHN}} 
			
			& \multicolumn{2}{l}{Constant} 
			& \textbf{0.0$\pm$0.0} & \textbf{0.0$\pm$0.0} & \textbf{0.0$\pm$0.0} & \textbf{0.0$\pm$0.0} & \textbf{0.0$\pm$0.0} & \textbf{0.0$\pm$0.0} \\
			
			& \multicolumn{2}{l}{Uniform} 
			& \textbf{0.0$\pm$0.0} & \textbf{0.0$\pm$0.0} & \textbf{0.0$\pm$0.0} & \textbf{0.0$\pm$0.0} & \textbf{0.0$\pm$0.0} & \textbf{0.0$\pm$0.0} \\

			& \multicolumn{2}{l}{Uniform-C(0.008)} 
			& \textbf{0.0$\pm$0.0} & 79.1$\pm$0.6 & \textbf{0.0$\pm$0.0} & \textbf{0.0$\pm$0.0} & 81.7$\pm$0.4 & \textbf{0.0$\pm$0.0} \\

			& \multicolumn{2}{l}{CelebA} 
			& \textbf{0.0$\pm$0.0} & \textbf{0.0$\pm$0.0} & \textbf{0.0$\pm$0.0}  & \textbf{0.0$\pm$0.0} & \textbf{0.0$\pm$0.0} & \textbf{0.0$\pm$0.0}  \\

			& \multicolumn{2}{l}{CelebA-C(0.08)}
			& \textbf{8.0$\pm$0.3} & 43.2$\pm$0.3 & 50.3$\pm$0.3 & \textbf{2.3$\pm$0.2} & 40.2$\pm$0.8 &50.2$\pm$0.3 \\

			& \multicolumn{2}{l}{CIFAR-10} 
			& \textbf{0.0$\pm$0.0} & \textbf{0.0$\pm$0.0} & \textbf{0.0$\pm$0.0} & \textbf{0.0$\pm$0.0} & \textbf{0.0$\pm$0.0} & \textbf{0.0$\pm$0.0} \\

			& \multicolumn{2}{l}{CIFAR-10-C(0.12)}
			& \textbf{24.0$\pm$0.5} & 64.1$\pm$0.3 & 43.2$\pm$0.4 & \textbf{18.7$\pm$0.5} & 70.9$\pm$0.3 & 40.6$\pm$0.6 \\

			& \multicolumn{2}{l}{CIFAR-100} 
			& \textbf{0.0$\pm$0.0} & \textbf{0.0$\pm$0.0} & \textbf{0.0$\pm$0.0} & \textbf{0.0$\pm$0.0} & \textbf{0.0$\pm$0.0} & \textbf{0.0$\pm$0.0} \\

			& \multicolumn{2}{l}{CIFAR-100-C(0.12)}
			& \textbf{23.3$\pm$0.4} & 61.0$\pm$0.3 & 34.6$\pm$0.3 & \textbf{19.4$\pm$0.5} & 68.2$\pm$0.2 & 29.2$\pm$0.7 \\

			& \multicolumn{2}{l}{ImageNet32} 
			& \textbf{0.0$\pm$0.0} & \textbf{0.0$\pm$0.0} &\textbf{0.0$\pm$0.0} & \textbf{0.0$\pm$0.0} & \textbf{0.0$\pm$0.0} & \textbf{0.0$\pm$0.0} \\
			& \multicolumn{2}{l}{ImageNet32-C(0.07)}
			& \textbf{9.3$\pm$0.2} & 53.7$\pm$0.6 & 30.8$\pm$0.3 & \textbf{4.0$\pm$0.3} & 56.4$\pm$0.6 & 23.8$\pm$0.4 \\
			& \multicolumn{2}{l}{LSUN}
			& \textbf{0.0$\pm$0.0} & \textbf{0.0$\pm$0.0} & \textbf{0.0$\pm$0.0} & \textbf{0.0$\pm$0.0} & \textbf{0.0$\pm$0.0} & \textbf{0.0$\pm$0.0} \\
			& \multicolumn{2}{l}{LSUN-C(0.06)}
			& \textbf{0.8$\pm$0.1} & 54.1$\pm$0.4 & 9.6$\pm$0.3 & \textbf{0.0$\pm$0.0} & 42.3$\pm$0.5 & 3.7$\pm$0.5\\
			\hline

			\multirow{11}*{\rotatebox{90}{CIFAR-10}}
			
			& \multicolumn{2}{l}{Constant}
			& \textbf{0.0$\pm$0.0} & \textbf{0.0$\pm$0.0} & 41.4$\pm$2.0 & \textbf{0.0$\pm$0.0} & \textbf{0.0$\pm$0.0} & 37.8$\pm$4.7 \\

			& \multicolumn{2}{l}{Uniform}
			& \textbf{0.0$\pm$0.0} & \textbf{0.0$\pm$0.0} & 87.6$\pm$1.2 & \textbf{0.0$\pm$0.0} & \textbf{0.0$\pm$0.0} & 89.2$\pm$3.3 \\
			& \multicolumn{2}{l}{Uniform-C(0.02)}
			& \textbf{0.0$\pm$0.0} & 85.9$\pm$0.5 & 48.0$\pm$2.3 & \textbf{0.0$\pm$0.0} & 87.5$\pm$1.0 & 41.3$\pm$3.4 \\

			& \multicolumn{2}{l}{CelebA}
			& 3.9$\pm$0.2 & \textbf{0.9$\pm$0.1} & 59.2$\pm$1.8 & 0.4$\pm$0.1 & \textbf{0.0$\pm$0.0} &55.1$\pm$4.1 \\

			& \multicolumn{2}{l}{CelebA-C(0.3)}
			& \textbf{23.3$\pm$0.7} & 65.5$\pm$0.6 & 46.0$\pm$2.7 & \textbf{13.4$\pm$0.5} & 69.5$\pm$0.8 & 50.4$\pm$3.0 \\

			& \multicolumn{2}{l}{ImageNet32} 
			& 18.9$\pm$0.6 & \textbf{3.8$\pm$0.3} & 53.5$\pm$0.9 & 12.3$\pm$0.6 & \textbf{0.8$\pm$0.1} & 55.4$\pm$3.2 \\

			& \multicolumn{2}{l}{ImageNet32-C(0.3)}
			& \textbf{35.3$\pm$0.6} & 57.1$\pm$0.5 &51.4$\pm$2.8 & \textbf{32.6$\pm$1.3} & 64.0$\pm$0.7 & 50.9$\pm$2.7 \\

			& \multicolumn{2}{l}{SVHN} 
			& 8.0$\pm$0.2 & \textbf{5.6$\pm$0.2} & 52.2$\pm$1.8 & 2.5$\pm$0.4 & \textbf{1.1$\pm$0.2} & 52.5$\pm$2.4 \\

			& \multicolumn{2}{l}{SVHN-C(2.0)}
			& \textbf{0.2$\pm$0.1} & 61.3$\pm$0.4 & 53.7$\pm$1.7 & \textbf{0.0$\pm$0.0} & 67.0$\pm$0.6 & 59.0$\pm$1.5 \\
			& \multicolumn{2}{l}{LSUN}
			& \textbf{0.0$\pm$0.0} & 1.2$\pm$0.1  & 65.9$\pm$0.8 & \textbf{0.0$\pm$0.0} &  \textbf{0.0$\pm$0.0} & 65.4$\pm$4.7 \\
			& \multicolumn{2}{l}{LSUN-C(0.3)}
			& \textbf{18.2$\pm$0.3} & 47.7$\pm$0.6 & 51.1$\pm$1.0 & \textbf{17.2$\pm$0.4} & 44.7$\pm$0.9 & 54.1$\pm$2.1\\
			\hline
			
			\multirow{10}*{\rotatebox{90}{CelebA}} 
			& \multicolumn{2}{l}{Constant}
			& \textbf{0.0$\pm$0.0} & \textbf{0.0$\pm$0.0} &\textbf{0.0$\pm$0.0}  & \textbf{0.0$\pm$0.0} & \textbf{0.0$\pm$0.0} &\textbf{0.0$\pm$0.0}\\
			
			& \multicolumn{2}{l}{Uniform}
			& \textbf{0.0$\pm$0.0} & \textbf{0.0$\pm$0.0} & \textbf{0.0$\pm$0.0} & \textbf{0.0$\pm$0.0} & \textbf{0.0$\pm$0.0} &\textbf{0.0$\pm$0.0} \\

			& \multicolumn{2}{l}{Uniform-C(0.012)}
			& \textbf{0.0$\pm$0.0} & 71.0$\pm$0.5 & 0.1$\pm$0.0 & \textbf{0.0$\pm$0.0} & 78.3$\pm$0.6 & \textbf{0.0$\pm$0.0}\\

			& \multicolumn{2}{l}{CIFAR-10}
			& \textbf{3.2$\pm$0.3} & 85.1$\pm$0.6 & 3.4$\pm$0.3 & \textbf{0.1$\pm$0.1} & 93.1$\pm$0.6 &0.4$\pm$0.2 \\

			& \multicolumn{2}{l}{CIFAR-100} 
			& \textbf{2.8$\pm$0.2} & 82.9$\pm$0.1 & 3.0$\pm$0.2 & \textbf{0.0$\pm$0.0} & 91.6$\pm$0.5 &0.3$\pm$0.1 \\

			& \multicolumn{2}{l}{ImageNet32}
			& \textbf{0.0$\pm$0.0} & 26.4$\pm$0.2 & \textbf{0.0$\pm$0.0} & \textbf{0.0$\pm$0.0} & 20.7$\pm$0.8 & \textbf{0.0$\pm$0.0} \\

			& \multicolumn{2}{l}{ImageNet32-C(0.2)}
			& \textbf{0.0$\pm$0.0} & 67.7$\pm$0.6 & \textbf{0.0$\pm$0.0} & \textbf{0.0$\pm$0.0} & 73.7$\pm$0.5  & \textbf{0.0$\pm$0.0}\\

			& \multicolumn{2}{l}{SVHN} 
			& \textbf{0.0$\pm$0.0} & 28.5$\pm$0.2 & \textbf{0.0$\pm$0.0} & \textbf{0.0$\pm$0.0} & 22.0$\pm$0.4 & \textbf{0.0$\pm$0.0} \\

			& \multicolumn{2}{l}{SVHN-C(1.8)} 
			& \textbf{0.0$\pm$0.0} & 90.0$\pm$0.4  & \textbf{0.0$\pm$0.0} & \textbf{0.0$\pm$0.0} & 97.0$\pm$0.3 & \textbf{0.0$\pm$0.0} \\
			& \multicolumn{2}{l}{LSUN}
			& \textbf{0.0$\pm$0.0} & 38.8$\pm$0.5 &  \textbf{0.0$\pm$0.0} & \textbf{0.0$\pm$0.0} & 34.6$\pm$0.3 & \textbf{0.0$\pm$0.0}\\
			\hline
			&\textbf{average} && \textbf{4.6}  &  33.9 & 33.9 & \textbf{3.2} & 34.0 & 19.5\\
			\bottomrule[1pt]   			
		\end{tabular}  
	\end{center}  
\end{table*}

\begin{table*} [ht]
	\scriptsize
	\caption{GAD Results (AUROC and AUPR  in percentage) of \PADmethod\ and Ty-test on Glow with batch sizes 2 and 4. The higher the better. 
	} 
	\label{tbl:glow_fashionmnist_svhn_cifar10_celeba_DOCR_TC_M_bs_2_4_S2}  
	\begin{center}  
		\begin{tabular}{cllr rr r|r rr r}  
			\toprule[1pt]
			\multirow{3}*{ID$\downarrow$}&\multirow{3}*{OOD$\downarrow$}&Batch size $\rightarrow$&
			\multicolumn{4}{c}{$m$=2}  &\multicolumn{4}{c}{$m$=4}
			\\
			\cline{3-11}   
			& & Method$\rightarrow$  & \multicolumn{2}{c}{\PADmethod} & \multicolumn{2}{c}{Ty-test} 
			& \multicolumn{2}{c}{\PADmethod} & \multicolumn{2}{c}{Ty-test} \\
			\cline{3-11} 
			&&Metric$\rightarrow$
			&  \multicolumn{1}{c}{AUROC} 
			& \multicolumn{1}{c}{AUPR}
			&  \multicolumn{1}{c}{AUROC} 
			& \multicolumn{1}{c}{AUPR}
			&  \multicolumn{1}{c}{AUROC} 
			& \multicolumn{1}{c}{AUPR}
			&  \multicolumn{1}{c}{AUROC} 
			& \multicolumn{1}{c}{AUPR}\\
			\hline
			
			\multirow{5}*{\rotatebox{90}{\tabincell{c}{Fash.}}}
			
			&\multicolumn{2}{l}{Constant} 
			\input{results/S2/glow_fashionmnist_vs_constant_bs_2_4_S2.tex}

			
			&\multicolumn{2}{l}{MNIST} 
			\input{results/S2/glow_fashionmnist_vs_mnist_bs_2_4_S2.tex}
			&\multicolumn{2}{l}{MNIST-C(10.0)}
			\input{results/S2/glow_fashionmnist_vs_mnist_divergence_bs_2_4_S2.tex}
			& \multicolumn{2}{l}{notMNIST} 
			\input{results/S2/glow_fashionmnist_vs_notmnist_bs_2_4_S2.tex}
			& \multicolumn{2}{l}{notMNIST-C(0.005)}
			\input{results/S2/glow_fashionmnist_vs_notmnist_gray_bs_2_4_S2.tex}\\
			\hline
			\multirow{13}*{\rotatebox{90}{SVHN}} 
			
			& \multicolumn{2}{l}{Constant} 
			\input{results/S2/glow_svhn_vs_constant_bs_2_4_S2.tex}
			&\multicolumn{2}{l}{Uniform} 
			\input{results/S2/glow_svhn_vs_uniform_noise_bs_2_4_S2.tex}
			
			&\multicolumn{2}{l}{Uniform-C(0.008)} 
			\input{results/S2/glow_svhn_vs_uniform_noise_gray_bs_2_4_S2.tex}
			& \multicolumn{2}{l}{CelebA} 
			\input{results/S2/glow_svhn_vs_celeba32_bs_2_4_S2.tex}
			& \multicolumn{2}{l}{CelebA-C(0.08)}
			\input{results/S2/glow_svhn_vs_celeba32_gray_bs_2_4_S2.tex}
			& \multicolumn{2}{l}{CIFAR-10} 
			\input{results/S2/glow_svhn_vs_cifar10_bs_2_4_S2.tex}
			& \multicolumn{2}{l}{CIFAR-10-C(0.12)}
			\input{results/S2/glow_svhn_vs_cifar10_gray_bs_2_4_S2.tex}
			& \multicolumn{2}{l}{CIFAR-100} 
			\input{results/S2/glow_svhn_vs_cifar100_bs_2_4_S2.tex}
			& \multicolumn{2}{l}{CIFAR-100-C(0.12)}
			\input{results/S2/glow_svhn_vs_cifar100_gray_bs_2_4_S2.tex}
			& \multicolumn{2}{l}{ImageNet32} 
			\input{results/S2/glow_svhn_vs_imagenet32_bs_2_4_S2.tex}
			& \multicolumn{2}{l}{ImageNet32-C(0.07)}
			\input{results/S2/glow_svhn_vs_imagenet32_gray_bs_2_10_S2.tex}\\
			& \multicolumn{2}{l}{LSUN}
			& \textbf{100.0$\pm$0.0}& \textbf{100.0$\pm$0.0} & \textbf{100.0$\pm$0.0}& \textbf{100.0$\pm$0.0}
			& \textbf{100.0$\pm$0.0}& \textbf{100.0$\pm$0.0}& \textbf{100.0$\pm$0.0}& \textbf{100.0$\pm$0.0}\\
			& \multicolumn{2}{l}{LSUN-C(0.06)}
			& \textbf{98.1$\pm$0.0} & \textbf{97.7$\pm$0.1} & 41.3$\pm$0.4 & 42.0$\pm$0.2 
			& \textbf{99.8$\pm$0.0} & \textbf{99.7$\pm$0.0} & 42.4$\pm$0.4 & 42.4$\pm$0.1\\
			\hline
			\multirow{11}*{\rotatebox{90}{CIFAR-10}}

			& \multicolumn{2}{l}{Constant}
			\input{results/S2/glow_cifar10_vs_Constant_bs_2_4_S2.tex}
			& \multicolumn{2}{l}{Uniform}
			& \textbf{100.0$\pm$0.0} & \textbf{100.0$\pm$0.0} & \textbf{100.0$\pm$0.0} & \textbf{100.0$\pm$0.0} & \textbf{100.0$\pm$0.0} & \textbf{100.0$\pm$0.0} & \textbf{100.0$\pm$0.0} & \textbf{100.0$\pm$0.0} \\
			& \multicolumn{2}{l}{Uniform-C(0.02)}
			\input{results/S2/glow_cifar10_vs_uniform_gray_bs_2_4_S2.tex}
			
			& \multicolumn{2}{l}{CelebA}
			\input{results/S2/glow_cifar10_vs_CelebA32_bs_2_4_S2.tex}
			& \multicolumn{2}{l}{CelebA-C(0.3)}
			\input{results/S2/glow_cifar10_vs_CelebA32-C0.3_bs_2_4_S2.tex}
			& \multicolumn{2}{l}{ImageNet32} 
			\input{results/S2/glow_cifar10_vs_Imagenet32_bs_2_4_S2.tex}
			& \multicolumn{2}{l}{ImageNet32-C(0.3)}
			\input{results/S2/glow_cifar10_vs_Imagenet32-C0.3_bs_2_4_S2.tex}
			& \multicolumn{2}{l}{SVHN} 
			\input{results/S2/glow_cifar10_vs_SVHN_bs_2_4_S2.tex}
			& \multicolumn{2}{l}{SVHN-C(2.0)}
			\input{results/S2/glow_cifar10_vs_SVHN-C2.0_bs_2_4_S2.tex}\\
			& \multicolumn{2}{l}{LSUN}
			& \textbf{100.0$\pm$0.0} & \textbf{100.0$\pm$0.0} & 97.4$\pm$0.0 & 97.8$\pm$0.0
			& \textbf{100.0$\pm$0.0} & \textbf{100.0$\pm$0.0} & 99.8$\pm$0.0 & 99.8$\pm$0.0\\
			& \multicolumn{2}{l}{LSUN-C(0.3)}
			& \textbf{86.8$\pm$0.1} & \textbf{87.6$\pm$0.2} & 46.6$\pm$0.4 & 45.5$\pm$0.2
			& \textbf{89.2$\pm$0.3} & \textbf{90.0$\pm$0.2} & 50.4$\pm$0.3 & 47.7$\pm$0.2\\
			\hline
			
			\multirow{11}*{\rotatebox{90}{CelebA}} 
			
			& \multicolumn{2}{l}{Constant}
			\input{results/S2/glow_celeba32_vs_constant_bs_2_4_S2.tex}
			
			& \multicolumn{2}{l}{Uniform}
			\input{results/S2/glow_celeba32_vs_uniform_noise_bs_2_4_S2.tex}
			
			& \multicolumn{2}{l}{Uniform-C(0.12)}
			\input{results/S2/glow_celeba32_vs_uniform_noise_gray_bs_2_4_S2.tex}
			& \multicolumn{2}{l}{CIFAR-10}
			\input{results/S2/glow_celeba32_vs_cifar10_bs_2_4_S2.tex}
			& \multicolumn{2}{l}{CIFAR-100} 
			\input{results/S2/glow_celeba32_vs_cifar100_bs_2_4_S2.tex}
			& \multicolumn{2}{l}{ImageNet32}
			\input{results/S2/glow_celeba32_vs_imagenet32_bs_2_4_S2.tex}
			& \multicolumn{2}{l}{ImageNet32-C(0.2)}
			& \textbf{99.2$\pm$0.0} & \textbf{99.3$\pm$0.0} & 34.6$\pm$0.2 & 40.3$\pm$0.1 & \textbf{100.0$\pm$0.0} & \textbf{100.0$\pm$0.0} & 28.4$\pm$0.2 & 37.4$\pm$0.0\\
			& \multicolumn{2}{l}{SVHN} 
			\input{results/S2/glow_celeba32_vs_svhn_bs_2_4_S2.tex}
			& \multicolumn{2}{l}{SVHN-C(1.8)} 
			\input{results/S2/glow_celeba32_vs_svhn_divergence_bs_2_4_S2.tex}\\
			& \multicolumn{2}{l}{LSUN}
			& \textbf{100.0$\pm$0.0} & \textbf{100.0$\pm$0.0} & 60.2$\pm$0.1 &  58.8$\pm$0.3 
			&\textbf{100.0$\pm$0.0} & \textbf{100.0$\pm$0.0} & 63.8$\pm$0.1 & 62.5$\pm$0.3\\
			\hline
			& \textbf{average} && \textbf{94.8} & \textbf{94.9} &64.4 & 65.2 & \textbf{97.0} & \textbf{97.0} & 64.7 & 68.6 \\
			\bottomrule[1pt]   			
		\end{tabular}  
	\end{center}  
\end{table*}

\begin{table*} [ht]
	\scriptsize
	\caption{GAD Results (AUROC and AUPR in percentage) of \PADmethod\ and GOD2KS on Glow with batch sizes 5 and 10. We run our method for 5 times. The higher  the better. We reference the results on all problems of GOD2KS reported in \cite{jiang2022revisiting}, where the authors did not report average results of multiple runs.
	} 
	\label{tbl:glow_fashionmnist_svhn_cifar10_celeba_DOCR_TC_M_bs_5_10_S2_vs_GOD2KS}  
	\begin{center}  
		\begin{tabular}{cllr rr r|r rr r}  
			\toprule[1pt]
			\multirow{3}*{ID$\downarrow$}&\multirow{3}*{OOD$\downarrow$}&Batch size$\rightarrow$&
			\multicolumn{4}{c}{$m$=5}  &\multicolumn{4}{c}{$m$=10}
			\\
			\cline{3-11}   
			& & Method$\rightarrow$& \multicolumn{2}{c}{\PADmethod} & \multicolumn{2}{c}{GOD2KS} 
			& \multicolumn{2}{c}{\PADmethod} & \multicolumn{2}{c}{GOD2KS} \\
			\cline{3-11} 
			&&Metric$\rightarrow$
			&  \multicolumn{1}{c}{AUROC} 
			& \multicolumn{1}{c}{AUPR}
			&  \multicolumn{1}{c}{AUROC} 
			& \multicolumn{1}{c}{AUPR}
			&  \multicolumn{1}{c}{AUROC} 
			& \multicolumn{1}{c}{AUPR}
			&  \multicolumn{1}{c}{AUROC} 
			& \multicolumn{1}{c}{AUPR}\\
			\hline
			\multirow{3}*{FashionMNIST}

			&\multicolumn{2}{l}{MNIST} 
			& \textbf{99.8$\pm$0.0} & \textbf{99.8$\pm$0.0} & 98 & 98 & \textbf{100.0$\pm$0.0} & \textbf{100.0$\pm$0.0} & \textbf{100} & \textbf{100} \\
			&\multicolumn{2}{l}{KMNIST}
			& \textbf{99.9$\pm$0.0} & \textbf{99.9$\pm$0.0} & 97 & 96 & \textbf{100.0$\pm$0.0} & \textbf{100.0$\pm$0.0} & \textbf{100} & \textbf{100} \\
			&\multicolumn{2}{l}{Omniglot}
			& \textbf{100.0$\pm$0.0} & \textbf{100.0$\pm$0.0} & \textbf{100} & \textbf{100} & \textbf{100.0$\pm$0.0} & \textbf{100.0$\pm$0.0} & \textbf{100} & \textbf{100} \\
			
			\hline
			\multirow{4}*{SVHN}

			& \multicolumn{2}{l}{CelebA} 
			& \textbf{100.0$\pm$0.0} & \textbf{100.0$\pm$0.0} & \textbf{100} & 99 & \textbf{100.0$\pm$0.0} & \textbf{100.0$\pm$0.0} & \textbf{100} & \textbf{100} \\
			
			& \multicolumn{2}{l}{CIFAR-10} 
			& \textbf{100.0$\pm$0.0} & \textbf{100.0$\pm$0.0} & 92 & 84 & \textbf{100.0$\pm$0.0} & \textbf{100.0$\pm$0.0} & 99 & 98 \\
			
			& \multicolumn{2}{l}{CIFAR-100} 
			& \textbf{100.0$\pm$0.0} & \textbf{100.0$\pm$0.0} & 93 & 86 & \textbf{100.0$\pm$0.0} & \textbf{100.0$\pm$0.0} & 99 & 98 \\
			& \multicolumn{2}{l}{LSUN}
			&  \textbf{100.0$\pm$0.0} &  \textbf{100.0$\pm$0.0} & 99 & 98 
			&  \textbf{100.0$\pm$0.0} &  \textbf{100.0$\pm$0.0} &  \textbf{100.0} &  \textbf{100.0}\\
			
			\hline
			\multirow{3}*{CIFAR-10}

			& \multicolumn{2}{l}{CelebA}
			& \textbf{99.2$\pm$0.1} & \textbf{99.4$\pm$0.1} & 86 & 92 & \textbf{100.0$\pm$0.0} & \textbf{100.0$\pm$0.0} & 96 & 98 \\
			
			& \multicolumn{2}{l}{SVHN} 
			& \textbf{97.6$\pm$0.2} & 97.8$\pm$0.2 & 96 & \textbf{98} & 99.8$\pm$0.0 & 99.8$\pm$0.0 & \textbf{100} & \textbf{100} \\
			& \multicolumn{2}{l}{LSUN}
			&  \textbf{100.0$\pm$0.0} & \textbf{100.0$\pm$0.0} & 60 & 58 & \textbf{100.0$\pm$0.0} & \textbf{100.0$\pm$0.0} & 58 & 56 \\
			
			\hline
			
			\multirow{4}*{CelebA}
			& \multicolumn{2}{l}{CIFAR-10}
			& \textbf{99.6$\pm$0.0} & \textbf{99.6$\pm$0.0} & 84 & 73 & \textbf{100.0$\pm$0.0} & \textbf{100.0$\pm$0.0} & 94 & 91 \\
			
			& \multicolumn{2}{l}{CIFAR-100} 
			& \textbf{99.8$\pm$0.0} & \textbf{99.8$\pm$0.0} & 82 & 71  & \textbf{100.0$\pm$0.0} & \textbf{100.0$\pm$0.0} & 94 & 90 \\
			
			& \multicolumn{2}{l}{SVHN} 
			& \textbf{100.0$\pm$0.0} & \textbf{100.0$\pm$0.0} & 97 & 98 & \textbf{100.0$\pm$0.0} & \textbf{100.0$\pm$0.0} & \textbf{100} & \textbf{100} \\
			& \multicolumn{2}{l}{LSUN}
			& \textbf{100.0$\pm$0.0} & \textbf{100.0$\pm$0.0} &85 & 75 & \textbf{100.0$\pm$0.0} & \textbf{100.0$\pm$0.0} & 96 & 92\\
			
			\hline
			&\textbf{average}& & \textbf{99.7} & \textbf{99.7} & 90.6 & 87.6 & \textbf{100.0} & \textbf{100.0} & 95.4 & 94.5\\
			\bottomrule[1pt]   			
		\end{tabular}  
	\end{center}  
\end{table*}

\textbf{Mixture of OOD Datasets}.
Table \ref{tbl:glow_mixture_ood_datasets_DOCR_TC_M_bs_5_10_S2} shows the results of \PADmethod\ when OOD dataset is a mixture of two of the three datasets: SVHN, CelebA, and CIFAR-10. 

\begin{table*} [t]
	\scriptsize
	\caption{GAD Results (AUROC and AUPR in percentage) of \PADmethod\ and Ty-test on Glow with batch size 5 and 10. For SVHN, CIFAR-10, and CelebA, we choose one dataset as ID data and the mixture of the other two datasets as OOD data.
	} 
	\label{tbl:glow_mixture_ood_datasets_DOCR_TC_M_bs_5_10_S2}  
	\begin{center}  
		\begin{tabular}{cllr rr r|r rr r}  
			\toprule[1pt]
			\multirow{3}*{ID$\downarrow$}&\multirow{3}*{OOD $\downarrow$}&Batch size$\rightarrow$&
			\multicolumn{4}{c}{$m$=5}  &\multicolumn{4}{c}{$m$=10}
			\\
			\cline{3-11}   
			& & Method$\rightarrow$& \multicolumn{2}{c}{\PADmethod} & \multicolumn{2}{c}{Ty-test} 
			& \multicolumn{2}{c}{\PADmethod} & \multicolumn{2}{c}{Ty-test} \\
			\cline{3-11} 
			&&Metric$\rightarrow$
			&  \multicolumn{1}{c}{AUROC} 
			& \multicolumn{1}{c}{AUPR}
			&  \multicolumn{1}{c}{AUROC} 
			& \multicolumn{1}{c}{AUPR}
			&  \multicolumn{1}{c}{AUROC} 
			& \multicolumn{1}{c}{AUPR}
			&  \multicolumn{1}{c}{AUROC} 
			& \multicolumn{1}{c}{AUPR}\\
			\hline
			SVHN 
			
			&\multicolumn{2}{l}{CelebA+CIFAR-10}
			\input{results/S2/glow_svhn_vs_mixed_celeba_cifar10_bs_5_10_S2.tex}\\
			CIFAR-10

			&\multicolumn{2}{l}{SVHN+CelebA}
			\input{results/S2/glow_cifar10_vs_mixture_celeba_svhn_bs_5_10_S2.tex}\\

			CelebA

			&\multicolumn{2}{l}{SVHN+CIFAR-10}
			\input{results/S2/glow_celeba32_vs_mixed_svhn_cifar10_bs_5_10_S2.tex}\\

			\hline
			&\textbf{average} & & \textbf{99.4 }& \textbf{99.5} & 60.5 & 66.4 & \textbf{99.9} & \textbf{100.0} & 54.9 & 63.5\\
			\bottomrule[1pt]   	
			
		\end{tabular}  
	\end{center}  
\end{table*}

\textbf{Ablation Study}.
Table \ref{tbl:glow_fashionmnist_svhn_cifar10_celeba_ablation_study} shows the results of ablation study. Except for CIFAR-10 vs ImageNet32-C(0.3), the order of performance is \PADmethod\ $>$ \GADmethod\ $>$ \GADmethod-all $>$ Ty-test. The only exception is CIFAR-10 vs ImageNet32-C(0.3). 
Note that KLOD only applies to GAD.

\textbf{One-vs-Rest}. Table \ref{tbl:glow_GAD_mnist_one_vs_rest} shows GAD Results (AUROC and AUPR  in percentage) of Glow on One-vs-Rest on MNIST. 

\begin{table*} [h]
	\vspace{-0pt}
	\scriptsize
	\caption{GAD Results (AUROC  in percentage) of ablation study. We compare four methods Ty-test, \GADmethod-all, \GADmethod, and \PADmethod. \PADmethod\ is the best one.  } 
	\label{tbl:glow_fashionmnist_svhn_cifar10_celeba_ablation_study}  
	\begin{center}  
		\begin{tabular}{cllr rr r|r rr r}  
			\toprule[1pt]
			\multirow{2}*{ID}&\multirow{2}*{OOD$\downarrow$}&Batch size& \multicolumn{4}{c}{$m$=10}  &\multicolumn{4}{c}{$m$=25}  \\
			\cline{3-11}   
			& & Method& Ty-test & \GADmethod-all & \GADmethod & \PADmethod & Ty-test & \GADmethod-all & \GADmethod & \PADmethod \\
			\cline{3-11} 
			\hline

			\multirow{5}*{\rotatebox{90}{\tabincell{c}{Fash.M}}} 
			&\multicolumn{2}{l}{Constant} 
			& 41.6$\pm$0.41 
			& \textbf{100.0$\pm$0.0}
			&   \textbf{100.0$\pm$0.0}  
			& \textbf{100.0$\pm$0.0}
			& 39.1$\pm$0.8
			& \textbf{100.0$\pm$0.0}
			&\textbf{100.0$\pm$0.0} 
			& \textbf{100.0$\pm$0.0}\\
			&\multicolumn{2}{l}{MNIST} 
			& 99.2$\pm$0.1 
			& \textbf{100.0$\pm$0.0}
			&   \textbf{100.0$\pm$0.0}  
			& \textbf{100.0$\pm$0.0}
			& \textbf{100.0$\pm$0.0}
			& \textbf{100.0$\pm$0.0}
			&\textbf{100.0$\pm$0.0} 
			& \textbf{100.0$\pm$0.0}\\
			
			&\multicolumn{2}{l}{MNIST-C(10.0)}
			&84.9$\pm$0.3
			&\textbf{100.0$\pm$0.0}
			&\textbf{100.0$\pm$0.0}
			&\textbf{100.0$\pm$0.0}
			&94.7$\pm$0.3
			&\textbf{100.0$\pm$0.0}
			&\textbf{100.0$\pm$0.0}
			&\textbf{100.0$\pm$0.0}
			\\
			& \multicolumn{2}{l}{notMNIST} 
			& 92.7$\pm$0.5
			& \textbf{100.0$\pm$0.0}
			&  \textbf{100.0$\pm$0.0} 
			& \textbf{100.0$\pm$0.0}
			& 98.9$\pm$0.2 
			& \textbf{100.0$\pm$0.0}
			& \textbf{100.0$\pm$0.0}
			&\textbf{100.0$\pm$0.0}\\
			& \multicolumn{2}{l}{notMNIST-C(0.005)}
			&7.0$\pm$0.6
			&\textbf{100.0$\pm$0.0}
			&\textbf{100.0$\pm$0.0}
			&\textbf{100.0$\pm$0.0}
			&2.7$\pm$0.2
			&\textbf{100.0$\pm$0.0}
			&\textbf{100.0$\pm$0.0}
			&\textbf{100.0$\pm$0.0}
			\\
			\hline
			\multirow{11}*{\rotatebox{90}{SVHN}} 
			& \multicolumn{2}{l}{Constant} 
			& \textbf{100.0$\pm$0.0}  
			& \textbf{100.0$\pm$0.0}
			&\textbf{100.0$\pm$0.0}
			&\textbf{100.0$\pm$0.0}
			&\textbf{100.0$\pm$0.0}
			& \textbf{100.0$\pm$0.0}
			&\textbf{100.0$\pm$0.0}
			&\textbf{100.0$\pm$0.0}
			\\
			
			& \multicolumn{2}{l}{Uniform} 
			& \textbf{100.0$\pm$0.0}  
			& \textbf{100.0$\pm$0.0}
			&\textbf{100.0$\pm$0.0}
			&\textbf{100.0$\pm$0.0}
			&\textbf{100.0$\pm$0.0}
			& \textbf{100.0$\pm$0.0}
			&\textbf{100.0$\pm$0.0}
			&\textbf{100.0$\pm$0.0}
			\\
			
			& \multicolumn{2}{l}{Uniform-C(0.008)} 
			& 11.8$\pm$0.3  
			& \textbf{100.0$\pm$0.0}
			&\textbf{100.0$\pm$0.0}
			&\textbf{100.0$\pm$0.0}
			& 5.2$\pm$0.6
			& \textbf{100.0$\pm$0.0}
			&\textbf{100.0$\pm$0.0}
			&\textbf{100.0$\pm$0.0}
			\\
			
			& \multicolumn{2}{l}{CelebA} 
			& \textbf{100.0$\pm$0.0}  
			& \textbf{100.0$\pm$0.0}
			&\textbf{100.0$\pm$0.0}
			&\textbf{100.0$\pm$0.0}
			&\textbf{100.0$\pm$0.0}
			& \textbf{100.0$\pm$0.0}
			&\textbf{100.0$\pm$0.0}
			&\textbf{100.0$\pm$0.0}
			\\
			& \multicolumn{2}{l}{CelebA-C(0.08)}
			&54.7$\pm$0.5
			&\textbf{100.0$\pm$0.0}
			&\textbf{100.0$\pm$0.0}
			&\textbf{100.0$\pm$0.0}
			&58.2$\pm$0.3
			& \textbf{100.0$\pm$0.0}
			&\textbf{100.0$\pm$0.0}  
			&\textbf{100.0$\pm$0.0}
			\\
			& \multicolumn{2}{l}{CIFAR10} 
			&\textbf{100.0$\pm$0.0}  
			&\textbf{100.0$\pm$0.0}
			&\textbf{100.0$\pm$0.0}
			&\textbf{100.0$\pm$0.0}
			&\textbf{100.0$\pm$0.0}
			&\textbf{100.0$\pm$0.0}
			&\textbf{100.0$\pm$0.0}
			&\textbf{100.0$\pm$0.0}
			\\
			& \multicolumn{2}{l}{CIFAR10-C(0.12)}
			&54.7$\pm$0.5
			& 86.2$\pm$0.3
			&\textbf{100.0$\pm$0.0}
			&\textbf{100.0$\pm$0.0}
			&12.6$\pm$0.9
			& 98.3$\pm$0.5
			&\textbf{99.1$\pm$0.3}
			&\textbf{100.0$\pm$0.0}
			\\
			& \multicolumn{2}{l}{CIFAR100} 
			& \textbf{100.0$\pm$0.0}  
			&\textbf{100.0$\pm$0.0}
			&\textbf{100.0$\pm$0.0}
			&\textbf{100.0$\pm$0.0}
			&\textbf{100.0$\pm$0.0}
			&\textbf{100.0$\pm$0.0}
			&\textbf{100.0$\pm$0.0}
			&\textbf{100.0$\pm$0.0}
			\\
			& \multicolumn{2}{l}{CIFAR100-C(0.12)}
			&26.9$\pm$1.3
			&86.0$\pm$0.8
			&\textbf{95.5$\pm$0.4}
			&\textbf{100.0$\pm$0.0}
			&12.0$\pm$1.1
			&96.2$\pm$1.0
			&\textbf{97.2$\pm$0.2}
			&\textbf{100.0$\pm$0.0}
			\\
			& \multicolumn{2}{l}{ImageNet32} 
			& \textbf{100.0$\pm$0.0}  
			& \textbf{100.0$\pm$0.0} 
			&\textbf{100.0$\pm$0.0}
			&\textbf{100.0$\pm$0.0} 
			&\textbf{100.0$\pm$0.0}
			&\textbf{100.0$\pm$0.0} 
			&\textbf{100.0$\pm$0.0}
			&\textbf{100.0$\pm$0.0} 
			\\
			& \multicolumn{2}{l}{ImageNet32-C(0.07)}
			&42.6$\pm$0.4
			&\textbf{100.0$\pm$0.0} 
			&\textbf{100.0$\pm$0.0}
			&\textbf{100.0$\pm$0.0} 
			&35.7$\pm$0.3
			&\textbf{100.0$\pm$0.0} 
			&\textbf{100.0$\pm$0.0}
			&\textbf{100.0$\pm$0.0} 
			\\
			\hline
			\multirow{9}*{\rotatebox{90}{CIFAR10}} 
			
			& \multicolumn{2}{l}{Constant}
			&\textbf{100.0$\pm$0.0}
			&\textbf{100.0$\pm$0.0}
			&\textbf{100.0$\pm$0.0}
			&\textbf{100.0$\pm$0.0}
			&\textbf{100.0$\pm$0.0}
			&\textbf{100.0$\pm$0.0}
			&\textbf{100.0$\pm$0.0}
			&\textbf{100.0$\pm$0.0}
			\\
			
			& \multicolumn{2}{l}{Uniform}
			&\textbf{100.0$\pm$0.0}
			&\textbf{100.0$\pm$0.0}
			&\textbf{100.0$\pm$0.0}
			&\textbf{100.0$\pm$0.0}
			&\textbf{100.0$\pm$0.0}
			&\textbf{100.0$\pm$0.0}
			&\textbf{100.0$\pm$0.0}
			&\textbf{100.0$\pm$0.0}
			\\
			
			& \multicolumn{2}{l}{Uniform-C(0.02)}
			&10.7$\pm$0.1
			&\textbf{100.0$\pm$0.0}
			&\textbf{100.0$\pm$0.0}
			&\textbf{100.0$\pm$0.0}
			&5.1$\pm$1.0
			&\textbf{100.0$\pm$0.0}
			&\textbf{100.0$\pm$0.0}
			&\textbf{100.0$\pm$0.0}
			\\

			& \multicolumn{2}{l}{CelebA}
			&\textbf{100.0$\pm$0.0}
			&\textbf{100.0$\pm$0.0}
			&\textbf{100.0$\pm$0.0}
			&\textbf{100.0$\pm$0.0}
			&\textbf{100.0$\pm$0.0}
			&\textbf{100.0$\pm$0.0}
			&\textbf{100.0$\pm$0.0}
			&\textbf{100.0$\pm$0.0}
			\\
			& \multicolumn{2}{l}{CelebA-C(0.3)}
			&23.4$\pm$5.3
			&\textbf{100.0$\pm$0.0}
			&\textbf{100.0$\pm$0.0}
			&\textbf{100.0$\pm$0.0}
			&12.6$\pm$0.7
			&\textbf{100.0$\pm$0.0}
			&\textbf{100.0$\pm$0.0}
			&\textbf{100.0$\pm$0.0}
			\\
			& \multicolumn{2}{l}{ImageNet32} 
			&\textbf{100.0$\pm$0.0}
			&99.7$\pm$0.1
			&99.3$\pm$0.0
			&94.7$\pm$0.1
			&\textbf{100.0$\pm$0.0}
			&95.3$\pm$0.7
			&99.0$\pm$0.3
			&98.9$\pm$0.4
			\\
			& \multicolumn{2}{l}{ImageNet32-C(0.3)}
			&31.7$\pm$0.7
			&98.4$\pm$0.2
			&\textbf{94.8$\pm$0.3}
			&73.1$\pm$1.5
			&15.0$\pm$1.0
			&97.5$\pm$0.3
			&\textbf{96.7$\pm$0.5}
			&76.9$\pm$1.3
			\\
			& \multicolumn{2}{l}{SVHN} 
			&\textbf{99.9$\pm$0.0}
			&96.7$\pm$0.2
			&99.1$\pm$0.0
			&99.8$\pm$0.1
			&\textbf{100.0$\pm$0.0}
			&87.6$\pm$0.5
			&99.6$\pm$0.1
			&\textbf{100.0$\pm$0.0}
			\\
			& \multicolumn{2}{l}{SVHN-C(2.0)}
			&26.7$\pm$0.6
			&\textbf{100.0$\pm$0.0}
			&\textbf{100.0$\pm$0.0}
			&\textbf{100.0$\pm$0.0}
			&58.2$\pm$0.2
			&\textbf{100.0$\pm$0.0}
			&\textbf{100.0$\pm$0.0}
			&\textbf{100.0$\pm$0.0}
			\\
			\hline
			\multirow{6}*{\rotatebox{90}{CelebA}} & \multicolumn{2}{l}{CIFAR10}
			&1.0$\pm$0.1
			&94.3$\pm$0.8
			&\textbf{99.8$\pm$0.0}
			&\textbf{100.0$\pm$0.0}
			&0.0$\pm$0.0
			&95.6$\pm$0.5
			&\textbf{100.0$\pm$0.0}
			&\textbf{100.0$\pm$0.0}
			\\
			& \multicolumn{2}{l}{CIFAR100} 
			&2.0$\pm$0.2
			&94.7$\pm$0.4
			&\textbf{99.8$\pm$0.0}
			&\textbf{100.0$\pm$0.0}
			&0.0$\pm$0.0
			&95.2$\pm$0.4
			&\textbf{100.0$\pm$0.0}
			&\textbf{100.0$\pm$0.0}
			\\
			& \multicolumn{2}{l}{ImageNet32}
			&87.9$\pm$0.3
			&\textbf{100.0$\pm$0.0}
			&\textbf{100.0$\pm$0.0}
			&\textbf{100.0$\pm$0.0}
			&96.7$\pm$0.4
			&\textbf{100.0$\pm$0.0}
			&\textbf{100.0$\pm$0.0}
			&\textbf{100.0$\pm$0.0}
			\\
			& \multicolumn{2}{l}{ImageNet32-C(0.2)}
			&18.2$\pm$0.3 
			&\textbf{100.0$\pm$0.0}
			&\textbf{100.0$\pm$0.0}
			&\textbf{100.0$\pm$0.0}
			&7.8$\pm$0.3
			&\textbf{100.0$\pm$0.0}
			&\textbf{100.0$\pm$0.0}
			&\textbf{100.0$\pm$0.0}
			\\
			& \multicolumn{2}{l}{SVHN} 
			&91.5$\pm$0.6
			&\textbf{100.0$\pm$0.0}
			&\textbf{100.0$\pm$0.0}
			&\textbf{100.0$\pm$0.0}
			&98.6$\pm$0.2
			&\textbf{100.0$\pm$0.0}
			&\textbf{100.0$\pm$0.0}
			&\textbf{100.0$\pm$0.0}
			\\
			& \multicolumn{2}{l}{SVHN-C(1.8)} 
			&1.4$\pm$0.2
			&\textbf{100.0$\pm$0.0}
			&\textbf{100.0$\pm$0.0}
			&\textbf{100.0$\pm$0.0}
			&0.0$\pm$0.0
			&\textbf{100.0$\pm$0.0}
			&\textbf{100.0$\pm$0.0}
			&\textbf{100.0$\pm$0.0}
			\\
			\hline
			&\multicolumn{2}{c}{\textbf{average except CIFAR10 vs ImageNet32-C(0.3)}} & 65.49 & 98.76 & 99.79 & \textbf{99.82} & 64.43 & 98.94 & 99.83 & \textbf{99.96}\\
			&\textbf{average} & & 64.40 & 98.75 & \textbf{99.63} & 98.95 & 62.84 & 98.89 & \textbf{99.73} &99.22\\
			
			\bottomrule[1pt]   			
		\end{tabular}  
	\end{center}  
\end{table*}

\begin{table*} [h]
	\vspace{-0pt}
	\scriptsize
	\caption{GAD Results (AUROC and AUPR  in percentage) of Glow on One-vs-Rest on MNIST. The higher is the better.} 
	\label{tbl:glow_GAD_mnist_one_vs_rest}  
	\begin{center}  
		\begin{tabular}{lr rr r|r rr r}  
			\toprule[1pt]
			Batch size&
			\multicolumn{4}{c}{$m$=5}  &\multicolumn{4}{c}{$m$=10}
			\\
			\hline   
			Method & \multicolumn{2}{c}{\PADmethod} & \multicolumn{2}{c}{Ty-test} 
			& \multicolumn{2}{c}{\PADmethod} & \multicolumn{2}{c}{Ty-test} \\
			\hline
			Metric
			& \multicolumn{1}{c}{AUROC} 
			& \multicolumn{1}{c}{AUPR}
			& \multicolumn{1}{c}{AUROC} 
			& \multicolumn{1}{c}{AUPR}
			& \multicolumn{1}{c}{AUROC} 
			& \multicolumn{1}{c}{AUPR}
			& \multicolumn{1}{c}{AUROC} 
			& \multicolumn{1}{c}{AUPR}\\
			\hline
			class 0 vs rest 
			\input{results/S2/glow_mnist_only_0_vs_mnist_except_0_bs_5_10_S2.tex}\\
			class 1 vs rest
			\input{results/S2/glow_mnist_only_1_vs_mnist_except_1_bs_5_10_S2.tex}\\
			class 2 vs rest
			\input{results/S2/glow_mnist_only_2_vs_mnist_except_2_bs_5_10_S2.tex}\\
			class 3 vs rest
			\input{results/S2/glow_mnist_only_3_vs_mnist_except_3_bs_5_10_S2.tex}\\
			class 4 vs rest
			\input{results/S2/glow_mnist_only_4_vs_mnist_except_4_bs_5_10_S2.tex}\\
			class 5 vs rest
			\input{results/S2/glow_mnist_only_5_vs_mnist_except_5_bs_5_10_S2.tex}\\
			class 6 vs rest
			\input{results/S2/glow_mnist_only_6_vs_mnist_except_6_bs_5_10_S2.tex}\\
			class 7 vs rest
			\input{results/S2/glow_mnist_only_7_vs_mnist_except_7_bs_5_10_S2.tex}\\
			class 8 vs rest
			\input{results/S2/glow_mnist_only_8_vs_mnist_except_8_bs_5_10_S2.tex}\\
			class 9 vs rest
			& \textbf{95.0$\pm$0.7} & \textbf{95.6$\pm$0.8} & 76.3$\pm$1.5 & 80.2$\pm$1.8 & \textbf{99.7$\pm$0.2} & \textbf{99.7$\pm$0.2} & 80.1$\pm$2.6 &  84.2$\pm$1.9
			\\
			\hline
			\textbf{average} & \textbf{85.4} & \textbf{85.8} & 77.4 & 80.8 & \textbf{93.1} & \textbf{93.4} & 81.2 & 84.0\\
			\bottomrule[1pt]   			
		\end{tabular}  
	\end{center}  
\end{table*}

\textbf{GlowGMM}.  \label{sec:class_condition_glow}
Figure \ref{fig:glowgmm_fashionmnist_bar} and Table \ref{tbl:cond_glow_fashionmnist_CTR_TC} shows the results of GAD on FashionMNIST. Our method achieves 100\% AUROC on average when the batch size is 25. The baseline only reaches 22.7\% AUROC.
Recent works have improved the accuracy of conditional Glow on classification problems \cite{izmailov2019semisupervised, atanov2019semiconditional}.
However, as long as GlowGMM does not achieve 100\% classification accuracy, the question proposed in the introduction remains.

Table \ref{tbl:cond_glow_fashionmnist_logpz_classification} shows the results of using $p(\bm{z})$ for one-vs-rest classification on FashionMNIST with GlowGMM. $p(\bm{z})$ is not a good criterion for OOD detection. For example, the AUROC for class 8 vs rest is only 55.5\%.

\textbf{Generating OOD data Using GlowGMM}.
In Section \ref{sec:KL_GAD}, we sample blurred OOD data with fitted Gaussian distribution. 
In GlowGMM, we can generate \textit{high-quality} OOD images with fitted Gaussian distribution from OOD representations.
Figure \ref{fig:glow_cond_fashionmnist_samples}(a) shows the generated images using noise sampled from each Gaussian component $\mathcal{N}_i(\bm{\mu}_i, diag(\bm{\sigma_i^2}))\ (1\leq i \leq 10)$ as prior. The  $i$-th column corresponds to the $i$-th Gaussian $\mathcal{N}_i$. Figure \ref{fig:glow_cond_fashionmnist_samples}(b) shows the generated images using the similar operation in Section \ref{sec:investigate_flow_latents}. For each $i$, we compute the representations $\{\z\}$ of the $((i+1)\%10)$-th class and normalize them under $\mathcal{N}_i(\bm{\mu}_i, diag(\bm{\sigma_i^2}))$ as $\bm{z}'=(\bm{z}-\bm{\mu}_i)/\bm{\sigma}_i$. We use the normalized representation $\{\z'\}$ to fit a Gaussian distribution $\widetilde{\mathcal{N}}_i(\widetilde{\bm{\mu}}_{i'}, \widetilde{\bm{\Sigma}}_{i'})$.
Then We sample $\varepsilon_{i'}  \sim \widetilde{\mathcal{N}}_i(\widetilde{\bm{\mu}}_{i'}, \widetilde{\bm{\Sigma}}_{i'})$  and unnormalize $\varepsilon_{i'}$ back using parameters of the $i$-th component as $\varepsilon_{i'} \cdot \bm{\sigma}_i + \bm{\mu}_i$. Finally, we compute $f^{-1}(\varepsilon_{i'} \cdot \bm{\sigma}_i + \bm{\mu}_i)$ to generate new images. As shown in Figure \ref{fig:glow_cond_fashionmnist_samples}(b), we can generate almost high quality images of the $((i+1)\%10)$-th class from the fitted Gaussian. 
These results verify that OOD representations reside in specific directions that can be characterized by the mean and covariance matrix of OOD representations. We did not conduct more experiments on OOD sampling because it is beyond the scope of this paper.


\begin{figure*}[h!]%
	\centering
	\includegraphics[width=4.5cm]{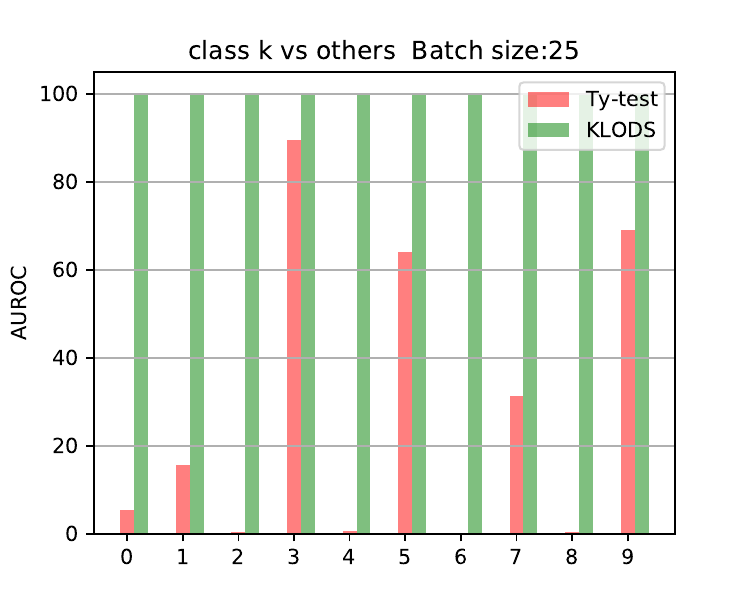}
	\vspace{-8pt}
	\caption{GAD results(AUROC) on GlowGMM trained on FashionMNIST. Numerical results are shown in Table \ref{tbl:cond_glow_fashionmnist_CTR_TC} in the supplementary material.}
	\label{fig:glowgmm_fashionmnist_bar}
\end{figure*}

\begin{table*}[t] 
	\vspace{-0pt}
	\scriptsize
	\caption{GAD results (AUROC and AUPR  in percentage) on GlowGMM trained on FashionMNIST.} 
	\vspace{-10pt}
	\label{tbl:cond_glow_fashionmnist_CTR_TC}  
	\begin{center}  
		\begin{tabular}{ l r r |r r }  
			\toprule[1pt]
			Batch size &\multicolumn{4}{ c }{$m$=25}  \\
			\hline  
			Method&\multicolumn{2}{ c }{\PADmethod} &\multicolumn{2}{ c }{Ty-test} \\
			\hline
			
			Metrics 
			& \multicolumn{1}{ c}{AUROC} 
			& \multicolumn{1}{c }{AUPR} 
			& \multicolumn{1}{ c}{AUROC} 
			& \multicolumn{1}{c }{AUPR}\\
			\hline
			
			class 0 vs rest 
			&\textbf{100.0$\pm$0.0} & \textbf{100.0$\pm$0.0}	
			& 5.4$\pm$1.6& 31.2$\pm$0.3
			\\
			class 1 vs rest 
			&\textbf{100.0$\pm$0.0} & \textbf{100.0$\pm$0.0} 	&15.7$\pm$2.4 
			&33.4$\pm$4.9\\
			class 2 vs rest 
			&\textbf{100.0$\pm$0.0} & \textbf{100.0$\pm$0.0} 	&0.5$\pm$0.5 
			&30.7$\pm$0.0 \\
			class 3 vs rest 
			&\textbf{99.9$\pm$0.1} & \textbf{99.9$\pm$0.1}	
			& 89.6$\pm$2.5 & 91.3$\pm$2.3 \\
			class 4 vs rest &
			\textbf{100.0$\pm$0.0} & \textbf{100.0$\pm$0.0} 	&0.7$\pm$0.6 &30.7$\pm$0.0 \\
			class 5 vs rest & 
			\textbf{100.0$\pm$0.0} & \textbf{100.0$\pm$0.0} 	&64.2$\pm$1.4 &66.4$\pm$2.9 \\
			class 6 vs rest 
			&\textbf{99.9$\pm$0.1} & \textbf{99.9$\pm$0.1}	&
			0.0$\pm$0.0 & 30.7$\pm$0.0\\
			class 7 vs rest  
			&\textbf{100.0$\pm$0.0}  & \textbf{100.0$\pm$0.0}  	
			&31.4$\pm$2.8 &46.6$\pm$3.3\\
			class 8 vs rest 
			&\textbf{100.0$\pm$0.0}  & \textbf{100.0$\pm$0.0} 	
			&0.4$\pm$0.5 &30.7$\pm$0.0 \\
			class 9 vs rest & 
			\textbf{100.0$\pm$0.0}  & \textbf{100.0$\pm$0.0}	
			& 69.0$\pm$3.6 &76.0$\pm$1.7 \\
			\hline
			\textbf{average} & \textbf{100} & \textbf{100} & 27.7 & 46.8\\
			\bottomrule[1pt] 
		\end{tabular}  
	\end{center}  
\end{table*}

\begin{figure*}[t]
	\centering
	\subfigure[]{
		\begin{minipage}[t]{3cm}
			\centering
			\includegraphics[width=3cm]{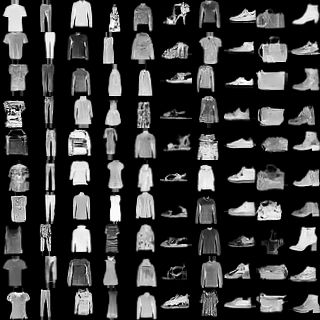}
			\label{fig:glow_cond_fashionmnist}
		\end{minipage}
	}
	\subfigure[]{
		\begin{minipage}[t]{3cm}
			\centering
			\includegraphics[width=3cm]{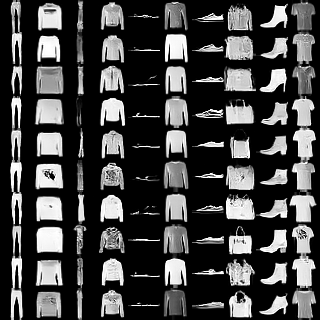}
			\label{fig:glow_cond_fashionmnist_sample_according_to_next_class_representations}
		\end{minipage}
	}	
	\vspace{-10pt}
	\caption{ 
		GlowGMM with 10 components trained on FashionMNIST. (a) sampling  from $\mathcal{N}_i(\bm{\mu}_i, diag(\bm{\sigma_i^2}))$. The  $i$-th column corresponds to Gaussian distribution $\mathcal{N}_i$. (b) For $\mathcal{N}_i$, we fit another Gaussian distribution $\widetilde{\mathcal{N}}_i(\widetilde{\bm{\mu}}_{i'}, \widetilde{\bm{\Sigma}}_{i'})$ using the normalized representations (by parameters of $\mathcal{N}_i$) of inputs of the $((i+1)\%10)$-th class. The $i$-th column shows images generated from $\widetilde{\mathcal{N}}_i$. 
	}
	\label{fig:glow_cond_fashionmnist_samples}
\end{figure*}

\begin{table*} [ht]
	\scriptsize
	\vspace{-0pt}
	\caption{GlowGMM trained on FashionMNIST. Use $p(\bm{z})$ as criterion for 1 vs rest classification.} 
	\label{tbl:cond_glow_fashionmnist_logpz_classification}  
	\begin{center}  
		\begin{tabular}{ l r r }  
			\toprule[1pt]
			Method & \multicolumn{2}{ c }{$ p(\bm{z})$}  \\
			\hline 
			Metrics &  \multicolumn{1}{ c}{AUROC} 
			& \multicolumn{1}{c }{AUPR}\\
			\hline
			
			class 0 vs rest &
			72.7$\pm$1.6 & 72.0$\pm$1.4 \\
			class 1 vs rest & 
			85.1$\pm$0.6 & 86.2$\pm$0.6\\
			class 2 vs rest & 
			74.8$\pm$4.5 & 76.9$\pm$4.0 \\
			class 3 vs rest & 
			68.9$\pm$4.7 & 71.2$\pm$4.5 \\
			class 4 vs rest & 
			77.1$\pm$2.1 & 78.4$\pm$3.2 \\
			class 5 vs rest & 
			71.7$\pm$1.4 & 71.9$\pm$1.2 \\
			class 6 vs rest & 
			73.5$\pm$7.8 & 73.7$\pm$8.6 \\
			class 7 vs rest & 
			86.9$\pm$0.4 & 88.6$\pm$0.4 \\
			class 8 vs rest & 
			55.5$\pm$0.9 & 53.8$\pm$0.5 \\
			class 9 vs rest& 
			86.6$\pm$0.3 & 87.1$\pm$0.3 \\
			\hline
			\textbf{average} & 75.3 & 76.0\\
			\bottomrule[1pt] 
		\end{tabular}  
	\end{center}  
\end{table*}

\subsection{GAD Results on VAE}

Figure \ref{fig:GAD_VAE_results_bar} and Table \ref{tbl:vae_fashionmnist_svhn_cifar10_CTR_TC} show GAD results on convolutional VAE on problems FashionMNIST/SVHN/CIFAR10 vs others.

Table \ref{tbl:vae_cifar10_cifar100_imagenet32_ctr_tc} shows the GAD results on CIFAR10 vs CIFAR100/ImageNet32.

\begin{figure*}[t]
	\centering
	\vspace{-0pt}
	\subfigure[]{
		\begin{minipage}[t]{4.15cm}
			\includegraphics[width=4.15cm]{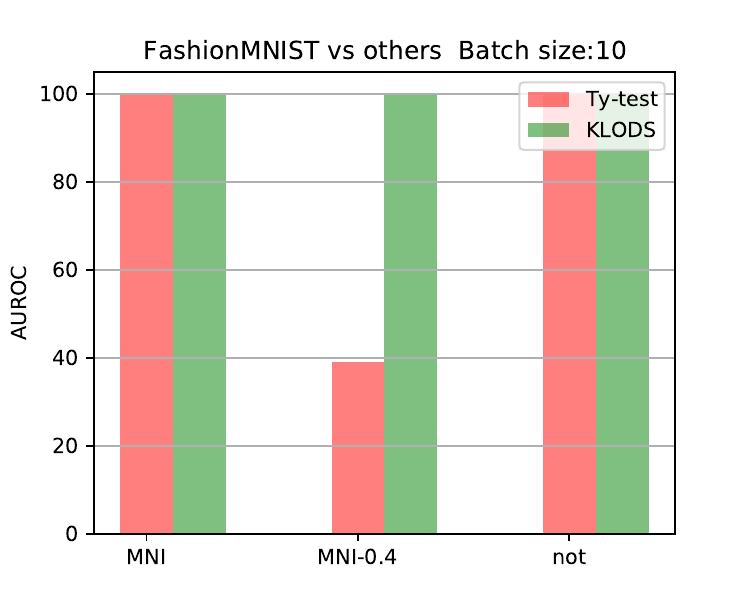}
			\label{fig:GAD_VAE_bar_fashion_vs_others_bs10}
		\end{minipage}
	}
	\subfigure[]{
		\begin{minipage}[t]{4.15cm}
			\centering
			\includegraphics[width=4.15cm]{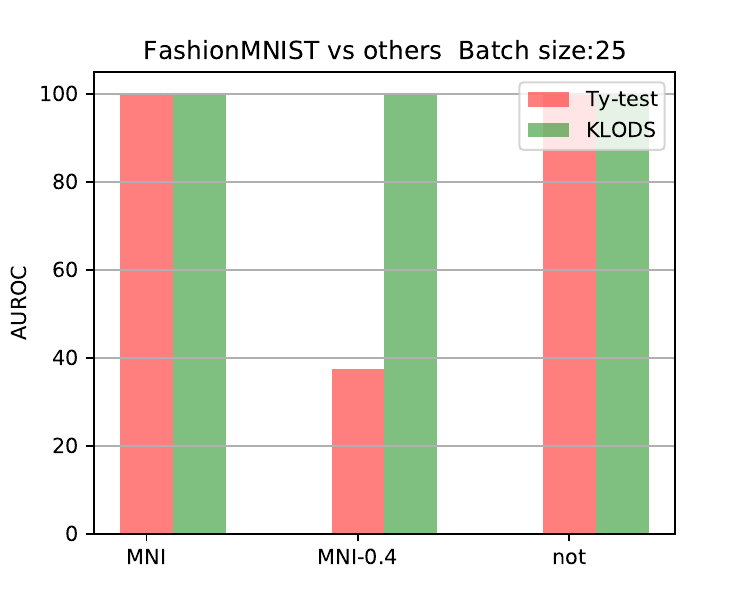}
			\label{fig:GAD_VAE_bar_fashion_vs_others_bs25}
		\end{minipage}%
	}%
	\subfigure[]{
		\begin{minipage}[t]{4.15cm}
			\includegraphics[width=4.15cm]{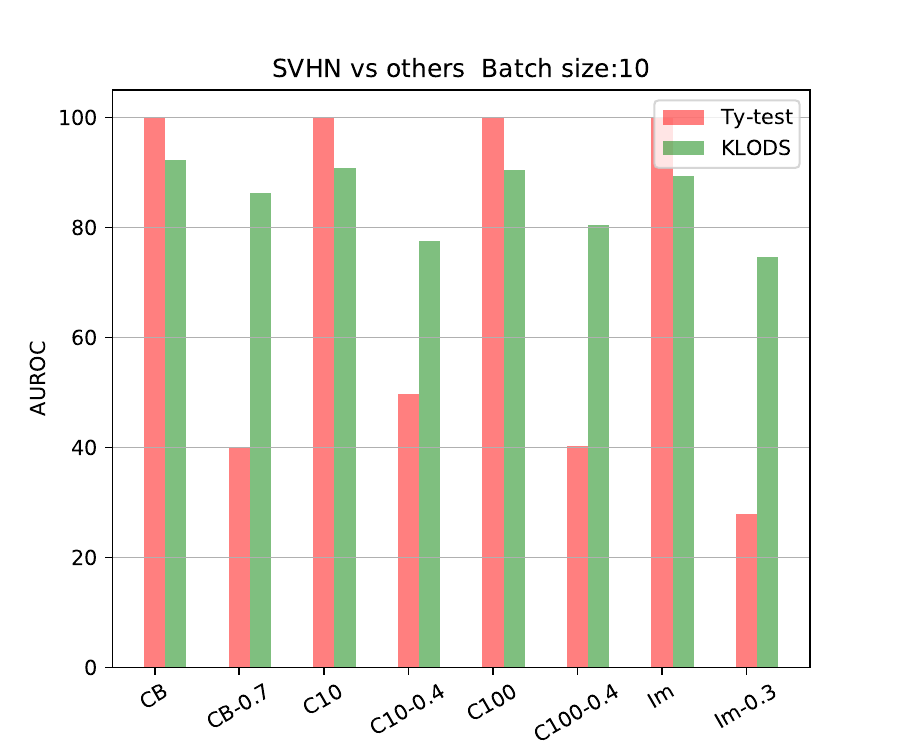}
			\label{fig:GAD_VAE_bar_svhn_vs_others_bs10}
		\end{minipage}
	}
	\vspace{-10pt}
	
	\subfigure[]{
		\begin{minipage}[t]{4.15cm}
			\includegraphics[width=4.15cm]{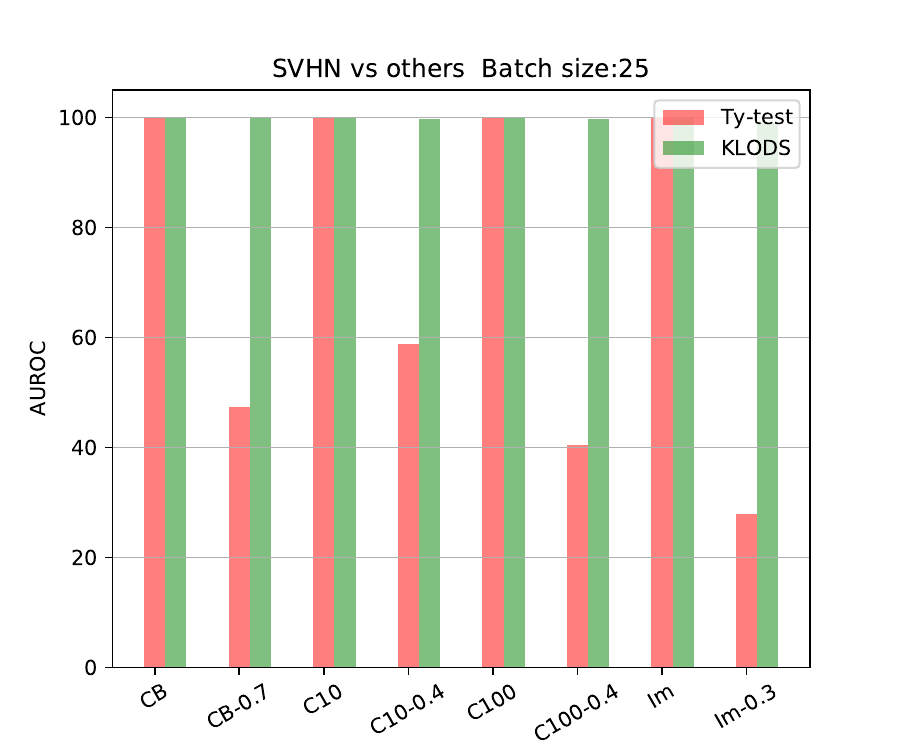}
			\label{fig:GAD_VAE_bar_svhn_vs_others_bs25}
		\end{minipage}
	}
	\subfigure[]{
		\begin{minipage}[t]{4.15cm}
			\includegraphics[width=4.15cm]{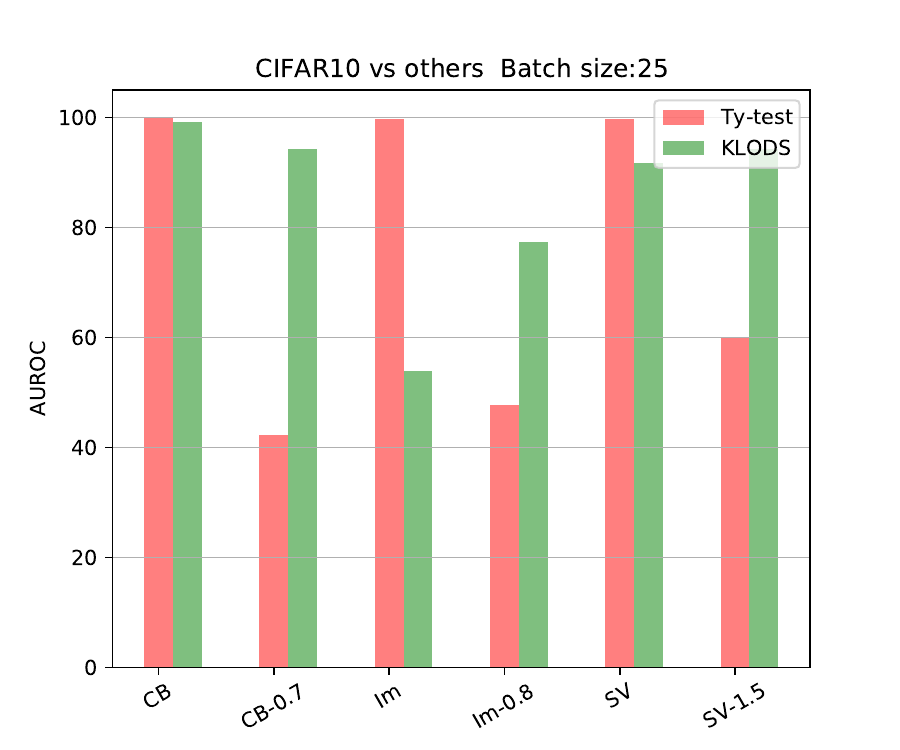}
			\label{fig:GAD_VAE_bar_cifar10_vs_others_bs25}
		\end{minipage}
	}
	\subfigure[]{
		\begin{minipage}[t]{4.15cm}
			\centering
			\includegraphics[width=4.15cm]{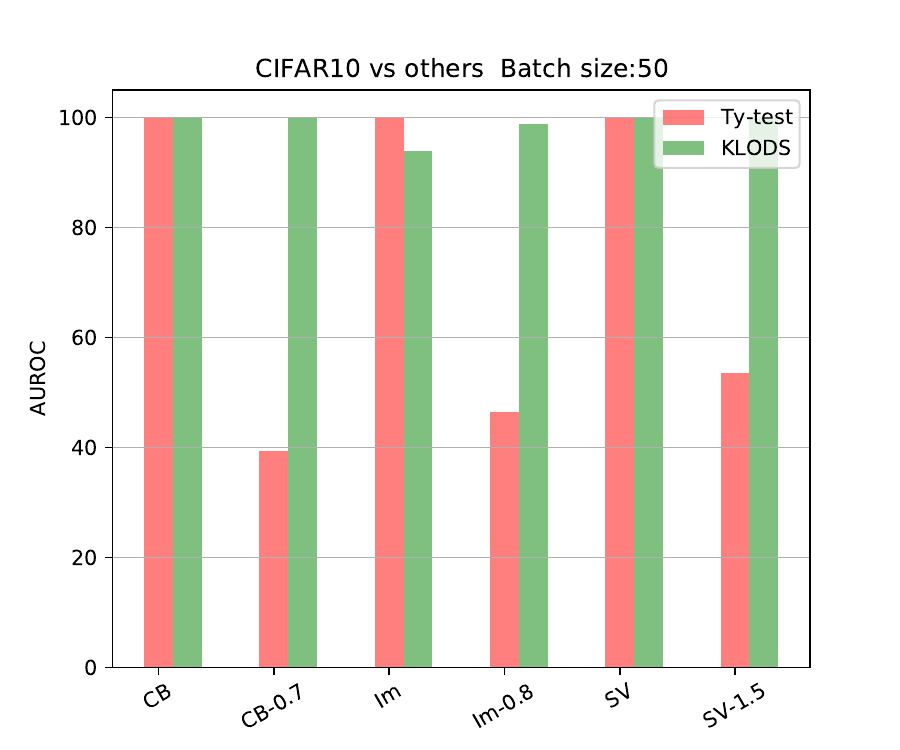}
			\label{fig:GAD_VAE_bar_cifar10_vs_others_bs50}
		\end{minipage}%
	}%
	\vspace{-10pt}
	\caption{GAD results (AUROC) on convolutional VAE. with batch sizes 5 and 10. The X-axis labeled with OOD datasets with abbreviated names. MNI: MNIST, not: notMNIST, CB: CelebA, C10/100: CIFAR-10/100, Im: ImageNet, SV: SVHN. The number $k$ after the dataset name indicates the dataset with adjusted contrast with a factor $k$. For example, CB-0.7 means CelebA-C(0.7). Numerical results are shown in Table \ref{tbl:vae_fashionmnist_svhn_cifar10_CTR_TC} in the supplementary material.}  
	\label{fig:GAD_VAE_results_bar}
\end{figure*}

\begin{table*} [t]
	\vspace{-10pt}
	\scriptsize
	\caption{GAD results (AUROC and AUPR  in percentage) of \GADmethod\ on VAE.} 
	\label{tbl:vae_fashionmnist_svhn_cifar10_CTR_TC}  
	\vspace{-10pt}
	\begin{center}  
		\begin{tabular}{ c l l r r r r | r r r r }  
			\toprule[1pt]
			\multirow{3}*{ID$\downarrow$}&\multirow{3}*{OOD$\downarrow$}&Batch size$\rightarrow$& \multicolumn{4}{ c }{$m$=10}  &\multicolumn{4}{ c }{$m$=25}  \\
			\cline{3-11}   
			& & Method$\rightarrow$& \multicolumn{2}{ c }{\GADmethod} & \multicolumn{2}{ c }{Ty-test} & \multicolumn{2}{ c }{\GADmethod} & \multicolumn{2}{ c }{Ty-test} \\
			\cline{3-11} 
			&&Metric$\rightarrow$
			&  \multicolumn{1}{ c}{AUROC} 
			& \multicolumn{1}{c }{AUPR}
			&  \multicolumn{1}{ c}{AUROC} 
			& \multicolumn{1}{c }{AUPR}
			&  \multicolumn{1}{ c}{AUROC} 
			& \multicolumn{1}{c }{AUPR}
			&  \multicolumn{1}{ c}{AUROC} 
			& \multicolumn{1}{c }{AUPR}\\
			\hline
			
			\multirow{3}*{\rotatebox{90}{\tabincell{c}{Fash.}}} & \multicolumn{2}{ l }{MNIST} 
			&99.7$\pm$0.1 
			&99.5$\pm$0.2
			&\textbf{100.0$\pm$0.0} 
			&\textbf{100.0$\pm$0.0}
			&\textbf{100.0$\pm$0.0}
			&\textbf{100.0$\pm$0.0} 
			&\textbf{100.0$\pm$0.0} 
			&\textbf{100.0$\pm$0.0}
			\\
			& \multicolumn{2}{ l }{MNIST-C(0.4)}
			&\textbf{99.8$\pm$0.0}
			&\textbf{99.8$\pm$0.0}
			&39.1$\pm$0.7
			&40.5$\pm$0.3
			&\textbf{100.0$\pm$0.0}
			&\textbf{100.0$\pm$0.0}
			&37.6$\pm$1.9
			&39.8$\pm$0.7
			\\			
			& \multicolumn{2}{ l }{notMNIST} 
			&\textbf{100.0$\pm$0.0} 
			&\textbf{100.0$\pm$0.0}
			&\textbf{100.0$\pm$0.0} 
			&\textbf{100.0$\pm$0.0} 
			&\textbf{100.0$\pm$0.0} 
			&\textbf{100.0$\pm$0.0} 
			&\textbf{100.0$\pm$0.0} 
			&\textbf{100.0$\pm$0.0} 
			\\
			\hline
			
			\multirow{8}*{\rotatebox{90}{\tabincell{c}{SVHN}}} &\multicolumn{2}{ l }{CelebA} 
			&92.2$\pm$0.6
			&82.3$\pm$1.1
			&\textbf{100.0$\pm$0.0}
			&\textbf{100.0$\pm$0.0}
			&\textbf{100.0$\pm$0.0}
			&\textbf{100.0$\pm$0.0}
			&\textbf{100.0$\pm$0.0}
			&\textbf{100.0$\pm$0.0}
			\\
			& \multicolumn{2}{ l }{CelebA-C(0.7)} 
			&\textbf{86.2$\pm$0.9}
			&\textbf{76.5$\pm$1.5}
			&39.9$\pm$1.2
			&41.2$\pm$0.5
			&\textbf{100.0$\pm$0.0}
			&\textbf{100.0$\pm$0.0}
			&47.4$\pm$1.5
			&44.3$\pm$0.7
			\\
			& \multicolumn{2}{ l }{CIFAR-10}
			&90.9$\pm$1.3
			&81.3$\pm$2.3
			&\textbf{100.0$\pm$0.0}
			&\textbf{100.0$\pm$0.0}
			&\textbf{100.0$\pm$0.0}
			&\textbf{100.0$\pm$0.0}
			&\textbf{100.0$\pm$0.0}
			&\textbf{100.0$\pm$0.0}
			\\
			& \multicolumn{2}{ l }{CIFAR-10-C(0.4)}
			&\textbf{77.6$\pm$8.8}
			&\textbf{69.9$\pm$1.3}
			&49.8$\pm$0.6
			&45.8$\pm$0.3
			&\textbf{99.7$\pm$0.2}
			&\textbf{99.6$\pm$0.3}
			&58.8$\pm$0.9
			&50.2$\pm$0.4
			\\
			& \multicolumn{2}{ l }{CIFAR-100}
			&90.4$\pm$0.4
			&80.3$\pm$0.6
			&\textbf{100.0$\pm$0.0}
			&\textbf{100.0$\pm$0.0}
			&\textbf{100.0$\pm$0.0}
			&\textbf{100.0$\pm$0.0}
			&\textbf{100.0$\pm$0.0}
			&\textbf{100.0$\pm$0.0}
			\\
			& \multicolumn{2}{ l }{CIFAR-100-C(0.4)}
			&\textbf{80.5$\pm$1.0}
			&\textbf{73.2$\pm$1.8}
			&40.3$\pm$0.8
			&40.7$\pm$1.3
			&\textbf{99.8$\pm$0.0}
			&\textbf{99.8$\pm$0.0}
			&40.5$\pm$0.4
			&41.3$\pm$0.2
			\\
			& \multicolumn{2}{ l }{ImageNet32} 
			&89.3$\pm$8.6
			&80.1$\pm$1.5
			&\textbf{100.0$\pm$0.0}
			&\textbf{100.0$\pm$0.0}
			&\textbf{100.0$\pm$0.0}
			&\textbf{100.0$\pm$0.0}
			&\textbf{100.0$\pm$0.0}
			&\textbf{100.0$\pm$0.0}
			\\
			& \multicolumn{2}{ l }{ImageNet32-C(0.3)}
			&\textbf{74.6$\pm$0.6}
			&\textbf{67.8$\pm$0.7}
			&27.9$\pm$1.0
			&36.5$\pm$0.3
			&\textbf{99.0$\pm$0.0}
			&\textbf{99.0$\pm$0.0}
			&27.9$\pm$1.0
			&36.5$\pm$0.3
			\\	
			\hline
			& \textbf{average} & & \textbf{89.2} & \textbf{82.8}&72.5 &73.2&\textbf{99.9}&\textbf{99.9}&73.8&73.8\\
			\hline
			ID $\downarrow$ & OOD$\downarrow$ & Batch size & \multicolumn{4}{ c }{$m$=25}  &\multicolumn{4}{ c }{$m$=50}  
			\\
			\hline{\tiny }
			\multirow{6}*{\rotatebox{90}{\tabincell{c}{CIFAR-10}}} & \multicolumn{2}{ l }{CelebA} 
			&99.1$\pm$0.4
			&99.1$\pm$0.4
			&\textbf{100.0$\pm$0.0}
			&\textbf{100.0$\pm$0.0}
			&\textbf{100.0$\pm$0.0}
			&\textbf{100.0$\pm$0.0}
			&\textbf{100.0$\pm$0.0}
			&\textbf{100.0$\pm$0.0}
			\\
			& \multicolumn{2}{ l }{CelebA-C(0.7)} 
			&\textbf{94.2$\pm$0.6}
			&\textbf{93.8$\pm$0.8}
			&42.3$\pm$1.1
			&42.8$\pm$0.6
			&\textbf{100.0$\pm$0.0}
			&\textbf{100.0$\pm$0.0}
			&39.3$\pm$2.0
			&41.1$\pm$1.0
			\\
			& \multicolumn{2}{ l }{ImageNet32} 
			&54.0$\pm$1.9
			&53.4$\pm$0.7
			&\textbf{99.8$\pm$0.1}
			&\textbf{99.8$\pm$0.1}
			&94.0$\pm$0.6
			&94.0$\pm$0.5
			&\textbf{100.0$\pm$0.0}
			&\textbf{100.0$\pm$0.0}
			\\
			& \multicolumn{2}{ l }{ImageNet32-C(0.8)}
			&\textbf{77.4$\pm$1.4}
			&\textbf{77.3$\pm$1.8}
			&47.8$\pm$1.5
			&48.0$\pm$1.5
			&\textbf{98.8$\pm$0.5}
			&\textbf{98.9$\pm$0.4}
			&46.4$\pm$1.7
			&46.8$\pm$1.2
			\\	
			& \multicolumn{2}{ l }{SVHN} 
			&91.8$\pm$1.5
			&91.1$\pm$2.3
			&\textbf{99.8$\pm$0.0}
			&\textbf{99.8$\pm$0.0}
			&\textbf{100.0$\pm$0.0}
			&\textbf{100.0$\pm$0.0}
			&\textbf{100.0$\pm$0.0}
			&\textbf{100.0$\pm$0.0}
			\\
			& \multicolumn{2}{ l }{SVHN-C(1.5)}
			&\textbf{94.2$\pm$1.5}
			&\textbf{91.1$\pm$2.3}
			&60.0$\pm$1.7
			&61.4$\pm$1.7
			&\textbf{100.0$\pm$0.0}
			&\textbf{100.0$\pm$0.0}
			&53.6$\pm$2.7
			&55.7$\pm$1.6
			\\
			\hline
			& \textbf{average} & & \textbf{85.1} &\textbf{84.3} & 75.0&75.3&\textbf{98.8}&\textbf{98.8}&73.2&73.9\\
			\bottomrule[1pt]   			
		\end{tabular}  
	\end{center}  
\end{table*}

\begin{table*}[ht] 
	\vspace{-10pt}
	\scriptsize
	\caption{GAD results (AUROC and AUPR) of \GADmethod\ without split representations on VAE trained on CIFAR10 and tested on CIFAR100. Each row is for one batch size.} 
	\label{tbl:vae_cifar10_cifar100_imagenet32_ctr_tc}  
	\vspace{-0pt}
	\begin{center}  
		\begin{tabular}{ l r r r r | r r r r }  
			\toprule[1pt]
			Problem & \multicolumn{4}{ c }{CIFAR10 vs CIFAR100}& \multicolumn{4}{ c }{CIFAR10 vs ImageNet32} \\
			\hline
			Method
			& \multicolumn{2}{ c }{\GADmethod}
			&\multicolumn{2}{ c }{Ty-test}  
			& \multicolumn{2}{ c }{\GADmethod}
			&\multicolumn{2}{ c }{Ty-test}  \\
			\hline
			Metric
			&  \multicolumn{1}{ c}{AUROC} 
			& \multicolumn{1}{c }{AUPR}
			&  \multicolumn{1}{ c}{AUROC} 
			& \multicolumn{1}{c }{AUPR}
			&  \multicolumn{1}{ c}{AUROC} 
			& \multicolumn{1}{c }{AUPR}
			&  \multicolumn{1}{ c}{AUROC} 
			& \multicolumn{1}{c }{AUPR} \\
			\hline
			
			$m$=50 
			&72.9$\pm$0.7
			&73.7$\pm$2.1
			&\textbf{73.8$\pm$0.5}
			&\textbf{74.3$\pm$1.8}
			&94.0$\pm$0.6
			&94.0$\pm$0.5
			&\textbf{100.0$\pm$0.0}
			&\textbf{100.0$\pm$0.0}
			\\
			$m$=100
			&\textbf{90.9$\pm$1.0}
			&\textbf{91.3$\pm$1.3}
			&82.6$\pm$0.5
			&83.5$\pm$1.1
			&99.9$\pm$0.2
			&99.9$\pm$0.2
			&\textbf{100.0$\pm$0.0}
			&\textbf{100.0$\pm$0.0}
			\\
			$m$=150 
			&\textbf{98.0$\pm$0.4}
			&\textbf{98.1$\pm$0.5}
			&88.4$\pm$1.3
			&88.6$\pm$2.3
			&\textbf{100.0$\pm$0.0}
			&\textbf{100.0$\pm$0.0}
			&\textbf{100.0$\pm$0.0}
			&\textbf{100.0$\pm$0.0}
			\\
			\bottomrule[1pt] 			
		\end{tabular}  
	\end{center}  
\end{table*}

Table \ref{tbl:vae_cifar10_reconstruction_probability} shows the results of using reconstruction probability $E_{z\sim q_{\phi}}[\log p_{\theta}(\bm{x}|\bm{z})]$ for OOD detection in VAE.

\begin{table*} [ht]
	\scriptsize
	\vspace{-0pt}
	\caption{VAE trained on CIFAR10. Use reconstruction probability for OOD data detection.} 
	\label{tbl:vae_cifar10_reconstruction_probability}  
	\begin{center}  
		\begin{tabular}{ l r r }  
			\toprule[1pt]
			Method & \multicolumn{2}{ c }{reconstruction probability}  \\
			\hline 
			Metrics &  \multicolumn{1}{ c}{AUROC} 
			& \multicolumn{1}{c }{AUPR}\\
			\hline
			SVHN
			&17.6$\pm$0.0
			&34.3$\pm$0.0
			\\
			CelebA
			&83.1$\pm$0.0
			&82.5$\pm$0.0
			\\
			ImageNet32
			&72.4$\pm$0.2
			&75.0$\pm$0.1
			\\
			CIFAR100
			&52.3$\pm$0.0
			&53.6$\pm$0.0
			\\
			\hline
			\textbf{average} & 56.4 &61.4\\
			\bottomrule[1pt] 
		\end{tabular}  
	\end{center}  
\end{table*}

\subsection{PAD results on Glow}\label{sec:PAD_appendix}
Table \ref{tbl:PAD_vs_confidence_loss_ODIN} shows PAD results (AUROC in percentage) of Joint confidence loss method, ODIN, Joint confidence loss method+ODIN, DoSE and KLODS.

\begin{table*}[!ht]
	\centering
	\caption{PAD results (AUROC in percentage) of Joint confidence loss method, ODIN, Joint confidence loss method+ODIN, DoSE, and KLODS. We use all the problems evaluated in the original publication of Joint confidence loss method \cite{confidenceloss2018ICLR}.  Our method outperforms all baselines.\\
		\textbf{*} The authors of Joint confidence loss method \cite{confidenceloss2018ICLR} did not report AUROC result of Joint confidence loss method+ODIN on SVHN vs TinyImageNet. Since joint confidence loss+ODIN is reported to be better than Joint confidence loss method, so we just use 100\% AUROC.}
	\label{tbl:PAD_vs_confidence_loss_ODIN}  
	\begin{tabular}{llcccccc}
		\toprule[1pt]
		ID & OOD & confidence loss & joint confidence loss & ODIN & joint confidence loss+ODIN& DoSE & KLODS\\
		\hline
		\multirow{3}*{SVHN} & CIFAR-10 & 83 & 97.6 & 94 & \textbf{99} &96.2 &  98.9 \\ 
		~ & TinyImageNet & 98 & 99.5 & 95 & \textbf{100*} & \textbf{100} & 99.8 \\ 
		~ & LSUN & 98.5 & 99.8 & 94 & \textbf{100} & 91.6 &  \textbf{100} \\ \hline
		\multirow{3}*{CIFAR-10} & SVHN & 46.5 & 67.5 & 85 & 85 & \textbf{95.5}  & 82.6 \\ 
		~ & TinyImageNet & 67 & 72.5 & 76 & \textbf{86} & 76.7 & 83.9 \\ 
		~ & LSUN & 62.5 & 76 & 78 & 87.5 & 98 & \textbf{98.9} \\ \hline
		&\textbf{average} & 75.9 & 85.5 & 87 & 92.9 & 93 & \textbf{94.0}\\
		\bottomrule[1pt]
	\end{tabular}
\end{table*}

\clearpage

\section{Figures} \label{sec:appendix_more_figures}
This section contains figures referred to in the other parts.

\begin{figure*}[t]
	\centering
	\vspace{-0pt}
	\subfigure[]{
		\begin{minipage}[t]{4.4cm}
			\centering
			\includegraphics[width=4.4cm]{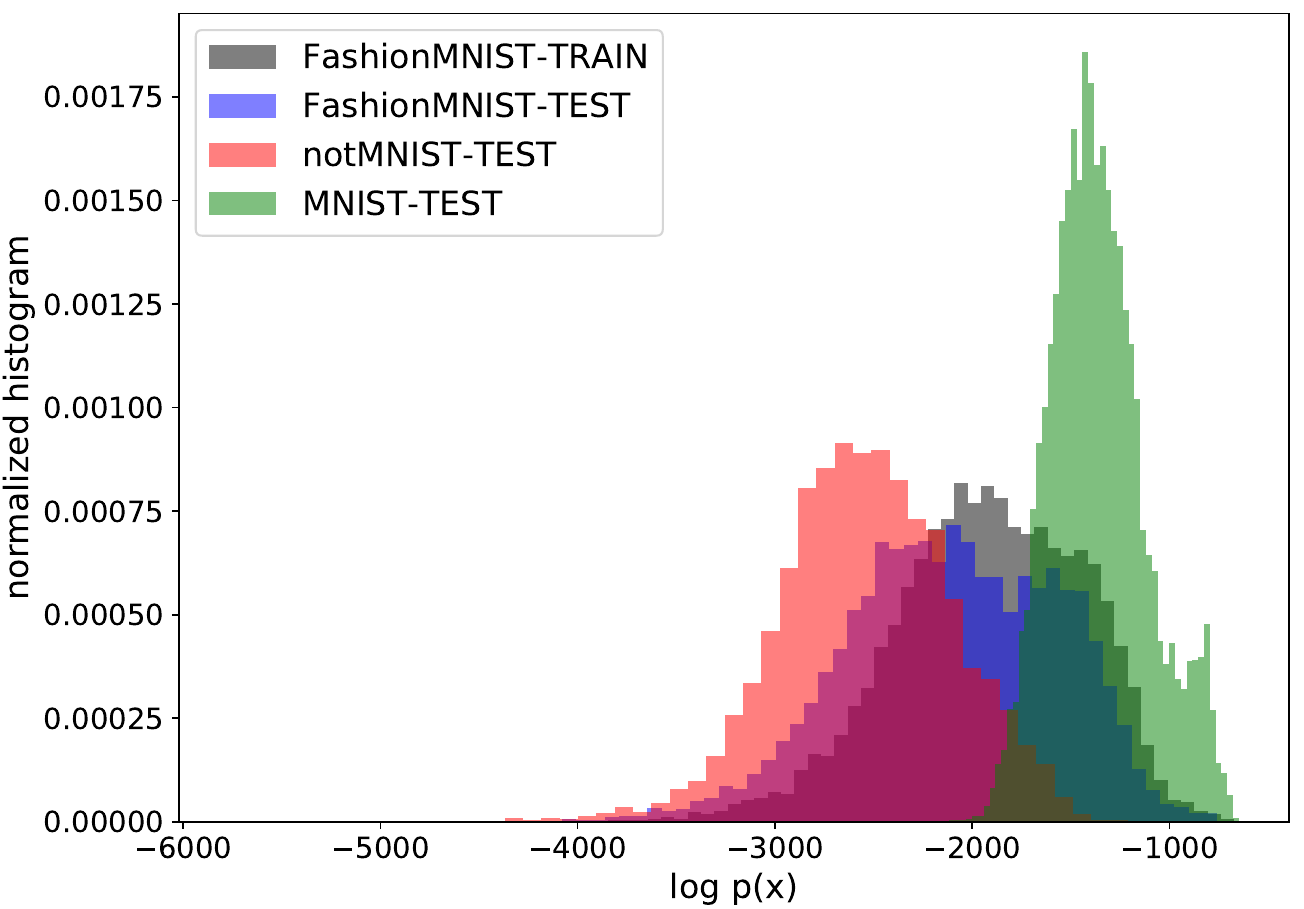}
			\label{fig:logpx_fashionmnist_notmnist_mnist}
		\end{minipage}%
	}%
	\subfigure[]{
		\begin{minipage}[t]{4.15cm}
			\includegraphics[width=4.15cm]{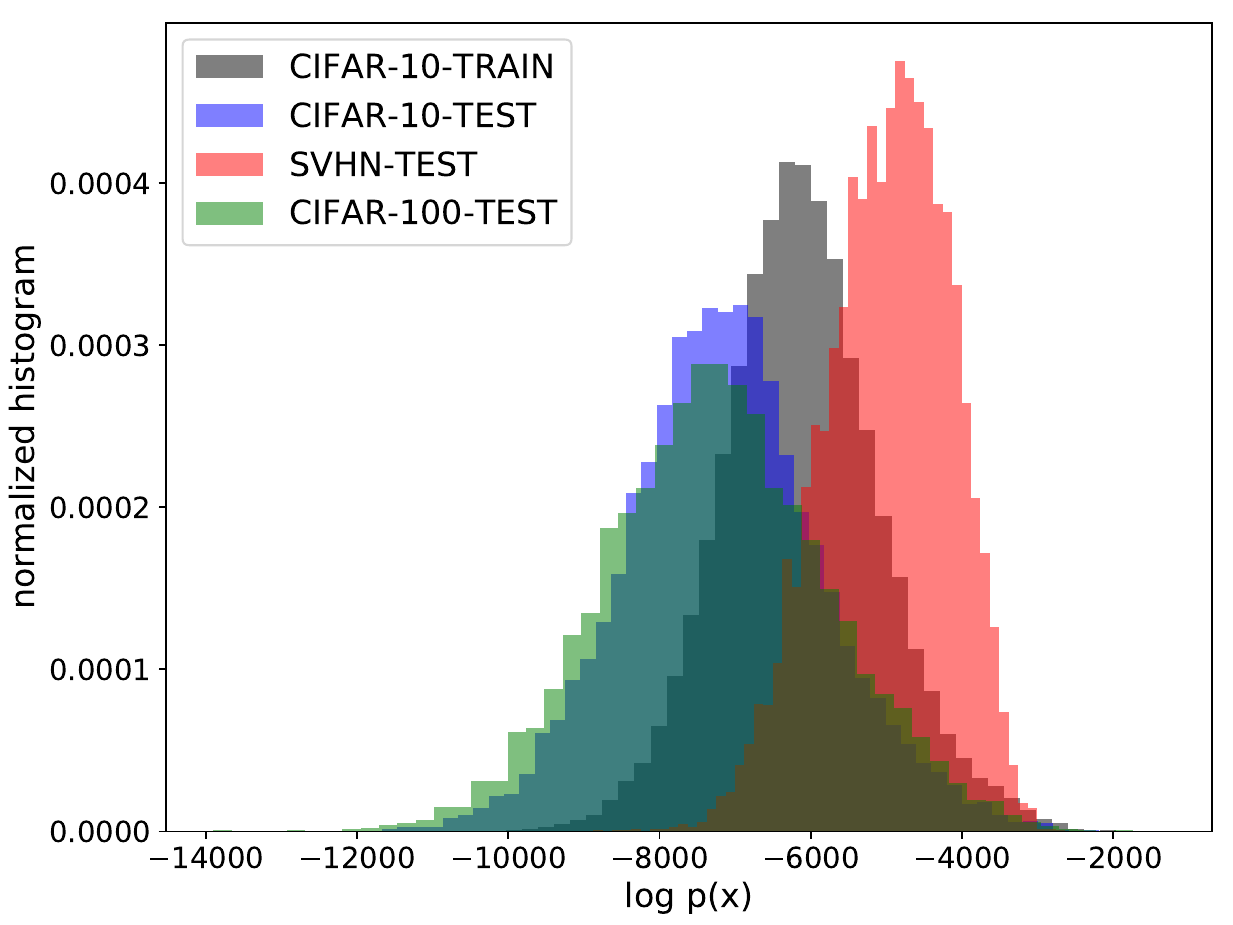}
			\label{fig:logpx_cifar10_cifar100_svhn}
		\end{minipage}
	}
	\subfigure[]{
		\begin{minipage}[t]{4.15cm}
			\includegraphics[width=4.15cm]{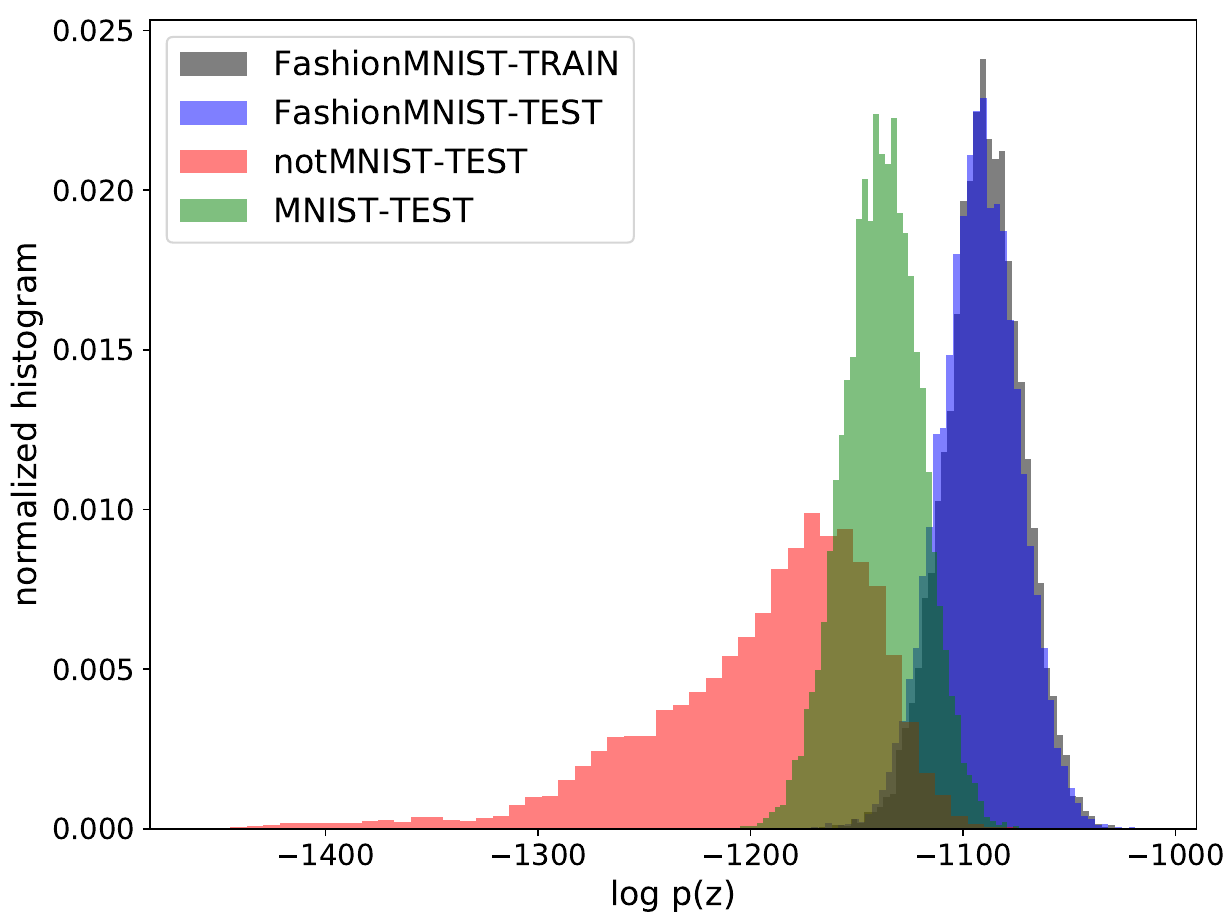}
			\label{fig:logpz_fashionmnist_notmnist_mnist}
		\end{minipage}
	}
	\subfigure[]{
		\begin{minipage}[t]{4.18cm}
			\includegraphics[width=4.18cm]{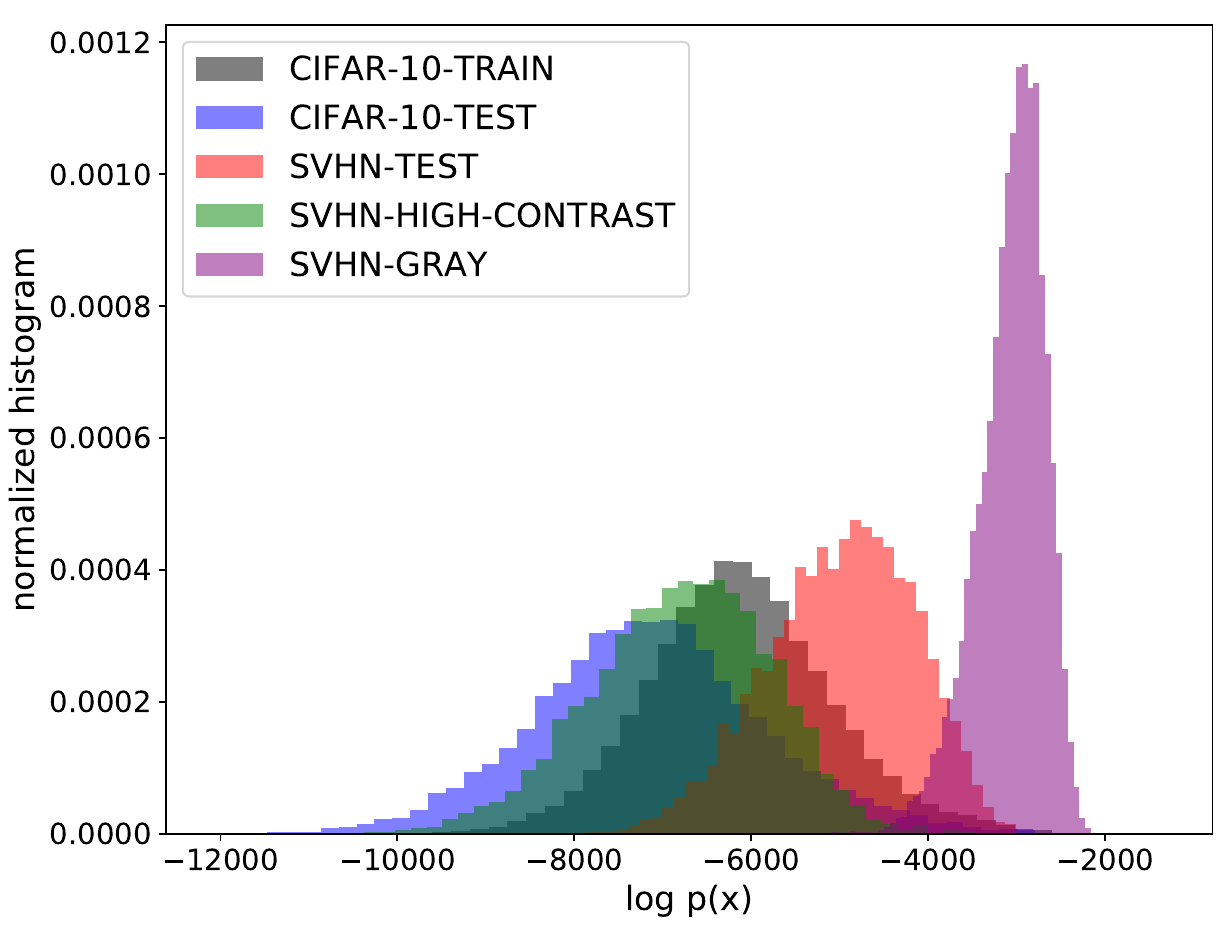}
			\label{fig:logpx_cifar10_svhn_contrast_glow}
		\end{minipage}
	}
	\vspace{-5pt}
	\caption{Distributions of likelihoods of ID dataset (train and test) and OOD dataset. (a) and (b) show the normalized histogram of $\log p(\bm{z})$ and $\log p(\bm{x})$  on Glow trained on FashionMNIST, respectively. (c) shows the normalized histogram of $\log p(\bm{x})$ on Glow trained on CIFAR-10. (d) shows that $\log p(\bm{x})$ of OOD data can be manipulated by adjusting the contrast of images. SVHN-HIGH-CONTRAST and SVHN-GRAY are SVHN with adjusted contrast by a factor of 2.0 and 0.3, respectively.}  
	\label{fig:logp_compare_problem}
\end{figure*}

\begin{figure*}[t]
	\centering
	\subfigure[]{
		\begin{minipage}[t]{5cm}
			\centering
			\includegraphics[width=5cm]{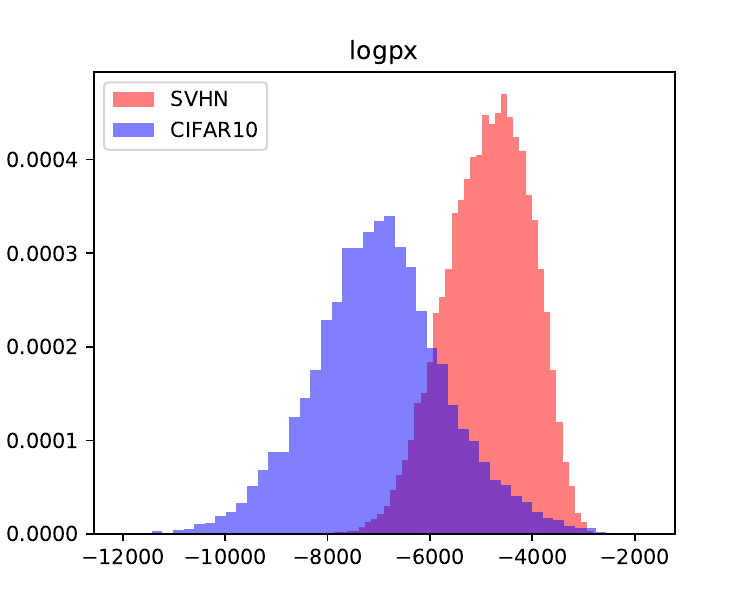}
			\label{fig:logpx_cifar10_vs_svhn_residual_flow}
		\end{minipage}
	}
	\subfigure[]{
		\begin{minipage}[t]{5cm}
			\centering
			\includegraphics[width=5cm]{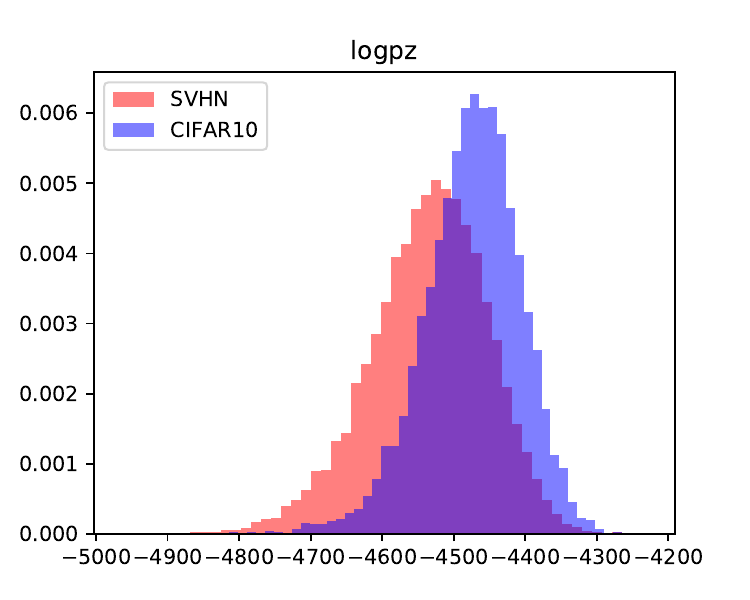}
			\label{fig:logpz_cifar10_vs_svhn_residual_flow}
		\end{minipage}
	}
	\subfigure[]{
		\begin{minipage}[t]{5cm}
			\centering
			\includegraphics[width=5cm]{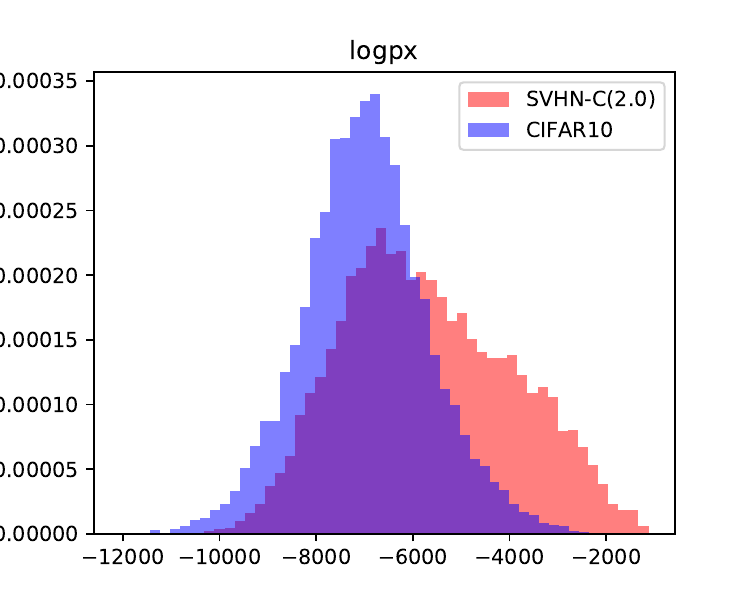}
			\label{fig:logpx_cifar10_vs_svhn_divergence_residual_flow}
		\end{minipage}
	}
	
	\caption{Residual flow trained on CIFAR-10 assigns (a)  higher $\log p(\x)$ for SVHN; (b) similar $\log p(\z)$ for SVHN; and (c) coinciding $\log p(\x)$ for SVHN with increased contrast with a factor of 2. We use the official implementation at \cite{residual_flow} and the model checkpoint released at \cite{residual_flow_ckp}.}
	\label{fig:logp_cifar10_vs_svhn_residual_flow}
\end{figure*}

\begin{figure*}[ht]
	\centering
	\includegraphics[width=13cm]{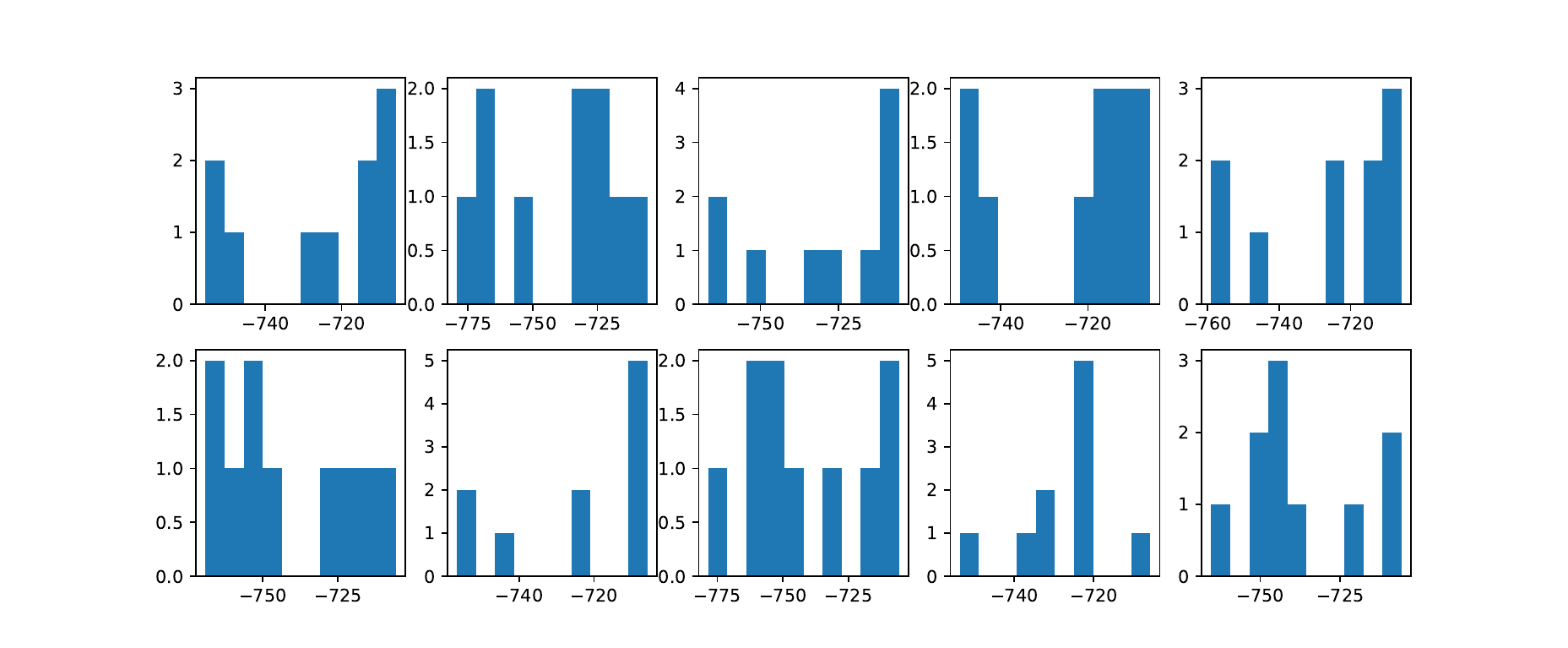}
	\caption{Train GlowGMM on FashionMNIST. The $i$-th subfigure shows the histogram of log-probabilities of 10 centroids under the $i$-th Gaussian component.	All log-probabilities are close to $768\times \log(1/\sqrt{2\pi})\approx -705.74$, which is the log-probability of the center of 768-dimensional standard Gaussian distribution. These results indicate that these centroids are close to each other.}
	\label{fig:cond_glow_logp_of_centroids_under_all_other_components}
\end{figure*}

\begin{figure*}[t]
	\centering	\includegraphics[width=9cm]{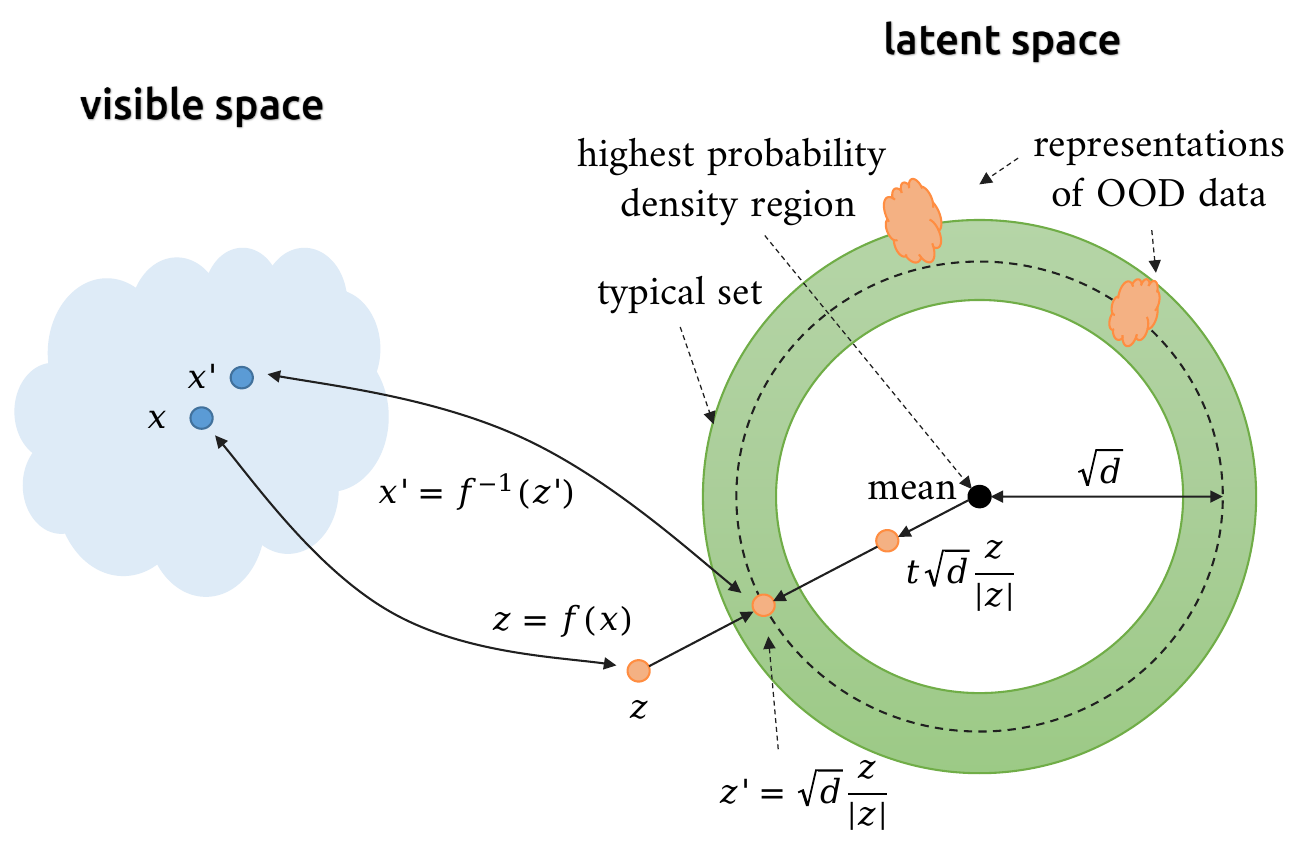}
	\caption{Rescaling $\z$ to the typical set of prior. For each input $\x$, we can compute $\z=f(\x)$, and rescale $\z$ to the typical set annulus of Gaussian prior as $\z'=\sqrt{d}\frac{\z}{|\z|}$. Then we map $\z'$ back to visible space as $\x'=f^{-1}(\z')$. We observe that $\x'$ is similar to $\x$. See Figure \ref{fig:sample_images_rescale_to_typical_set_trained_glow_on_fashionmnist} for examples.
	}
	\label{fig:typical_set}
\end{figure*}

\begin{figure*}[t]
	\centering
	\subfigure[]{
		\begin{minipage}[t]{3.5cm}
			\centering
			\includegraphics[width=3.5cm]{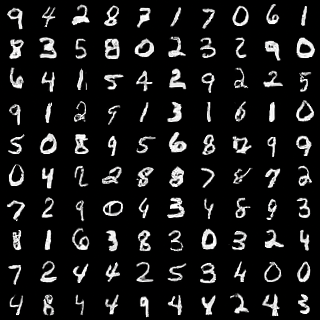}
			\label{fig:scale_to_typical_set_mnist_encodings_glow_fashionmnist}
		\end{minipage}
	}
	\subfigure[]{
		\begin{minipage}[t]{3.5cm}
			\centering
			\includegraphics[width=3.5cm]{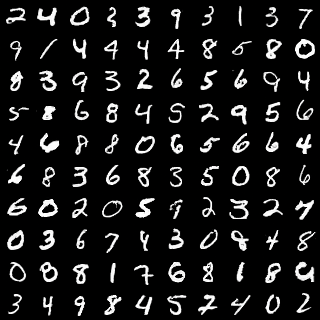}
			\label{fig:scale_to_typical_set_mnist_encodings_glow_fashionmnist}
		\end{minipage}
	}
	\subfigure[]{
		\begin{minipage}[t]{3.5cm}
			\centering
			\includegraphics[width=3.5cm]{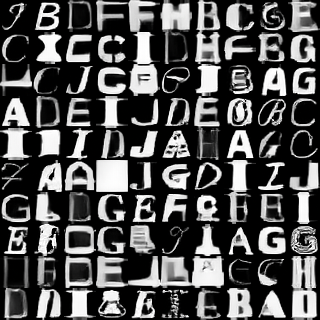}
			\label{fig:scale_to_typical_set_notmnist_encodings_glow_fashionmnist}
		\end{minipage}
	}
	
	\caption{Train Glow on FashionMNIST and test on (a), (b) MNIST and (c) notMNIST . We scale the representations of OOD dataset to the typical set of prior Gaussian distribution. The scaled latent vectors still correspond to nearly the same images. (a) and (c): rescale only the last scale of OOD representation to the typical set of prior. The first and second scales are kept as standard Gaussian noise. (b) rescale the OOD representations at all scales to the typical set of prior.}
	\label{fig:sample_images_rescale_to_typical_set_trained_glow_on_fashionmnist}
\end{figure*}

\begin{figure*}[htbp]
	\centering
	\subfigure{
		\begin{minipage}[t]{4.2cm}
			\centering
			\includegraphics[width=4.2cm]{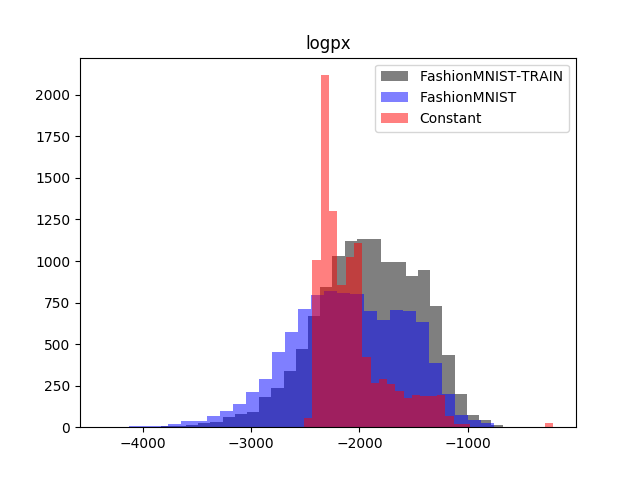}
			\label{fig:logpx_glow_fashionmnist_mnist}
		\end{minipage}
	}
	\subfigure{
		\begin{minipage}[t]{4.2cm}
			\centering
			\includegraphics[width=4.2cm]{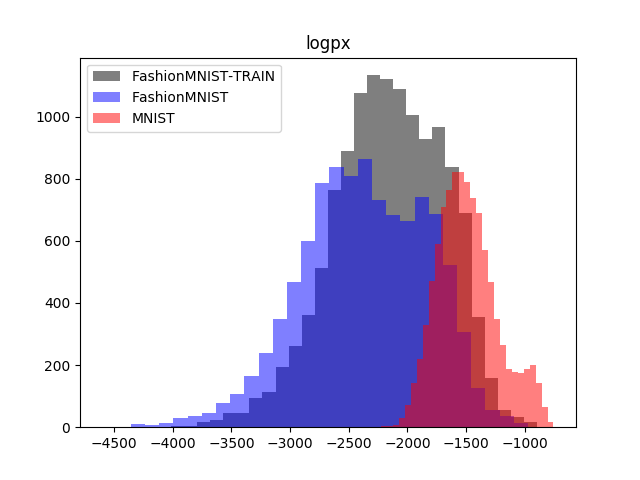}
			\label{fig:logpx_glow_fashionmnist_mnist}
		\end{minipage}
	}
	\subfigure{
		\begin{minipage}[t]{4.2cm}
			\centering
			\includegraphics[width=4.2cm]{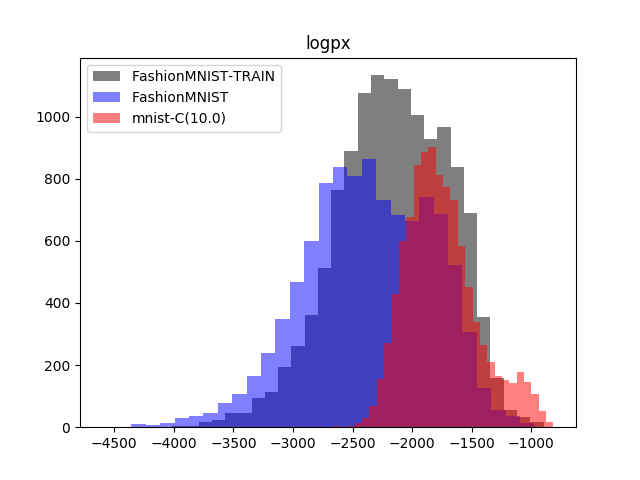}
			\label{fig:logpx_glow_fashionmnist_mnist_divergence}
		\end{minipage}
	}
	
	\subfigure{
		\begin{minipage}[t]{4.2cm}
			\centering
			\includegraphics[width=4.2cm]{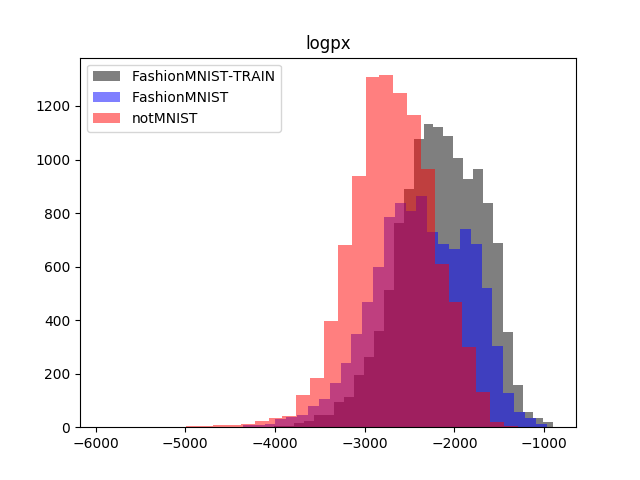}
			\label{fig:logpx_glow_fashionmnist_vs_notmnist}
		\end{minipage}
	}
	\subfigure{
		\begin{minipage}[t]{4.2cm}
			\centering
			\includegraphics[width=4.2cm]{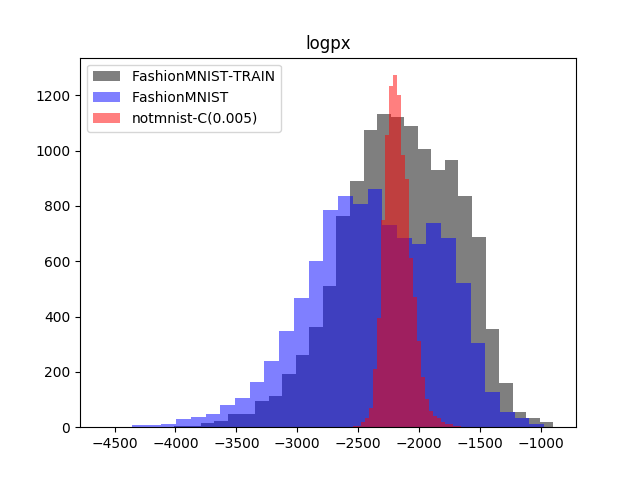}
			\label{fig:logpx_glow_fashionmnist_vs_notmnist_gray}
		\end{minipage}
	}
	
	\caption{Glow trained on FashionMNIST. Histogram of $\log p(\bm{x})$. We can manipulate the likelihood distribution of OOD dataset by adjusting the contrast. ``-C(\textit{k})'' means the dataset with adjusted contrast by a factor of \textit{k}.}
	\label{fig:logpx_glow_fashionmnist_vs_others}
\end{figure*}

\begin{figure*}[htbp]
	\centering
	\subfigure{
		\begin{minipage}[t]{4.2cm}
			\centering
			\includegraphics[width=4.2cm]{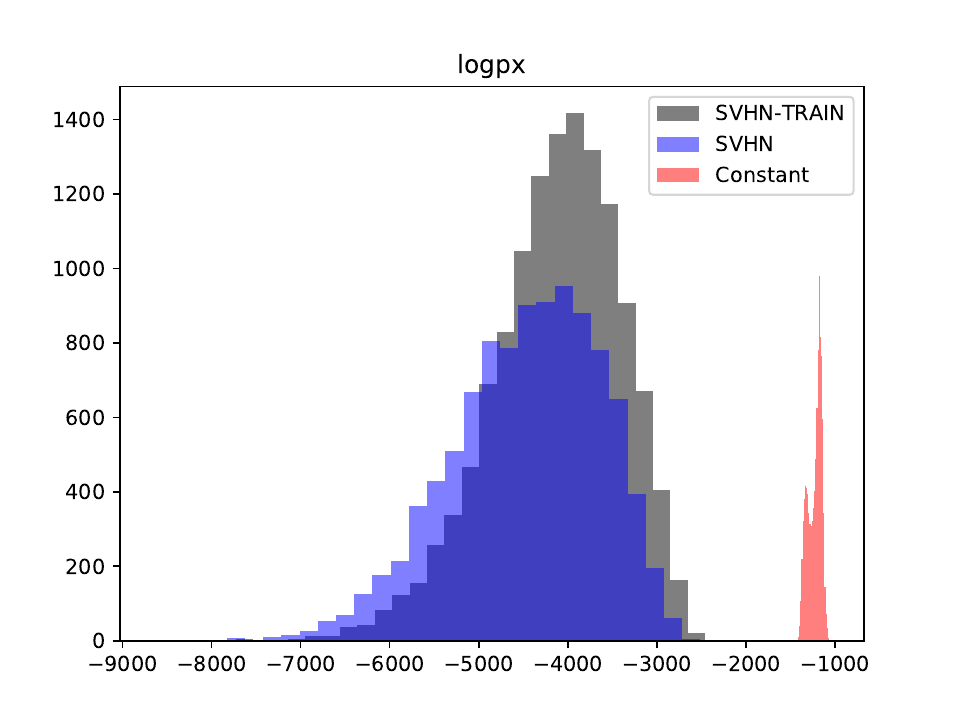}
		\end{minipage}
	}
	\subfigure{
		\begin{minipage}[t]{4.2cm}
			\centering
			\includegraphics[width=4.2cm]{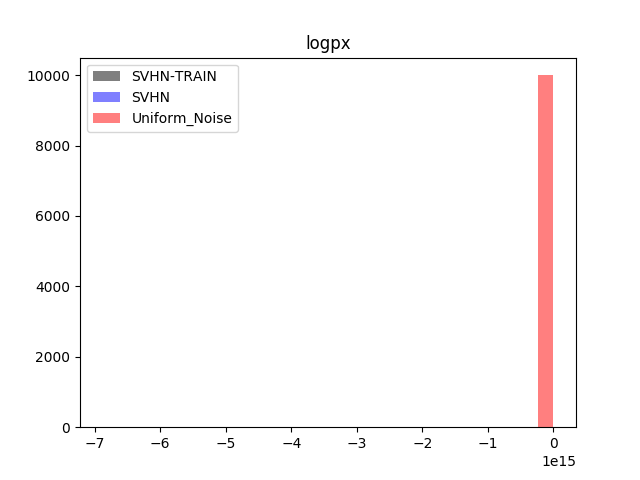}
		\end{minipage}
	}
	\subfigure{
		\begin{minipage}[t]{4.2cm}
			\centering
			\includegraphics[width=4.2cm]{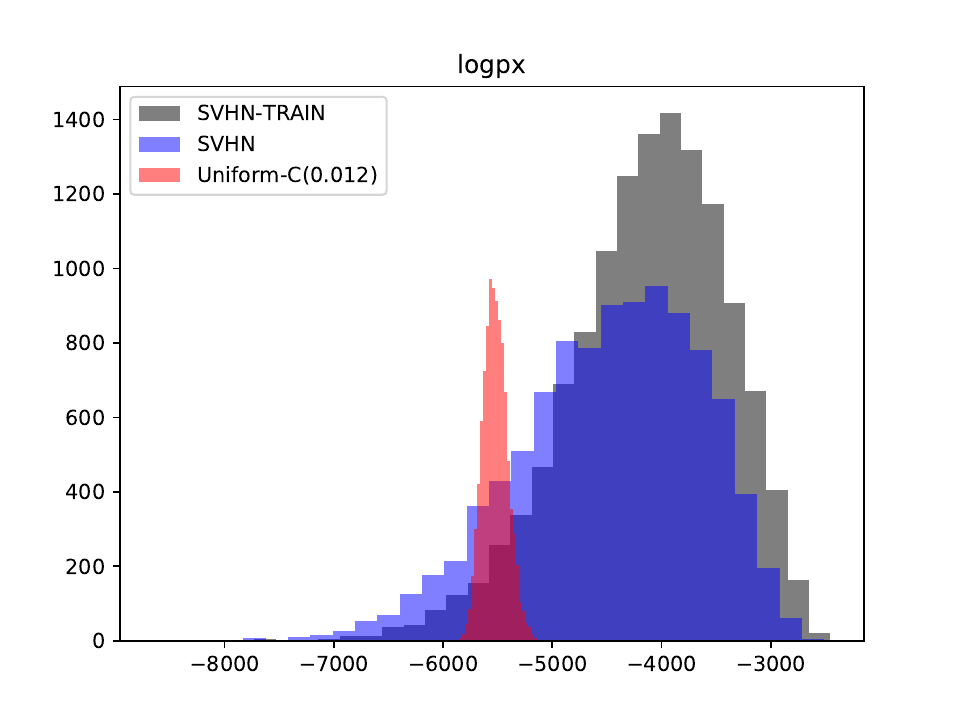}
		\end{minipage}
	}
	
	\subfigure{
		\begin{minipage}[t]{4.2cm}
			\centering
			\includegraphics[width=4.2cm]{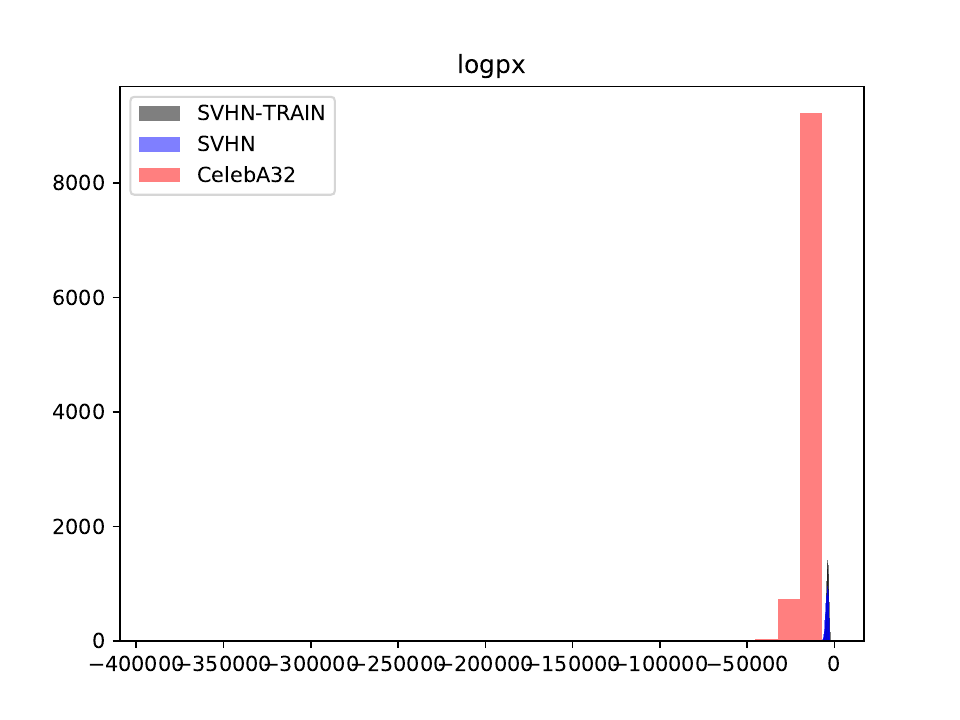}
		\end{minipage}
	}
	\subfigure{
		\begin{minipage}[t]{4.2cm}
			\centering
			\includegraphics[width=4.2cm]{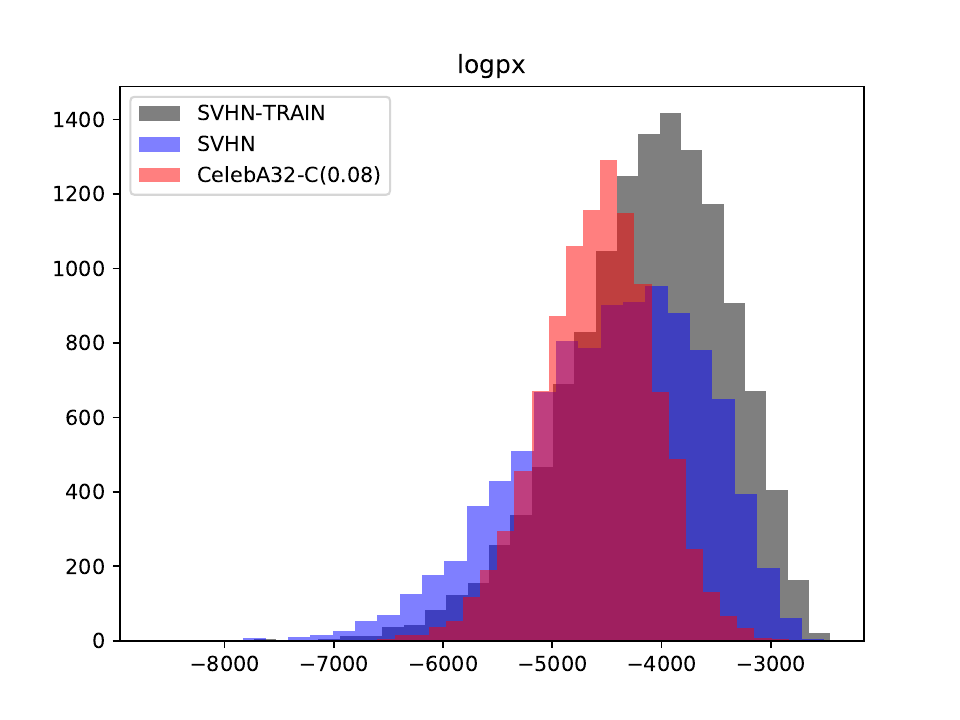}
		\end{minipage}
	}
	\subfigure{
		\begin{minipage}[t]{4.2cm}
			\centering
			\includegraphics[width=4.2cm]{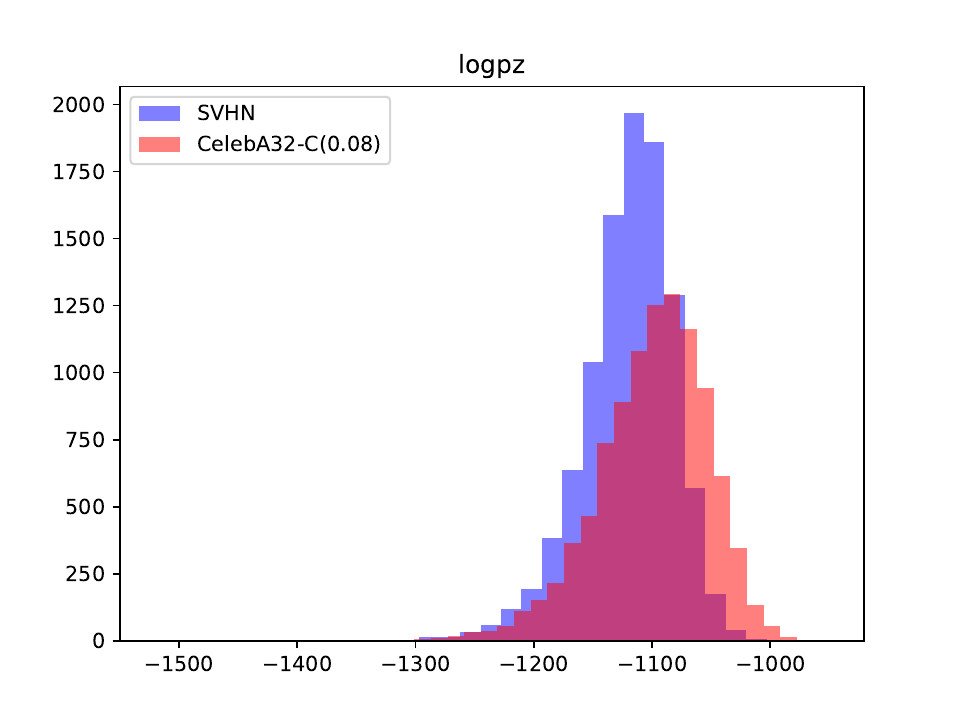}
		\end{minipage}
	}

	\vspace{-0pt}
	\subfigure{
		\begin{minipage}[t]{4.2cm}
			\centering
			\includegraphics[width=4.2cm]{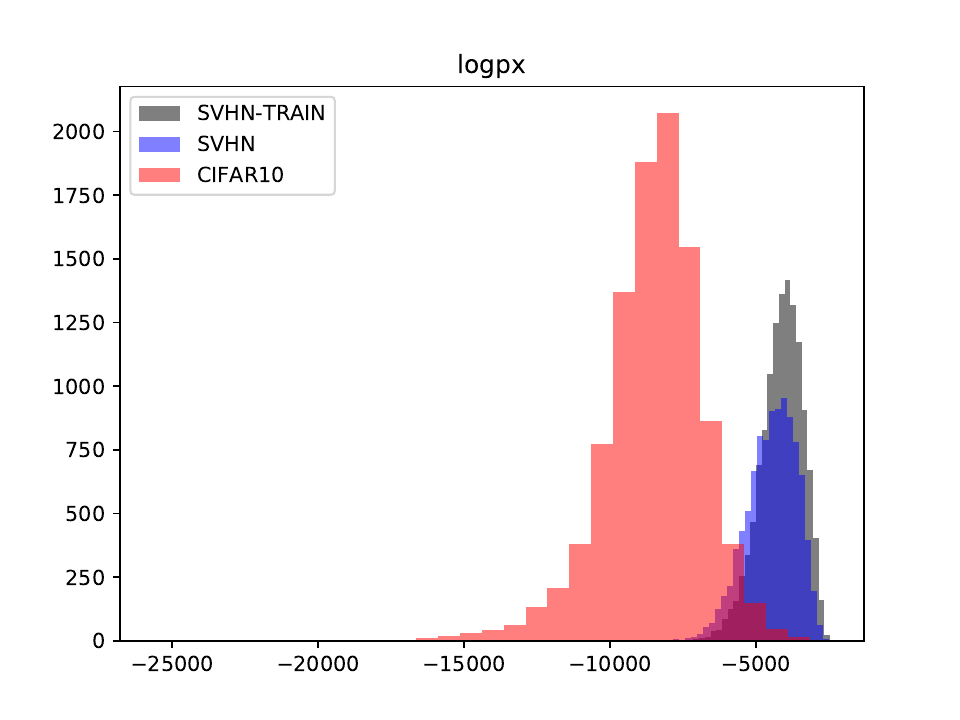}
		\end{minipage}
	}
	\subfigure{
		\begin{minipage}[t]{4.2cm}
			\centering
			\includegraphics[width=4.2cm]{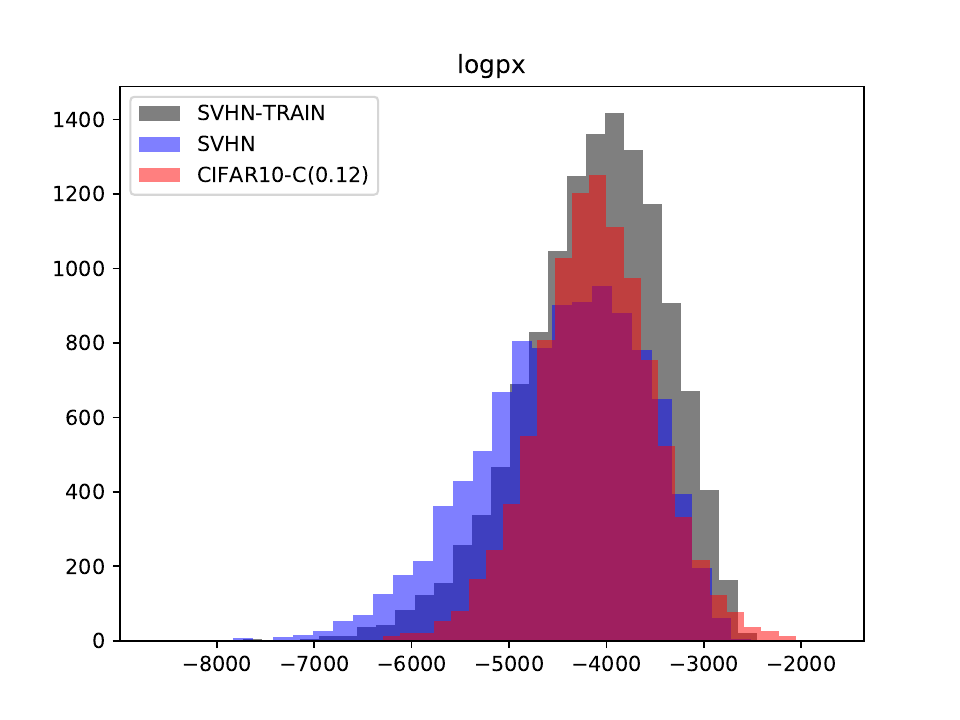}
		\end{minipage}
	}
	\subfigure{
		\begin{minipage}[t]{4.2cm}
			\centering
			\includegraphics[width=4.2cm]{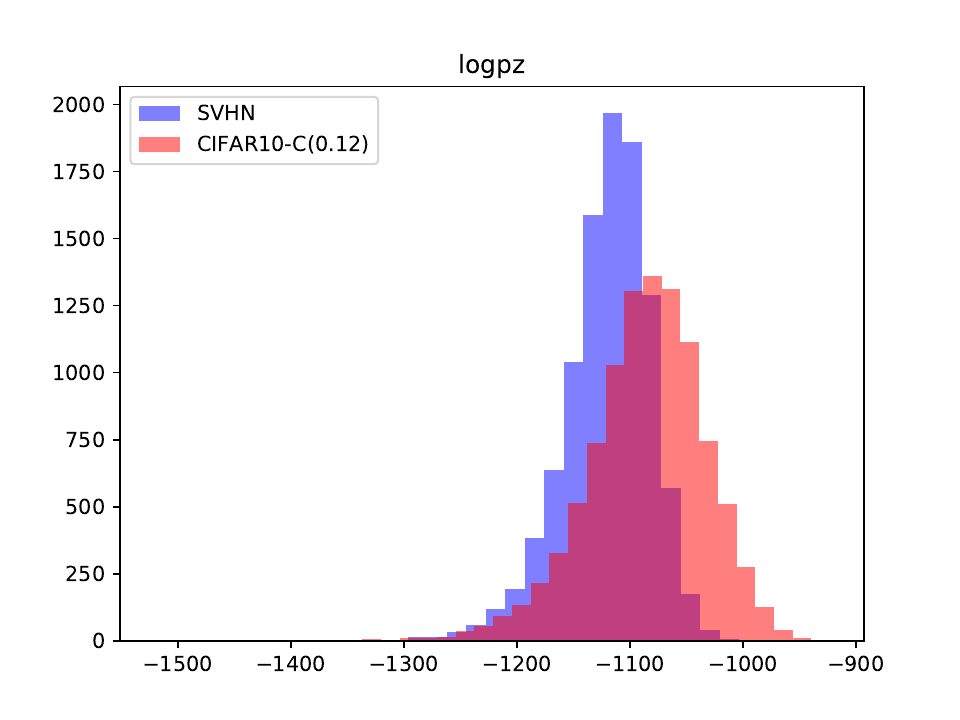}
		\end{minipage}
	}
	
	\vspace{-10pt}
	\subfigure{
		\begin{minipage}[t]{4.2cm}
			\centering
			\includegraphics[width=4.2cm]{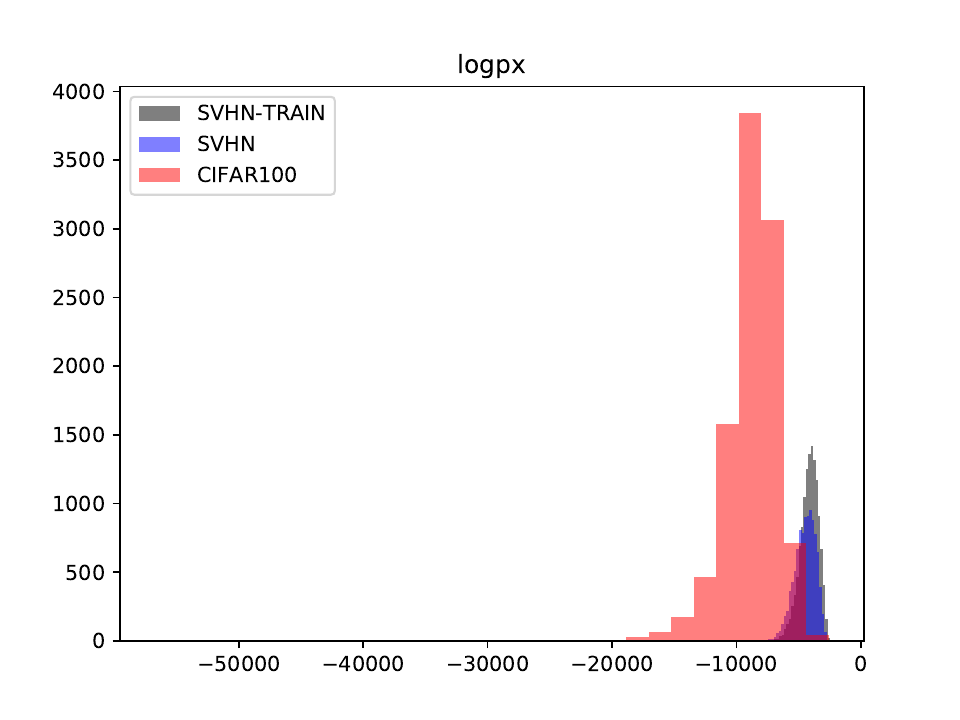}
		\end{minipage}
	}
	\subfigure{
		\begin{minipage}[t]{4.2cm}
			\centering
			\includegraphics[width=4.2cm]{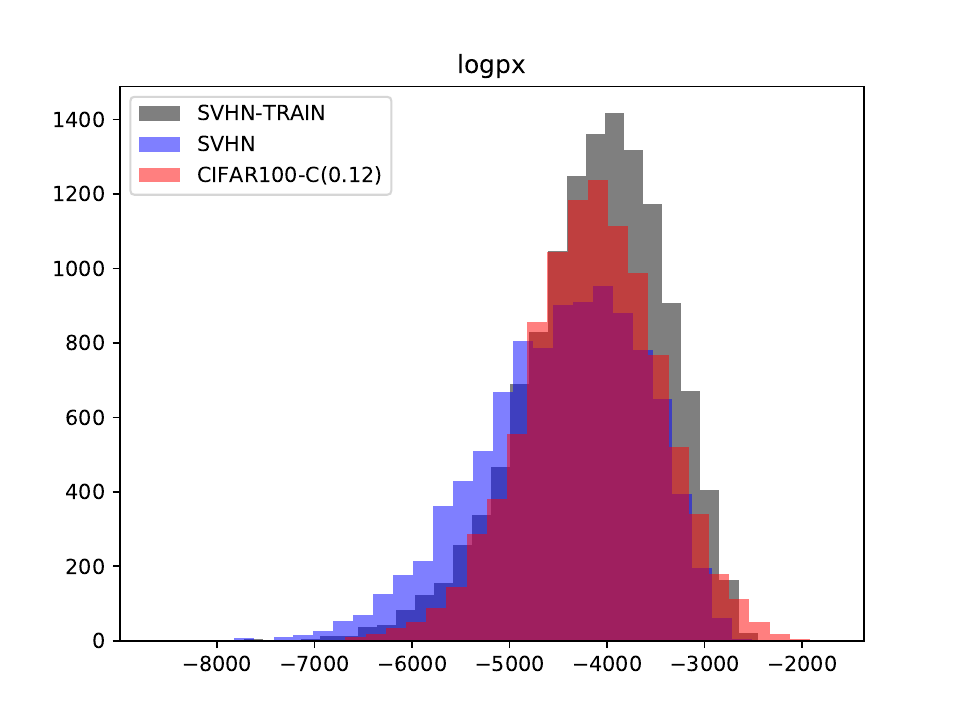}
		\end{minipage}
	}
	\subfigure{
		\begin{minipage}[t]{4.2cm}
			\centering
			\includegraphics[width=4.2cm]{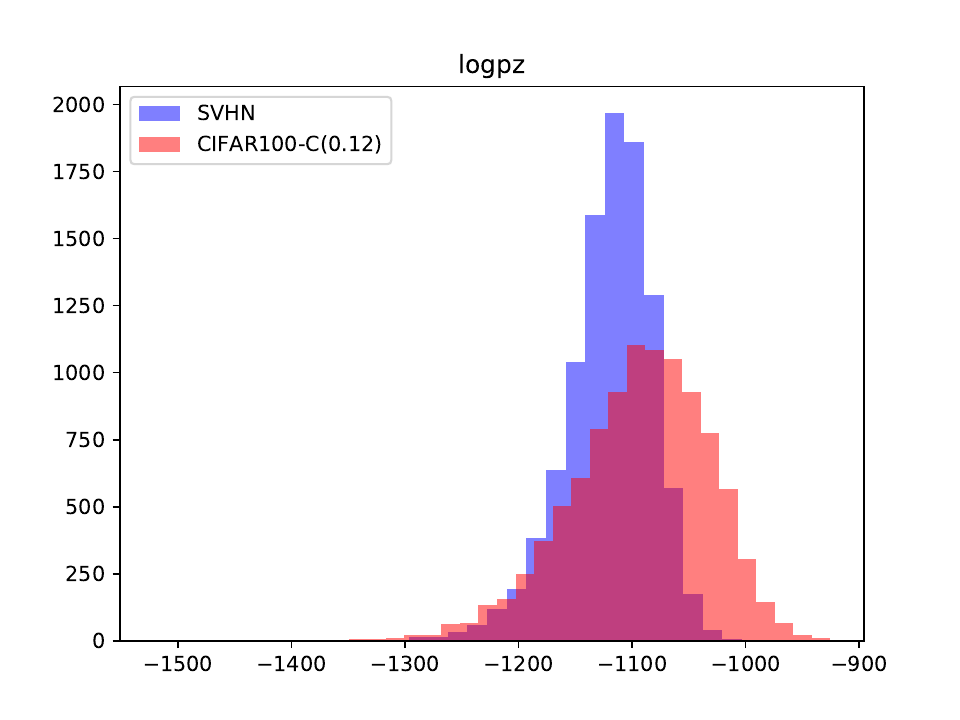}
		\end{minipage}
	}
	
	\vspace{-0pt}
	\subfigure{
		\begin{minipage}[t]{4.2cm}
			\centering
			\includegraphics[width=4.2cm]{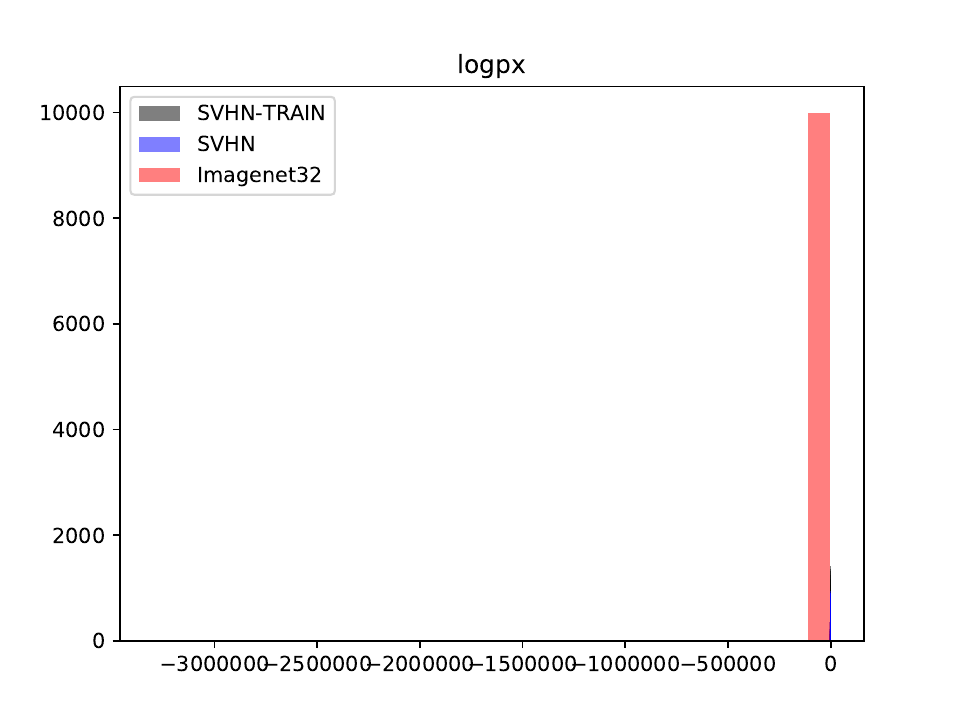}
		\end{minipage}
	}
	\subfigure{
		\begin{minipage}[t]{4.2cm}
			\centering
			\includegraphics[width=4.2cm]{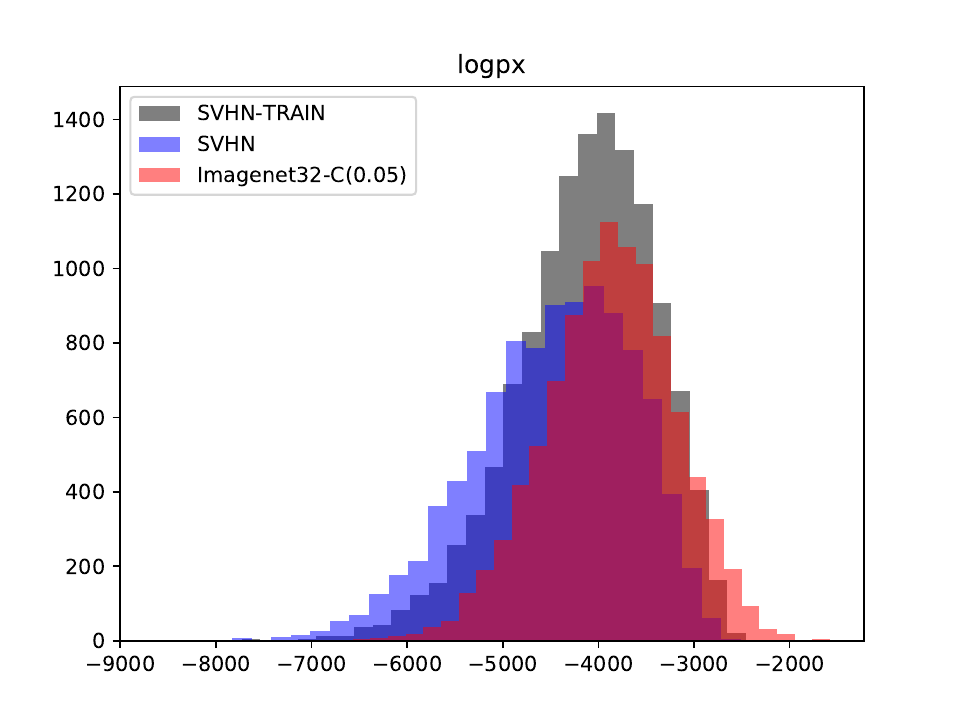}
		\end{minipage}
	}	
	\subfigure{
		\begin{minipage}[t]{4.2cm}
			\centering
			\includegraphics[width=4.2cm]{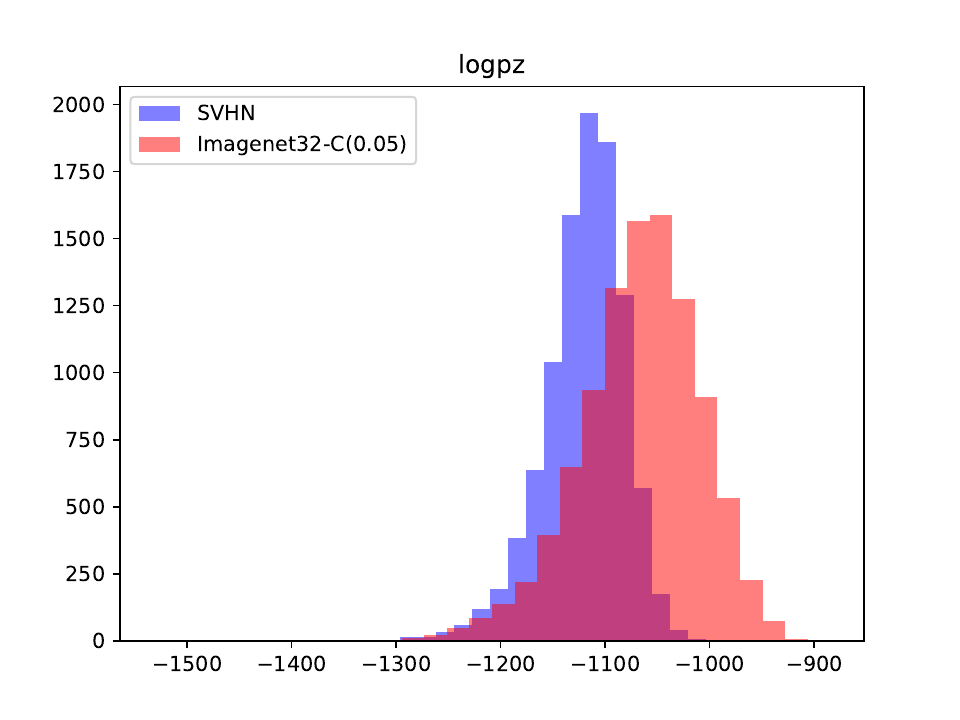}
		\end{minipage}
	}	
	
	\subfigure{
		\begin{minipage}[t]{4.2cm}
			\centering
			\includegraphics[width=4.2cm]{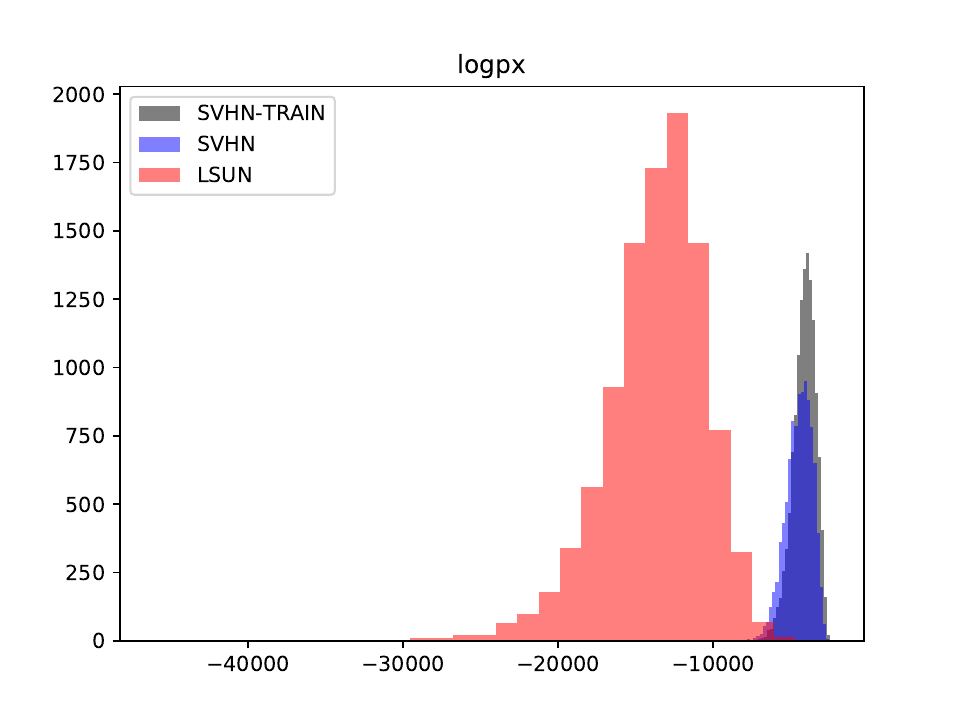}
		\end{minipage}
	}	
	\subfigure{
		\begin{minipage}[t]{4.2cm}
			\centering
			\includegraphics[width=4.2cm]{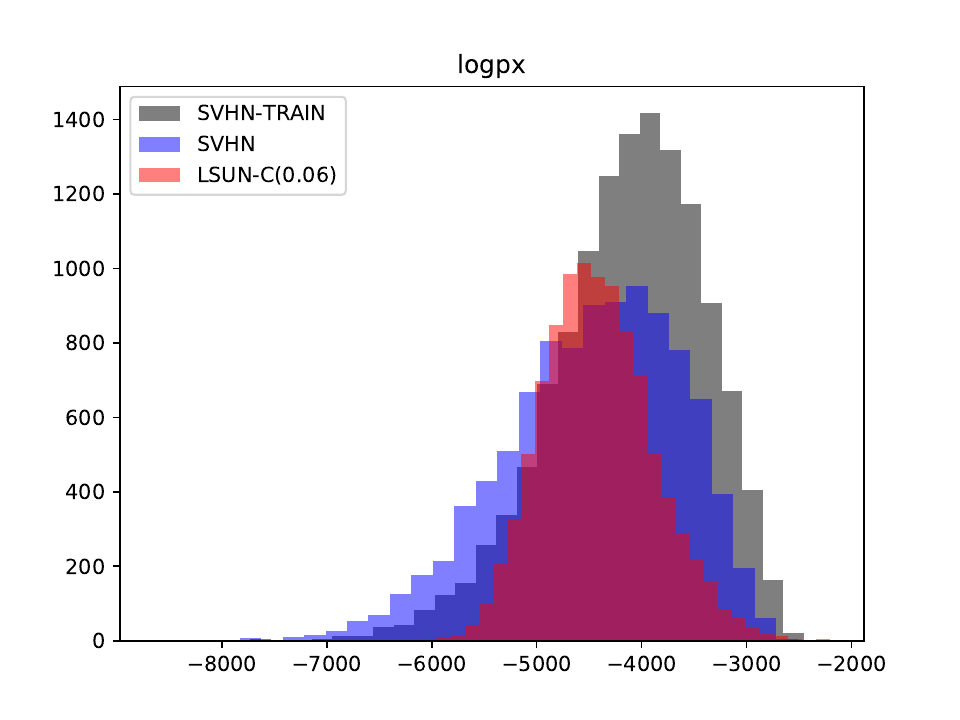}
		\end{minipage}
	}	
	\subfigure{
		\begin{minipage}[t]{4.2cm}
			\centering
			\includegraphics[width=4.2cm]{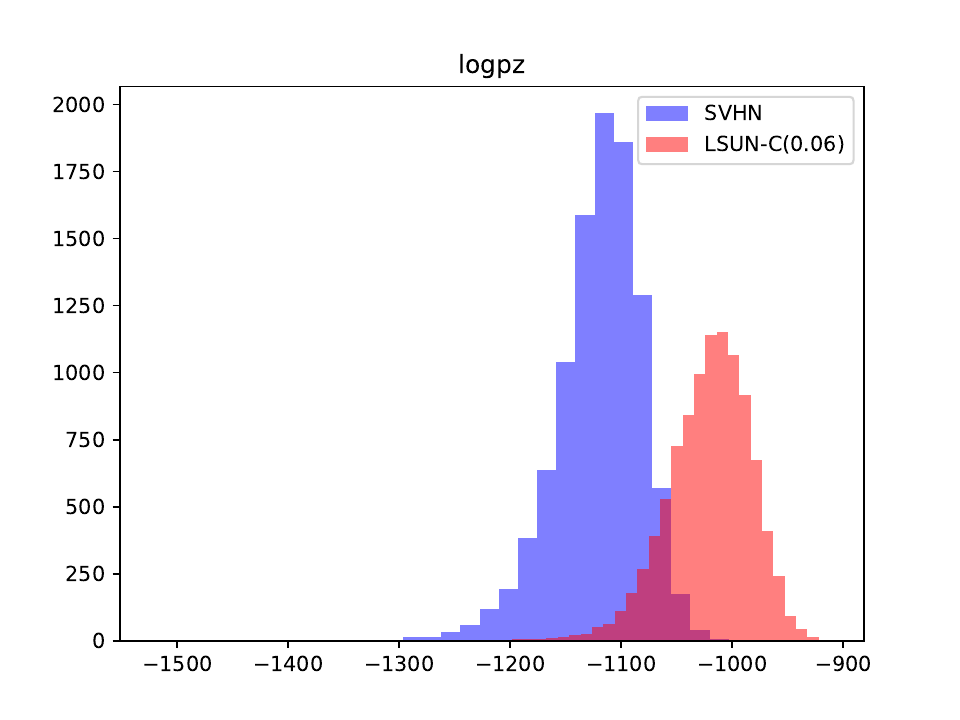}
		\end{minipage}
	}
	
	\caption{Official Glow trained on SVHN. Histogram of $\log p(\bm{x})$, $\log p(\bm{z})$, and $\log p(\bm{x})$ contributed by the last scale. We can manipulate the likelihood distribution of OOD dataset by adjusting the contrast. ``-C(\textit{k})'' means the dataset with adjusted contrast by a factor of \textit{k}. Note that the distribution of $\log p(\x)$ of the last scale and $\log p(\z)$ have a similar shape. This is because the log-determinant of the last scale is similar for every data point in the same dataset. We do not observe this phenomenon in Glow trained on CIFAR-10.}
	\label{fig:logpx_glow_SVHN_vs_others}
\end{figure*}

\begin{figure*}[htbp]
	\centering

	\vspace{-0pt}
	\subfigure{
		\begin{minipage}[t]{4.2cm}
			\centering
			\includegraphics[width=4.2cm]{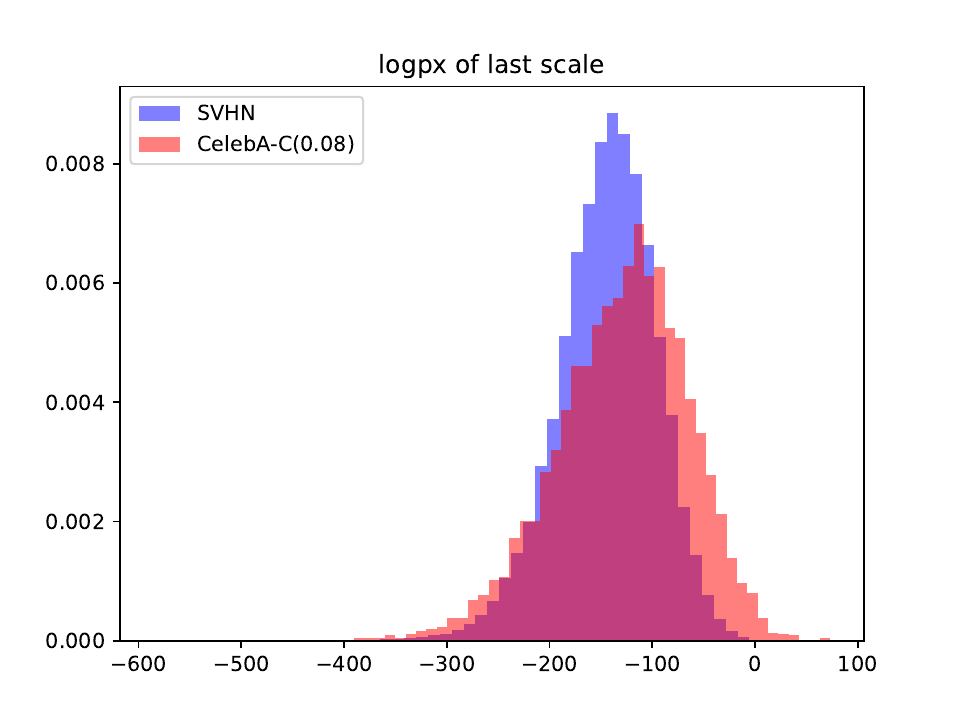}
		\end{minipage}
	}
	\subfigure{
		\begin{minipage}[t]{4.2cm}
			\centering
			\includegraphics[width=4.2cm]{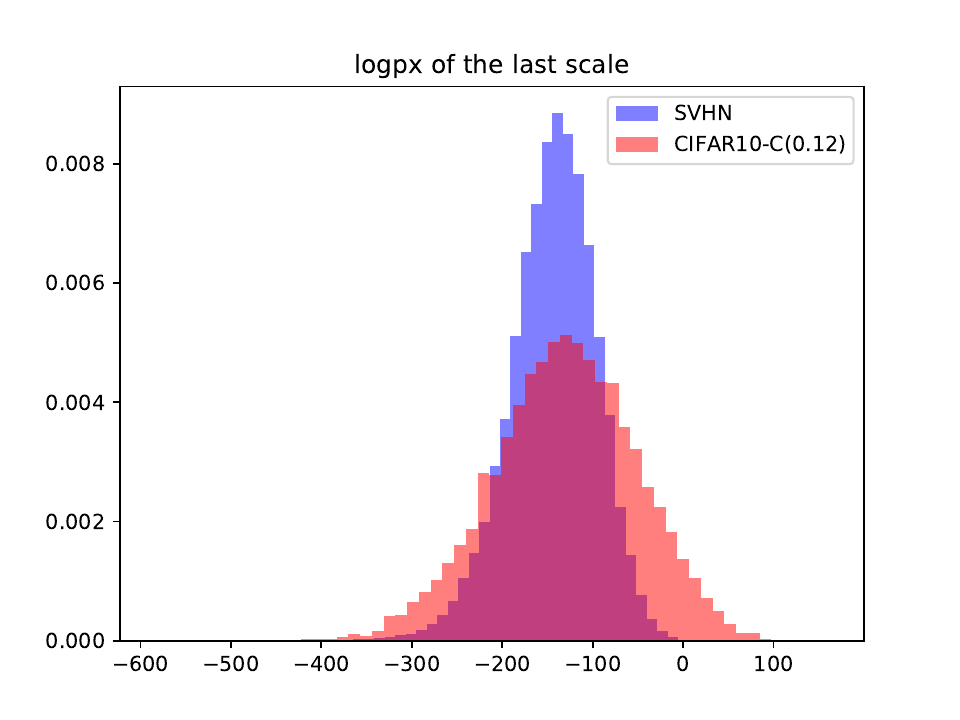}
		\end{minipage}
	}	
	\subfigure{
		\begin{minipage}[t]{4.2cm}
			\centering
			\includegraphics[width=4.2cm]{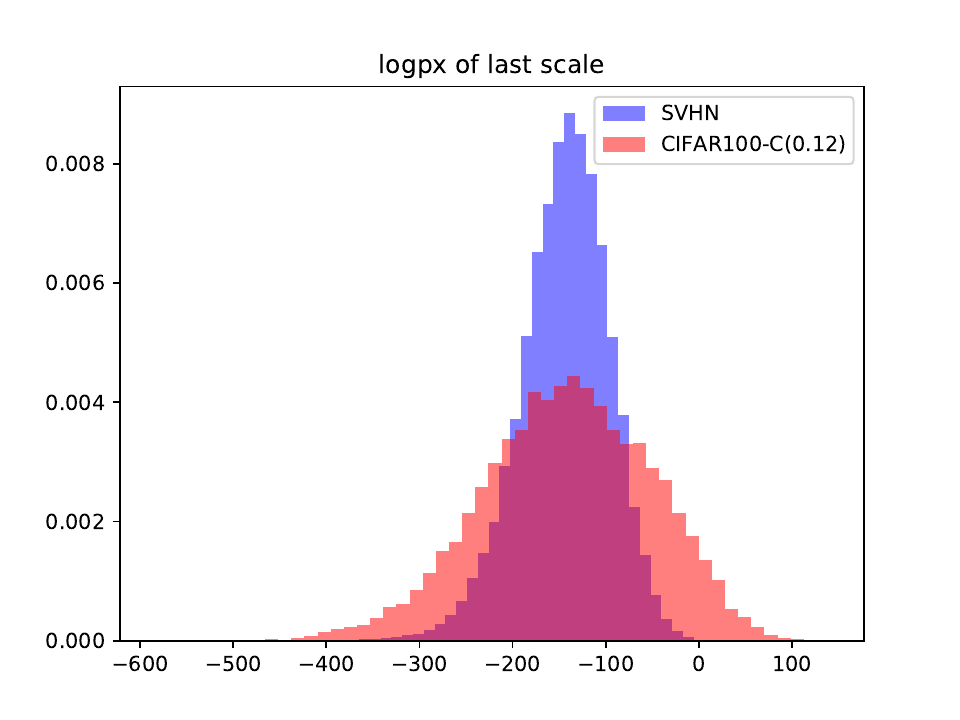}
		\end{minipage}
	}		
	\subfigure{
		\begin{minipage}[t]{4.2cm}
			\centering
			\includegraphics[width=4.2cm]{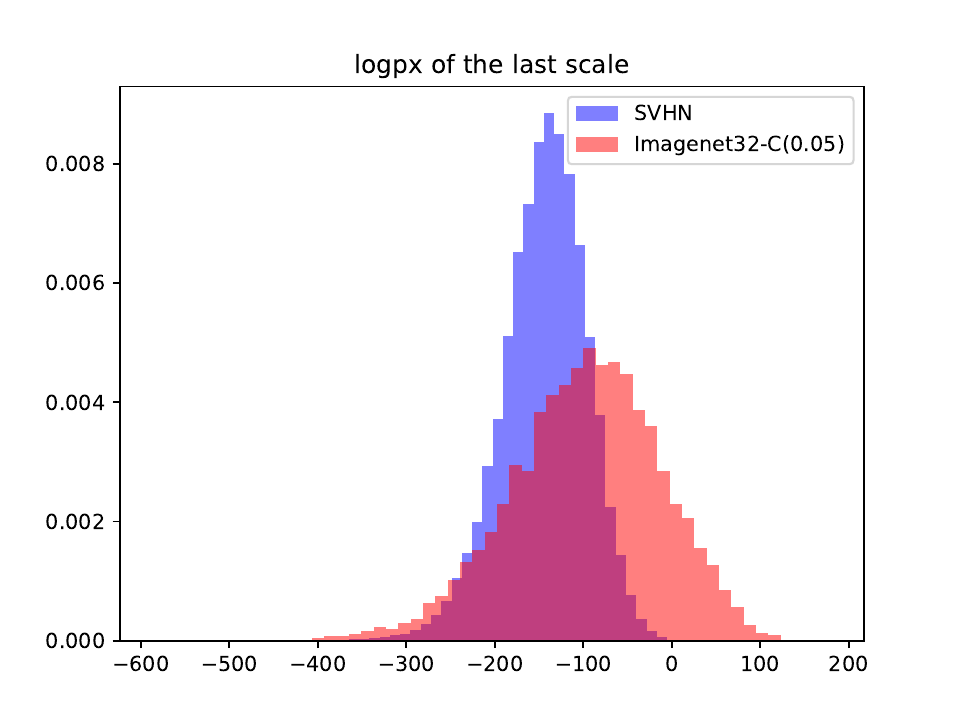}
		\end{minipage}
	}	
	\caption{To reproduce the results of baseline method $L_{last}$ precisely, we use the implementation of Glow model and checkpoint trained on SVHN released by the authors of  $L_{last}$. The $\log p(\x)$ of the last scale of OOD and ID data also coincide under data manipulation. We can also see that the $\log p(\x)$ of the last scale even becomes positive for OOD data. The authors also said that the metric used by $L_{last}$ could not be explained as log-likelihood for OOD data.}
	\label{fig:logpx_last_scale_glow_SVHN_vs_others_hierarchical_model}
\end{figure*}

\begin{figure*}[htbp]
	\centering
	\subfigure{
		\begin{minipage}[t]{4.2cm}
			\centering
			\includegraphics[width=4.2cm]{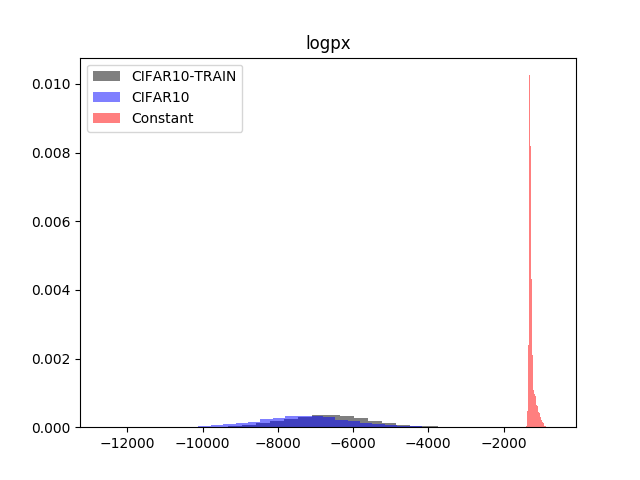}
		\end{minipage}
	}
	\subfigure{
		\begin{minipage}[t]{4.2cm}
			\centering
			\includegraphics[width=4.2cm]{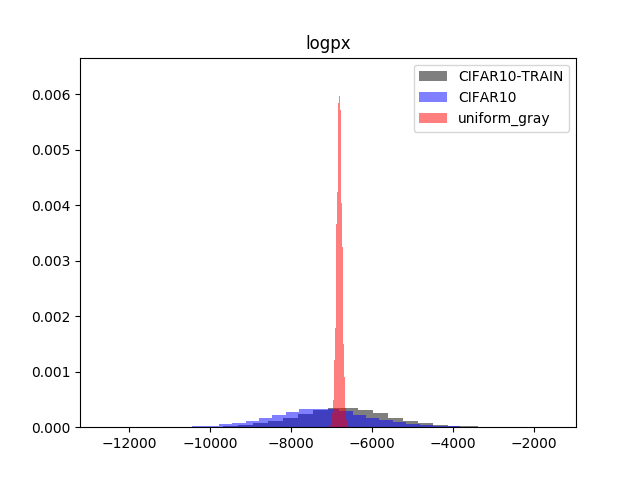}
		\end{minipage}
	}
	\subfigure{
		\begin{minipage}[t]{4.2cm}
			\centering
			\includegraphics[width=4.2cm]{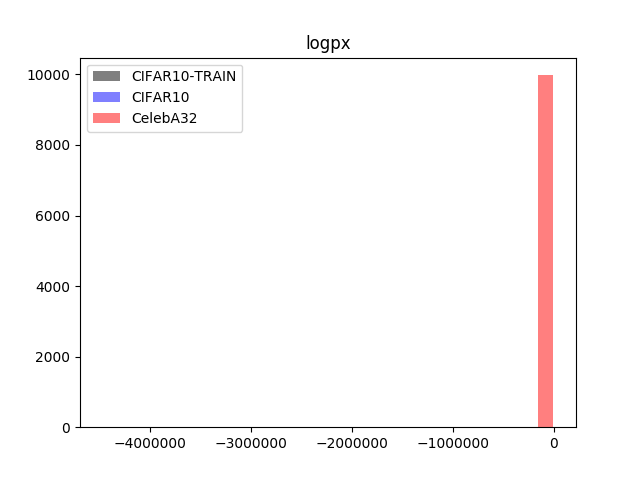}
		\end{minipage}
	}
	\subfigure{
		\begin{minipage}[t]{4.2cm}
			\centering
			\includegraphics[width=4.2cm]{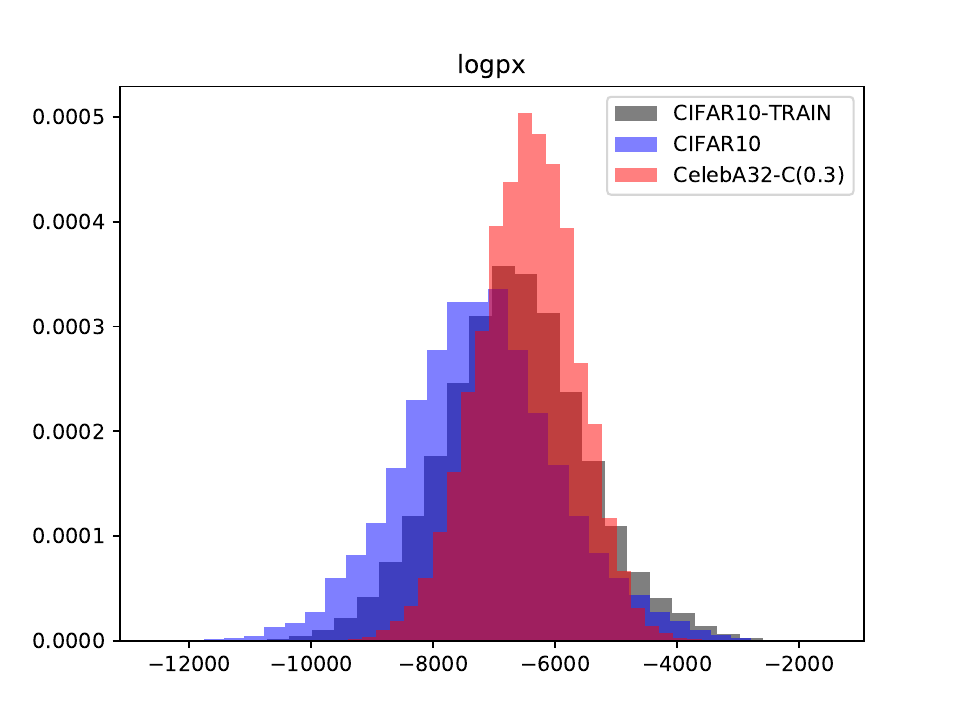}
		\end{minipage}
	}
	
	\subfigure{
		\begin{minipage}[t]{4.2cm}
			\centering
			\includegraphics[width=4.2cm]{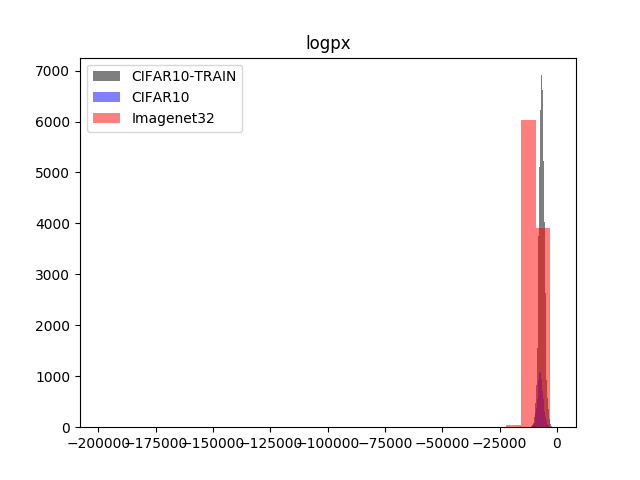}
		\end{minipage}
	}
	\subfigure{
		\begin{minipage}[t]{4.2cm}
			\centering
			\includegraphics[width=4.2cm]{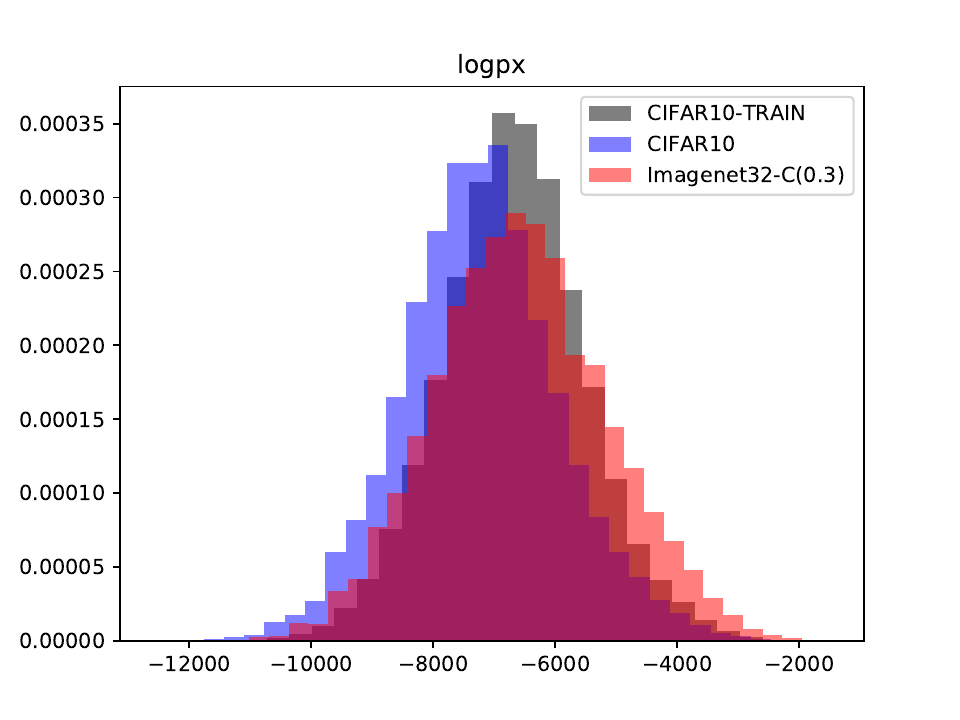}
		\end{minipage}
	}
	\vspace{-10pt}
	\subfigure{
		\begin{minipage}[t]{4.2cm}
			\centering
			\includegraphics[width=4.2cm]{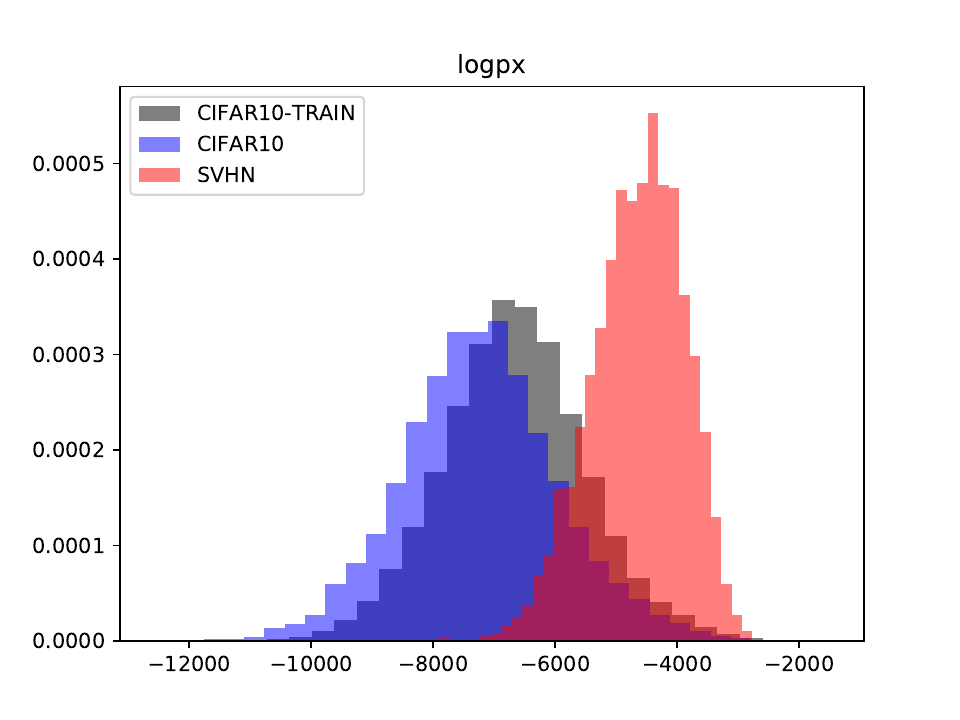}
		\end{minipage}
	}
	
	\subfigure{
		\begin{minipage}[t]{4.2cm}
			\centering
			\includegraphics[width=4.2cm]{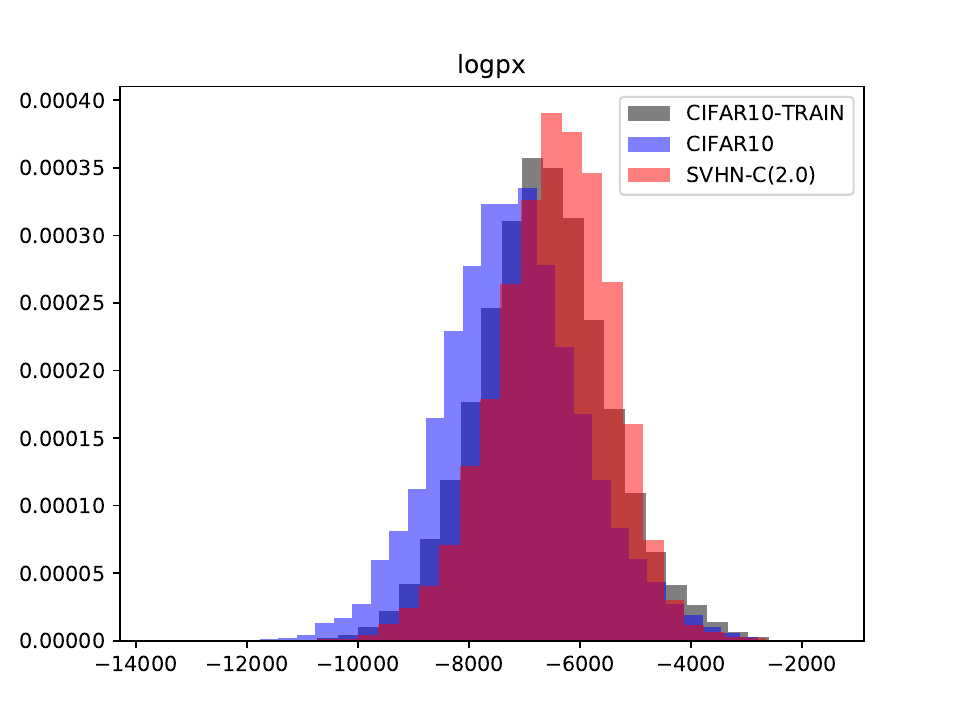}
		\end{minipage}
	}
	\vspace{-0pt}
	\subfigure{
		\begin{minipage}[t]{4.2cm}
			\centering
			\includegraphics[width=4.2cm]{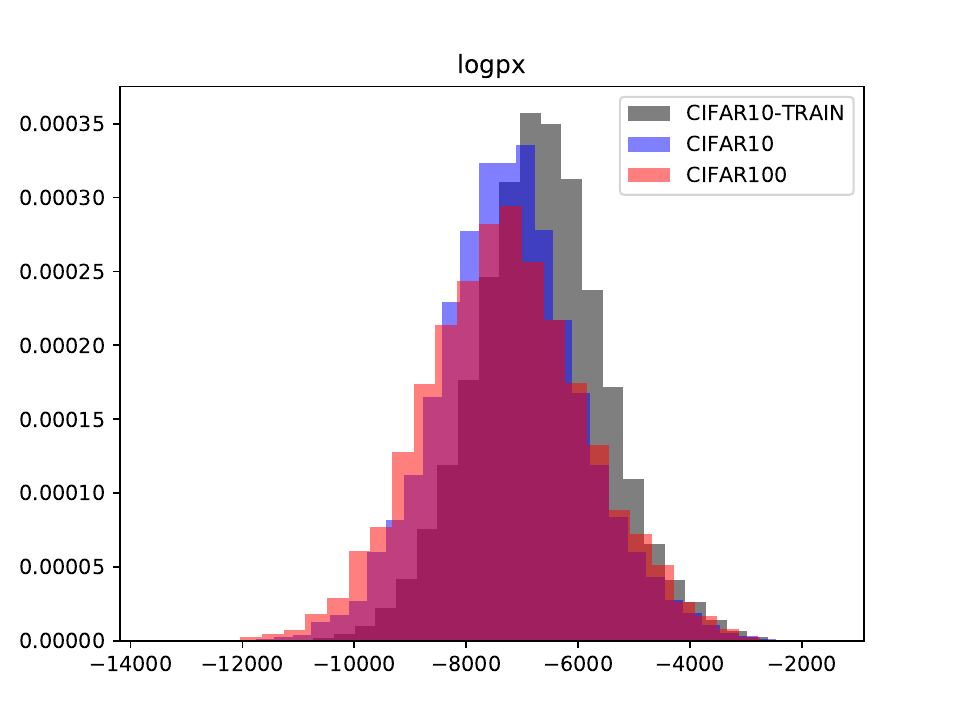}
		\end{minipage}
	}	
	\subfigure{
		\begin{minipage}[t]{4.2cm}
			\centering
			\includegraphics[width=4.2cm]{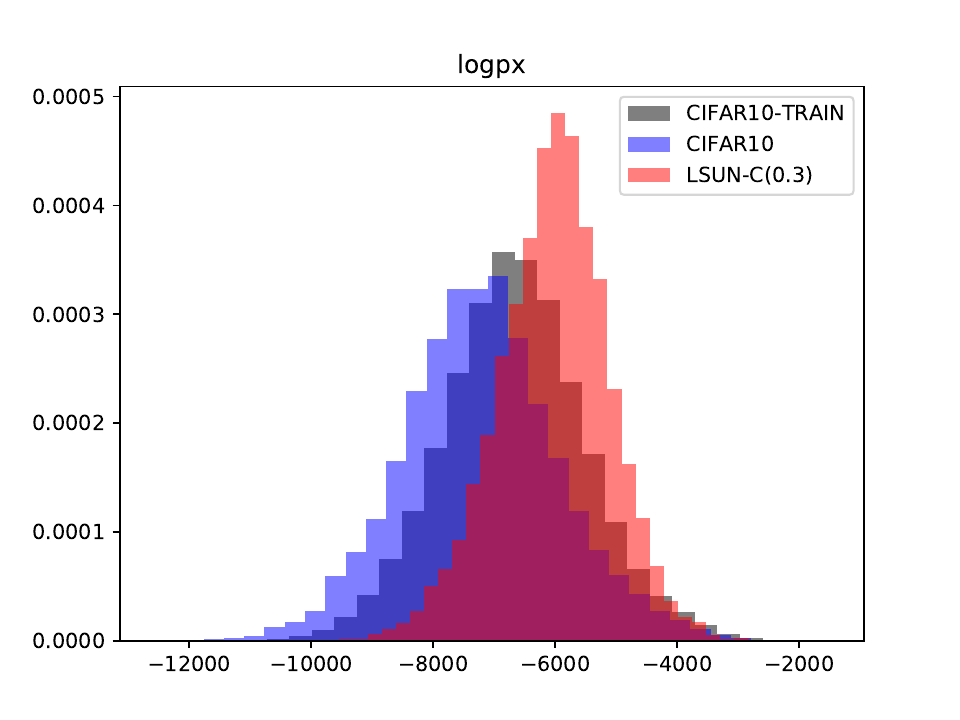}
		\end{minipage}
	}			
	
	\caption{Glow trained on CIFAR10. Histogram of $\log p(\bm{x})$.  We can manipulate the likelihood distribution of OOD dataset by adjusting the contrast. ``-C(\textit{k})'' means the dataset with adjusted contrast by a factor of \textit{k}. The ranges of $\log p(\bm{x})$ of CelebA and LSUN are too large to break the scale of the figure. For CIFAR10 vs Uniform, $\log p(\bm{x})$ of Uniform are too small. }
	\label{fig:logpx_glow_cifar10_vs_others}
\end{figure*}

\begin{figure*}[htbp]
	\centering
	\subfigure{
		\begin{minipage}[t]{4.2cm}
			\centering
			\includegraphics[width=4.2cm]{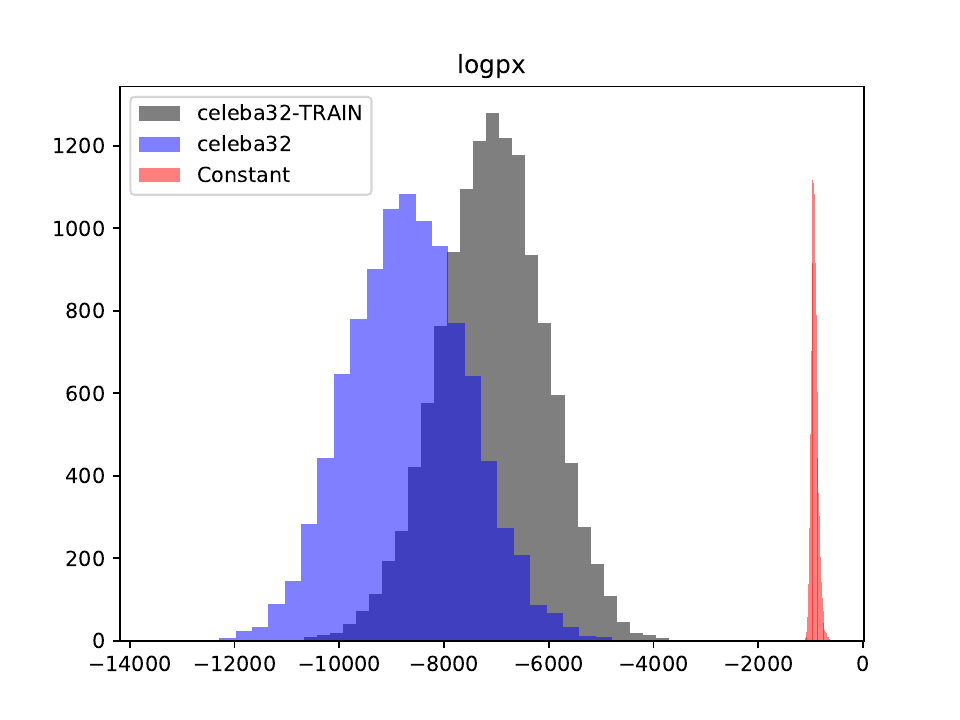}
		\end{minipage}
	}
	\subfigure{
		\begin{minipage}[t]{4.2cm}
			\centering
			\includegraphics[width=4.2cm]{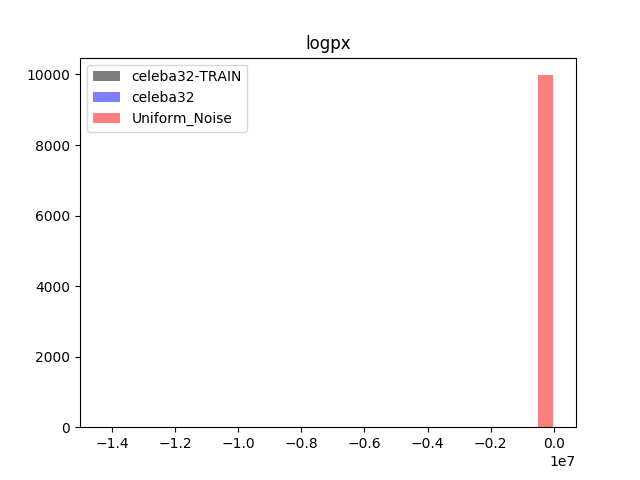}
		\end{minipage}
	}
	\subfigure{
		\begin{minipage}[t]{4.2cm}
			\centering
			\includegraphics[width=4.2cm]{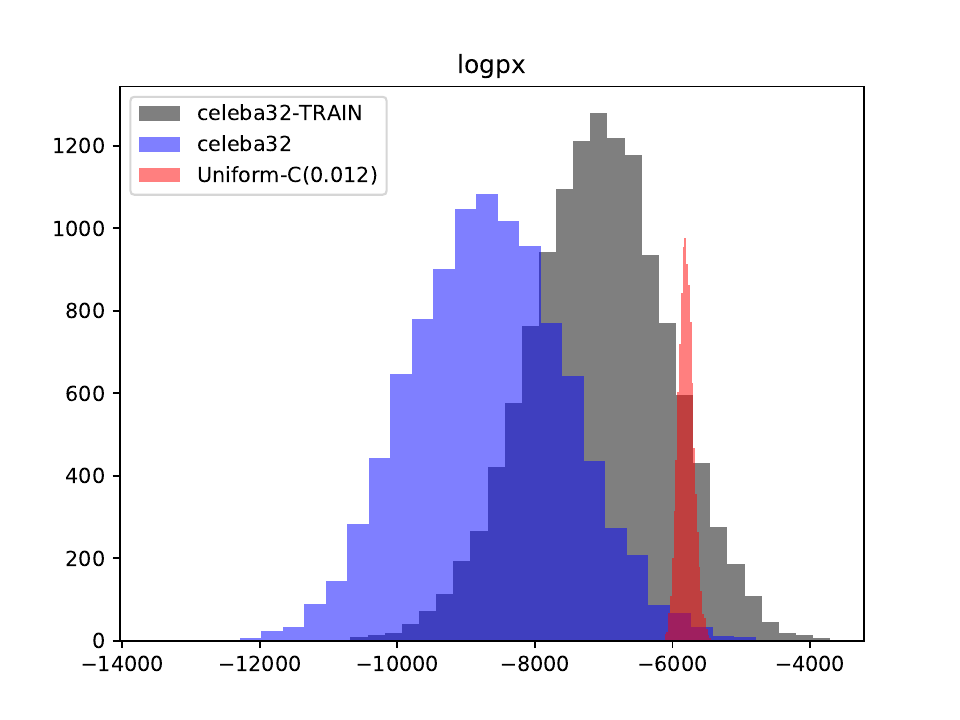}
		\end{minipage}
	}
	\subfigure{
		\begin{minipage}[t]{4.2cm}
			\centering
			\includegraphics[width=4.2cm]{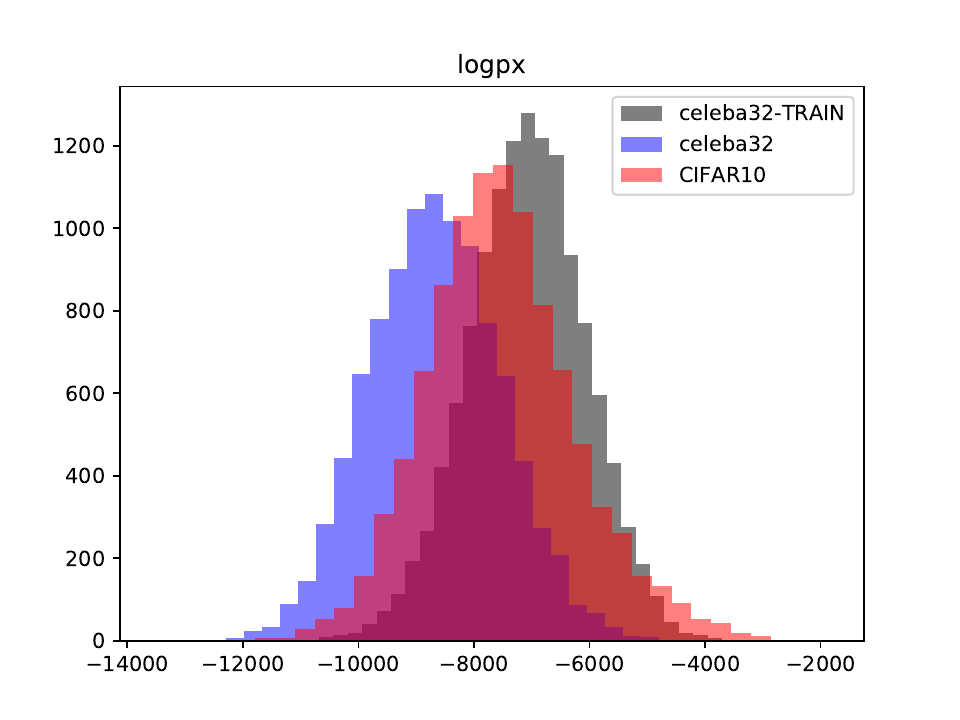}
		\end{minipage}
	}
	
	\subfigure{
		\begin{minipage}[t]{4.2cm}
			\centering
			\includegraphics[width=4.2cm]{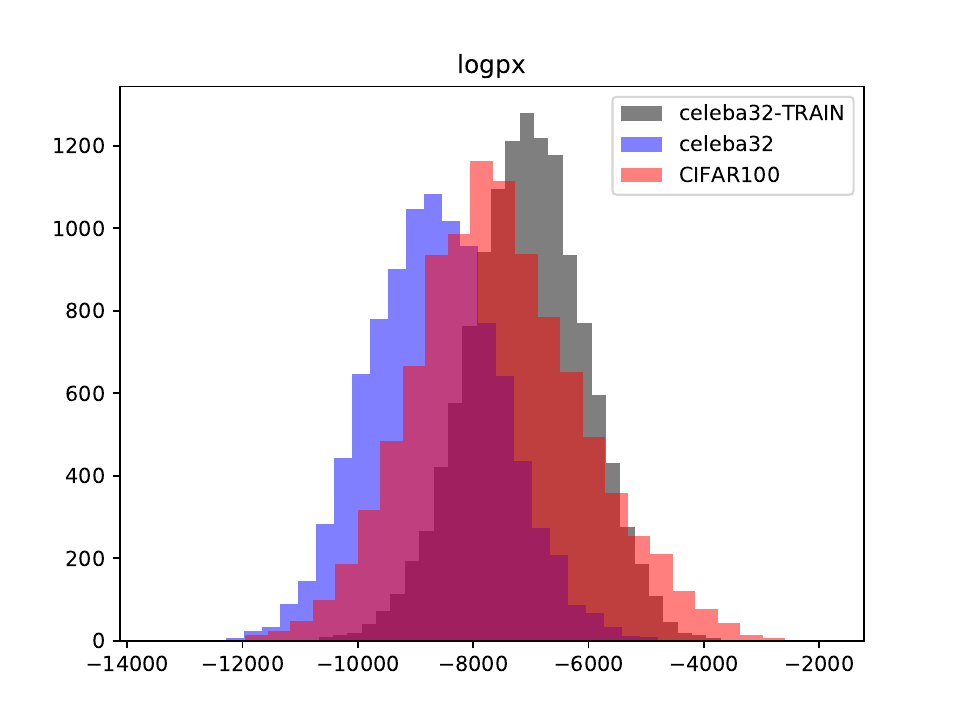}
		\end{minipage}
	}
	\vspace{-0pt}
	\subfigure{
		\begin{minipage}[t]{4.2cm}
			\centering
			\includegraphics[width=4.2cm]{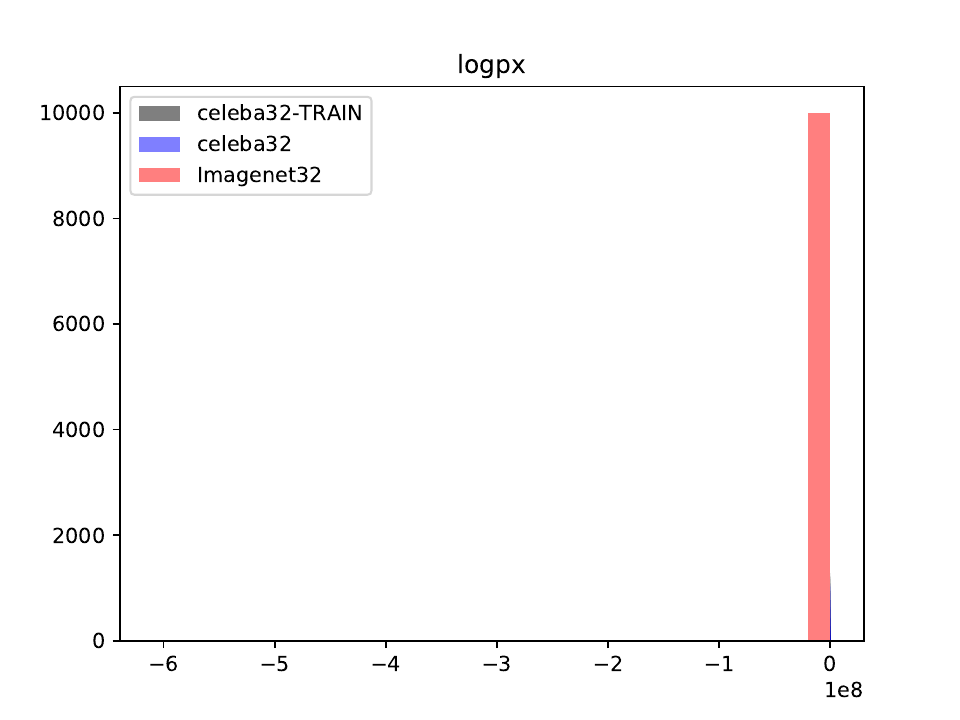}
		\end{minipage}
	}
	\subfigure{
		\begin{minipage}[t]{4.2cm}
			\centering
			\includegraphics[width=4.2cm]{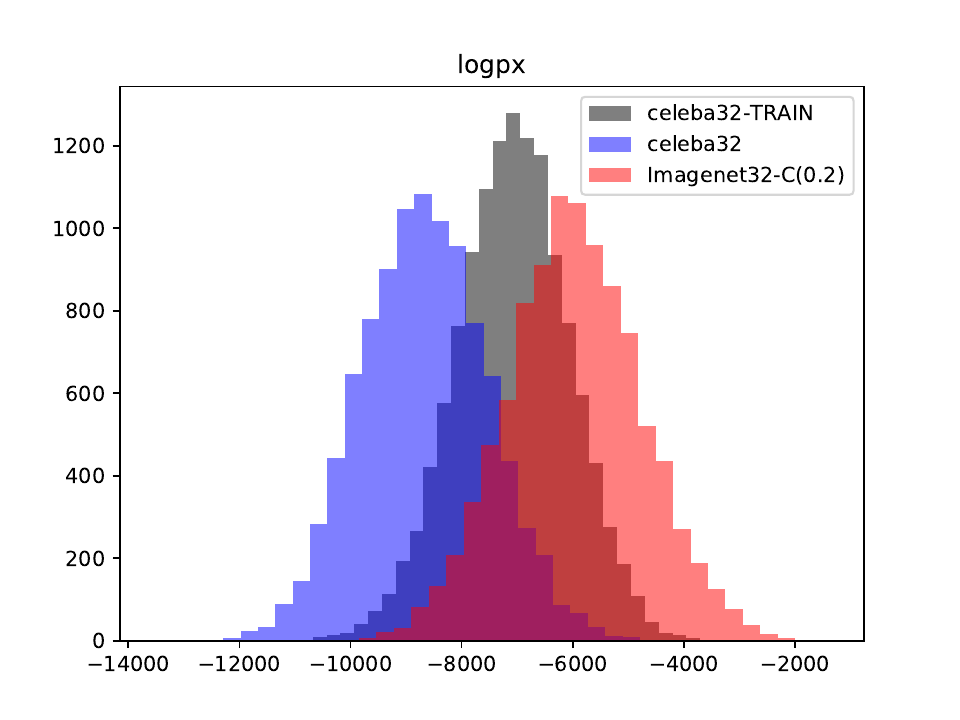}
		\end{minipage}
	}
	
	\vspace{-10pt}
	\subfigure{
		\begin{minipage}[t]{4.2cm}
			\centering
			\includegraphics[width=4.2cm]{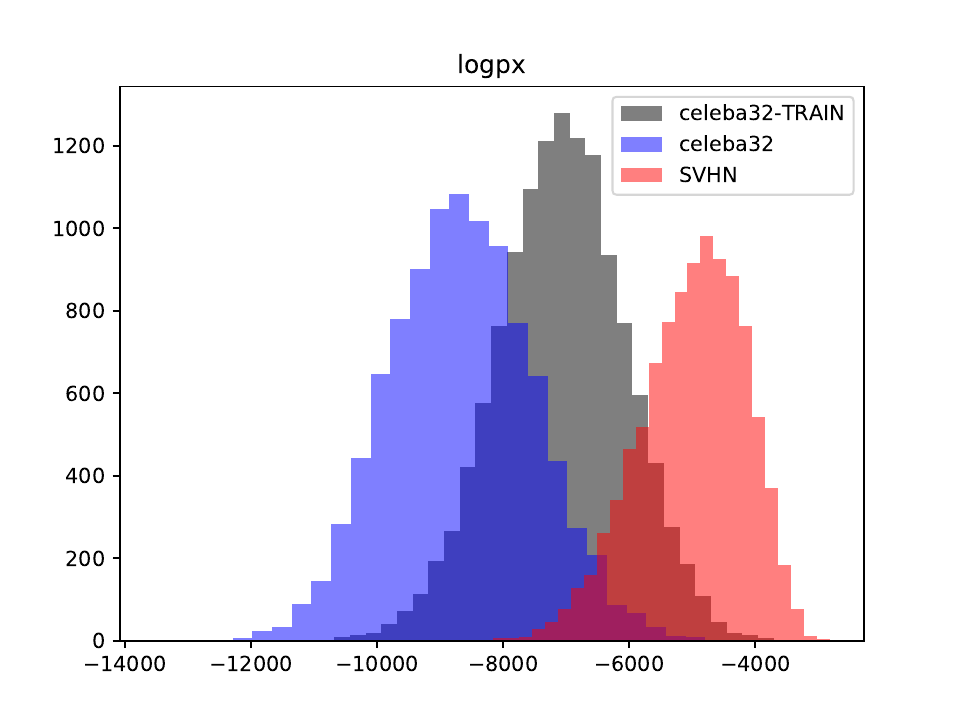}
		\end{minipage}
	}
	\subfigure{
		\begin{minipage}[t]{4.2cm}
			\centering
			\includegraphics[width=4.2cm]{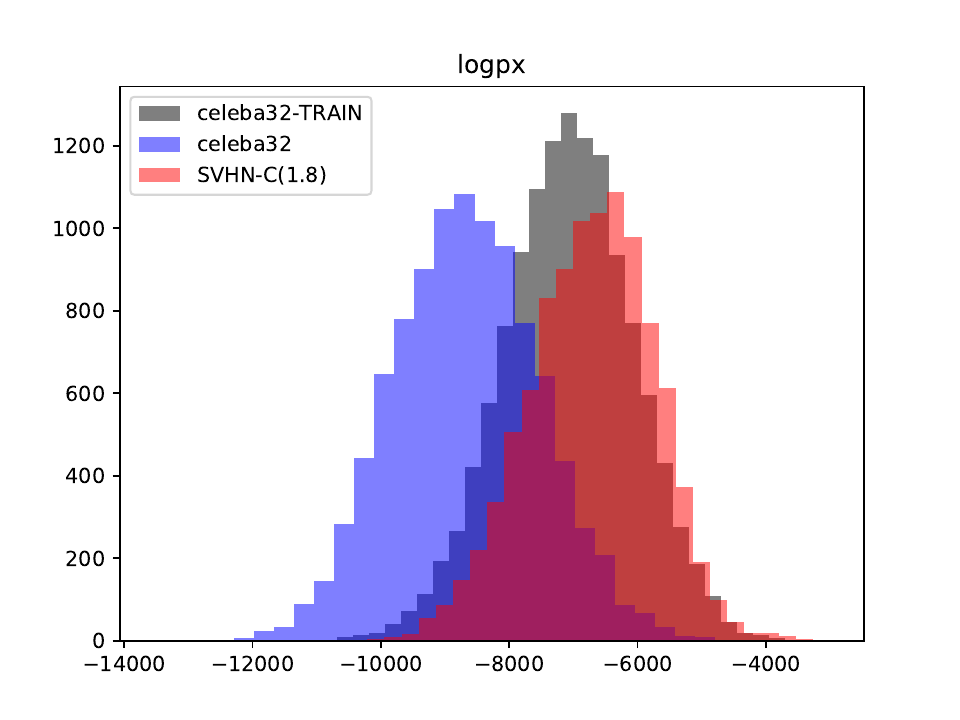}
		\end{minipage}
	}
	\subfigure{
		\begin{minipage}[t]{4.2cm}
			\centering
			\includegraphics[width=4.2cm]{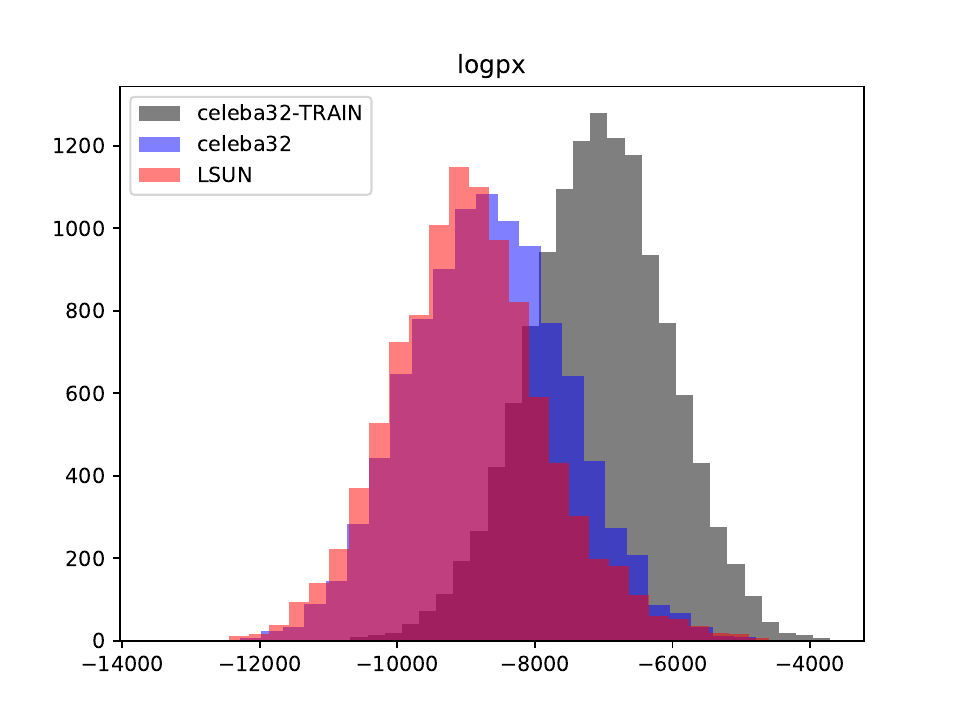}
		\end{minipage}
	}
	
	\caption{Glow trained on CelebA. Histogram of $\log p(\bm{x})$.  We can manipulate the likelihood distribution of OOD dataset by adjusting the contrast. ``-C(\textit{k})'' means the dataset with adjusted contrast by a factor of \textit{k}. It is hard to make the likelihoods of train and test split of CelebA fit well on the official Glow model. 
	}
	\label{fig:logpx_glow_celeba_vs_others}
\end{figure*}
\begin{figure*}[htbp]
	\centering
	\includegraphics[width=7cm]{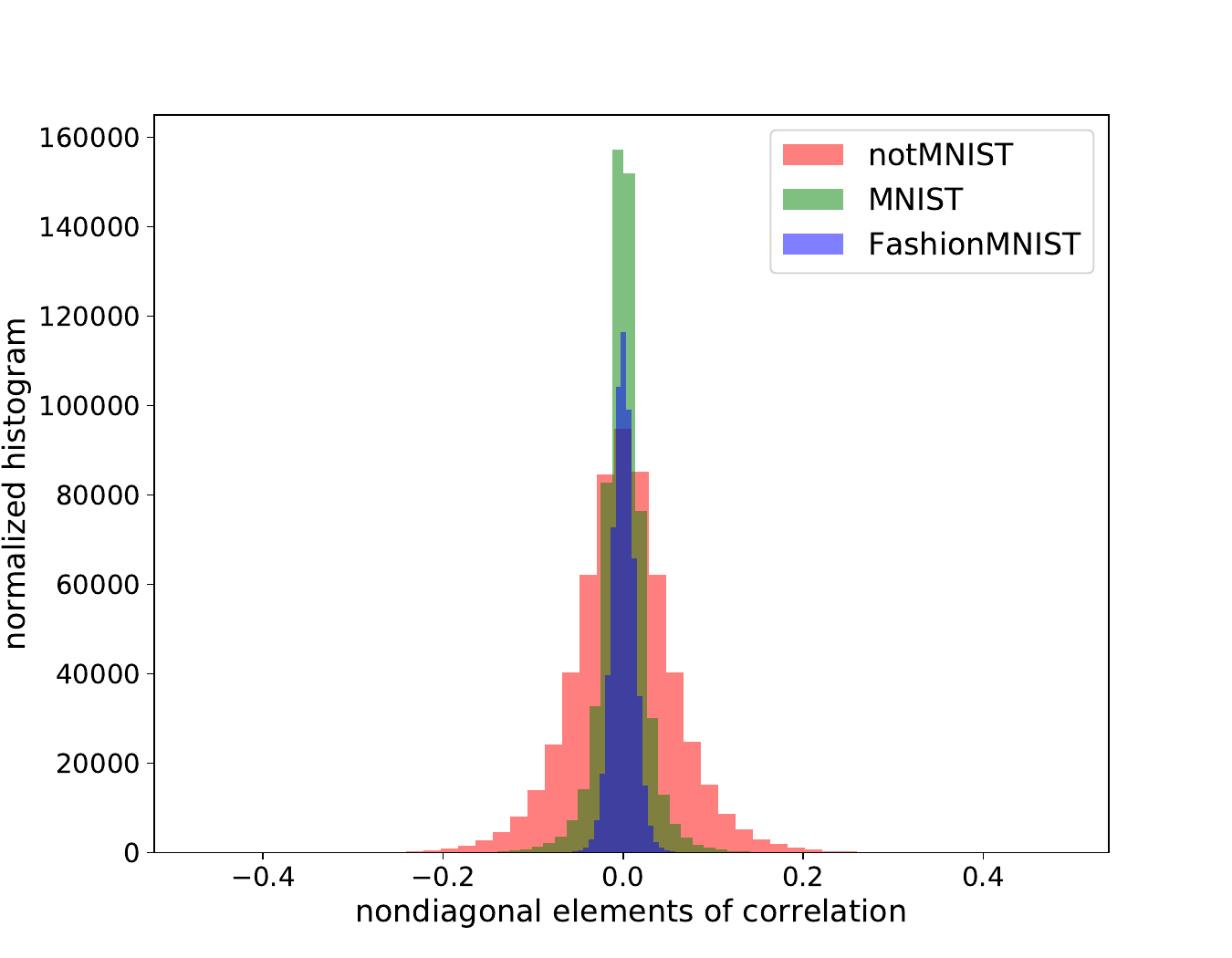}
	\caption{Glow trained on FashionMNIST and tested on MNIST/notMNIST. Histogram of non-diagonal elements in the correlation coefficient of representations.}
	\label{fig:nondiagonal_correlation_coefficient_fashionmnist_mnist_notmnist_glow} 
\end{figure*}

\begin{figure*}[htbp]
	\centering
	\includegraphics[width=10cm]{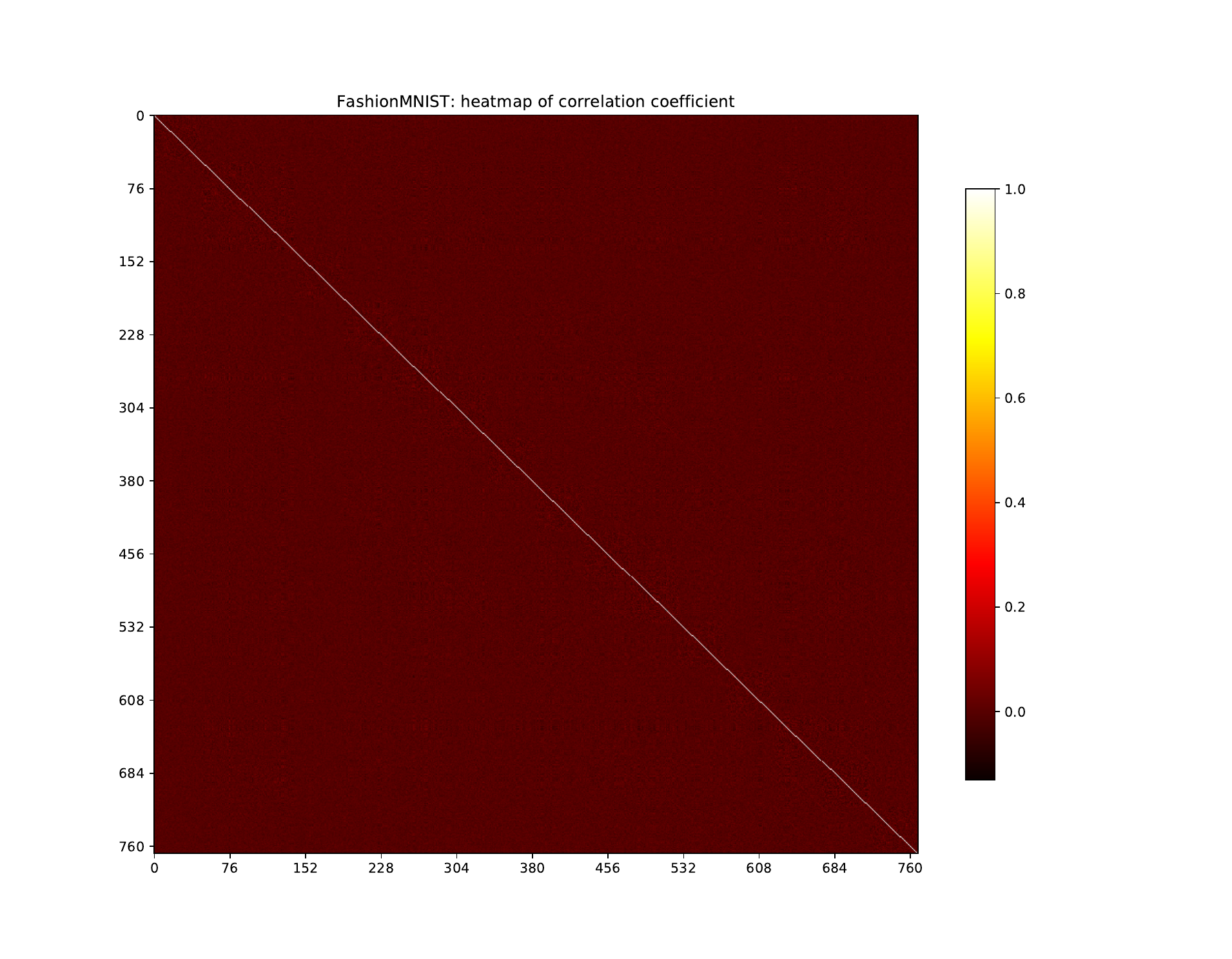}
	\caption{Glow trained on FashionMNIST. Heatmap of correlation of FashionMNIST representations.}
	\label{fig:heatmap_fashionmnist_correlation_coefficient}
\end{figure*}

\begin{figure*}[htbp]
	\vspace{-20pt}
	\centering
	\includegraphics[width=10cm]{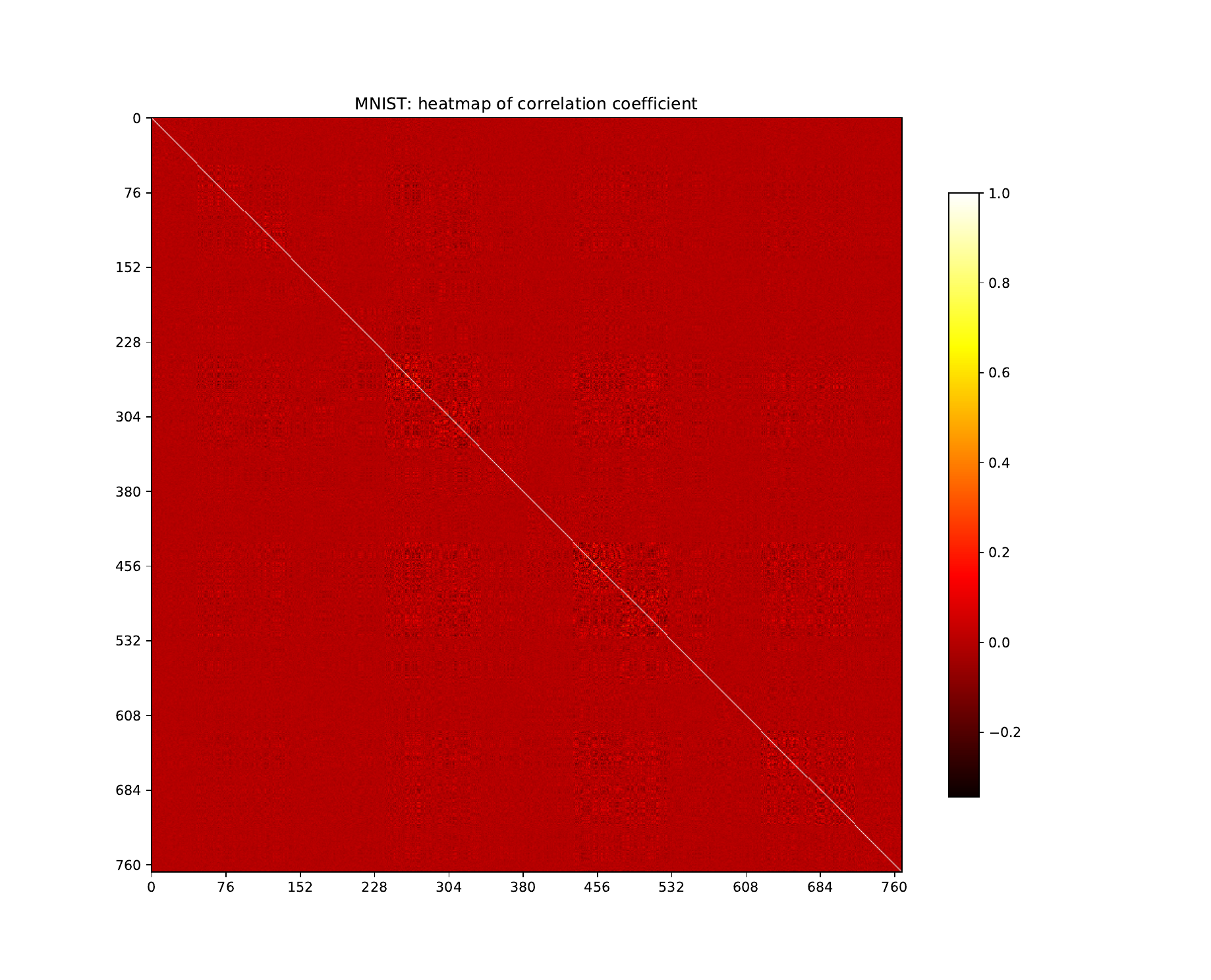}
	\caption{Glow trained on FashionMNIST. Heatmap of correlation of MNIST representations.}
	\label{fig:heatmap_mnist_correlation_coefficient}
\end{figure*}

\begin{figure*}[htbp]
	\centering
	\includegraphics[width=10cm]{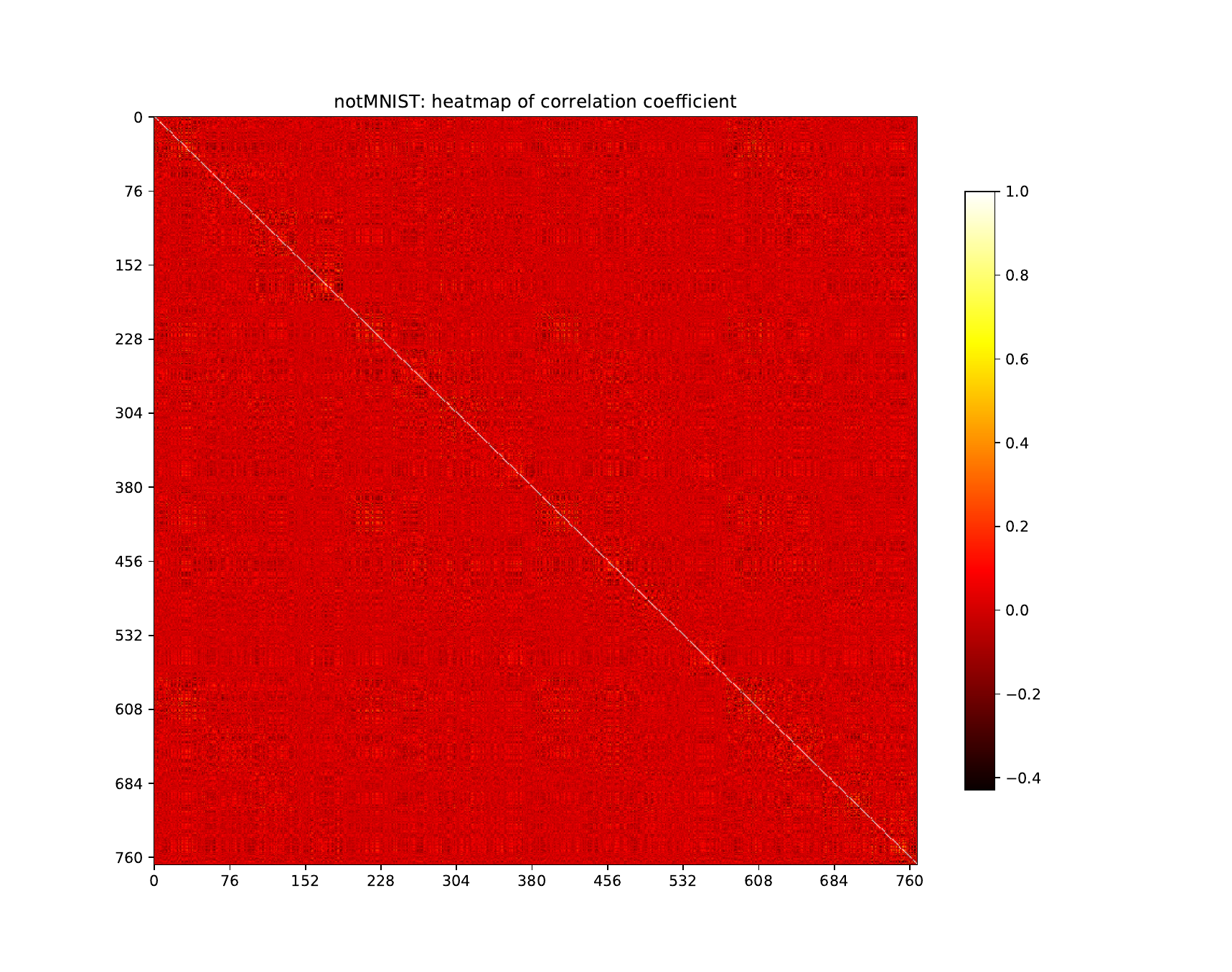}
	\caption{Glow trained on FashionMNIST. Heatmap of correlation of notMNIST representations.}
	\label{fig:heatmap_notmnist_correlation_coefficient}
\end{figure*}
%

\begin{figure*}[htbp]
	\centering
	\subfigure{
		\begin{minipage}[t]{4.2cm}
			\centering
			\includegraphics[width=4.2cm]{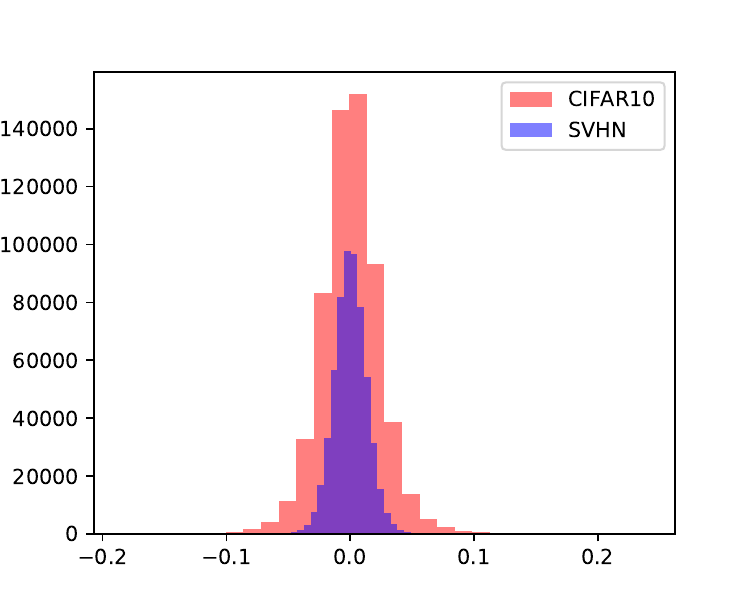}
		\end{minipage}
	}
	\subfigure{
		\begin{minipage}[t]{4.2cm}
			\centering
			\includegraphics[width=4.2cm]{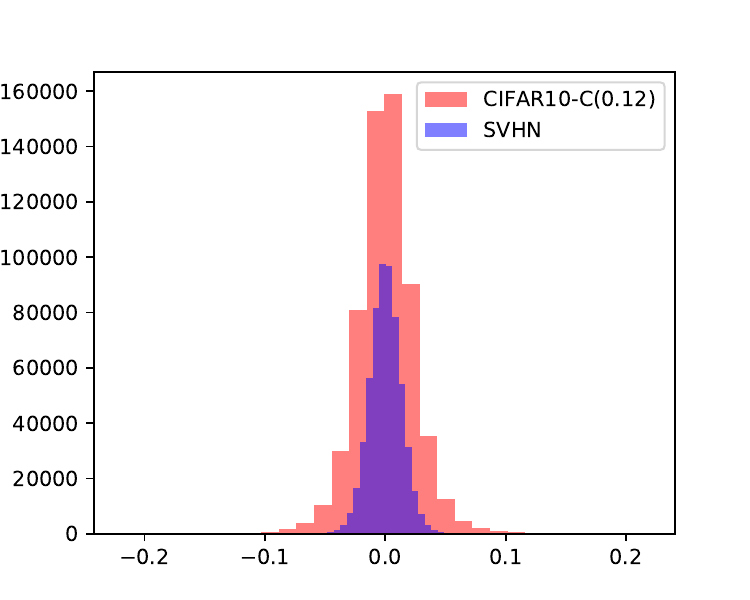}
		\end{minipage}
	}
	\subfigure{
		\begin{minipage}[t]{4.2cm}
			\centering
			\includegraphics[width=4.2cm]{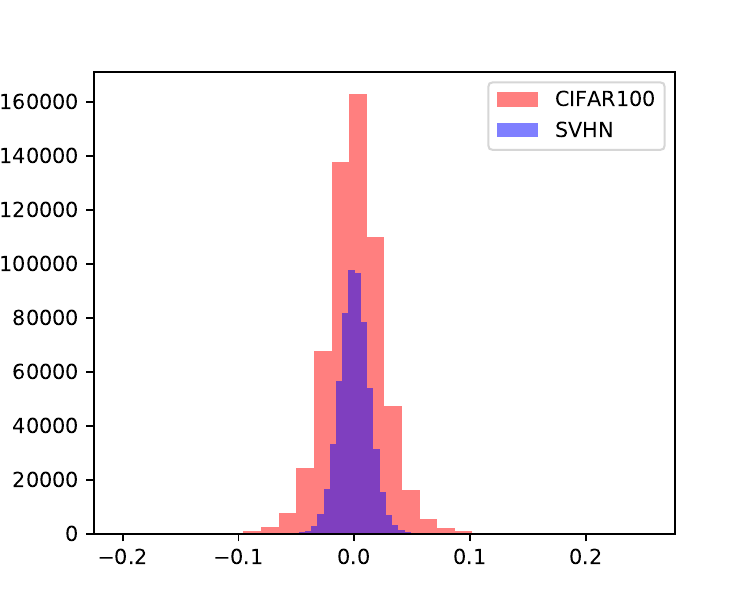}
		\end{minipage}
	}
	\subfigure{
		\begin{minipage}[t]{4.2cm}
			\centering
			\includegraphics[width=4.2cm]{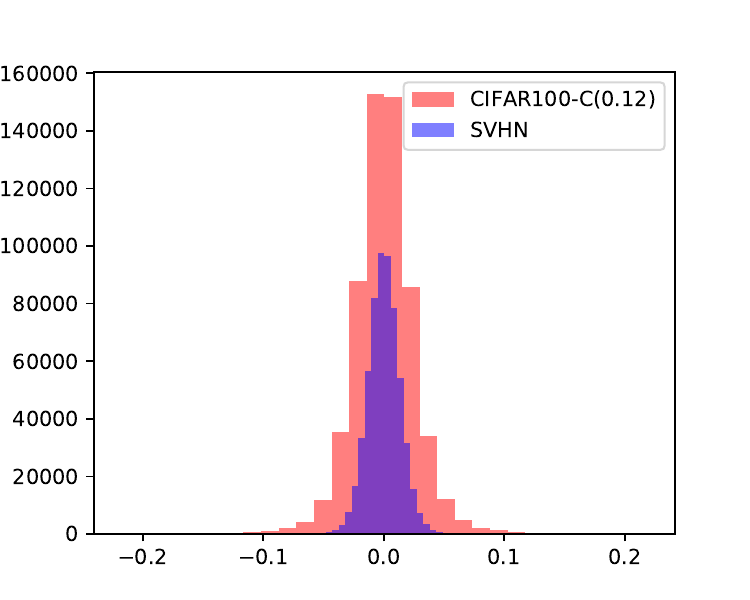}
		\end{minipage}
	}
	
	\vspace{-10pt}
	\subfigure{
		\begin{minipage}[t]{4.2cm}
			\centering
			\includegraphics[width=4.2cm]{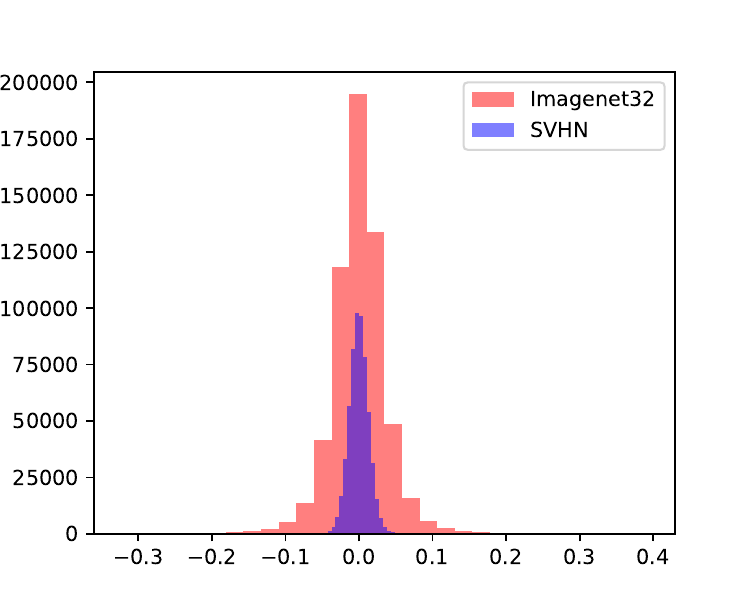}
		\end{minipage}
	}
	\subfigure{
		\begin{minipage}[t]{4.2cm}
			\centering
			\includegraphics[width=4.2cm]{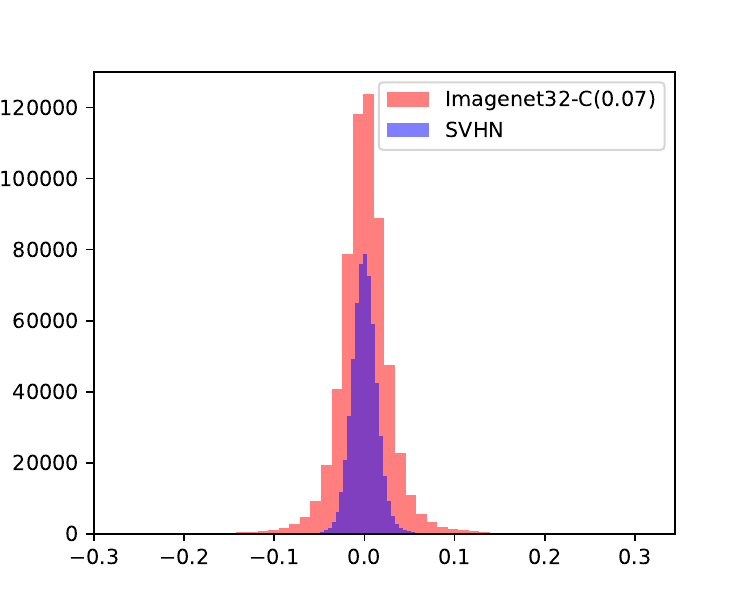}
		\end{minipage}
	}
	\subfigure{
		\begin{minipage}[t]{4.2cm}
			\centering
			\includegraphics[width=4.2cm]{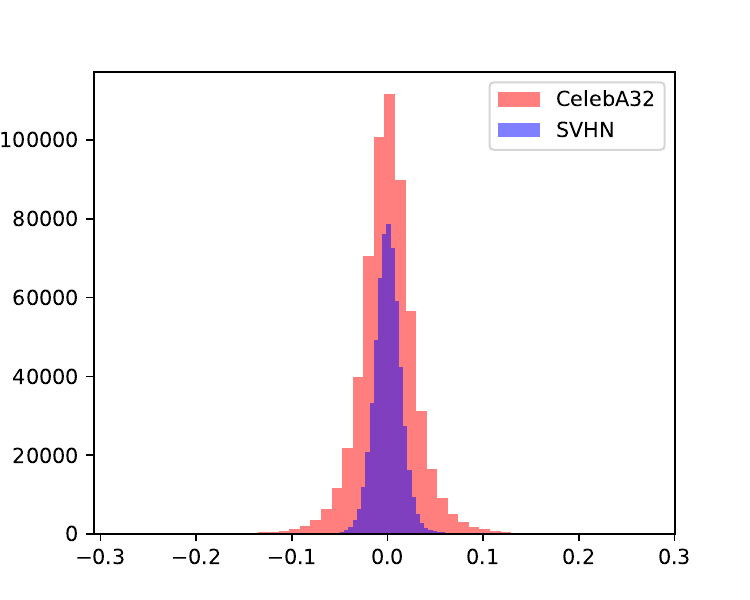}
		\end{minipage}
	}
	
	\subfigure{
		\begin{minipage}[t]{4.2cm}
			\centering
			\includegraphics[width=4.2cm]{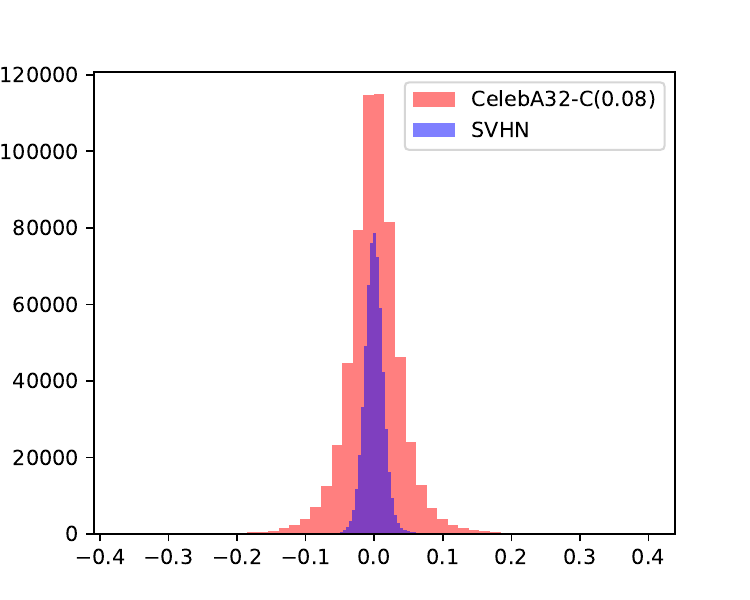}
		\end{minipage}
	}
	\subfigure{
		\begin{minipage}[t]{4.2cm}
			\centering
			\includegraphics[width=4.2cm]{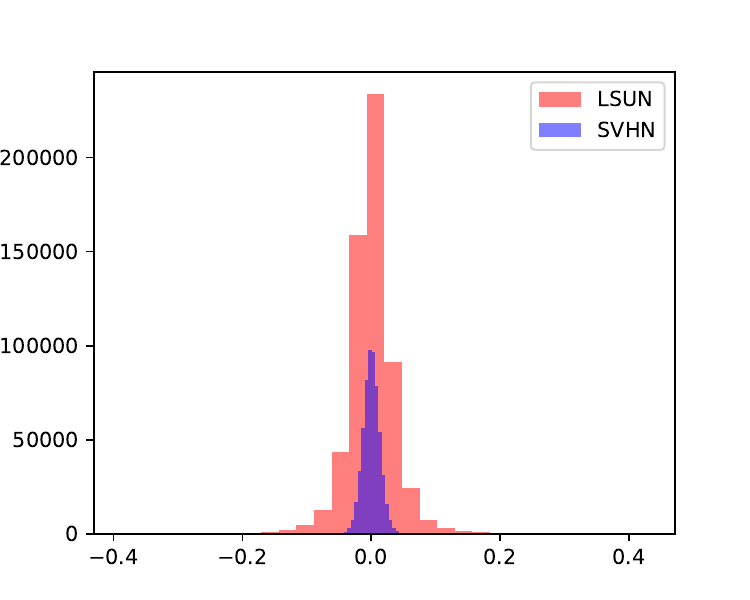}
		\end{minipage}
	}
	\subfigure{
		\begin{minipage}[t]{4.2cm}
			\centering
			\includegraphics[width=4.2cm]{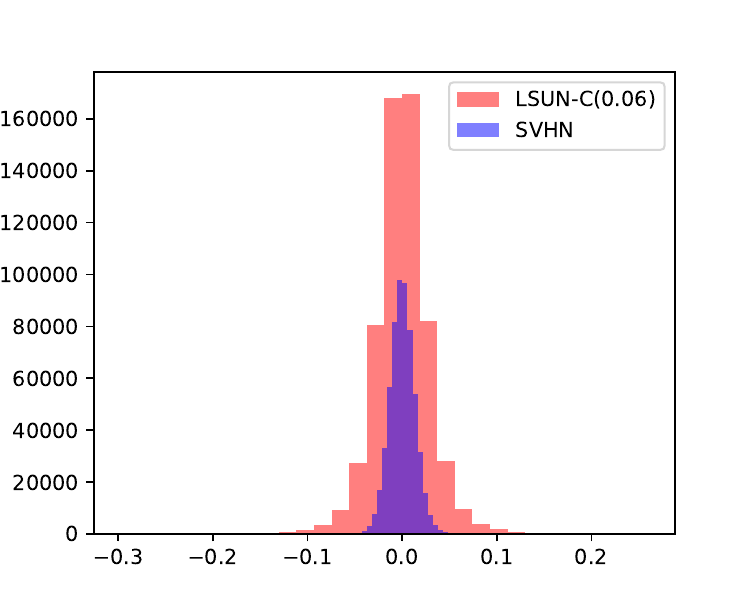}
		\end{minipage}
	}
	\caption{Glow trained on SVHN. Histogram of non-diagonal elements of correlation of representations.}
	\label{fig:histo_correlation_train_svhn_test_cifar}
\end{figure*}

%
%

%

\begin{figure*}[htbp]
	\centering
	\vspace{-0pt}
	\subfigure{
		\begin{minipage}[t]{4.2cm}
			\centering
			\includegraphics[width=4.2cm]{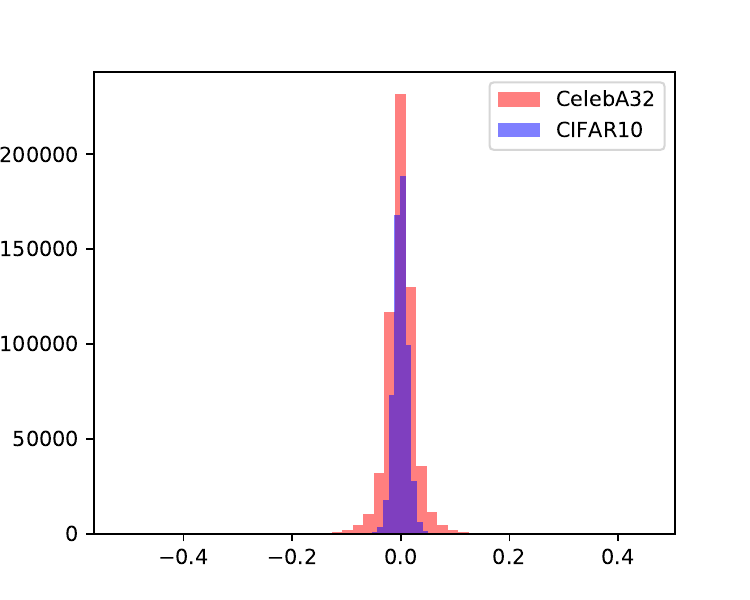}
		\end{minipage}
	}
	\subfigure{
		\begin{minipage}[t]{4.2cm}
			\centering
			\includegraphics[width=4.2cm]{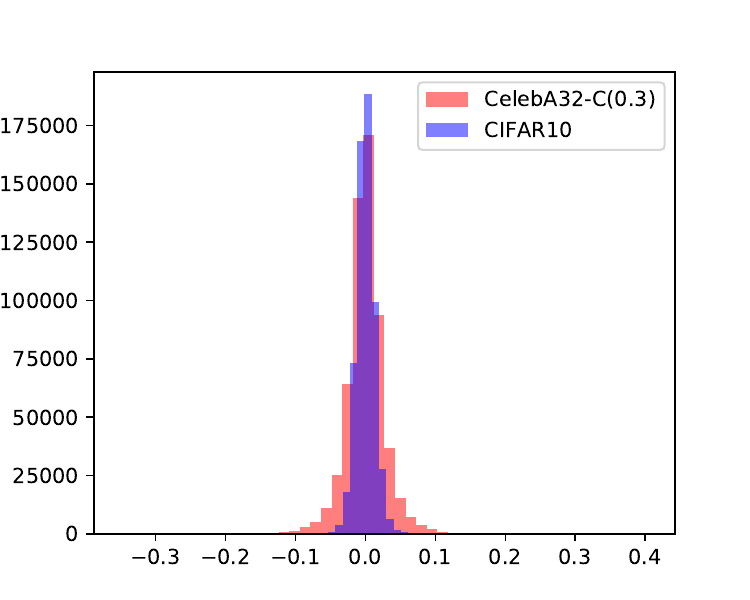}
		\end{minipage}
	}
	\subfigure{
		\begin{minipage}[t]{4.2cm}
			\centering
			\includegraphics[width=4.2cm]{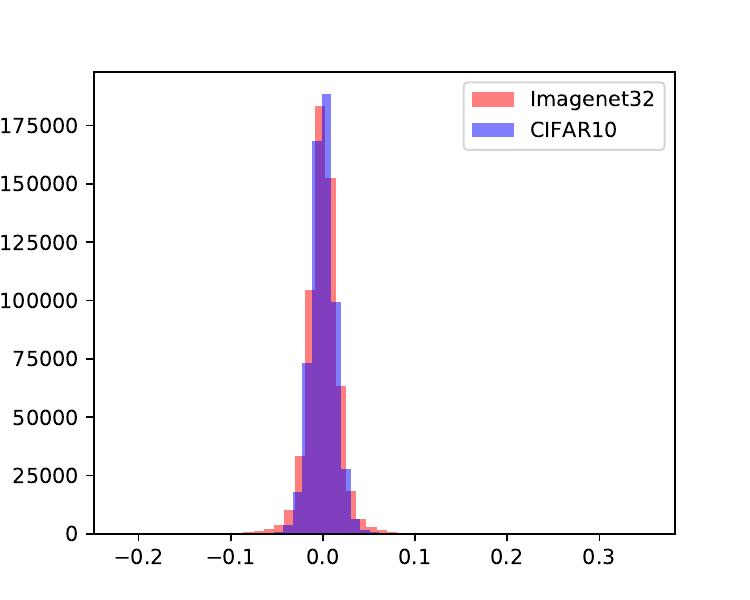}
		\end{minipage}
	}
	\subfigure{
		\begin{minipage}[t]{4.2cm}
			\centering
			\includegraphics[width=4.2cm]{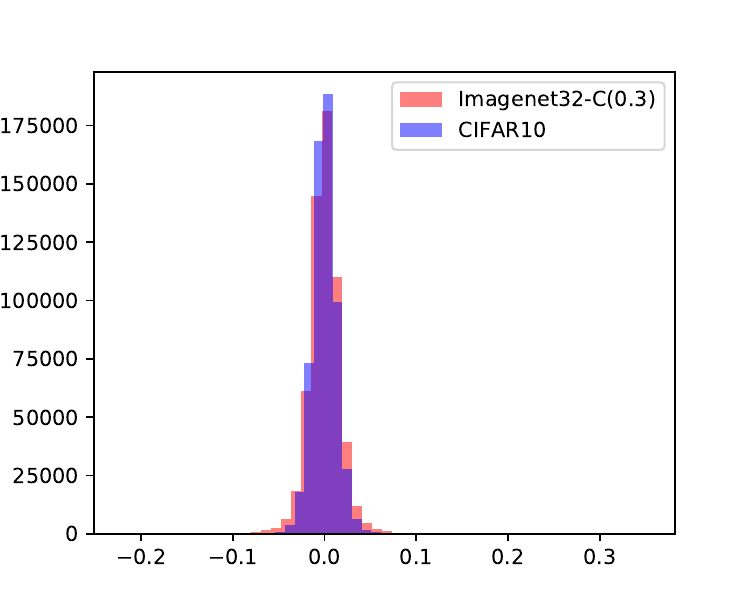}
		\end{minipage}
	}
	
	\vspace{-10pt}
	\subfigure[]{
		\begin{minipage}[t]{4.2cm}
			\centering
			\includegraphics[width=4.2cm]{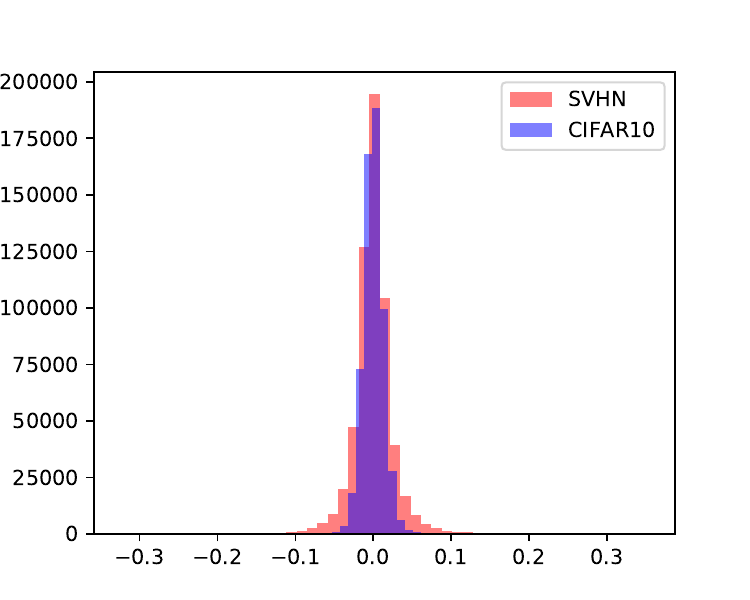}
		\end{minipage}
	}
	\subfigure[]{
		\begin{minipage}[t]{4.2cm}
			\centering
			\includegraphics[width=4.2cm]{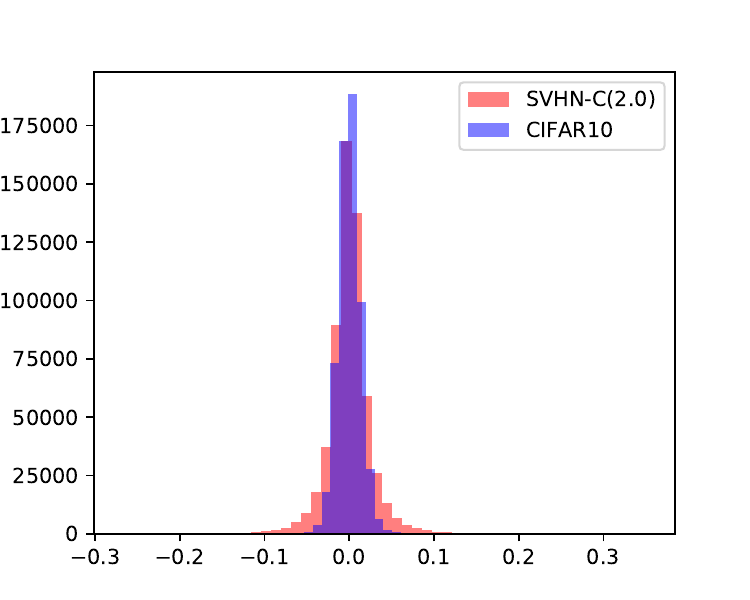}
		\end{minipage}
	}	
	\vspace{-0pt}
	\subfigure[]{
		\begin{minipage}[t]{4.2cm}
			\centering
			\includegraphics[width=4.2cm]{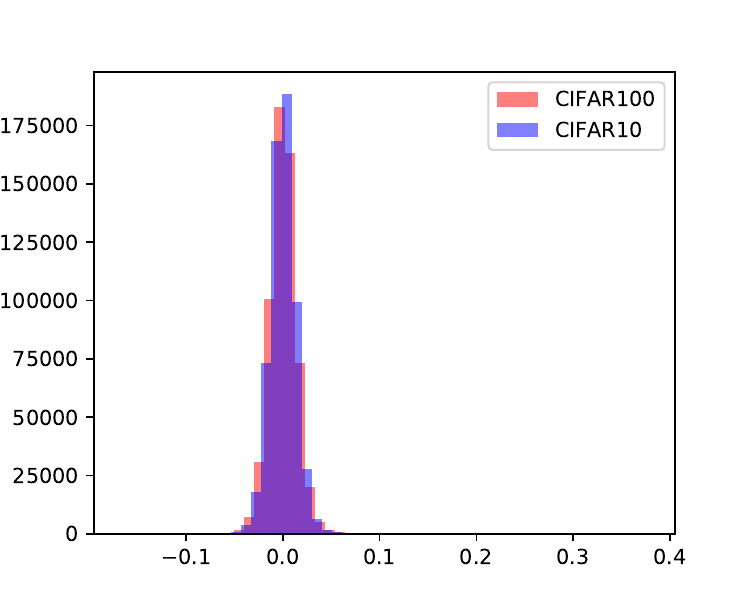}
		\end{minipage}
	}	
	\vspace{-0pt}
	\caption{Glow trained on CIFAR10. Histogram of non-diagonal elements of correlation of representations.}
	\label{fig:histo_correlation_train_cifar10_test_others}
\end{figure*}

\begin{figure*}[htbp]
	\centering
	\subfigure{
		\begin{minipage}[t]{4.2cm}
			\centering
			\includegraphics[width=4.2cm]{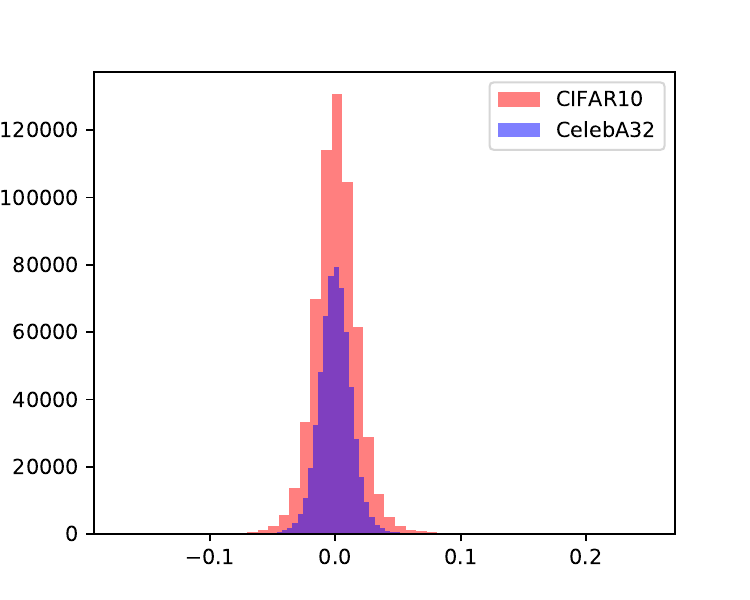}
		\end{minipage}
	}
	\subfigure{
		\begin{minipage}[t]{4.2cm}
			\centering
			\includegraphics[width=4.2cm]{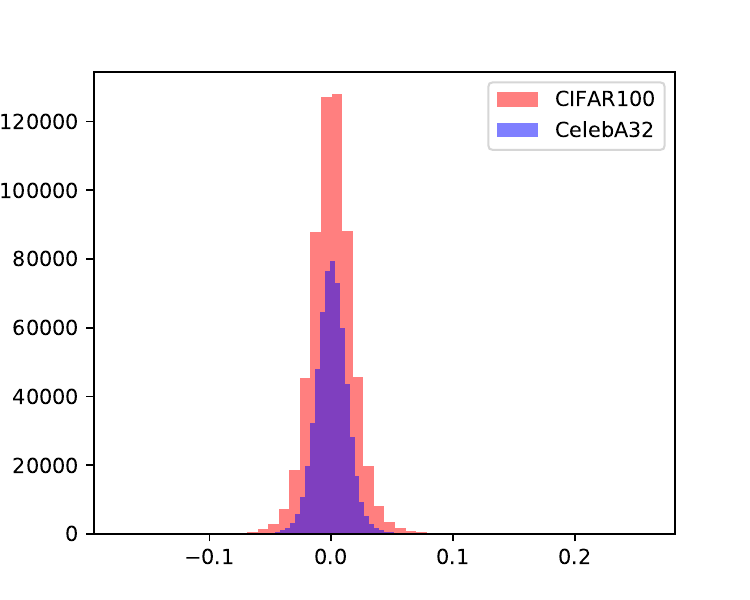}
		\end{minipage}
	}
	\subfigure{
		\begin{minipage}[t]{4.2cm}
			\centering
			\includegraphics[width=4.2cm]{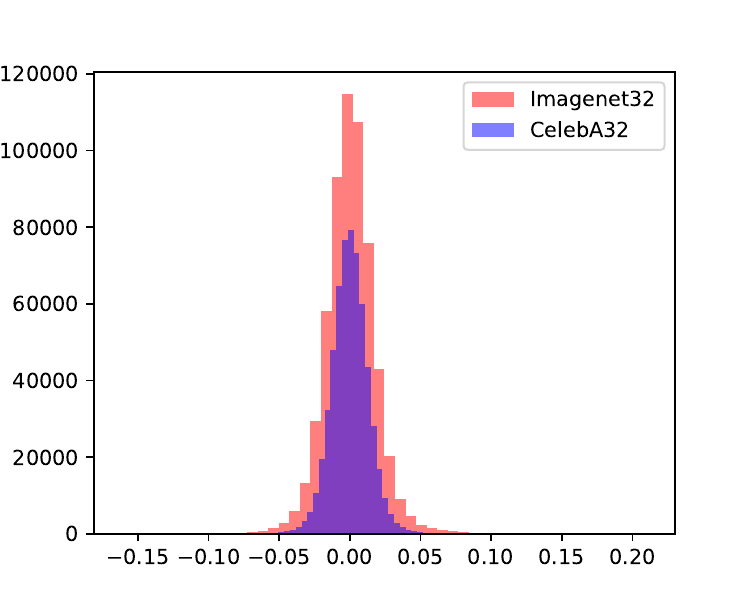}
		\end{minipage}
	}
	
	\vspace{-10pt}
	\subfigure{
		\begin{minipage}[t]{4.2cm}
			\centering
			\includegraphics[width=4.2cm]{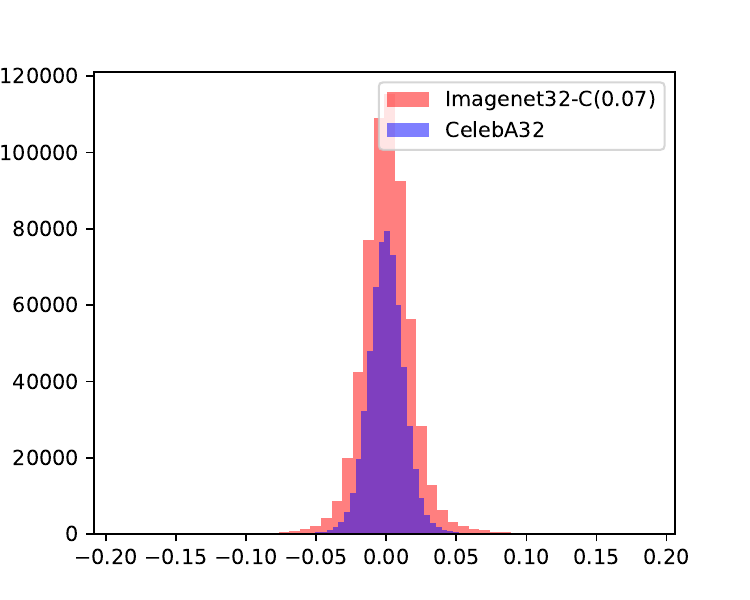}
		\end{minipage}
	}
	\subfigure{
		\begin{minipage}[t]{4.2cm}
			\centering
			\includegraphics[width=4.2cm]{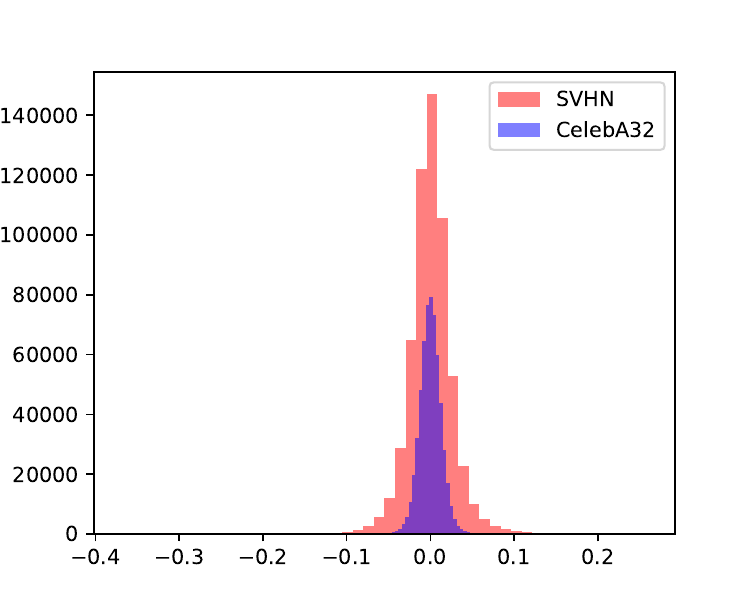}
		\end{minipage}
	}
	\subfigure{
		\begin{minipage}[t]{4.2cm}
			\centering
			\includegraphics[width=4.2cm]{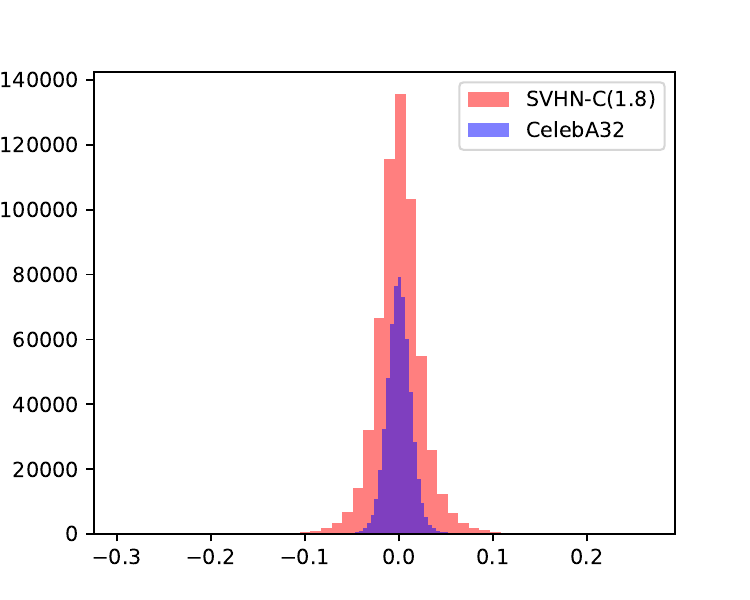}
		\end{minipage}
	}	
	\caption{Glow trained on CelebA. Histogram of non-diagonal elements of correlation of representations.}
	\label{fig:histo_correlation_train_celeba_test_others}
\end{figure*}

\begin{figure*}[htbp]
	\centering
	\subfigure[]{
		\begin{minipage}[t]{8cm}
			\centering
			\includegraphics[width=8cm]{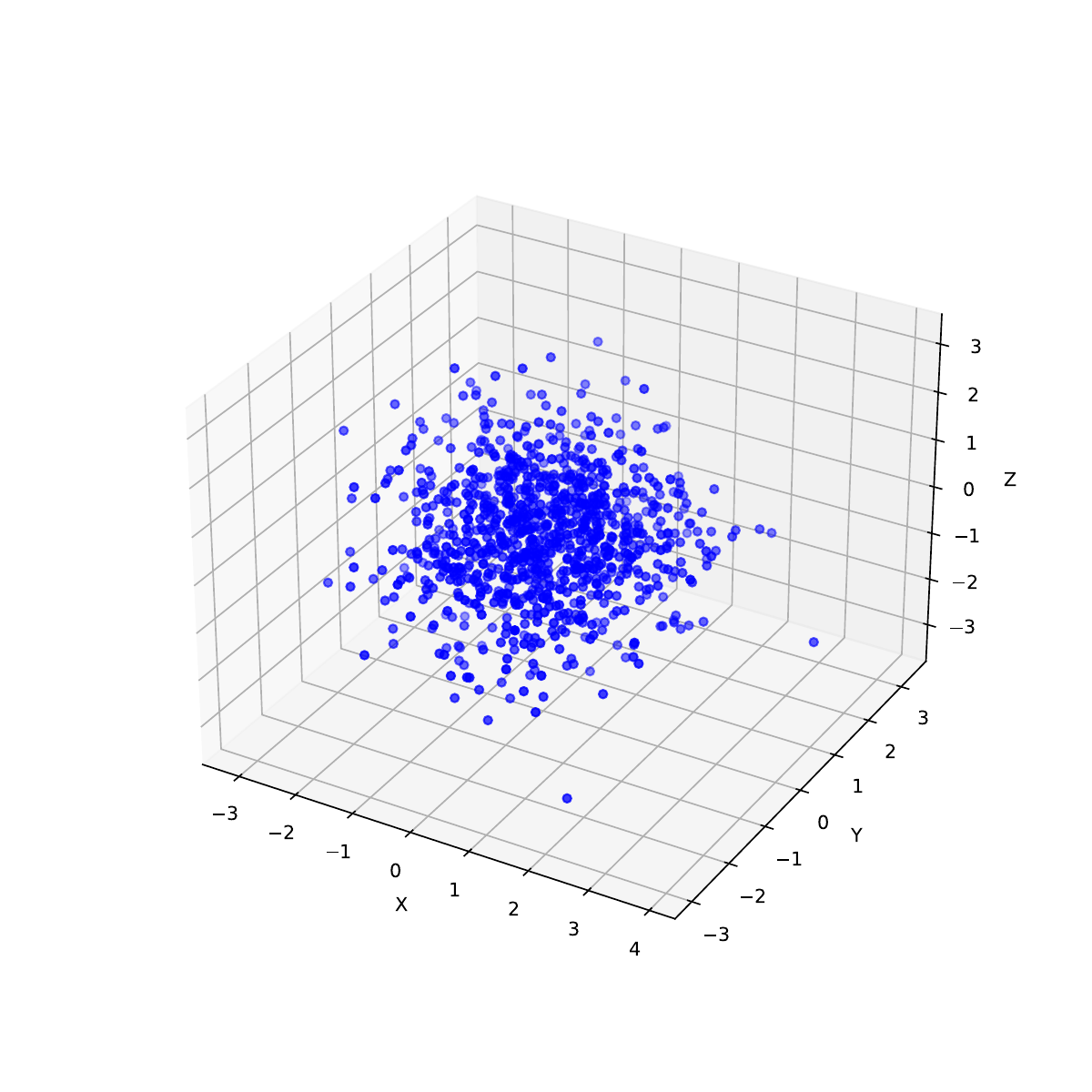}
			\label{fig:3d_nodirection}
		\end{minipage}
	}
	\subfigure[]{
		\begin{minipage}[t]{8cm}
			\centering
			\includegraphics[width=8cm]{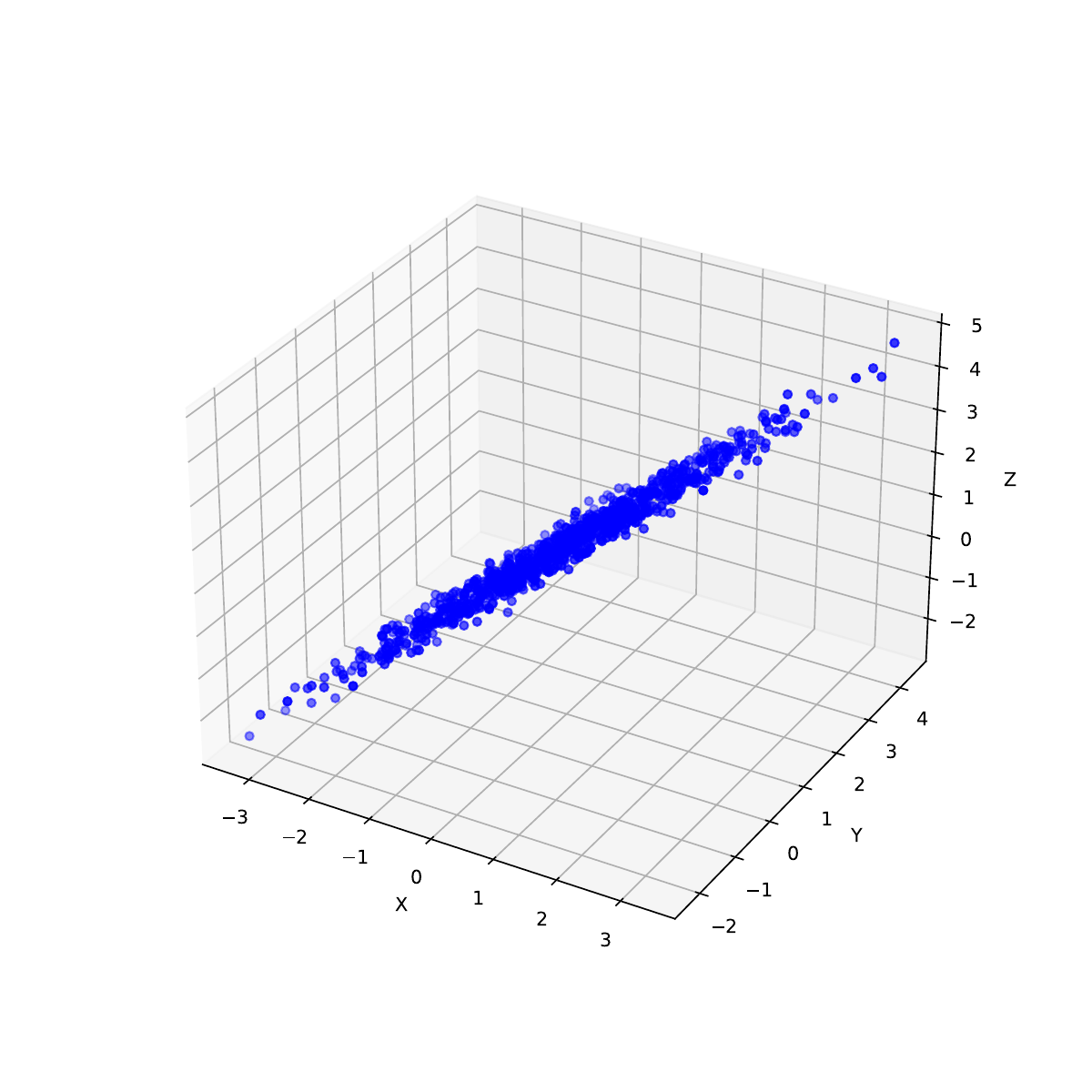}
			\label{fig:3d_direction}
		\end{minipage}
	}
	
	\caption{Samples from 3-d Gaussian distribution $\n(\mu, \Sigma)$. The mean $\mu$ and covariance matrix $\Sigma$ determines where the data locate in. (a) $\mu=(0,0,0)$,
		$\Sigma=((1,0,0)^{\top},(0,1,0)^{\top},(0,0,1)^{\top})$. (b) $\mu=(0,1,1)$, $\Sigma=((1,0.98,0.98)^{\top},(0.98,1,0.98)^{\top},(0.98,0.98,1)^{\top})$.
	}
	\label{fig:3d_gaussian}
\end{figure*}
%

\begin{figure*}%
	\centering
	\includegraphics[width=4.5cm]{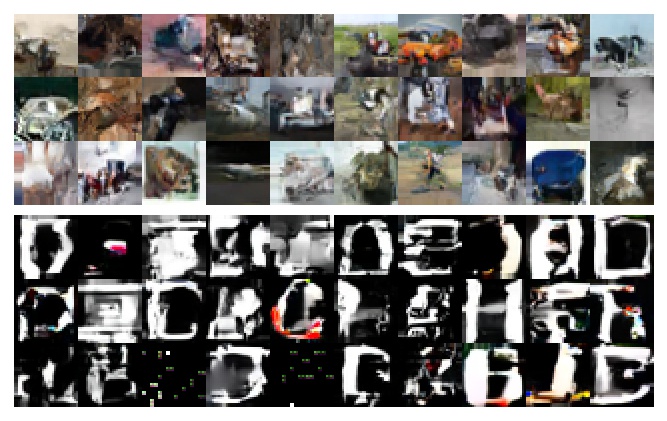}
	\vspace{-8pt}
	\caption{Glow trained on CIFAR-10. Generated images from prior (up), fitted Gaussian distribution from the representations of OOD dataset notMNIST (down).}
	\label{fig:sampled_notmnist_trained_on_fashionmnist}
\end{figure*}

\begin{figure*}[htbp]
	\centering
	\subfigure[]{
		\begin{minipage}[t]{4cm}
			\centering
			\includegraphics[width=4cm]{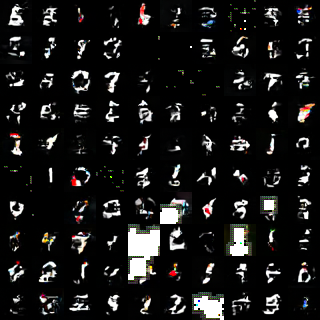}
			\label{fig:sampled_mnist_trained_on_cifar10}
		\end{minipage}
	}
	\subfigure[]{
		\begin{minipage}[t]{4cm}
			\centering
			\includegraphics[width=4cm]{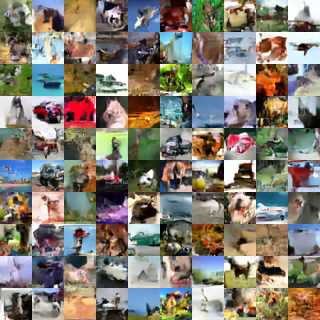}
			\label{fig:sampled_cifar100_trained_on_cifar10}
		\end{minipage}
	}
	\subfigure[]{
		\begin{minipage}[t]{4cm}
			\centering
			\includegraphics[width=4cm]{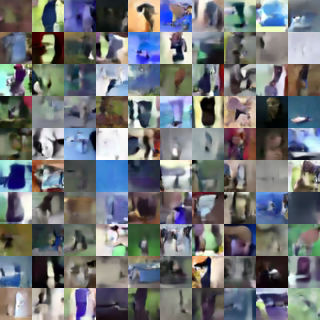}
			\label{fig:sampled_svhn_trained_on_cifar10}
		\end{minipage}
	}
	
	\subfigure[]{
		\begin{minipage}[t]{4cm}
			\centering
			\includegraphics[width=4cm]{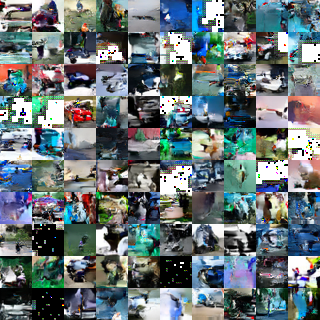}
			\label{fig:sampled_imagenet32_trained_on_cifar10}
		\end{minipage}
	}
	\subfigure[]{
		\begin{minipage}[t]{4cm}
			\centering
			\includegraphics[width=4cm]{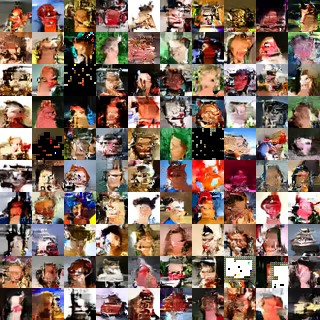}
			\label{fig:sampled_celeba_trained_on_cifar10}
		\end{minipage}
	}
	
	\caption{Glow trained on CIFAR10. Generated images according to the fitted Gaussian distribution from representations of (a) MNIST; (b) CIFAR100; (c) SVHN; (d) ImageNet32; (e) CelebA. We replicate MNIST into three channels and pad zeros for consistency. These results demonstrate that the covariance of representations contains important information of an OOD dataset.}
	\label{fig:sample_images_trained_on_cifar10}
\end{figure*}

\begin{figure*}[htbp]
	\centering
	\subfigure[]{
		\begin{minipage}[t]{4cm}
			\centering
			\includegraphics[width=4cm]{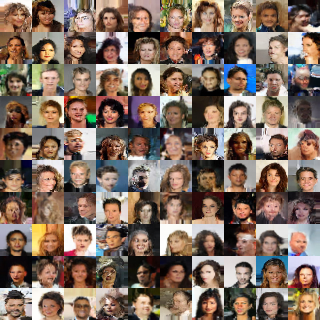}
			\label{fig:sampled_celeba_trained_on_celeba}
		\end{minipage}
	}
	\subfigure[]{
		\begin{minipage}[t]{4cm}
			\centering
			\includegraphics[width=4cm]{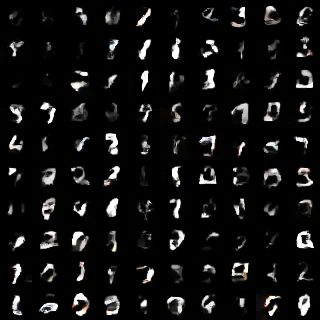}
			\label{fig:sampled_mnist_trained_on_celeba}
		\end{minipage}
	}
	\subfigure[]{
		\begin{minipage}[t]{4cm}
			\centering
			\includegraphics[width=4cm]{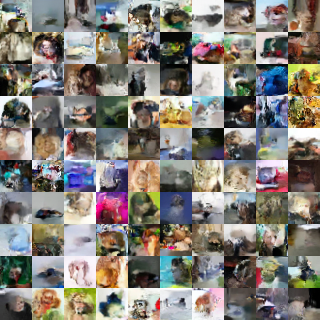}
			\label{fig:sampled_cifar10_trained_on_celeba}
		\end{minipage}
	}
	
	\caption{Glow trained on CelebA32$\times$32, sampling according to (a) standard Gaussian distribution; (b) fitted Gaussian distribution from MNIST representations; (c) fitted Gaussian distribution from CIFAR10 representations.}
	\label{fig:sample_images_trained_on_celeba}
\end{figure*}

\begin{figure*}[htbp]
	\centering
	\includegraphics[width=4cm]{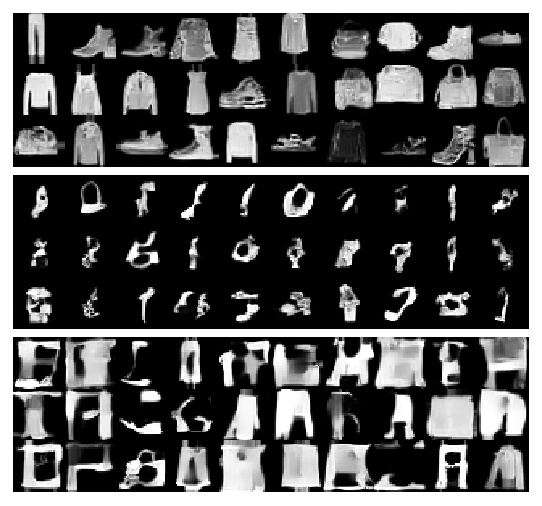}
	\vspace{-0pt}
	\caption{Glow trained on FashionMNIST. Sampling according to prior (up), fitted Gaussian distribution from representations of MNSIT (middle) and notMNIST (down).}
	\label{fig:sampled_mnist_notmnist_trained_on_fashionmnist}
\end{figure*}

\begin{figure*}[htbp]
	\centering
	\subfigure[]{
		\begin{minipage}[t]{5cm}
			\centering
			\includegraphics[width=5cm]{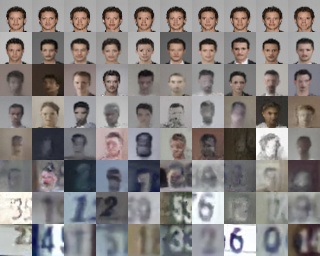}
			\label{fig:glow_trained_on_celeba_temperature_svhn}
		\end{minipage}
	}
	\subfigure[]{
		\begin{minipage}[t]{5cm}
			\centering
			\includegraphics[width=5cm]{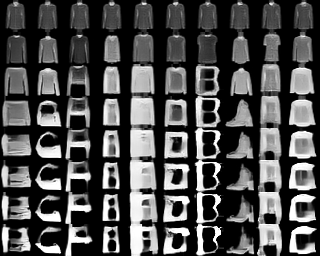}
			\label{fig:temperature_fashionmnist_to_notmnist}
		\end{minipage}
	}
	
	\caption{(a) Train Glow on CelebA and sample from the fitted Gaussian distribution of SVHN. (b) Train on FashionMNIST and sample from the fitted Gaussian distribution of notMNIST. From top to down, the sampled noises from Gaussian distribution are scaled by temperature 0, 0.25, 0.5, 0.6, 0.7, 0.8, 0.9, 1.0, respectively.
	}
	\label{fig:glow_trained_on_ID_temperature_OOD}
\end{figure*}

\begin{figure*}[ht]
	\centering
	\subfigure[FashionMNIST vs MNIST]{
		\begin{minipage}[t]{5cm}
			\centering
			\includegraphics[width=5cm]{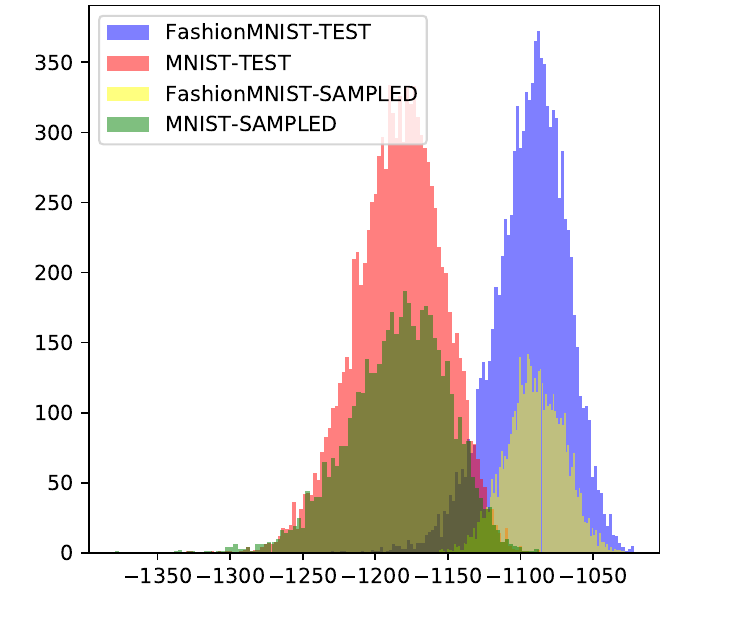}
			\label{fig:logpz_fashionmnist_mnist_with_sample}
		\end{minipage}
	}
	\subfigure[FashionMNIST vs notMNIST]{
		\begin{minipage}[t]{5cm}
			\centering
			\includegraphics[width=5cm]{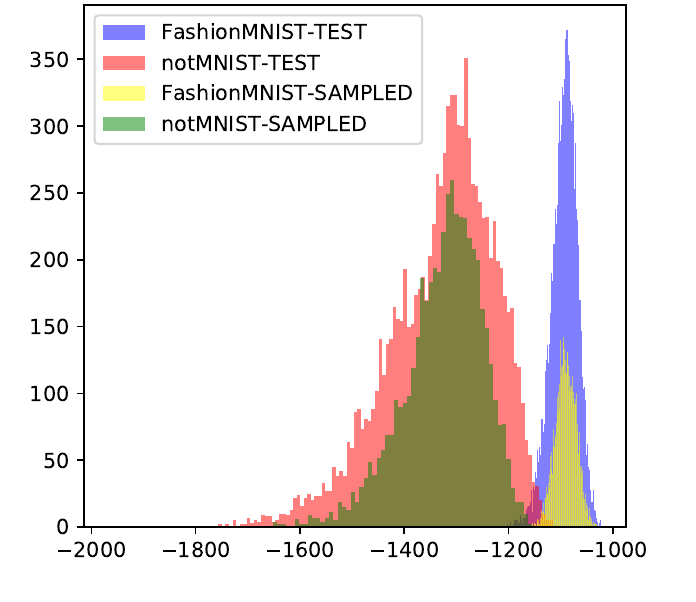}
			\label{fig:logpz_fashionmnist_notmnist_with_sample}
		\end{minipage}
	}
	
	\caption{Glow trained on FashionMNIST. Histogram of $\log p(\bm{z})$ of (a) FashionMNIST vs MNIST, (b) FashionMNIST vs notMNIST under Glow. The green part corresponds to the $\log p(\bm{z})$ of noises sampled from the fitted Gaussian distribution of OOD datasets. }
	\label{fig:logpz_datasets_with_sample}
\end{figure*}

\begin{figure*}[htbp]
	\centering
	\subfigure[]{
		\begin{minipage}[t]{4.2cm}
			\centering
			\includegraphics[width=4.2cm]{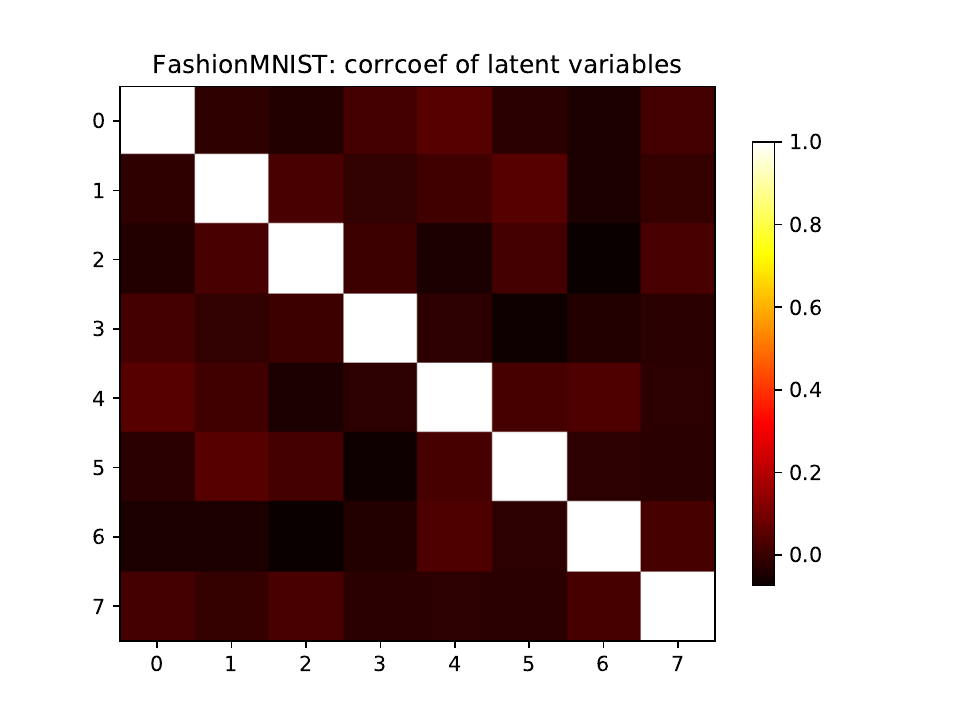}
			\label{fig:heatmap_fashionmnist_correlation_coefficient_VAE}
		\end{minipage}
	}
	\subfigure[]{
		\begin{minipage}[t]{4.2cm}
			\centering
			\includegraphics[width=4.2cm]{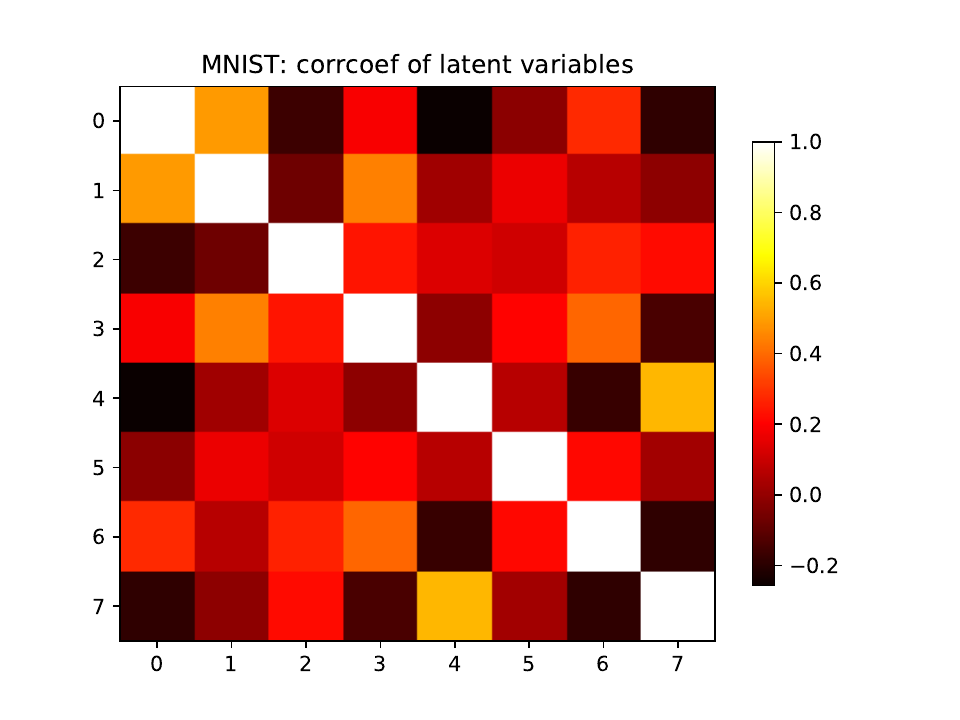}
			\label{fig:heatmap_mnist_corrcoef_VAE}
		\end{minipage}
	}
	\subfigure[]{
		\begin{minipage}[t]{4.2cm}
			\centering
			\includegraphics[width=4.2cm]{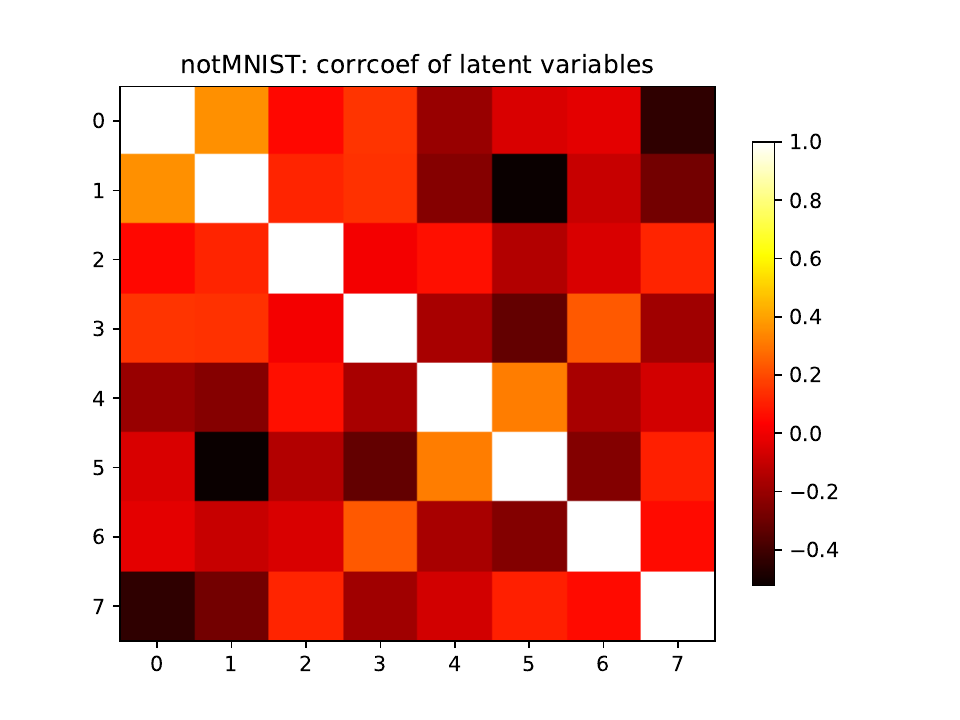}
			\label{fig:heatmap_notmnist_corrcoef_VAE}
		\end{minipage}
	}
	\subfigure[]{
		\begin{minipage}[t]{4.2cm}
			\centering
			\includegraphics[width=4.2cm]{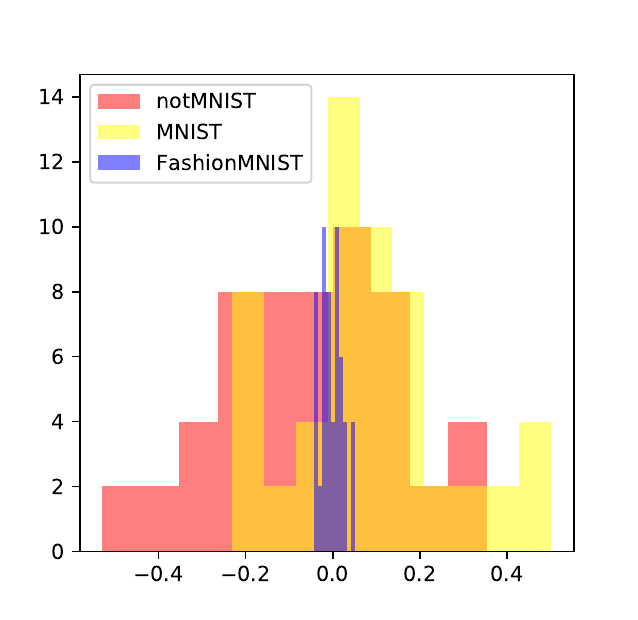}
			\label{fig:heatmap_notmnist_corrcoef_VAE}
		\end{minipage}
	}
	\caption{VAE trained on FashionMNIST. Heatmap of correlation of (a)FashionMNIST (b)MNIST (c) notMNIST representations. (d) Histogram of non-diagonal elements of correlation of sampled representations.}
	\label{fig:heatmap_correlation_fashionmnist_mnist_notmnist_VAE}
\end{figure*}

%

\begin{figure*}[htbp]
	\centering
	\subfigure{
		\begin{minipage}[t]{4.2cm}
			\centering
			\includegraphics[width=4.2cm]{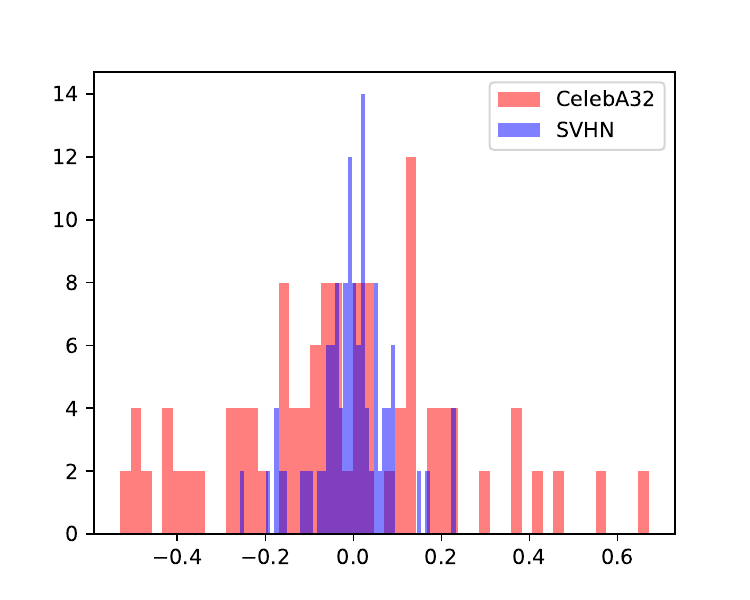}
		\end{minipage}
	}
	\subfigure{
		\begin{minipage}[t]{4.2cm}
			\centering
			\includegraphics[width=4.2cm]{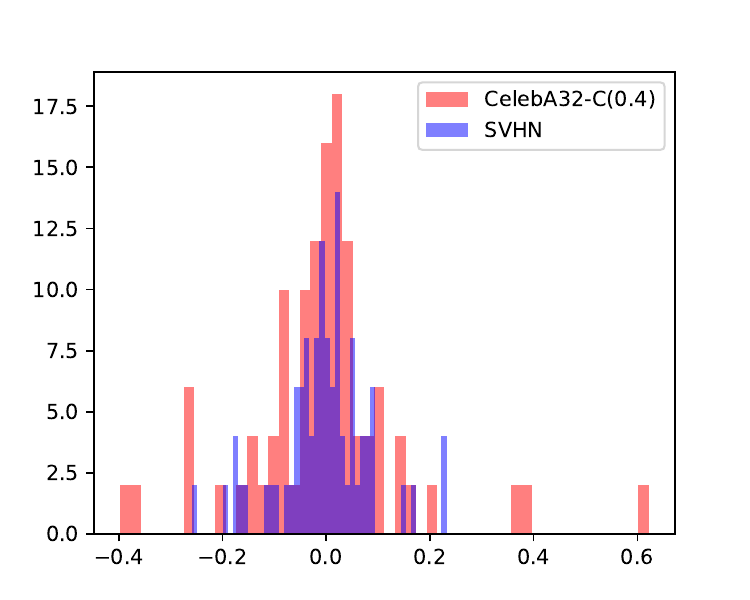}
		\end{minipage}
	}
	\subfigure{
		\begin{minipage}[t]{4.2cm}
			\centering
			\includegraphics[width=4.2cm]{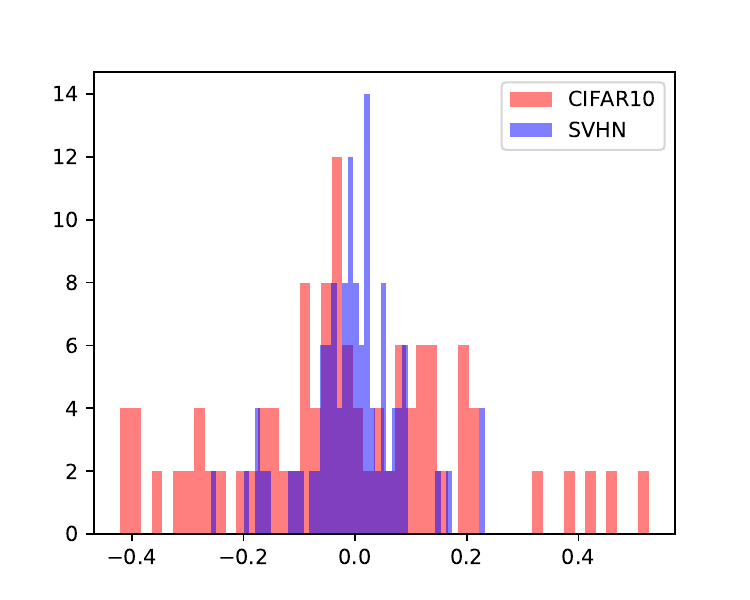}
		\end{minipage}
	}
	\subfigure{
		\begin{minipage}[t]{4.2cm}
			\centering
			\includegraphics[width=4.2cm]{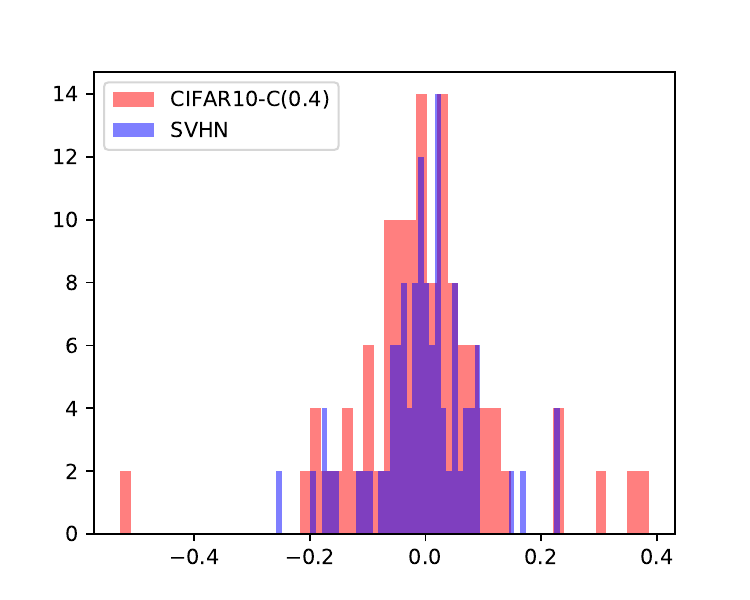}
		\end{minipage}
	}
	
	\vspace{-10pt}
	\subfigure{
		\begin{minipage}[t]{4.2cm}
			\centering
			\includegraphics[width=4.2cm]{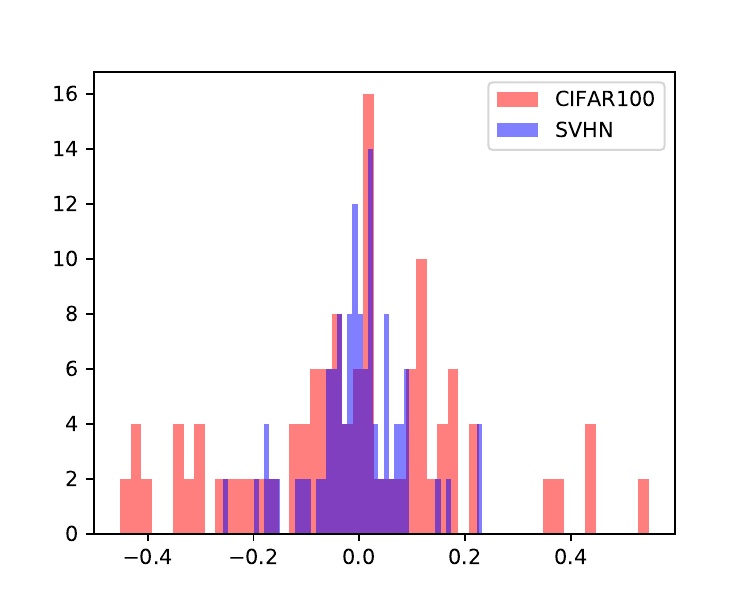}
		\end{minipage}
	}
	\subfigure{
		\begin{minipage}[t]{4.2cm}
			\centering
			\includegraphics[width=4.2cm]{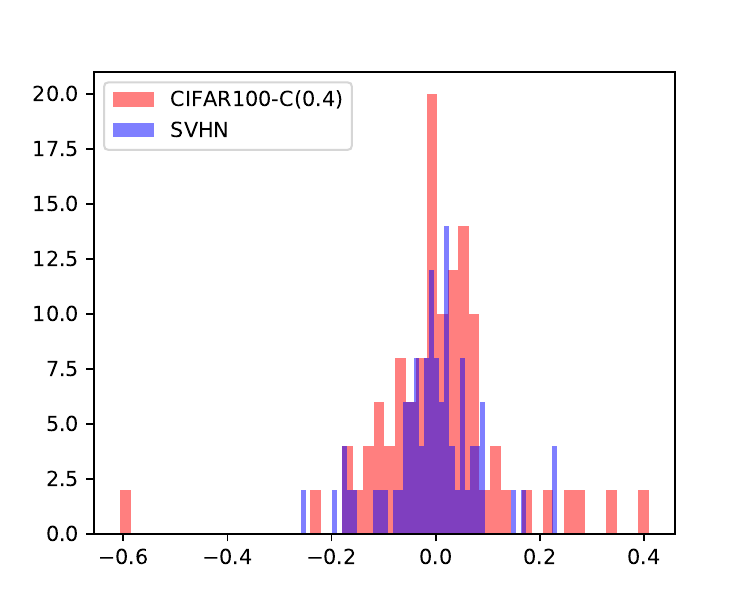}
		\end{minipage}
	}
	\subfigure{
		\begin{minipage}[t]{4.2cm}
			\centering
			\includegraphics[width=4.2cm]{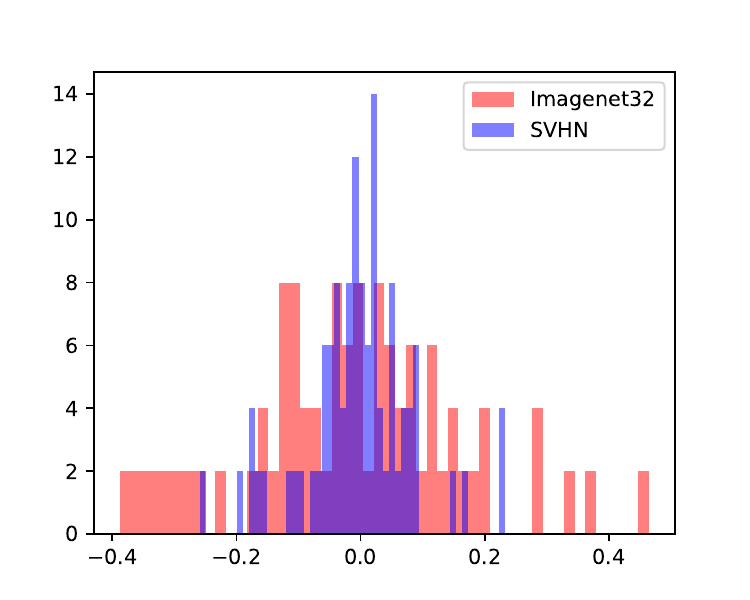}
		\end{minipage}
	}
	\subfigure{
		\begin{minipage}[t]{4.2cm}
			\centering
			\includegraphics[width=4.2cm]{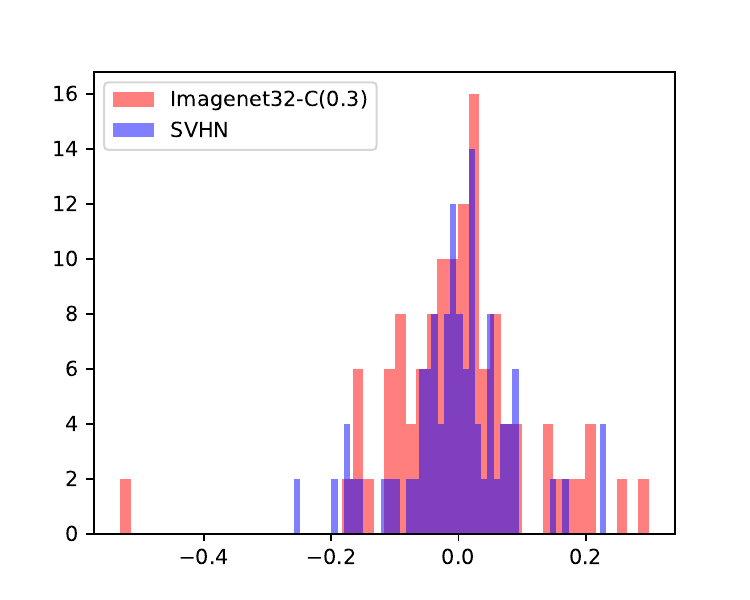}
		\end{minipage}
	}
	\caption{VAE trained on SVHN. Histogram of non-diagonal elements of correlation of sampled representations.}
	\label{fig:histo_correlation_VAE_train_svhn}
\end{figure*}

\begin{figure*}[htbp]
	\centering
	\subfigure{
		\begin{minipage}[t]{4.2cm}
			\centering
			\includegraphics[width=4.2cm]{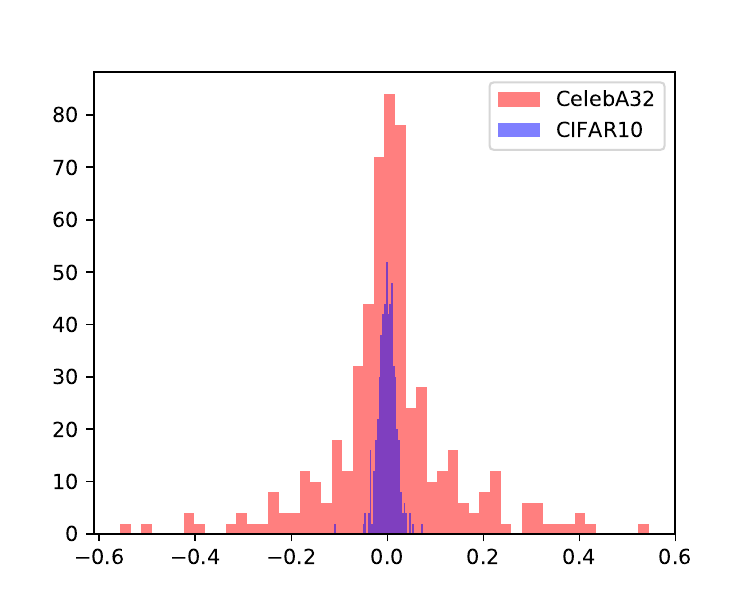}
		\end{minipage}
	}
	\subfigure{
		\begin{minipage}[t]{4.2cm}
			\centering
			\includegraphics[width=4.2cm]{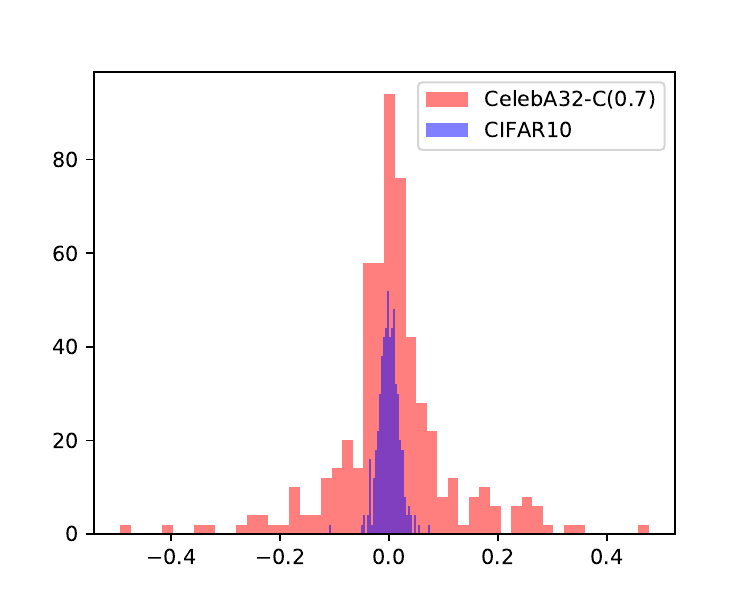}
		\end{minipage}
	}
	\subfigure{
		\begin{minipage}[t]{4.2cm}
			\centering
			\includegraphics[width=4.2cm]{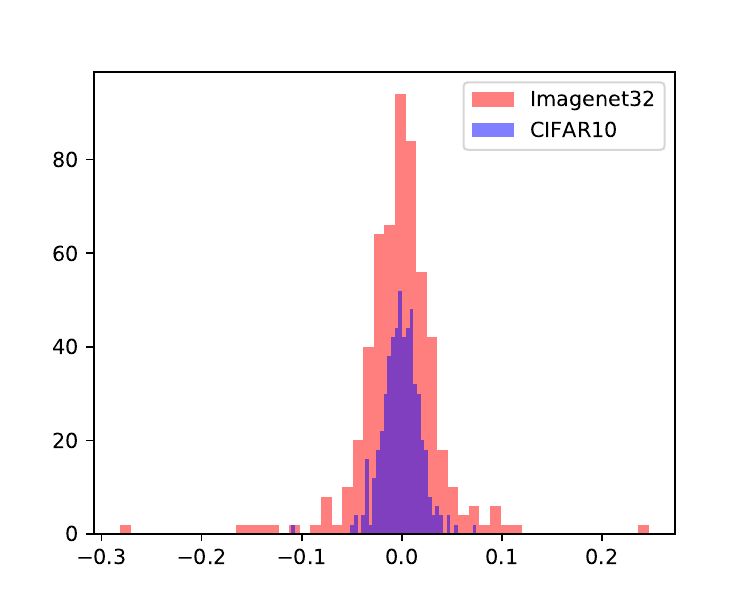}
		\end{minipage}
	}
	\subfigure{
		\begin{minipage}[t]{4.2cm}
			\centering
			\includegraphics[width=4.2cm]{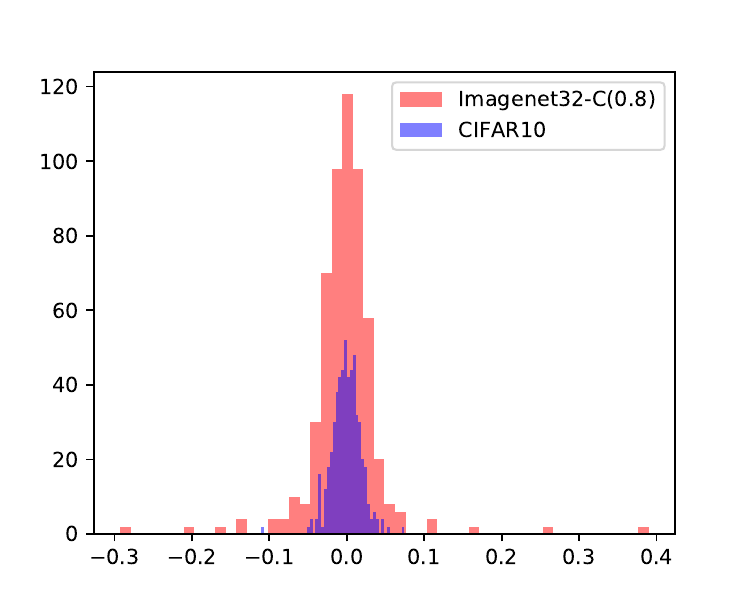}
		\end{minipage}
	}
	
	\vspace{-10pt}
	\subfigure{
		\begin{minipage}[t]{4.2cm}
			\centering
			\includegraphics[width=4.2cm]{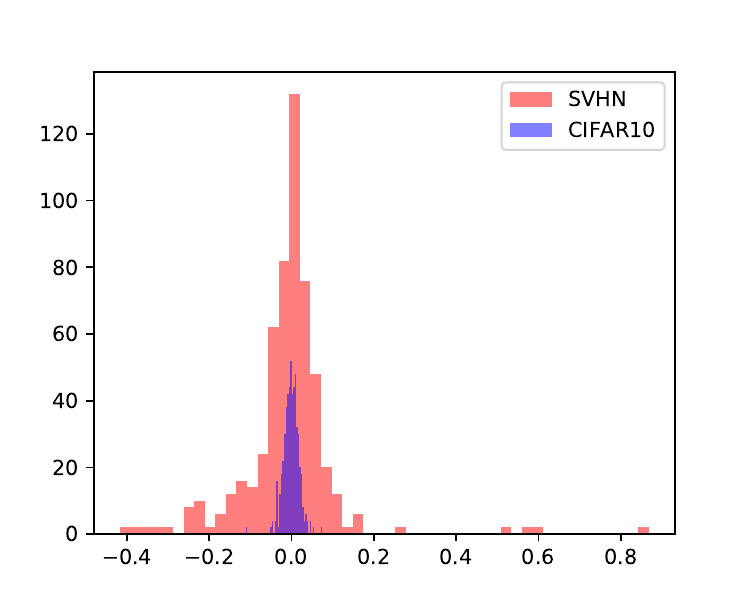}
		\end{minipage}
	}
	\subfigure{
		\begin{minipage}[t]{4.2cm}
			\centering
			\includegraphics[width=4.2cm]{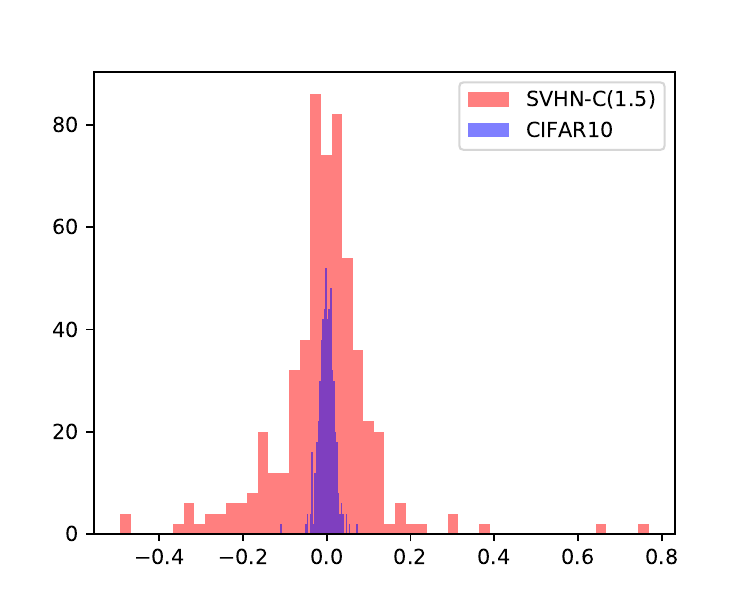}
		\end{minipage}
	}
	\subfigure{
		\begin{minipage}[t]{4.2cm}
			\centering
			\includegraphics[width=4.2cm]{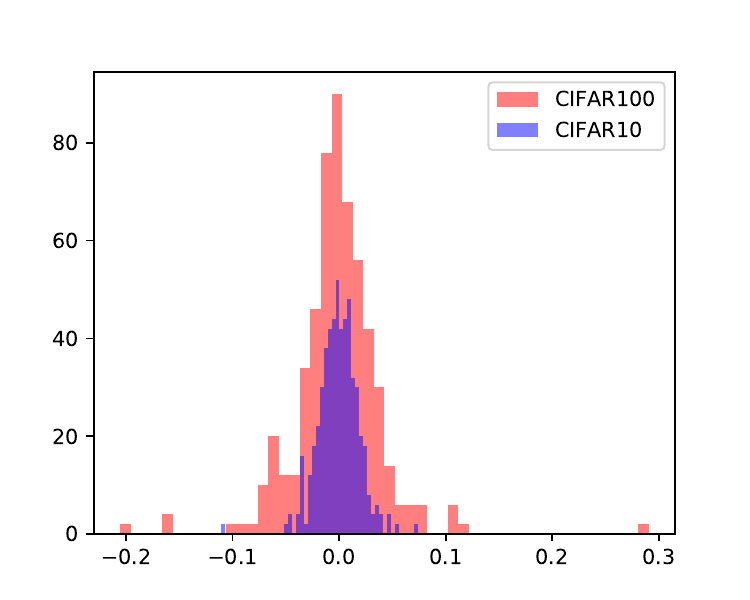}
		\end{minipage}
	}
	\caption{VAE trained on CIFAR10. Histogram of non-diagonal elements of correlation of sampled representations.}
	\label{fig:histo_correlation_VAE_train_cifar10}
\end{figure*}

\begin{figure*}[t]
	\centering
	\includegraphics[width=10cm]{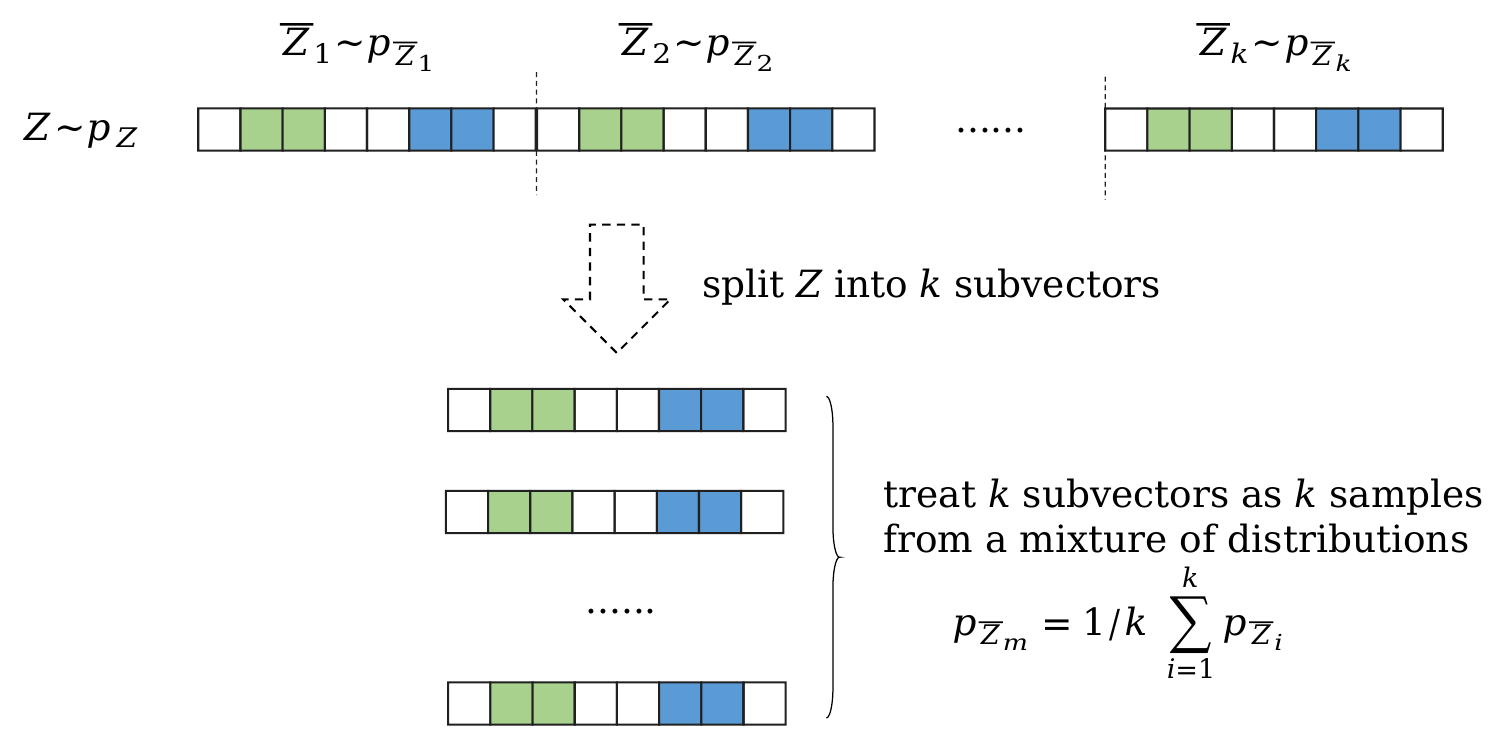}
	\caption{Split a random vector $Z\sim p_Z$ into $k$ subvectors $\bar{Z}_i\sim p_{\bar{Z}_i}$ ($1\leq i \leq k$). We treat $k$ subvectors as $k$ samples from a mixture of distributions $p_{\bar{Z}_i}=1/k\Sigma_{i=1}^k p_{\bar{Z}_i}$. In the figure, we use the same color to indicate neighboring pixels that are strongly correlated. For example, if the second element $\bar{Z}_{i,2}$ and the third element $\bar{Z}_{i,3}$ are strongly correlated for all $1\leq i\leq k$, we can say that $\bar{Z}_{m,2}$ and $\bar{Z}_{m,3}$ are also strongly correlated. This is why we can leverage local pixel dependence in our method.}
	\label{fig:splitting_strategy}
\end{figure*}

\begin{figure*}[t]
	\centering
	%
		%
		
		\begin{minipage}[t]{12cm}
			\centering
			\includegraphics[width=12cm]{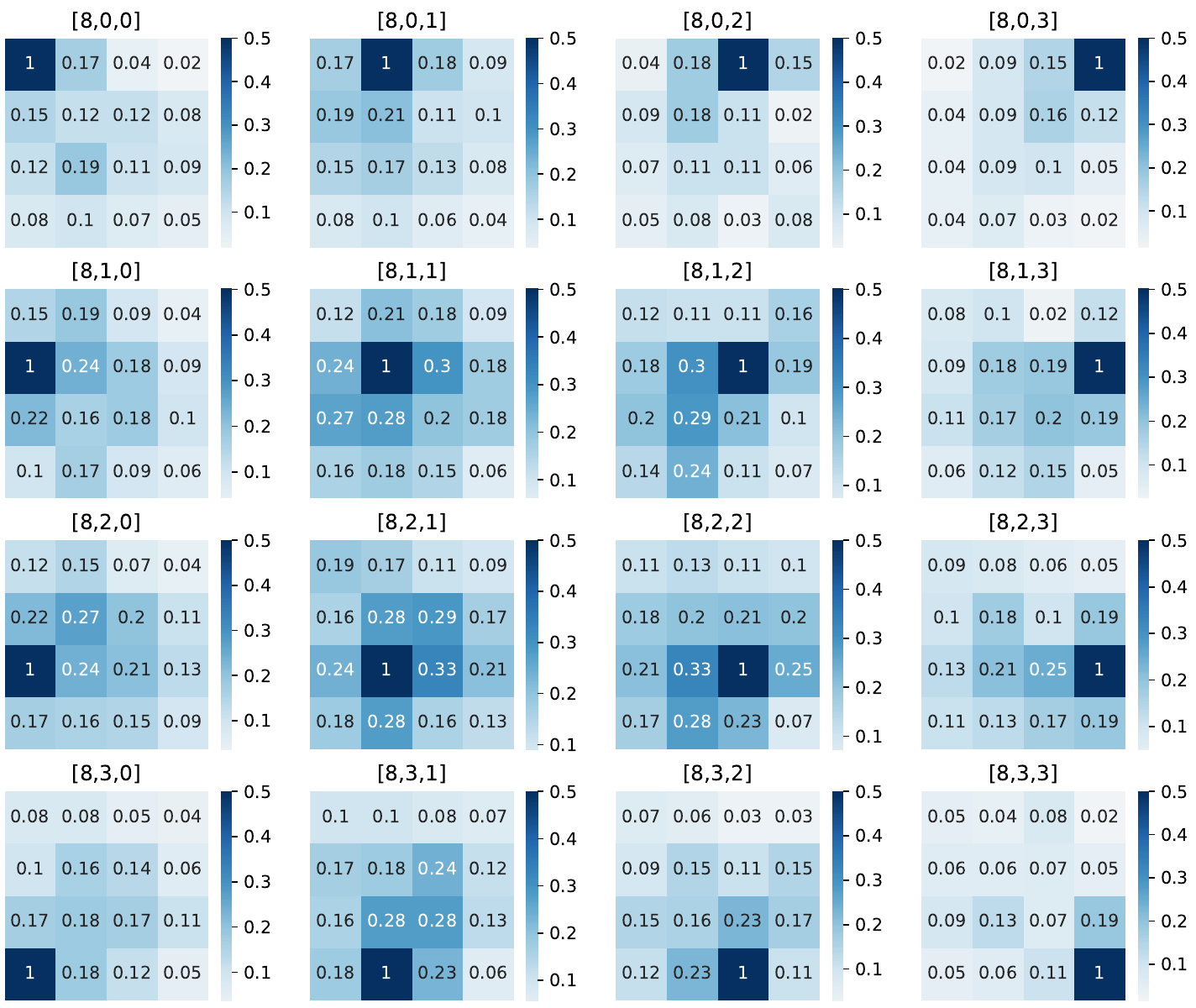}
		\end{minipage}
	\caption{Train Glow on SVHN and test on ImageNet32. We randomly select the 8-th channel. The subfigure at $i$-th row and $j$-th column shows the correlation between the pixel at position $(i,j)$ and all other pixels. Adjacent pixels tend to have stronger correlation.}
	\label{fig:heatmap_correlation_between_pixels_cifar10_imagenet_under_svhn_glow}
\end{figure*}

\begin{figure*}[htbp]
	\centering
	\subfigure[SVHN]{
		\begin{minipage}[t]{5cm}
			\centering
			\includegraphics[width=5cm]{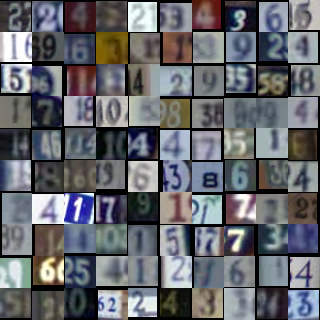}
			\label{fig:svhn_examples}
		\end{minipage}
	}
	\subfigure[SVHN with increased contrast by a factor of 2, have lower likelihood]{
		\begin{minipage}[t]{5cm}
			\centering
			\includegraphics[width=5cm]{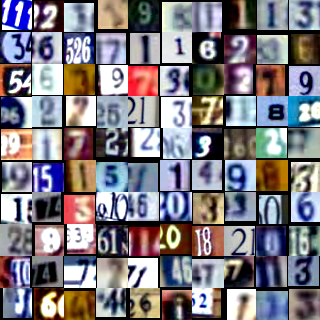}
			\label{fig:svhn_contrast_examples}
		\end{minipage}
	}
	
	\subfigure[CelebA32]{
		\begin{minipage}[t]{5cm}
			\centering
			\includegraphics[width=5cm]{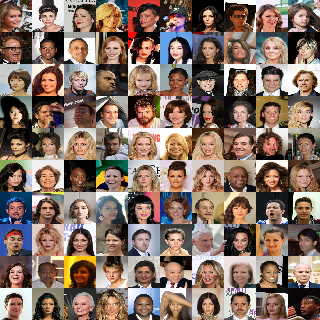}
			\label{fig:celeba32_examples}
		\end{minipage}
	}
	\subfigure[CelebA32 with decreased contrast by a factor of 0.3, have higher likelihood]{
		\begin{minipage}[t]{5cm}
			\centering
			\includegraphics[width=5cm]{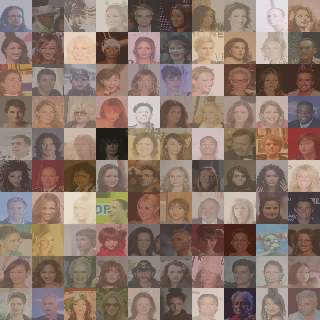}
			\label{fig:celeba32_gray_examples}
		\end{minipage}
	}

	\subfigure[ImageNet32]{
		\begin{minipage}[t]{5cm}
			\centering
			\includegraphics[width=5cm]{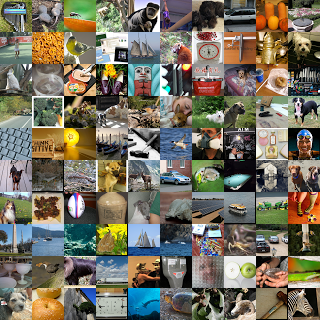}
			\label{fig:imagenet32_examples}
		\end{minipage}
	}
	\subfigure[ImageNet32 with decreased contrast by a factor of 0.3, have higher likelihood]{
		\begin{minipage}[t]{5cm}
			\centering
			\includegraphics[width=5cm]{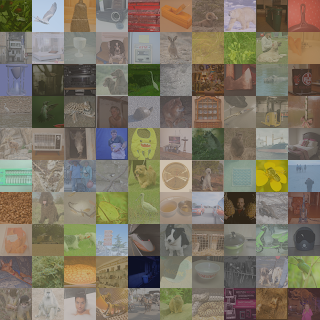}
			\label{fig:imagenet32_gray_examples}
		\end{minipage}
	}
	
	\caption{Examples of datasets and their mutations. Under Glow trained on CIFAR10, these mutated datasets have a similar likelihood distribution with CIFAR10 test split.}
	\label{fig:dataset_examples}
\end{figure*}


\begin{figure*}[htbp]
	\centering
	\subfigure[]{
		\begin{minipage}[t]{4.5cm}
			\centering
			\includegraphics[width=4.5cm]{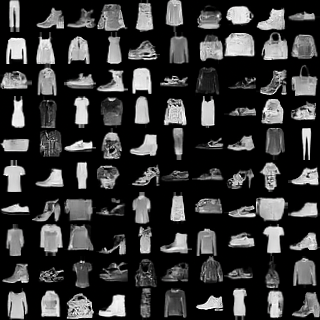}
			\label{fig:generated_fashionmnist_glow}
		\end{minipage}
	}
	\subfigure[]{
		\begin{minipage}[t]{4.5cm}
			\centering
			\includegraphics[width=4.5cm]{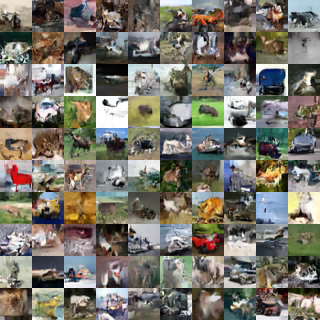}
			\label{fig:generated_cifar10_glow}
		\end{minipage}
	}
	\subfigure[]{
		\begin{minipage}[t]{4.5cm}
			\centering
			\includegraphics[width=4.5cm]{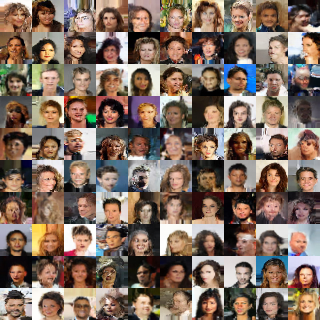}
			\label{fig:generated_celeba_glow}
		\end{minipage}
	}

	\caption{Generated images from Glow trained on (a)FashionMNIST; (b)CIFAR-10; (c)CelebA32. }
	\label{fig:generated_images_glow}
\end{figure*}


\begin{figure*}[t]
	\centering
	\includegraphics[width=11cm]{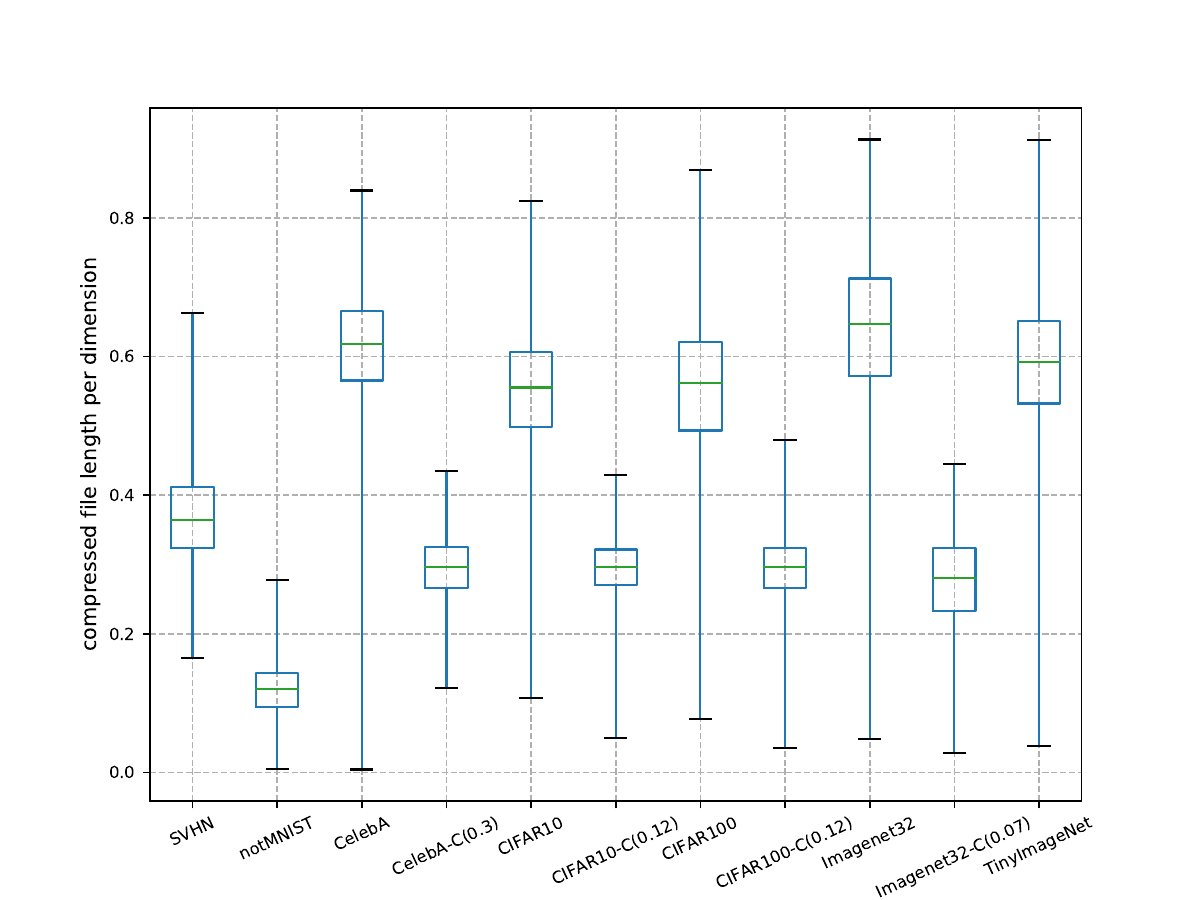}
	\caption{The distributions of complexity estimated by the lengths of compressed files of datasets. We use FLIF as compressor and compute lengths in bits per dimension. Datasets with decreased contrast has lower complexity. }
	\label{fig:complexity_length}
\end{figure*}

\end{document}

%% file: sec3.tex
\vspace{-10pt}
\section{Explaining Why Cannot Sample OOD Data}\label{sec:whynotexplain}
In this section, we explain why we cannot sample OOD data from two perspectives. 
Based on these analyses, we will derive our OOD detection method in Section \ref{sec:GADmethod}.

Figure \ref{fig:theory_and_algorithm} shows the overview of our analysis of the KL divergence in flow-based model for a certain case (discussed in Subsection \ref{sec:gaussian_case}). The top half of Figure 
\ref{fig:flowchart} in the supplementary material also summarizes our discussion in this section. 
Please refer to Figure \ref{fig:theory_and_algorithm} and Figure 
\ref{fig:flowchart} in the supplementary material when reading this section. 
\vspace{-10pt}
\subsection{\textbf{Explanation 1: Divergence Perspective}}
Our analysis involves the following distributions: the distributions of ID data ($p_X$) and OOD data ($q_X$), the distributions of ID representations ($p_Z$) and  OOD representations ($q_Z$), the prior $p_Z^r$, and the model induced distribution $p_X^r$ such that $Z_r\sim p_Z^r$ and $X_r=f^{-1}(Z_r)\sim p^r_X$.  Table \ref{tbl:notations} in the supplementary material summarizes the notations involved in our analysis and how they influence each other. In this subsection, we first discuss the general case. Then we conduct further analysis for  \textit{Category I} problems (smaller/similar variance, higher likelihoods).

\vspace{-5pt}
\subsubsection{General case}\label{sec:general_case}

We can analyze the KL divergence in flow-based model in the following steps.
\begin{figure}[t!]
	\centering
	\centering
	\includegraphics[width=8cm]{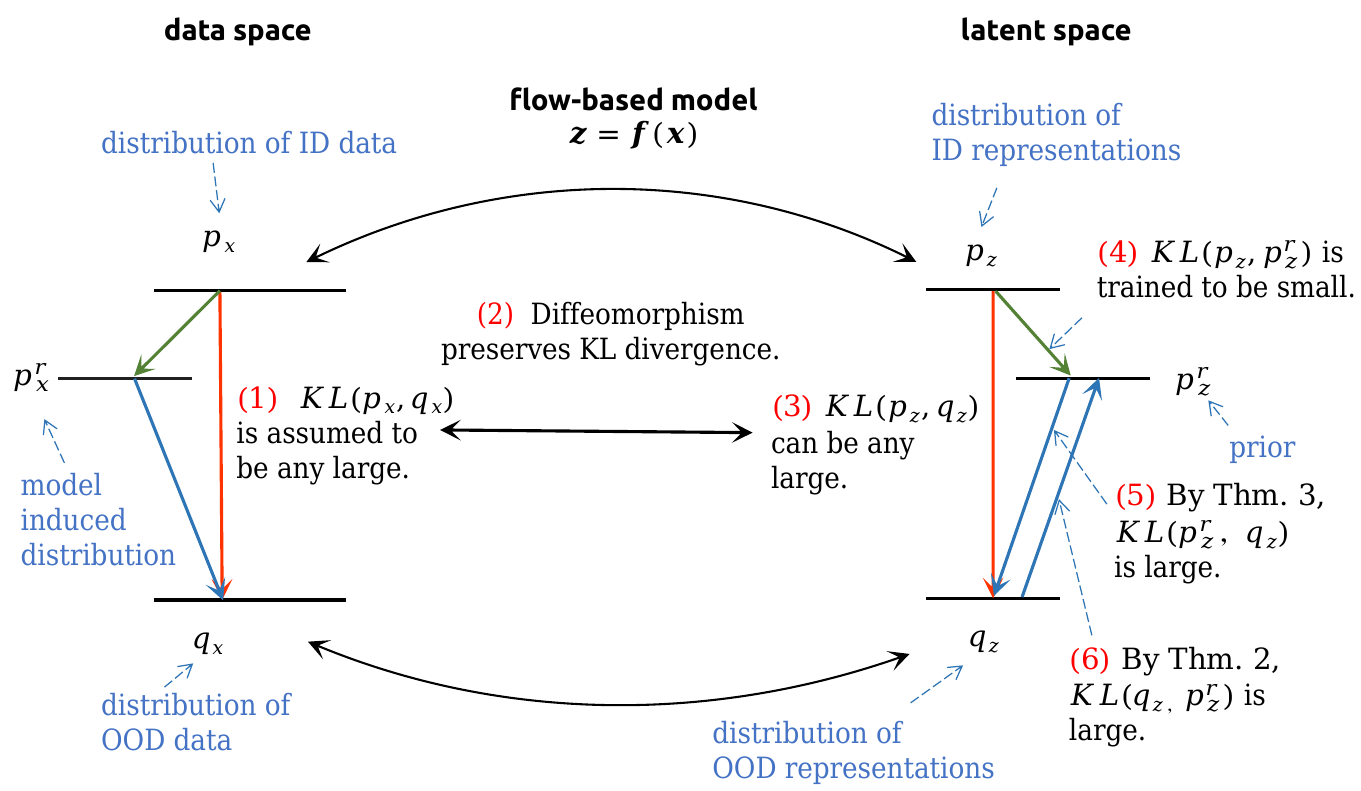}
	\caption{The key steps of our analysis for Gaussian case (Subsection \ref{sec:gaussian_case}). Arrows represent KL divergences. 
	}
	\label{fig:theory_and_algorithm} 
\end{figure}

%
\begin{enumerate}[(1)]
	\item \label{step1}
	We treat ID and OOD datasets as samples from different unknown distributions. 
	Therefore, 	it is reasonable to consider the following assumption.
	\begin{assumption}\label{assum:1}
		\textit{The KL divergence between the distributions of ID and OOD datasets is large.} 
	\end{assumption}
	So we can assume both $KL(p_X||q_X)$ and $KL(q_X||p_X)$ can be any large.

	\item 
	According to the following Theorem \ref{thm:all_enough_distance_guarantee}, we know diffeomorphism preserves KL divergence.
	\begin{theorem} \label{thm:all_enough_distance_guarantee}
		(See \cite{nielsen2018elementary}) Given a diffeomorphism $\bm{z}=f(\bm{x})$, let $X_1\sim p_X$, $X_2\sim q_X$, $Z_1=f(X_1)\sim p_Z$ and  $Z_2=f(X_2)\sim q_Z$. 
		Let $D_{\phi}^{h}$ be a $(h,\phi)$-divergence measure, 
		\vspace{-5pt}
		$$D_{\phi}^{h}(p_X,q_X)=D_{\phi}^{h}(p_Z,q_Z)$$
	\end{theorem}
		\vspace{-5pt}
	\begin{proof}\label{thm:proof}
		KL divergence is a member of the $(h,\phi)$-divergence family (See Section \ref{sec:appendix_divergence} in the supplementary material).  The proof of Theorem \ref{thm:all_enough_distance_guarantee} relies on diffeomorphisms. See \cite{nielsen2018elementary} for proof. 
	\end{proof}
	
	Thus, we can know $KL(p_X||q_X)=KL(p_Z||q_Z)$ is large. 
	\item \label{itm:small_forwardKL}
	We can suppose the model is expressible enough and trained by maximum likelihood estimation. 
	This is equal to minimizing  forward KL divergence $KL(p_X||p^r_X)$ \cite{papamakarios2019flow_model_survey}.
	By Theorem \ref{thm:all_enough_distance_guarantee}, we also have $KL(p_X||p^r_X)=KL(p_Z||p^r_Z)$. Thus, $KL(p_Z||p^r_Z)$ is small.
	\item  
	KL divergence is not symmetric and does not satisfy the triangle inequality (\textit{i.e.}, not a proper statistical distance) \footnote{For example, we can construct two distributions $p$ and $q$ such that $KL(p||q)$ is any small but $KL(q||p)$ is any large.}.  Otherwise, we would know that the reverse KL divergence $KL(p^r_Z||p_Z)$ is small and that $KL(q_Z||p^r_Z)$ is large by triangle inequality. Researchers have investigated other statistical divergences in different contexts  \cite{fGAN, improved_WGANs, pardo2018statistical}.
	However, flow-based model is usually trained by minimizing KL divergence. In order to explain the phenomenon of flow-based model, we should conduct further analysis on KL divergence.
	In this paper, we seek stronger conclusions for a special case.
	
	We perform generalized Shapiro-Wilk test for multivariate normality \cite{ComparisonNormalityTest2011} on representations. As shown in Table \ref{tbl:SW_test} in the supplementary material, ID representations always have high $p$-values. This indicates that ID representations always manifest strong normality.
	Therefore, we can use a Gaussian distribution $\n_p$ to approximate $p_Z$ and have $KL(p_Z||p^r_Z) \approx KL(\n_p||p^r_Z)$. 
	Now we can apply the following Theorem \ref{thm:duality_small_KL_general} which reveals the approximate symmetry of small KL divergence between Gaussian distributions.
	
	\begin{theorem}{(\textbf{Approximate symmetry of small KL divergence between Gaussian distributions})}\label{thm:duality_small_KL_general}
		For any $n$-dimensional Gaussian distributions $\n(\m_1,\s_1)$ and $\n(\m_2,\s_2)$, if $\fKLnn\leq \e  $ $(\varepsilon\geq 0)$, 
			\vspace{-5pt}
		\begin{align}\label{equ:supremum_duality_kl_maintext}
			\bKLnn  
			\leq 	 \varepsilon + 2\varepsilon^{1.5} + O(\varepsilon^2)
		\end{align}
		\vspace{-5pt}
	\end{theorem}
	%
	%

	\begin{proof}
		The proof is too long. See our manuscript \cite{zhang2021properties} for details. 
		Importantly, the supremum is independent of the dimension $n$. So Theorem \ref{thm:duality_small_KL_general} is applicable to high-dimensional problems (\textit{e.g.}, flow-based model). 
	\end{proof}
	 By Theorem \ref{thm:duality_small_KL_general}, we can know the reverse KL divergence $KL(p^r_Z||\n_p)\approx KL(p^r_Z||p_Z)$ must be small too. 
	Thus, we can consider the following assumption. 
	\begin{assumption}\label{assum:2}
		\textit{The distribution of ID representations and the  prior are close enough.}
	\end{assumption}
	\item 
	Now that the forward and reverse KL divergence between $p_Z$ and prior $p_Z^r$ are both small, we can consider a stronger assumption $p_Z\approx p_Z^r$. Thus, we have $KL(q_Z||p_Z^r) \approx KL(q_Z||p_Z)$. In step \ref{step1}, we have known $KL(q_X||p_X)=KL(q_Z||p_Z)$ is large, so $KL(q_Z||p_Z^r)$ is large too.
\end{enumerate}


	\vspace{-5pt}
\subsubsection{The Gaussian case} \label{sec:gaussian_case}
In the above Step (4), we use a strong assumption $p_Z\approx p^r_Z$. 
In fact, for \textit{Category I} problems  (smaller/similar variance, higher/similar likelihoods), we do not need such assumption. 
The results of normality test on OOD representations demonstrate OOD representations in all \textit{Category I} problems except for SVHN vs Constant have $p$-values greater than 0.05 (see Table \ref{tbl:SW_test} in supplementary material). It seems that OOD datasets ``sitting inside'' the training data are also ``Gaussianized'' along with the training data. 
As far as we know, we are the first to observe this phenomenon.

Based on this observation, we can conduct more analysis using the following Theorem \ref{thm:triangle_n1_n2_n3}, which reveals that KL divergence between Gaussian distributions follows a relaxed triangle inequality.
\begin{theorem}{(\textbf{Relaxed triangle inequality})}\label{thm:triangle_n1_n2_n3}
	For any three $n$-dimensional Gaussian distributions $\mathcal{N}(\bm{\mu}_i,\bm{\Sigma}_i)$ $(i\in\{1,2,3\})$  
	such that $KL(\mathcal{N}(\bm{\mu}_1,\bm{\Sigma}_1)||\mathcal{N}(\bm{\mu}_2,\bm{\Sigma}_2))\leq \varepsilon_1$ and 
	$ KL(\mathcal{N}(\bm{\mu}_2,\bm{\Sigma}_2)||\mathcal{N}(\bm{\mu}_3,\bm{\Sigma}_3))\leq \varepsilon_2$ for small $\varepsilon_1, \varepsilon_2\ge 0$, 
		\vspace{-5pt}
	\begin{align}
		&KL((\mathcal{N}(\bm{\mu}_1,\bm{\Sigma}_1)||\mathcal{N}(\bm{\mu}_3,\bm{\Sigma}_3))\nonumber \\
		< & 3\varepsilon_1+3\varepsilon_2+2\sqrt{\varepsilon_1\varepsilon_2}+o(\varepsilon_1)+o(\varepsilon_2)
	\end{align}
		\vspace{-10pt}
	
\end{theorem}
\begin{proof}
	The proof is complex and too long. See our work \cite{zhang2021properties} for details. 
	The bound is small for small $\varepsilon_1, \varepsilon_2$ and is 0 when $\varepsilon_1=\varepsilon_2=0$. Similarly, the bound is independent of the dimension $n$ and applicable to high-dimensional problems.
\end{proof}

As shown in Figure \ref{fig:theory_and_algorithm} and Figure 
\ref{fig:flowchart} in the supplementary material, when $q_Z$ is Gaussian-like, we can use a Gaussian distribution $\n_q$ to approximate $q_Z$ and have $KL(q_Z||p^r_Z) \approx KL(\n_q||p^r_Z)$, $KL(p_Z||q_Z)\approx KL(p_Z||\mathcal{N}_q)$. Now that $KL(p_Z||q_Z)$ is large and $KL(p_Z||p^r_Z)$ is small. According to the relaxed triangle inequality in Theorem \ref{thm:triangle_n1_n2_n3}, $KL(p^r_Z||\n_q)$ must not be small. Furthermore, we can apply  Theorem \ref{thm:duality_small_KL_general} on $KL(p^r_Z||\n_q)$ and know that $KL(\n_q||p^r_Z)$	is large. Finally, we know $KL(q_Z||p^r_Z)$ is large too. 	
	\vspace{-5pt}
\subsubsection{Summary}\label{sec:answer1}
Overall, we can explain why we cannot sample OOD data from the divergence perspective.

\begin{framed}
\textbf{\textit{Answer 1 to Q1}}: The KL divergence between the distribution of OOD representations and prior is large regardless of when the likelihoods of OOD data are higher, lower, or coinciding with that of ID data. So it is hard to sample OOD-like data from the model.
\end{framed}

\vspace{-10pt}
\subsection{\textbf{Explanation 2: Geometric Perspective}}\label{sec:KL_decomposition}
We can obtain another explanation from a geometric perspective based on the analysis in the last subsection.
The first step is to use the following Theorem \ref{thm:decompose_KL_ID} to decompose forward KL divergence.
Besides, we will use Theorem \ref{thm:decompose_KL_ID} to derive OOD detection method in Section \ref{sec:GADmethod}.
\begin{theorem}\label{thm:decompose_KL_ID}
	Let  $X\sim p^*_{X}$ be an $n$-dimensional random vector, $X_i\sim p^*_{X_i}$ be the $i$-th dimensional element of $X$. Then 
	\begin{align}\label{equ:decompose_KL_basic}
		& KL(p^*_X||\mathcal{N}(0,I_n))\\
		=&\underbrace{KL(p^*_X||\prod_{i=1}^n p^*_{X_i}(x))}_{I_d[p^*_X]\atop \text{total\ correlation}  } + \underbrace{\sum_{i=1}^n KL(p^*_{X_i}||\mathcal{N}(0,1))}_{  D_d[p^*_X]=\sum_{i=1}^n D_d^i[p^*_{X_i}]\atop \text{dimensional-wise\ KL\ divergence}}
	\end{align} 
\end{theorem}
\begin{proof}
	We can decompose KL divergence as in \cite{chen2018isolating}.
	See Section \ref{sec:proof_decompose_kl_ID} in the supplementary material for proof. 
\end{proof}

Theorem \ref{thm:decompose_KL_ID} decomposes forward KL divergence into two non-negative parts: $I_d$ is total correlation (generalized mutual information) measuring the mutual dependence between dimensions \cite{mutual_info_esti_2013}; $D_d$ is dimension-wise KL divergence between the marginal distribution of each dimension and prior. We use $[p^*_X]$ to denote one term is computed from $p^*_X$.

Theorem \ref{thm:decompose_KL_ID} can help us further investigate the forward KL divergence. 
For ID data, we have known that $KL(p_Z||p^r_Z)$ is small. Applying Theorem \ref{thm:decompose_KL_ID} to $KL(p_Z||p^r_Z)$, we can know the total correlation $I_d[p_Z]$ must be small. This indicates that ID data tends to have independent representations.
On the contrary, for OOD data, 
a large $KL(q_Z||p^r_Z)$ allows a large total correlation $I_d[q_Z]$.
Although it is hard to estimate total correlation \cite{mutual_info_esti_2013}, we can use an alternative dependence measure, \textit{i.e.}, the most commonly used correlation coefficient, to investigate the linear dependency.
We train Glow on FashionMNIST and test on MNIST/notMNIST. Figure \ref{fig:nondiagonal_correlation_coefficient_fashionmnist_mnist_notmnist_glow} in the supplementary material shows the histogram of the non-diagonal elements in the correlation matrix of representations. We can see that OOD representations are more correlated. In fact, this happens for all the problems in our experiments. See Figure \ref{fig:heatmap_fashionmnist_correlation_coefficient} to \ref{fig:histo_correlation_train_celeba_test_others} in the supplementary material for more details.

From a geometric perspective, \textit{a high correlation between dimensions indicates  the representations of OOD dataset locate in specific directions} \cite{13WaysLookCorr} (see Figure \ref{fig:3d_gaussian} in the supplementary material for a 3-d example).  
It is hard to obtain data on specific directions in high dimensional space when sampling from standard Gaussian distribution. 

\textbf{Sampling OOD Data}.
To verify the above conclusion further, we have tried to restore the information of OOD dataset from the covariance of OOD representations.
Ordinarily, after training a flow-based model $f$, we sample noise $\varepsilon\sim \n(0,I)$ and feed back to the model, we can generate new image $f^{-1}(\varepsilon)$ seeming like training data. 
Now we feed the model with an OOD dataset and fit a Gaussian distributions $\n(\widetilde{\bm{\mu}},\widetilde{\bm{\Sigma}})$ from OOD representations, where $\widetilde{\bm{\mu}}$ and $\widetilde{\bm{\Sigma}}$ are the sample mean and covariance of OOD representations, respectively.
Then we sample noise $\varepsilon' \sim \n(\widetilde{\bm{\mu}},\widetilde{\bm{\Sigma}})$ and generate new image $f^{-1}(\varepsilon')$. We find that these generated images are meaningful OOD data.
For example, we train Glow on CIFAR-10 and perform the above OOD sampling using notMNIST as OOD dataset (gray-scale images are preprocessed for consistency, see Subsection \ref{sec:exp_setting}). 
As shown in Figure \ref{fig:sampled_notmnist_trained_on_fashionmnist} in the supplementary material, we can generate images similar to notMNIST, although the images are blurred. In this way, using a single Glow model trained on one training dataset, we can generate images like multiple OOD datasets, including MNIST, notMNIST, SVHN, CelebA, \textit{etc}, as long as we replace prior with the fitted Gaussian from the representations of the corresponding dataset (See Figure \ref{fig:sample_images_trained_on_cifar10}$\sim$
\ref{fig:sampled_mnist_notmnist_trained_on_fashionmnist} in the supplementary material for details). 
\textit{These results demonstrate that OOD representations reside in specific directions that can be partially characterized by the mean and covariance of OOD representations}. 
Such a similar phenomenon is also reported in \cite{gambardella2019transflow}, where Gambardella \textit{et al.} only use the mean of OOD representations. Their manuscript \cite{gambardella2019transflow} is released contemporaneously with the first edition of this paper. 


Furthermore, we scale the norm of OOD representations with different factors. The decoded images also vary from ID data to OOD data gradually. See Figure \ref{fig:glow_trained_on_ID_temperature_OOD} in the supplementary material for details. 
Overall, this leads to the second answer to Q1.
\begin{framed}
	\textbf{\textit{Answer 2 to Q1}}: OOD representations locate in specific directions with specific norms. The mean and covariance of OOD representations partially characterizes such specific directions. In high dimensional space, it is hard to sample data in specific directions from standard Gaussian distribution (prior) regardless of whether  these data reside in the typical set or not.
\end{framed}

\textbf{Note}.
In the proposed question Q1 ``why we cannot sample OOD data from the model'', we mean we cannot generate OOD data when sampling noise $\varepsilon$ from prior.
In this section, we sample OOD data from flow-based model with fitted Gaussian distribution from OOD representations. This does not contradict the proposed question Q1 because we need the mean and covariance of OOD representations in advance. More research on sampling OOD data is beyond the scope of this paper. We will explore this direction in the future.

%% file: results/S2/glow_fashionmnist_vs_constant_bs_5_10_S2.tex
& \textbf{100.0$\pm$0.0} & \textbf{100.0$\pm$0.0} & 42.1$\pm$0.3 & 42.1$\pm$0.2 & \textbf{100.0$\pm$0.0} & \textbf{100.0$\pm$0.0} & 41.7$\pm$0.5 & 41.9$\pm$0.2 \\

%% file: results/S2/glow_fashionmnist_vs_mnist_bs_5_10_S2.tex
& \textbf{99.8$\pm$0.0} & \textbf{99.8$\pm$0.0} & 97.6$\pm$0.1 & 95.8$\pm$0.5 & \textbf{100.0$\pm$0.0} & \textbf{100.0$\pm$0.0} & 99.7$\pm$0.1 & 99.6$\pm$0.1 \\

%% file: results/S2/glow_fashionmnist_vs_mnist_divergence_bs_5_10_S2.tex
& \textbf{100.0$\pm$0.0} & \textbf{100.0$\pm$0.0} & 88.2$\pm$0.3 & 81.8$\pm$0.2 & \textbf{100.0$\pm$0.0} & \textbf{100.0$\pm$0.0} & 95.8$\pm$0.5 & 93.5$\pm$1.2 \\

%% file: results/S2/glow_fashionmnist_vs_notmnist_bs_5_10_S2.tex
& \textbf{100.0$\pm$0.0} & \textbf{100.0$\pm$0.0} & 77.5$\pm$0.3 & 74.6$\pm$0.4 & \textbf{100.0$\pm$0.0} & \textbf{100.0$\pm$0.0} & 87.1$\pm$0.2 & 85.4$\pm$0.4 \\

%% file: results/S2/glow_fashionmnist_vs_notmnist_gray_bs_5_10_S2.tex
& \textbf{100.0$\pm$0.0} & \textbf{100.0$\pm$0.0} & 25.0$\pm$0.6 & 35.8$\pm$0.2 & \textbf{100.0$\pm$0.0} & \textbf{100.0$\pm$0.0} & 23.8$\pm$0.4 & 35.5$\pm$0.1 

%% file: results/S2/glow_svhn_vs_constant_bs_5_10_S2.tex
& \textbf{100.0$\pm$0.0} & \textbf{100.0$\pm$0.0} & \textbf{100.0$\pm$0.0} & \textbf{100.0$\pm$0.0} & \textbf{100.0$\pm$0.0} & \textbf{100.0$\pm$0.0} & \textbf{100.0$\pm$0.0} & \textbf{100.0$\pm$0.0} \\

%% file: results/S2/glow_svhn_vs_uniform_noise_bs_5_10_S2.tex
& \textbf{100.0$\pm$0.0} & \textbf{100.0$\pm$0.0} & \textbf{100.0$\pm$0.0} & \textbf{100.0$\pm$0.0} & \textbf{100.0$\pm$0.0} & \textbf{100.0$\pm$0.0} & \textbf{100.0$\pm$0.0} & \textbf{100.0$\pm$0.0} \\

%% file: results/S2/glow_svhn_vs_uniform_noise_gray_bs_5_10_S2.tex
& \textbf{100.0$\pm$0.0} & \textbf{100.0$\pm$0.0} & 13.5$\pm$0.5 & 33.0$\pm$0.1 & \textbf{100.0$\pm$0.0} & \textbf{100.0$\pm$0.0} & 11.1$\pm$0.5 & 32.6$\pm$0.1 \\

%% file: results/S2/glow_svhn_vs_celeba32_bs_5_10_S2.tex
& \textbf{100.0$\pm$0.0} & \textbf{100.0$\pm$0.0} & \textbf{100.0$\pm$0.0} & \textbf{100.0$\pm$0.0} & \textbf{100.0$\pm$0.0} & \textbf{100.0$\pm$0.0} & \textbf{100.0$\pm$0.0} & \textbf{100.0$\pm$0.0} \\

%% file: results/S2/glow_svhn_vs_celeba32_gray_bs_5_10_S2.tex
& \textbf{99.7$\pm$0.0} & \textbf{99.7$\pm$0.0} & 50.7$\pm$0.7 & 47.0$\pm$0.3 & \textbf{100.0$\pm$0.0} & \textbf{100.0$\pm$0.0} & 55.2$\pm$0.4 & 49.1$\pm$0.3 \\

%% file: results/S2/glow_svhn_vs_cifar10_bs_5_10_S2.tex
& \textbf{100.0$\pm$0.0} & \textbf{100.0$\pm$0.0} & \textbf{100.0$\pm$0.0} & \textbf{100.0$\pm$0.0} & \textbf{100.0$\pm$0.0} & \textbf{100.0$\pm$0.0} & \textbf{100.0$\pm$0.0} & \textbf{100.0$\pm$0.0} \\

%% file: results/S2/glow_svhn_vs_cifar10_gray_bs_5_10_S2.tex
& \textbf{97.0$\pm$0.2} & \textbf{97.4$\pm$0.2} & 31.6$\pm$0.5 & 37.9$\pm$0.2 & \textbf{99.3$\pm$0.1} & \textbf{99.4$\pm$0.1} & 25.0$\pm$0.3 & 35.6$\pm$0.1 \\

%% file: results/S2/glow_svhn_vs_cifar100_bs_5_10_S2.tex
& \textbf{100.0$\pm$0.0} & \textbf{100.0$\pm$0.0} & \textbf{100.0$\pm$0.0} & \textbf{100.0$\pm$0.0} & \textbf{100.0$\pm$0.0} & \textbf{100.0$\pm$0.0} & \textbf{100.0$\pm$0.0} & \textbf{100.0$\pm$0.0} \\

%% file: results/S2/glow_svhn_vs_cifar100_gray_bs_5_10_S2.tex
& \textbf{96.9$\pm$0.1} & \textbf{97.3$\pm$0.1} & 35.3$\pm$0.5 & 39.4$\pm$0.2 & \textbf{98.9$\pm$0.3} & \textbf{99.0$\pm$0.3} & 27.2$\pm$0.8 & 36.3$\pm$0.2 \\

%% file: results/S2/glow_svhn_vs_imagenet32_bs_5_10_S2.tex
& \textbf{100.0$\pm$0.0} & \textbf{100.0$\pm$0.0} & \textbf{100.0$\pm$0.0} & \textbf{100.0$\pm$0.0} & \textbf{100.0$\pm$0.0} & \textbf{100.0$\pm$0.0} & \textbf{100.0$\pm$0.0} & \textbf{100.0$\pm$0.0} \\

%% file: results/S2/glow_svhn_vs_imagenet32_gray_bs_5_10_S2.tex
& \textbf{99.8$\pm$0.0} & \textbf{99.8$\pm$0.0} & 45.5$\pm$0.9 & 46.0$\pm$0.5 & \textbf{100.0$\pm$0.0} & \textbf{100.0$\pm$0.0} & 42.1$\pm$0.7 & 44.1$\pm$0.5 

%% file: results/S2/glow_cifar10_vs_Constant_bs_5_10_S2.tex
& \textbf{100.0$\pm$0.0} & \textbf{100.0$\pm$0.0} & \textbf{100.0$\pm$0.0} & \textbf{100.0$\pm$0.0} & \textbf{100.0$\pm$0.0} & \textbf{100.0$\pm$0.0} & \textbf{100.0$\pm$0.0} & \textbf{100.0$\pm$0.0} \\

%% file: results/S2/glow_cifar10_vs_uniform_gray_bs_5_10_S2.tex
& \textbf{100.0$\pm$0.0} & \textbf{100.0$\pm$0.0} & 11.5$\pm$0.0 & 32.9$\pm$0.0 & \textbf{100.0$\pm$0.0} & \textbf{100.0$\pm$0.0} & 9.3$\pm$0.0 & 32.5$\pm$0.0 \\

%% file: results/S2/glow_cifar10_vs_CelebA32_bs_5_10_S2.tex
& 99.2$\pm$0.1 & 99.4$\pm$0.1 & \textbf{100.0$\pm$0.0} & \textbf{100.0$\pm$0.0} & \textbf{100.0$\pm$0.0} & \textbf{100.0$\pm$0.0} & \textbf{100.0$\pm$0.0} & \textbf{100.0$\pm$0.0} \\

%% file: results/S2/glow_cifar10_vs_CelebA32-C0.3_bs_5_10_S2.tex
& \textbf{84.3$\pm$0.3} & \textbf{84.4$\pm$0.4} & 28.4$\pm$0.5 & 36.7$\pm$0.2 & \textbf{94.5$\pm$0.3} & \textbf{94.7$\pm$0.3} & 23.5$\pm$0.5 & 35.2$\pm$0.1 \\

%% file: results/S2/glow_cifar10_vs_Imagenet32_bs_5_10_S2.tex
& 90.0$\pm$0.2 & 92.1$\pm$0.1 & \textbf{99.2$\pm$0.1} & \textbf{99.3$\pm$0.1} & 95.0$\pm$0.4 & 96.2$\pm$0.2 & \textbf{100.0$\pm$0.0} & \textbf{100.0$\pm$0.0} \\

%% file: results/S2/glow_cifar10_vs_Imagenet32-C0.3_bs_5_10_S2.tex
& \textbf{72.0$\pm$0.3} & \textbf{72.6$\pm$0.4} & 40.9$\pm$0.4 & 43.2$\pm$0.2 & \textbf{74.3$\pm$0.6} & \textbf{74.8$\pm$0.8} & 32.0$\pm$0.7 & 38.5$\pm$0.3 \\

%% file: results/S2/glow_cifar10_vs_SVHN_bs_5_10_S2.tex
& 97.6$\pm$0.2 & 97.8$\pm$0.2 & \textbf{98.6$\pm$0.1} & \textbf{98.4$\pm$0.1} & 99.8$\pm$0.0 & 99.8$\pm$0.0 & \textbf{99.9$\pm$0.1} & \textbf{99.9$\pm$0.1} \\

%% file: results/S2/glow_cifar10_vs_SVHN-C2.0_bs_5_10_S2.tex
& \textbf{100.0$\pm$0.0} & \textbf{100.0$\pm$0.0} & 33.5$\pm$0.4 & 61.0$\pm$0.2 & \textbf{100.0$\pm$0.0} & \textbf{100.0$\pm$0.0} & 27.2$\pm$0.5 & 58.2$\pm$0.1 

%% file: results/S2/glow_celeba32_vs_constant_bs_5_10_S2.tex
& \textbf{100.0$\pm$0.0} & \textbf{100.0$\pm$0.0} & \textbf{100.0$\pm$0.0} & \textbf{100.0$\pm$0.0} & \textbf{100.0$\pm$0.0} & \textbf{100.0$\pm$0.0} & \textbf{100.0$\pm$0.0} & \textbf{100.0$\pm$0.0} \\

%% file: results/S2/glow_celeba32_vs_uniform_noise_bs_5_10_S2.tex
& \textbf{100.0$\pm$0.0} & \textbf{100.0$\pm$0.0} & \textbf{100.0$\pm$0.0} & \textbf{100.0$\pm$0.0} & \textbf{100.0$\pm$0.0} & \textbf{100.0$\pm$0.0} & \textbf{100.0$\pm$0.0} & \textbf{100.0$\pm$0.0} \\

%% file: results/S2/glow_celeba32_vs_uniform_noise_gray_bs_5_10_S2.tex
& \textbf{100.0$\pm$0.0} & \textbf{100.0$\pm$0.0} & 36.2$\pm$0.7 & 39.6$\pm$0.2 & \textbf{100.0$\pm$0.0} & \textbf{100.0$\pm$0.0} & 30.9$\pm$0.7 & 37.9$\pm$0.2 \\

%% file: results/S2/glow_celeba32_vs_cifar10_bs_5_10_S2.tex
& \textbf{99.6$\pm$0.0} & \textbf{99.6$\pm$0.0} & 7.2$\pm$0.2 & 31.4$\pm$0.0 & \textbf{100.0$\pm$0.0} & \textbf{100.0$\pm$0.0} & 1.7$\pm$0.1 & 30.8$\pm$0.0 \\

%% file: results/S2/glow_celeba32_vs_cifar100_bs_5_10_S2.tex
& \textbf{99.8$\pm$0.0} & \textbf{99.8$\pm$0.0} & 9.5$\pm$0.3 & 31.8$\pm$0.1 & \textbf{100.0$\pm$0.0} & \textbf{100.0$\pm$0.0} & 2.9$\pm$0.2 & 30.9$\pm$0.0 \\

%% file: results/S2/glow_celeba32_vs_imagenet32_bs_5_10_S2.tex
& \textbf{100.0$\pm$0.0} & \textbf{100.0$\pm$0.0} & 78.1$\pm$0.4 & 85.6$\pm$0.3 & \textbf{100.0$\pm$0.0} & \textbf{100.0$\pm$0.0} & 83.9$\pm$0.4 & 89.6$\pm$0.2 \\

%% file: results/S2/glow_celeba32_vs_svhn_bs_5_10_S2.tex
& \textbf{100.0$\pm$0.0} & \textbf{100.0$\pm$0.0} & 78.7$\pm$0.3 & 73.3$\pm$0.9 & \textbf{100.0$\pm$0.0} & \textbf{100.0$\pm$0.0} & 86.6$\pm$0.8 & 83.3$\pm$1.4 \\

%% file: results/S2/glow_celeba32_vs_svhn_divergence_bs_5_10_S2.tex
& \textbf{100.0$\pm$0.0} & \textbf{100.0$\pm$0.0} & 3.5$\pm$0.2 & 31.0$\pm$0.0 & \textbf{100.0$\pm$0.0} & \textbf{100.0$\pm$0.0} & 0.5$\pm$0.1 & 30.7$\pm$0.0 

%% file: results/S2/glow_fashionmnist_vs_constant_bs_2_4_S2.tex
& \textbf{100.0$\pm$0.0} & \textbf{100.0$\pm$0.0} & 40.6$\pm$0.4 & 41.3$\pm$0.2 & \textbf{100.0$\pm$0.0} & \textbf{100.0$\pm$0.0} & 41.0$\pm$0.5 & 41.6$\pm$0.2 \\

%% file: results/S2/glow_fashionmnist_vs_mnist_bs_2_4_S2.tex
& \textbf{91.6$\pm$0.1} & \textbf{91.8$\pm$0.2} & 88.7$\pm$0.2 & 81.1$\pm$0.4 & \textbf{99.4$\pm$0.0} & \textbf{99.4$\pm$0.0} & 96.1$\pm$0.1 & 93.2$\pm$0.1 \\

%% file: results/S2/glow_fashionmnist_vs_mnist_divergence_bs_2_4_S2.tex
& \textbf{97.2$\pm$0.1} & \textbf{97.3$\pm$0.1} & 74.0$\pm$0.3 & 65.2$\pm$0.2 & \textbf{100.0$\pm$0.0} & \textbf{100.0$\pm$0.0} & 85.4$\pm$0.2 & 77.4$\pm$0.5 \\

%% file: results/S2/glow_fashionmnist_vs_notmnist_bs_2_4_S2.tex
& \textbf{99.2$\pm$0.0} & \textbf{99.4$\pm$0.0} & 64.0$\pm$0.3 & 61.8$\pm$0.3 & \textbf{100.0$\pm$0.0} & \textbf{100.0$\pm$0.0} & 74.3$\pm$0.4 & 71.2$\pm$0.3 \\

%% file: results/S2/glow_fashionmnist_vs_notmnist_gray_bs_2_4_S2.tex
& \textbf{100.0$\pm$0.0} & \textbf{100.0$\pm$0.0} & 23.2$\pm$0.2 & 35.3$\pm$0.0 & \textbf{100.0$\pm$0.0} & \textbf{100.0$\pm$0.0} & 24.8$\pm$0.3 & 35.7$\pm$0.1 

%% file: results/S2/glow_svhn_vs_constant_bs_2_4_S2.tex
& \textbf{100.0$\pm$0.0} & \textbf{100.0$\pm$0.0} & \textbf{100.0$\pm$0.0} & \textbf{100.0$\pm$0.0} & \textbf{100.0$\pm$0.0} & \textbf{100.0$\pm$0.0} & \textbf{100.0$\pm$0.0} & \textbf{100.0$\pm$0.0} \\

%% file: results/S2/glow_svhn_vs_uniform_noise_bs_2_4_S2.tex
& \textbf{100.0$\pm$0.0} & \textbf{100.0$\pm$0.0} & \textbf{100.0$\pm$0.0} & \textbf{100.0$\pm$0.0} & \textbf{100.0$\pm$0.0} & \textbf{100.0$\pm$0.0} & \textbf{100.0$\pm$0.0} & \textbf{100.0$\pm$0.0} \\

%% file: results/S2/glow_svhn_vs_uniform_noise_gray_bs_2_4_S2.tex
& \textbf{99.9$\pm$0.0} & \textbf{99.8$\pm$0.0} & 14.6$\pm$0.5 & 33.2$\pm$0.1 & \textbf{100.0$\pm$0.0} & \textbf{100.0$\pm$0.0} & 14.5$\pm$0.4 & 33.2$\pm$0.1 \\

%% file: results/S2/glow_svhn_vs_celeba32_bs_2_4_S2.tex
& \textbf{100.0$\pm$0.0} & \textbf{100.0$\pm$0.0} & \textbf{100.0$\pm$0.0} & \textbf{100.0$\pm$0.0} & \textbf{100.0$\pm$0.0} & \textbf{100.0$\pm$0.0} & \textbf{100.0$\pm$0.0} & \textbf{100.0$\pm$0.0} \\

%% file: results/S2/glow_svhn_vs_celeba32_gray_bs_2_4_S2.tex
& \textbf{90.1$\pm$0.2} & \textbf{88.6$\pm$0.4} & 49.0$\pm$0.3 & 46.6$\pm$0.2 & \textbf{96.2$\pm$0.2} & \textbf{95.1$\pm$0.3} & 55.8$\pm$0.4 & 50.4$\pm$0.3 \\

%% file: results/S2/glow_svhn_vs_cifar10_bs_2_4_S2.tex
& \textbf{100.0$\pm$0.0} & \textbf{100.0$\pm$0.0} & 99.9$\pm$0.0 & 99.9$\pm$0.0 & \textbf{100.0$\pm$0.0} & \textbf{100.0$\pm$0.0} & \textbf{100.0$\pm$0.0} & \textbf{100.0$\pm$0.0} \\

%% file: results/S2/glow_svhn_vs_cifar10_gray_bs_2_4_S2.tex
& \textbf{77.2$\pm$0.2} & \textbf{76.1$\pm$0.2} & 36.6$\pm$0.2 & 40.0$\pm$0.1 & \textbf{81.9$\pm$0.3} & \textbf{80.6$\pm$0.2} & 32.9$\pm$0.8 & 38.3$\pm$0.3 \\

%% file: results/S2/glow_svhn_vs_cifar100_bs_2_4_S2.tex
& \textbf{99.9$\pm$0.0} & \textbf{99.9$\pm$0.0} & 99.9$\pm$0.0 & 99.9$\pm$0.0 & \textbf{100.0$\pm$0.0} & \textbf{100.0$\pm$0.0} & 100.0$\pm$0.0 & 100.0$\pm$0.0 \\

%% file: results/S2/glow_svhn_vs_cifar100_gray_bs_2_4_S2.tex
& \textbf{79.8$\pm$0.3} & \textbf{79.3$\pm$0.3} & 40.6$\pm$0.4 & 42.1$\pm$0.2 & \textbf{83.5$\pm$0.2} & \textbf{82.7$\pm$0.1} & 36.5$\pm$0.7 & 39.9$\pm$0.3 \\

%% file: results/S2/glow_svhn_vs_imagenet32_bs_2_4_S2.tex
& \textbf{100.0$\pm$0.0} & \textbf{100.0$\pm$0.0} & \textbf{100.0$\pm$0.0} & \textbf{100.0$\pm$0.0} & \textbf{100.0$\pm$0.0} & \textbf{100.0$\pm$0.0} & \textbf{100.0$\pm$0.0} & \textbf{100.0$\pm$0.0} \\

%% file: results/S2/glow_svhn_vs_imagenet32_gray_bs_2_10_S2.tex
& \textbf{97.8$\pm$0.1} & \textbf{98.1$\pm$0.1} & 48.4$\pm$0.3 & 48.2$\pm$0.1 & \textbf{100.0$\pm$0.0} & \textbf{100.0$\pm$0.0} & 42.5$\pm$0.3 & 44.1$\pm$0.1 

%% file: results/S2/glow_cifar10_vs_Constant_bs_2_4_S2.tex
& \textbf{100.0$\pm$0.0} & \textbf{100.0$\pm$0.0} & \textbf{100.0$\pm$0.0} & \textbf{100.0$\pm$0.0} & \textbf{100.0$\pm$0.0} & \textbf{100.0$\pm$0.0} & \textbf{100.0$\pm$0.0} & \textbf{100.0$\pm$0.0} \\

%% file: results/S2/glow_cifar10_vs_uniform_gray_bs_2_4_S2.tex
& \textbf{100.0$\pm$0.0} & \textbf{100.0$\pm$0.0} & 9.9$\pm$0.0 & 32.5$\pm$0.0 & \textbf{100.0$\pm$0.0} & \textbf{100.0$\pm$0.0} & 11.2$\pm$0.0 & 32.8$\pm$0.0 \\

%% file: results/S2/glow_cifar10_vs_CelebA32_bs_2_4_S2.tex
& 93.3$\pm$0.1 & 94.6$\pm$0.1 & \textbf{98.0$\pm$0.1} & \textbf{98.1$\pm$0.0} & 98.4$\pm$0.1 & 98.7$\pm$0.1 & \textbf{99.9$\pm$0.0} & \textbf{99.9$\pm$0.0} \\

%% file: results/S2/glow_cifar10_vs_CelebA32-C0.3_bs_2_4_S2.tex
& \textbf{72.4$\pm$0.3} & \textbf{71.6$\pm$0.3} & 32.4$\pm$0.3 & 38.1$\pm$0.1 & \textbf{81.3$\pm$0.3} & \textbf{81.2$\pm$0.3} & 29.7$\pm$0.4 & 37.1$\pm$0.1 \\

%% file: results/S2/glow_cifar10_vs_Imagenet32_bs_2_4_S2.tex
& 82.2$\pm$0.2 & 85.2$\pm$0.1 & \textbf{93.1$\pm$0.2} & \textbf{94.6$\pm$0.2} & 87.8$\pm$0.2 & 90.2$\pm$0.2 & \textbf{98.3$\pm$0.2} & \textbf{98.7$\pm$0.1} \\

%% file: results/S2/glow_cifar10_vs_Imagenet32-C0.3_bs_2_4_S2.tex
& \textbf{68.2$\pm$0.1} & \textbf{69.1$\pm$0.3} & 47.8$\pm$0.3 & 48.0$\pm$0.2 & \textbf{70.2$\pm$0.3} & \textbf{71.0$\pm$0.2} & 42.6$\pm$0.9 & 44.0$\pm$0.6 \\

%% file: results/S2/glow_cifar10_vs_SVHN_bs_2_4_S2.tex
& 90.0$\pm$0.1 & \textbf{90.7$\pm$0.2} & \textbf{91.2$\pm$0.1} & 88.1$\pm$0.3 & 96.2$\pm$0.1 & 96.5$\pm$0.1 & \textbf{97.6$\pm$0.1} & \textbf{96.8$\pm$0.2} \\

%% file: results/S2/glow_cifar10_vs_SVHN-C2.0_bs_2_4_S2.tex
& \textbf{99.1$\pm$0.1} & \textbf{99.2$\pm$0.0} & 39.2$\pm$0.1 & 64.0$\pm$0.1 & \textbf{100.0$\pm$0.0} & \textbf{100.0$\pm$0.0} & 35.2$\pm$0.5 & 61.9$\pm$0.2

%% file: results/S2/glow_celeba32_vs_constant_bs_2_4_S2.tex
& \textbf{100.0$\pm$0.0} & \textbf{100.0$\pm$0.0} & \textbf{100.0$\pm$0.0} & \textbf{100.0$\pm$0.0} & \textbf{100.0$\pm$0.0} & \textbf{100.0$\pm$0.0} & \textbf{100.0$\pm$0.0} & \textbf{100.0$\pm$0.0} \\

%% file: results/S2/glow_celeba32_vs_uniform_noise_bs_2_4_S2.tex
& \textbf{100.0$\pm$0.0} & \textbf{100.0$\pm$0.0} & \textbf{100.0$\pm$0.0} & \textbf{100.0$\pm$0.0} & \textbf{100.0$\pm$0.0} & \textbf{100.0$\pm$0.0} & \textbf{100.0$\pm$0.0} & \textbf{100.0$\pm$0.0} \\

%% file: results/S2/glow_celeba32_vs_uniform_noise_gray_bs_2_4_S2.tex
& \textbf{99.2$\pm$0.0} & \textbf{99.3$\pm$0.0} & 35.1$\pm$0.2 & 39.3$\pm$0.1 & \textbf{100.0$\pm$0.0} & \textbf{100.0$\pm$0.0} & 29.4$\pm$0.4 & 37.5$\pm$0.1 \\

%% file: results/S2/glow_celeba32_vs_cifar10_bs_2_4_S2.tex
& \textbf{86.3$\pm$0.2} & \textbf{86.4$\pm$0.2} & 21.4$\pm$0.2 & 34.6$\pm$0.1 & \textbf{98.5$\pm$0.1} & \textbf{98.6$\pm$0.1} & 10.2$\pm$0.3 & 31.9$\pm$0.1 \\

%% file: results/S2/glow_celeba32_vs_cifar100_bs_2_4_S2.tex
& \textbf{89.6$\pm$0.2} & \textbf{90.0$\pm$0.2} & 25.0$\pm$0.2 & 35.9$\pm$0.0 & \textbf{99.2$\pm$0.1} & \textbf{99.3$\pm$0.1} & 12.9$\pm$0.1 & 32.5$\pm$0.0 \\

%% file: results/S2/glow_celeba32_vs_imagenet32_bs_2_4_S2.tex
& \textbf{100.0$\pm$0.0} & \textbf{100.0$\pm$0.0} & 76.4$\pm$0.3 & 83.5$\pm$0.1 & \textbf{100.0$\pm$0.0} & \textbf{100.0$\pm$0.0} & 76.9$\pm$0.4 & 84.5$\pm$0.2 \\

%% file: results/S2/glow_celeba32_vs_svhn_bs_2_4_S2.tex
& \textbf{99.9$\pm$0.0} & \textbf{99.9$\pm$0.0} & 69.3$\pm$0.1 & 62.5$\pm$0.1 & \textbf{100.0$\pm$0.0} & \textbf{100.0$\pm$0.0} & 76.0$\pm$0.2 & 70.5$\pm$0.4 \\

%% file: results/S2/glow_celeba32_vs_svhn_divergence_bs_2_4_S2.tex
& \textbf{100.0$\pm$0.0} & \textbf{100.0$\pm$0.0} & 14.5$\pm$0.2 & 32.9$\pm$0.1 & \textbf{100.0$\pm$0.0} & \textbf{100.0$\pm$0.0} & 5.6$\pm$0.2 & 31.3$\pm$0.0 

%% file: results/S2/glow_svhn_vs_mixed_celeba_cifar10_bs_5_10_S2.tex
& \textbf{100.0$\pm$0.0} & \textbf{100.0$\pm$0.0} & \textbf{100.0$\pm$0.0} & \textbf{100.0$\pm$0.0} & \textbf{100.0$\pm$0.0} & \textbf{100.0$\pm$0.0} & \textbf{100.0$\pm$0.0} & \textbf{100.0$\pm$0.0}

%% file: results/S2/glow_cifar10_vs_mixture_celeba_svhn_bs_5_10_S2.tex
& \textbf{98.2$\pm$0.2} & \textbf{98.5$\pm$0.2} & 60.8$\pm$0.3 & 64.4$\pm$0.5 & \textbf{99.8$\pm$0.0} & \textbf{99.9$\pm$0.0} & 52.8$\pm$1.0 & 58.1$\pm$1.1 

%% file: results/S2/glow_celeba32_vs_mixed_svhn_cifar10_bs_5_10_S2.tex
& \textbf{100.0$\pm$0.0} & \textbf{100.0$\pm$0.0} & 20.7$\pm$0.2 & 34.9$\pm$0.1 & \textbf{100.0$\pm$0.0} & \textbf{100.0$\pm$0.0} & 11.8$\pm$0.5 & 32.3$\pm$0.1 

%% file: results/S2/glow_mnist_only_0_vs_mnist_except_0_bs_5_10_S2.tex
& \textbf{92.9$\pm$1.2} & \textbf{93.4$\pm$1.1} & 86.5$\pm$1.5 & 89.9$\pm$1.1 & \textbf{99.1$\pm$0.6} & \textbf{99.2$\pm$0.5} & 95.7$\pm$1.6 & 96.8$\pm$1.0

%% file: results/S2/glow_mnist_only_1_vs_mnist_except_1_bs_5_10_S2.tex
& 96.6$\pm$0.5 & 96.9$\pm$0.6 & \textbf{100.0$\pm$0.0} & \textbf{100.0$\pm$0.0} & 99.8$\pm$0.1 & 99.8$\pm$0.1 & \textbf{100.0$\pm$0.0} & \textbf{100.0$\pm$0.0}

%% file: results/S2/glow_mnist_only_2_vs_mnist_except_2_bs_5_10_S2.tex
& \textbf{85.0$\pm$0.7} & \textbf{85.8$\pm$0.9} & 69.8$\pm$1.5 & 74.8$\pm$1.8 & \textbf{95.0$\pm$1.2} & \textbf{95.2$\pm$1.2} & 76.7$\pm$1.4 & 81.7$\pm$1.4

%% file: results/S2/glow_mnist_only_3_vs_mnist_except_3_bs_5_10_S2.tex
& 65.5$\pm$2.6 & 64.9$\pm$3.2 & \textbf{65.7$\pm$1.5} & \textbf{70.0$\pm$1.3} & \textbf{76.2$\pm$1.5} & \textbf{77.0$\pm$1.1} & 68.1$\pm$2.1 & 72.7$\pm$1.7

%% file: results/S2/glow_mnist_only_4_vs_mnist_except_4_bs_5_10_S2.tex
& \textbf{94.8$\pm$0.6} & \textbf{95.3$\pm$0.5} & 77.4$\pm$0.8 & 82.2$\pm$1.4 & \textbf{99.8$\pm$0.1} & \textbf{99.8$\pm$0.1} & 81.6$\pm$2.1 & 86.0$\pm$1.8

%% file: results/S2/glow_mnist_only_5_vs_mnist_except_5_bs_5_10_S2.tex
& \textbf{66.3$\pm$2.3} & \textbf{66.1$\pm$2.6} & 61.6$\pm$2.2 & 64.6$\pm$2.2 & \textbf{78.3$\pm$2.4} & \textbf{78.6$\pm$3.0} & 62.6$\pm$2.7 & 64.6$\pm$2.8

%% file: results/S2/glow_mnist_only_6_vs_mnist_except_6_bs_5_10_S2.tex
& \textbf{88.0$\pm$1.0} & \textbf{88.4$\pm$0.9} & 64.3$\pm$1.7 & 66.9$\pm$1.1 & \textbf{98.0$\pm$0.7} & \textbf{98.2$\pm$0.7} & 58.6$\pm$1.2 & 63.0$\pm$1.2

%% file: results/S2/glow_mnist_only_7_vs_mnist_except_7_bs_5_10_S2.tex
& \textbf{94.0$\pm$0.5} & \textbf{94.3$\pm$0.6} & 91.0$\pm$0.3 & 93.2$\pm$0.1 & \textbf{99.2$\pm$0.3} & \textbf{99.3$\pm$0.3} & 96.3$\pm$1.1 & 97.4$\pm$0.6

%% file: results/S2/glow_mnist_only_8_vs_mnist_except_8_bs_5_10_S2.tex
& 76.0$\pm$1.4 & 76.8$\pm$1.6 & \textbf{81.2$\pm$1.6} & \textbf{85.7$\pm$1.3} & 86.2$\pm$1.4 & 86.7$\pm$1.8 & \textbf{92.1$\pm$0.8} & \textbf{93.7$\pm$1.0}